\documentclass[twoside,11pt]{article}

\usepackage{jmlr2e}
\usepackage{natbib}

\usepackage{afterpage}
\usepackage{mathrsfs}
\usepackage{appendix}
\usepackage{hyperref}
\usepackage{amsmath}
\usepackage{amssymb}
\usepackage{graphicx}
\usepackage{amsfonts}
\usepackage{latexsym}
\usepackage{pdfsync}
\usepackage{subfigure}
\usepackage{color}
\usepackage{tabularx}
\usepackage[english]{babel}
\usepackage{array}
\usepackage{algorithm}
\usepackage{algorithmic}
\usepackage{esint}

\usepackage{hyperref}
 \usepackage{multirow}

\usepackage{tikz}

\DeclareFontFamily{OT1}{pzc}{}
\DeclareFontShape{OT1}{pzc}{m}{it}{<-> s * [1.10] pzcmi7t}{}
\DeclareMathAlphabet{\mathpzc}{OT1}{pzc}{m}{it}

\newcommand{\commentout}[1]{}

\newcommand{\nwc}{\newcommand}
\nwc{\ba}{\begin{array}}
\nwc{\bal}{\begin{align}}
\nwc{\bea}{\begin{eqnarray}}
\nwc{\beq}{\begin{eqnarray}}
\nwc{\bean}{\begin{eqnarray*}}
\nwc{\beqn}{\begin{eqnarray*}}
\nwc{\beqast}{\begin{eqnarray*}}

\nwc{\ea}{\end{array}}
\nwc{\eal}{\end{align}}
\nwc{\eea}{\end{eqnarray}}
\nwc{\eeq}{\end{eqnarray}}
\nwc{\eean}{\end{eqnarray*}}
\nwc{\eeqn}{\end{eqnarray*}}
\nwc{\eeqast}{\end{eqnarray*}}

\nwc{\ep}{\varepsilon}
\nwc{\ept}{\epsilon}

\nwc{\nwt}{\newtheorem}

\newcommand{\RR}{\mathbb{R}}
\newcommand{\ZZ}{\mathbb{Z}}
\newcommand{\EE}{\mathbb{E}}
\newcommand{\PP}{\mathbb{P}}
\nwc{\calC}{\mathcal{C}}
\nwc{\calO}{\mathcal{O}}
\nwc{\calJ}{\mathcal{J}}
\nwc{\calM}{\mathcal{M}}
\nwc{\calX}{\mathcal{X}}
\nwc{\calD}{\mathcal{D}}
\nwc{\calE}{\mathcal{E}}
\nwc{\calP}{\mathcal{P}}
\nwc{\calH}{\mathcal{H}}
\nwc{\calK}{\mathcal{K}}
\nwc{\calT}{\mathcal{T}}
\nwc{\tcalT}{\tilde{\calT}}
\nwc{\calA}{\mathcal{A}}
\nwc{\calB}{\mathcal{B}}
\nwc{\calI}{\mathcal{I}}
\nwc{\calQ}{\mathcal{Q}}
\nwc{\calF}{\mathcal{F}}
\nwc{\AS}{\calA_s}
\nwc{\Agamma}{\calA_{\gamma}}
\nwc{\BS}{\calB_s}
\nwc{\Gammas}{\Gamma_s}
\nwc{\bbI}{\mathbb{I}}
\nwc{\tgamma}{\tilde{\gamma}}
\nwc{\talpha}{\tilde{\alpha}}
\nwc{\btheta}{\boldsymbol \theta}

\nwc{\proj}{{\rm Proj}}
\nwc{\Vol}{{\rm Vol}}
\nwc{\diam}{{\rm diam}}

\nwc{\lognn}{\frac{\log n}{n}}

\nwc{\Lam}{\Lambda}
\nwc{\hLam}{\widehat{\Lambda}}

\nwc{\lcol}{\left\|}
\nwc{\rcol}{\right\|}

\nwc{\hcalD}{\widehat{\calD}}
\nwc{\hcalP}{\widehat{\calP}}
\nwc{\hc}{\widehat{c}}
\nwc{\hP}{\widehat{\calP}}
\nwc{\hS}{\widehat{\Sigma}}
\nwc{\hV}{\widehat{V}}
\nwc{\halpha}{\widehat{\alpha}}
\nwc{\hrho}{\widehat{\rho}}
\nwc{\hcalT}{\widehat{\calT}}

\nwc{\argmin}{{\rm argmin}}
\nwc{\trace}{{\rm trace}}
\nwc{\rank}{{\rm rank}}
\nwc{\intdim}{{\rm intdim}}
\nwc{\din}{d_{\rm in}}

\nwc{\cjk}{c_{j,k}}
\nwc{\hcjk}{\hc_{j,k}}
\nwc{\Sjk}{\Sigma_{j,k}}
\nwc{\hSjk}{\hS_{j,k}}
\nwc{\Cjk}{C_{j,k}}
\nwc{\calPjk}{\calP_{j,k}}
\nwc{\hcalPjk}{\widehat{\calP}_{j,k}}
\nwc{\Vjk}{V_{j,k}}
\nwc{\hVjk}{\widehat{V}_{j,k}}
\nwc{\Deltajk}{\Delta_{j,k}}
\nwc{\hDeltajk}{\widehat{\Delta}_{j,k}}
\nwc{\bDeltajk}{\bar{\Delta}_{j,k}}
\nwc{\jkp}{{j+1,k'}}
\nwc{\chijk}{\mathbf{1}_{j,k}}
\nwc{\calMjk}{\calM_{j,k}}
\nwc{\calQjk}{\calQ_{j,k}}
\nwc{\hcalQjk}{\widehat{\calQ}_{j,k}}
\nwc{\njk}{n_{j,k}}
\nwc{\hnjk}{\widehat{n}_{j,k}}

\nwc{\jstar}{j^*}

\nwc{\cjx}{c_{j,x}}
\nwc{\Vjx}{{V_{j,x}}}
\nwc{\Vjox}{{V_{j+1,x}}}
\nwc{\Sjx}{{S_{j,x}}}
\nwc{\SSjk}{{S_{j,k}}}
\nwc{\Sjxp}{{S_{j,x}^\perp}}
\nwc{\Sjox}{{S_{j+1,x}}}
\nwc{\Ujx}{{U_{j,x}}}
\nwc{\Ujox}{{U_{j+1,x}}}
\nwc{\calS}{\mathcal{S}}
\nwc{\calSjx}{\calS_{j,x}}

\nwc{\hn}{\widehat{n}}
\nwc{\hcjx}{\hc_{j,x}}
\nwc{\hVjx}{\hV_{j,x}}
\nwc{\hVjox}{\hV_{j+1,x}}
\nwc{\hSjx}{\widehat{S}_{j,x}}
\nwc{\hSSjk}{\widehat{S}_{j,k}}
\nwc{\hSjxp}{\widehat{S}_{j,x}^\perp}
\nwc{\hSjox}{\widehat{S}_{j+1,x}}
\nwc{\hUjx}{\widehat{U}_{j,x}}
\nwc{\hUjox}{\widehat{U}_{j+1,x}}
\nwc{\hcalS}{\widehat{\mathcal{S}}}
\nwc{\hcalSjx}{\hcalS_{j,x}}
\nwc{\calSjk}{\calS_{j,k}}
\nwc{\hcalSjk}{\hcalS_{j,k}}
\nwc{\hcalM}{\widetilde{\calM}}

\nwc{\Child}{{\mathscr{C}}}
\nwc{\Desendant}{{\mathscr{D}}}
\nwc{\amax}{a_{\rm max}}
\nwc{\amin}{a_{\rm min}}
\nwc{\jmax}{j_{\rm max}}
\nwc{\jmin}{j_{\rm min}}
\nwc{\fs}{\lfloor s\rfloor}
\nwc{\Voronoi}{{\rm Voronoi}}
\nwc{\ttheta}{\tilde\theta}

\nwc{\calTn}{\calT^n}

\nwc{\hcalTtaun}{\hcalT_{\tau_n}}
\nwc{\calTrhoeta}{\calT_{(\rho,\eta)}}
\nwc{\hcalTeta} {\hcalT_{\eta}}
\nwc{\calTeta}{\calT_{\eta}}
\nwc{\calTtaunb}{\calT_{\tau_n/b}}
\nwc{\calTbtaun}{\calT_{b\tau_n}}
\nwc{\calTrhobtaun}{\calT_{(\rho,b\tau_n)}}
\nwc{\calTrhotaunob}{\calT_{(\rho,\tau_n^o/b)}}

\nwc{\hLameta}{\hLam_{\eta}}
\nwc{\Lameta}{\Lam_{\eta}}
\nwc{\Lamrhoeta}{\Lam_{(\rho,\eta)}}
\nwc{\Lambtaun}{\Lam_{b\tau_n}}
\nwc{\Lamtaunb}{\Lam_{\tau_n/b}}
\nwc{\Lamrhobtaun}{\Lam_{(\rho,b\tau_n)}}
\nwc{\hLamtaun}{\hLam_{\tau_n}}

\nwc{\hcalTtauno}{\hcalT_{\tau_n^o}}
\nwc{\calTtaunob}{\calT_{\tau_n^o/b}}

\nwc{\hLamtauno}{\hLam_{\tau_n^o}}
\nwc{\Lambtauno}{\Lam_{b{\tau_n^o}}}
\nwc{\Lamtaunob}{\Lam_{{\tau_n^o}/b}}

\jmlrheading{}{}{}{}{}{Adaptive Geometric Multiscale Approximations}

\ShortHeadings{Adaptive Geometric Multiscale Approximations for Intrinsically Low-dimensional Data}{Liao and Maggioni}
\firstpageno{1}

\begin{document}

\title{Adaptive Geometric Multiscale Approximations for Intrinsically Low-dimensional Data}

\author{\name Wenjing Liao$^1$ \email wliao@math.jhu.edu	 \\
       \AND
       \name Mauro Maggioni$^{1,2,3}$ \email mauro.maggioni@jhu.edu \\
       \addr $^1$Departments of Mathematics, $^2$Applied Mathematics and Statistics, \\$^3$The Institute for Data Intensive Engineering and Science\\
       Johns Hopkins University, 3400 N. Charles Street, Baltimore, MD 21218, USA
}

\editor{}

\maketitle

\begin{abstract}
We consider the problem of efficiently approximating and encoding high-dimensional data sampled from a probability distribution $\rho$ in $\mathbb{R}^D$, that is nearly supported on a $d$-dimensional set $\mathcal{M}$ - for example supported on a $d$-dimensional Riemannian manifold.
Geometric Multi-Resolution Analysis (GMRA) provides a robust and computationally efficient procedure to construct low-dimensional geometric approximations of $\calM$ at varying resolutions. We introduce a thresholding algorithm on the geometric wavelet coefficients, leading to what we call adaptive GMRA approximations. We show that these data-driven, empirical approximations perform well, when the threshold is chosen as a suitable universal function of the number of samples $n$, on a wide variety of measures $\rho$, that are allowed to exhibit different regularity at different scales and locations, thereby efficiently encoding data from more complex measures than those supported on manifolds.
These approximations yield a data-driven dictionary, together with a fast transform mapping data to coefficients, and an inverse of such a map.  The algorithms for both the dictionary construction and the transforms have complexity $C n \log n$ with the constant linear in $D$ and exponential in $d$. Our work therefore establishes adaptive GMRA as a fast dictionary learning algorithm with approximation guarantees.
We include several numerical experiments on both synthetic and real data, confirming our theoretical results and demonstrating the effectiveness of adaptive GMRA.
\end{abstract}

\ \\




\section{Introduction}

We model a data set as $n$ i.i.d. samples $\calX_n := \{x_i\}_{i=1}^n$ from a probability measure $\rho$ in $\RR^D$.
We make the assumption that $\rho$ is supported on or near a set $\calM$ of dimension $d \ll D$, and consider the problem, given $\calX_n$, of learning a data-dependent dictionary that efficiently encodes data sampled from $\rho$.

In order to circumvent the curse of dimensionality, a popular model for data is sparsity: we say that the data is $k$-sparse on a suitable dictionary (i.e. a collection of vectors) $\Phi=\{\varphi_i\}_{i=1}^m\subset\mathbb{R}^D$ if each data point $x\in\mathbb{R}^d$ may be expressed as a linear combination of at most $k$ elements of $\Phi$. Clearly the case of interest is $k\ll D$.
These \textit{sparse representations} have been used in a variety of statistical signal processing tasks, compressed sensing, learning \citep[see e.g.][among many others]{PotterElad,peyre:sparsetexture,Lewicki98learningovercomplete,Kreutz-Delgado:2003:DLA:643335.643340,DBLP:journals/tit/MaurerP10,chen:33,DD:CompressedSensing,Aharon05ksvd,Tao:DantzigEstimator}, 
and spurred much research about how to learn data-adaptive dictionaries
\citep[see][and references therein]{gribonval:hal-00918142,Vainsencher:2011:SCD:1953048.2078210,DBLP:journals/tit/MaurerP10}. 
The algorithms used in dictionary learning are often computationally demanding, being based on high-dimensional non-convex optimization \citep{Mairal:OnlineLearningSparseCoding}.
These approaches have the strength of being very general, with minimal assumptions made on geometry of the dictionary or on the distribution from which the samples are generated. This ``worst-case'' approach incurs bounds dependent upon the ambient dimension $D$ in general (even in the standard case of data lying on one hyperplane). 


In \cite{MMS:NoisyDictionaryLearning} we proposed to attack the dictionary learning problem under geometric assumptions on the data, namely that the data lies close to a low-dimensional set $\calM$.
There are of course various possible geometric assumptions, the simplest one being that $\calM$ is a single $d$-dimensional subspace. For this model Principal Component Analysis (PCA) \citep[see][]{Pearson_PCA,Hotelling_PCA1,Hotelling_PCA2} is an effective tool to estimate the underlying plane. 
More generally, one may assume that data lie on a union of several low-dimensional planes instead of a single one. The problem of estimating multiple planes, called subspace clustering, is more challenging
\citep[see][]{RANSAC-FB-ACM81,KS-Ho-CVPR03,GPCA-VMS-PAMI05, LSA-YP-ECCV06,ALC-MDHW-PAMI07,GPCA-MYDF-SiamRev08,sccFoCM09,SSC-EV-CVPR09,LBF-ZSWL-CVPR10,Liu2010,CM:CVPR2011}. 
This model was shown effective in  various applications, including image processing \citep{RANSAC-FB-ACM81}, computer vision \citep{KS-Ho-CVPR03} and motion segmentation \citep{LSA-YP-ECCV06}. 

A different type of geometric model gives rise to manifold learning, where $\calM$ is assumed to be a $d$-dimensional manifold isometrically embedded in $\RR^D$, see \citep{isomap,RSLLE,belkin:nc,DG_HessianEigenmaps,DiffusionPNAS,DiffusionPNAS2,ZhangZha} and many others.
It is of interest to move beyond this model to even more general geometric models, for example where the regularity of the manifold is reduced, and data is not forced to lie exactly on a manifold, but only close to it.

Geometric Multi-Resolution Analysis (GMRA) was proposed in  \citet{CM:MGM2}, and its finite-sample performance was analyzed in \citet{MMS:NoisyDictionaryLearning}. 
In GMRA, geometric approximations of $\calM$ are constructed with multiscale techniques that have their roots in geometric measure theory, harmonic analysis and approximation theory.
GMRA performs a multiscale tree decomposition of data and build multiscale low-dimensional geometric approximations to $\calM$.
Given data, we run the cover tree algorithm \citep{LangfordICML06-CoverTree} to obtain a multiscale tree in which every node is a subset of $\calM$, called a dyadic cell, and all dyadic cells at a fixed scale form a partition of $\calM$. 
After the tree is constructed, we perform PCA on the data in each cell to locally approximate $\calM$ by the $d$-dimensional principal subspace so that every point in that cell is only encoded by the $d$ coefficients in principal directions. 
At a fixed scale $\calM$ is approximated by a piecewise
linear set.
%
In \cite{CM:MGM2} the performance of GMRA for volume measures on a $\mathcal{C}^{s}, s \in (1,2]$ Riemannian manifold was analyzed in the continuous case (no sampling), and the effectiveness of GMRA was demonstrated empirically on simulated and real-world data.
In \citet{MMS:NoisyDictionaryLearning}, the approximation error of $\calM$ was estimated in the non-asymptotic regime with $n$ i.i.d. samples from a measure $\rho$, satisfying certain technical assumptions, supported on a tube of a $\calC^2$ manifold of dimension $d$ isometrically embedded in $\RR^D$.  The probability bounds in \cite{MMS:NoisyDictionaryLearning} depend on $n$ and $d$, but not on $D$, successfully avoiding the curse of dimensionality caused by the ambient dimension. The assumption that $\rho$ is supported in a tube around a manifold can account for noise and does not force the data to lie exactly on a smooth low-dimensional manifold.

In \citet{CM:MGM2} and \citet{MMS:NoisyDictionaryLearning}, GMRA approximations are constructed on uniform partitions in which all the cells have similar diameters.
However, when the regularity, such as smoothness or curvature, weighted by the $\rho$ measure, of $\calM$ varies at different scales and locations, uniform partitions are not optimal.
Inspired by the adaptive methods in classical multi-resolution analysis \citep[see][among many others]{DJ1,DJ2,CohenDGO,DeVore:UniversalAlgorithmsLearningTheoryI,DeVore:UniversalAlgorithmsLearningTheoryII}, we propose an adaptive version of GMRA to construct low-dimensional geometric approximations of $\calM$ on an adaptive partition and provide finite sample performance guarantee for a much larger class of geometric structures $\calM$ in comparison with \citet{MMS:NoisyDictionaryLearning}. 

Our main result (Theorem \ref{thm3}) in this paper may be summarized as follows: \commentout{
\textcolor{blue}{statement in terms of $\ep$:}
suppose that the probability measure $\rho$ is supported on or near a compact $d$-dimensional Riemannian manifold $\calM \hookrightarrow \RR^D$ ($d\ge 3$). Let $s$ be a regularity parameter of $\rho$ in Definition \ref{defBs}, which allows us to consider $\calM$'s and $\rho$'s with highly nonuniform regularity.
For a given accuracy $\ep$, and regularity parameter $s$, if $n_\ep \gtrsim (1/\ep)^{\frac{2s+d-2}{s}}\log (1/\ep)$ i.id. samples are taken from $\rho$, then, with high probability, adaptive GMRA outputs a dictionary $\widehat{\Phi}_\ep = \{\widehat\phi_i\}_{i \in \calJ_\ep}$, an encoding operator $\hcalD_\ep: \RR^D \rightarrow \RR^{\calJ_\ep}$ and a decoding operator $\hcalD_\ep^{-1}: \RR^{\calJ_\ep} \rightarrow \RR^D$. For every $x \in \RR^D$, $\|\hcalD_\ep x\|_0 \le d+1$ (i.e. only $d+1$ entries are non-zero), and the Mean Squared Error (MSE) satisfies 
$${\rm MSE} := \EE_{x \sim \rho}[\| x -\hcalD_\ep^{-1} \hcalD_\ep x\|^2] \lesssim \ep^2.$$
As for computational complexity, constructing $\widehat\Phi_\ep$ takes $\calO((C^d+d^2)D \ep^{-\frac{2s+d-2}{s}}\log(1/\ep))$ and computing $\hcalD_\ep x$ takes $\calO(d(D+d^2)\log(1/\ep))$. While we state the results, including the computational complexity, in terms of the requested accuracy $\ep$, there are equivalent formulations in terms of the sample size $n$. We will use the latter format in the rest of the paper.}
Let $\rho$ be a probability measure supported on or near a compact $d$-dimensional Riemannian manifold $\calM \hookrightarrow \RR^D$, with $d\ge 3$. 
Let $\rho$ admit a multiscale decomposition satisfying the technical assumptions A1-A5 in section \ref{sectree} below.
Given $n$ i.i.d. samples are taken from $\rho$, the intrinsic dimension $d$, and a parameter $\kappa$ large enough, adaptive GMRA outputs a dictionary $\widehat{\Phi}_n = \{\widehat\phi_i\}_{i \in \calJ_n}$, an encoding operator $\hcalD_n: \RR^D \rightarrow \RR^{\calJ_n}$ and a decoding operator $\hcalD_n^{-1}: \RR^{\calJ_n} \rightarrow \RR^D$. With high probability, for every $x \in \RR^D$, $\|\hcalD_n x\|_0 \le d+1$ (i.e. only $d+1$ entries are non-zero), and the Mean Squared Error (MSE) satisfies 
$${\rm MSE} := \EE_{x \sim \rho}[\| x -\hcalD_n^{-1} \hcalD_n x\|^2] \lesssim \left( \lognn\right)^{\frac{2s}{2s+d-2}}.$$
Here $s$ is a regularity parameter of $\rho$ as in definition \ref{defBs}, which allows us to consider $\calM$'s and $\rho$'s with nonuniform regularity, varying at different locations and scales. Note that the algorithm does not need to know $s$, but it automatically adapts to obtain a rate that depends on $s$. We believe, but do not prove, that this rate is indeed optimal.
As for computational complexity, constructing $\widehat\Phi_n$ takes $\calO((C^d+d^2)D n \log n)$ and computing $\hcalD_n x$ only takes $\calO(d(D+d^2)\log n)$, which means we have a fast transform mapping data to their sparse encoding on the dictionary.

In adaptive GMRA, the dictionary is composed of the low-dimensional planes on adaptive partitions and the encoding operator transforms a point to the local $d$ principal coefficients of the data in a piece of the partition.
We state this results in terms of encoding and decoding to stress that learning the geometry in fact yields efficient representations of data, which may be used 
for performing signal processing tasks in a domain where the data admit a sparse representation, e.g. in compressive sensing or estimation problems \citep[see][]{IM:GMRA_CS,6410789,ArminWakin}. 
Adaptive GMRA is designed towards robustness, both in the sense of tolerance to noise and to model error (i.e. data not lying on a manifold). We assume $d$ is given throughout this paper. If not, we refer to \citet{LMR:MGM1, MM:MultiscaleDimensionalityEstimationAAAI,MM:MultiscaleDimensionalityEstimationSSP} for the estimation of intrinsic dimensionality.


 The paper is organized as follows. Our main results, including the construction of GMRA, adaptive GMRA and their finite sample analysis, are presented in Section \ref{secmain}. We show numerical experiments in Section \ref{secnum}. The detailed analysis of GMRA and adaptive GMRA is presented in Section \ref{secGMRA}. 
In Section \ref{secdisex}, we discuss the computational complexity of adaptive GMRA and extend our work to adaptive orthogonal GMRA.
Proofs are postponed till the appendix.

\noindent{\bf{Notation}}.
We will introduce some basic notation here. $f \lesssim g$ means that there exists a constant $C$ independent on any variable upon which $f$ and $g$ depend, such that $f \le C g$. Similarly for $\gtrsim$. $f \asymp g$ means that both $f \lesssim g$ and $f \gtrsim g$ hold. 
The cardinality of a set $A$ is denoted by $\#A$.
For $x \in \RR^D$, $\|x\|$ denotes the Euclidean norm 
and $B_r(x)$ denotes the Euclidean ball of radius $r$ centered at $x$.
Given a subspace $V\in \RR^D$, we denote its dimension by $\dim(V)$ and the orthogonal projection onto $V$ by $\proj_V$. If $A$ is a linear operator on $\RR^D$, $||A||$ is its operator norm. 
The identity operator is denoted by $\bbI$. 

\section{Main results}
\label{secmain}

GMRA was proposed in \citet{CM:MGM2} to efficiently represent points on or near a low-dimensional manifold in high dimensions.
We refer the reader to that paper for details of the construction, and we summarize here the main ideas in order to keep the presentation self-contained.
The construction of GMRA involves the following steps:
\begin{enumerate}
\item[(i)] construct  a {\em multiscale tree $\boldsymbol\calT$} and the associated decomposition of $\calM$ into nested cells $\{\Cjk\}_{k\in\calK_j,j\in\ZZ}$ where $j$ represents scale and $k$ location;

\item[(ii)] perform {\em local PCA} on each $\Cjk$: let the mean (``center'')  be $\cjk$ and the $d$-dim principal subspace $\Vjk$. Define $\calP^{j,k}(x) := c^{j,k}+\proj_{V^{j,k}}(x-c^{j,k})$. 

\item[(iii)] construct a {\em \lq\lq difference\rq\rq\ subspace $W_{j+1,k'}$} capturing $\calP^{j,k}(\Cjk)-\calP^{j+1,k'}(C_{j+1,k'})$, for each $C_{j+1,k'}\subseteq\Cjk$ (these quantities are associated with the refinement criterion in adaptive GMRA).
\end{enumerate}
$\calM$ may be approximated, at each scale $j$, by its projection $\calP_{\Lambda_j}$ onto the family of linear sets $\Lambda_j:=\{\calPjk(\Cjk)\}_{k\in\calK_j}$.  For example, linear approximations of the S manifold at scale $6$ and $10$ are displayed in Figure \ref{FigDemo1}. In a variety of distances, $\calP_{\Lambda_j}(\calM)\rightarrow\calM$.
\begin{figure}[t]
 \centering
  \subfigure[S manifold]{
 \includegraphics[width=4.8cm,height=4cm]{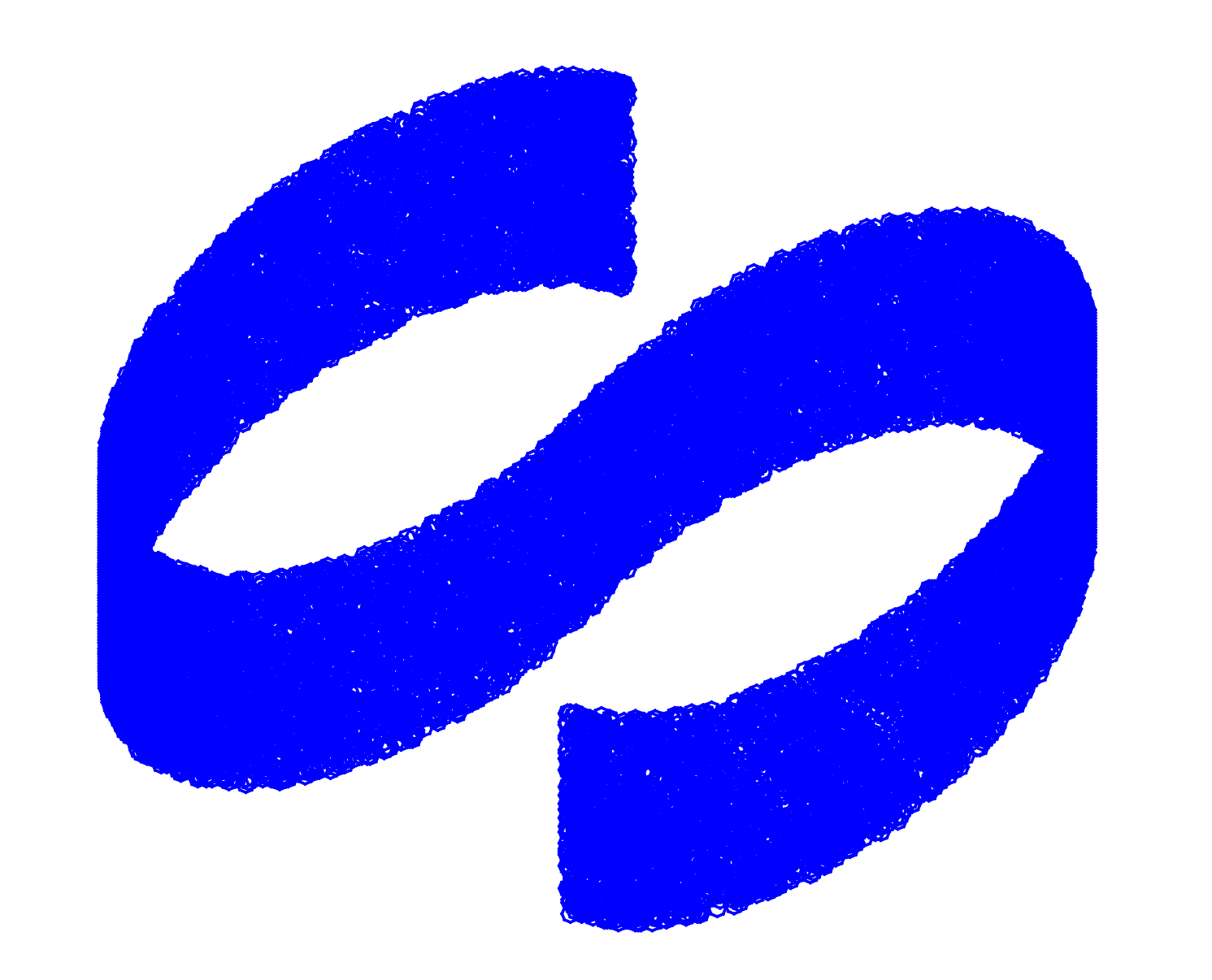}
  }
  \hspace{-0.3cm}
  \subfigure[scale 6]{
 \includegraphics[width=4.8cm,height=4cm]{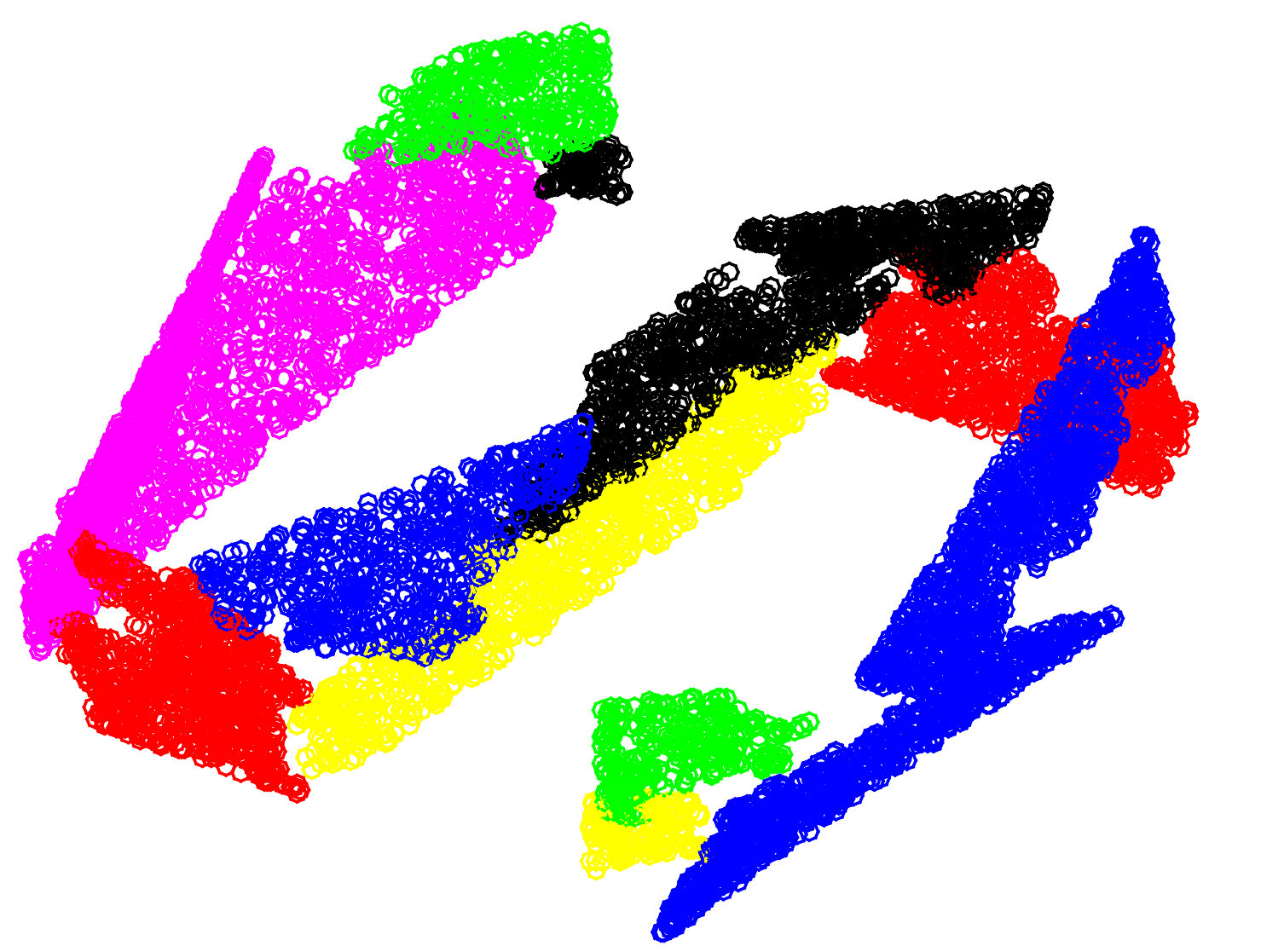}
  }
  \hspace{-0.3cm}
  \subfigure[scale 10]{
 \includegraphics[width=4.8cm,height=4cm]{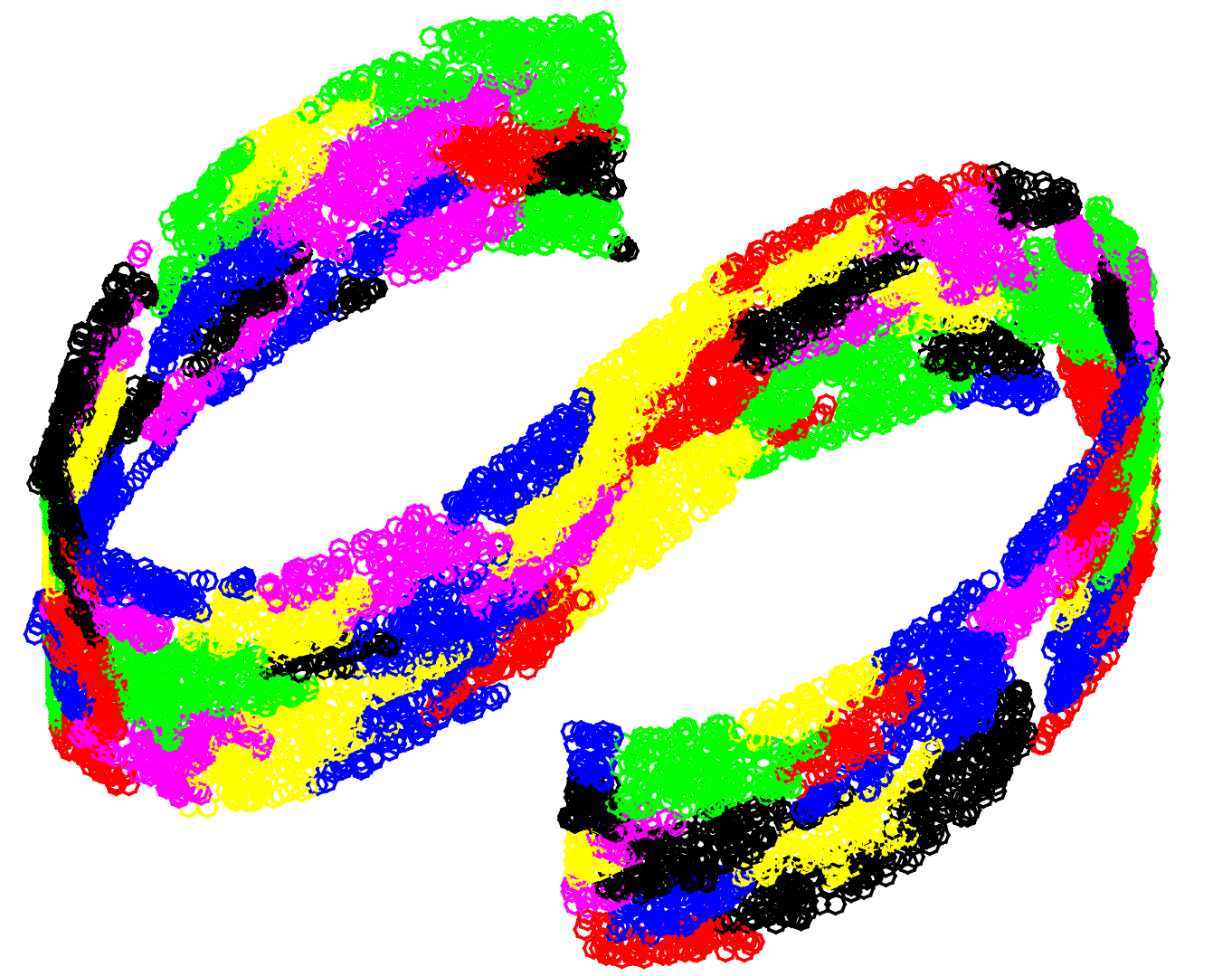}
  }
    \caption{(a) S manifold; (b,c) Linear approximations at scale $6,10$.
    }
  \label{FigDemo1}
\end{figure}
In practice $\mathcal{M}$ is unknown, and the construction above is carried over on training data, and its result is random with the training samples. Naturally we are interested in the performance of the construction on new samples. This is analyzed in a setting of ``smooth manifold+noise'' in \citet{MMS:NoisyDictionaryLearning}.
When the regularity (such as smoothness or curvature) of $\calM$ varies at different locations and scales, linear approximations  on fixed uniform partitions are not optimal. Inspired by adaptive methods in classical multi-resolution analysis \citep[see][]{CohenDGO,DeVore:UniversalAlgorithmsLearningTheoryI,DeVore:UniversalAlgorithmsLearningTheoryII}, we propose an adaptive version of GMRA which learns adaptive and near-optimal approximations.

We will start with the multiscale tree decomposition in Section \ref{sectree} and present GMRA and adaptive GMRA in Section \ref{secLMR:MGM1} and \ref{secALMR:MGM1} respectively. 

\subsection{Multiscale partitions and trees}
\label{sectree}

A multiscale set of partitions of $\calM$ with respect to probability measure $\rho$ is a family of sets $\{C_{j,k}\}_{k\in\calK_j,j\in \ZZ}$, called dyadic cells, satisfying Assumptions (A1-A5) below for all integers $j \ge \jmin$:

\begin{description}
\item[(A1)] for any $k\in \calK_{j}$ and $k'\in \calK_{j+1}$, either $C_{j+1,k'}\subseteq C_{j,k}$ or $\rho(C_{j+1,k'}\cap C_{j,k}) = 0$. We denote the children of $\Cjk$ by $\Child(C_{j,k}) = \{C_{j+1,k'}: C_{j+1,k'} \subseteq C_{j,k}\}$. We assume that $\amin\le\#\Child(C_{j,k})\le\amax$.
Also for every $\Cjk$, there exists a unique $k' \in \calK_{j-1}$ such that $C_{j,k} \subseteq C_{j-1,k'}$. 
We call $C_{j-1,k'}$ the parent of $C_{j,k}$.

\item[(A2)] $\rho(\calM \setminus \cup_{k\in\calK_j}C_{j,k}) = 0$, i.e. $\Lam_j :=\{C_{j,k}\}_{k\in\calK_j}$ is a cover for $\calM$.

\item[(A3)] $\exists\theta_1>0\,:\,\#\Lam_j  \le 2^{jd}/\theta_1$.

\item[(A4)] $\exists\theta_2>0$ such that, if $x$ is drawn from $\rho|_{\Cjk}$, then a.s. $\|x-c_{j,k}\| \le \theta_2 2^{-j}.$

\item[(A5)] Let $\lambda_1^{j,k} \ge \lambda_2^{j,k} \ge \ldots \ge\lambda_D^{j,k}$ be the eigenvalues of the covariance matrix $\Sjk$ of $\rho|_{\Cjk}$, defined in Table \ref{TableGMRA}. Then:
\begin{itemize}
\item[(i)] $\exists\theta_3>0$ such that $\forall j\ge \jmin$ and $k \in \calK_j$, $\lambda_d^{j,k} \ge \theta_3 {2^{-2j}}/{d}$,

\item[(ii)] $\exists \theta_4 \in (0,1)$ such that $\lambda_{d+1}^{j,k} \le \theta_4 \lambda_{d}^{j,k}$. 

\end{itemize}
\end{description}
(A1) implies that the $\{\Cjk\}_{k\in\calK_j,j\ge\jmin}$ are associated with a tree structure, and with some abuse of notation we call the above tree decompositions.
(A1)-(A5) are natural assumptions, easily satisfied by natural multiscale decompositions when $\calM$ is a $d$-dimensional manifold isometrically embedded in $\RR^D$: see the work \citep[][]{MMS:NoisyDictionaryLearning} for a detailed discussion.
(A2) guarantees that the cells at scale $j$ form a partition of $\calM$; (A3) says that there are at most $2^{jd}/\theta_1$ dyadic cells at scale $j$. (A4) ensures $\diam(\Cjk) \lesssim 2^{-j}$. When $\calM$ is a $d$-dimensional manifold, (A5)(i) is the condition that the best rank $d$ approximation to $\Sjk$ is close to the covariance matrix of a $d$-dimensional Euclidean ball, while (A5)(ii) imposes that the $(d+1)$-th eigenvalue is smaller that the $d$-th eigenvalue, i.e. the set has significantly larger variances in $d$ directions than in all the remaining ones.

We will construct such $\{\Cjk\}_{k\in\calK_j,j\ge\jmin}$ in Section \ref{secCoverTree}.
In our construction (A1-A4) is satisfied when $\rho$ a doubling probability measure\footnote{$\rho$ is doubling if there exists $C_1>0$ such that  $C_1^{-1} r^d \le \rho(\calM \cap B_r(x)) \le C_1 r^d$ for any $x \in \calM$ and $r>0$. $C_1$ is called the doubling constant of $\rho$.} \citep[see][]{MR1096400,DengHan}. 
If we further assume that $\calM$ is a $d$-dimensional $\calC^{s},s \in (1,2]$ closed Riemannian manifold isometrically embedded in $\RR^D$, then (A5) is satisfied as well (See Proposition \ref{propcovertree2}). 

\commentout{
\textcolor{blue}{OLD:}
More generally, $\rho$ may be supported on a tube of the manifold.
For example, let $\calM$ be a $d$-dimensional $\calC^{1+\alpha}$ Riemannian manifold with reach $\tau$ \footnote{The reach of $\calM$ is the largest number $t$ to have the property that any point at a distance $r<t$ from $\calM$ has a unique nearest point in $\calM$.} isometrically embedded in $\RR^D$ and 
$\calM_\sigma$ \footnote{$\calM_\sigma = \{y\in \RR^D : \inf_{x\in \calM} \|x-y\| < \sigma\}$.}
be the $\sigma$ tube of $\calM$. 
Assume that $\rho$ is the volume measure on $\calM_\sigma$ (or $\rho$ and the volume measure are absolutely continuous with respect to each other)
and
 $\Cjk$ uniformly contains a ball of radius $\asymp 2^{-j}$ for all $k \in \calK_j, j \ge \jmin$ (i.e., there exists $c(\rho)>0$ such that on every $\Cjk$, there exists a point $z_{j,k}$ such that $B(z_{j,k},c2^{-j}) \cap \calM_\sigma \subset \Cjk$).
Then as long as $2^{-j} > 8 \sigma$ and $3\cdot 2^{-j}+\sigma < \tau/8$, 
we have: (i) $\theta_1 \ge \frac{\Vol(B_d(0,1))}{2^{4d}\Vol(\calM)}(\frac{\tau-\sigma}{\tau+\sigma})^d$  where $B_d(0,1)$ is the Euclidean ball in $\RR^d$ of radius $1$ centered at the origin \citep[see][Theorem 7]{MMS:NoisyDictionaryLearning};
(ii) Our tree construction by the cover tree algorithm guarantees that with high probability \textcolor{blue}{$\theta_2 \le 4$} for $j \lesssim (1/d) \log n$;
(iii) $\theta_3 \ge (1/2)^{4d+8}(1+\sigma/\tau)^{-d}$  \citep[see][Theorem 7]{MMS:NoisyDictionaryLearning} and $\sum_{l=d+1}^D \lambda_l^{j,k} \lesssim  \sigma^2 + 2^{-2(1+\alpha)j}$ \citep[see][Theorem 2.3]{CM:MGM2}.
}

Some notation: a master tree $\calT$ is associated with $\{\Cjk\}_{k\in\calK_j,j\ge\jmin}$ (using property (A1)), constructed on $\calM$; since $\Cjk$'s at scale $j$ have similar diameters, $\Lam_j := \{ \Cjk\}_{k \in \calK_j}$ is called a uniform partition at scale $j$. 
It may happen that at the coarsest scales conditions (A3)-(A5) are satisfied but with very poor constants $\bf\theta$: it will be clear that in all that follows we may discard a few coarse scales, and only work at scales that are fine enough and for which (A3)-(A5) truly capture the local geometry of $\mathcal{M}$.

A proper subtree $\tcalT$ of $\calT$ is a collection of nodes of $\calT$ with the properties: (i) the root node is in $\tcalT$, (ii) if $C_{j,k}$ is in $\tcalT$ then the parent of $\Cjk$ is also in $\tcalT$. Any finite proper subtree $\tcalT$ is associated with a unique partition $\Lam = \Lam(\tcalT)$ which consists of its outer leaves, by which we mean those $C_{j,k} \in \calT$ such that $C_{j,k} \notin \tcalT$ but its parent is in $\tcalT$.

\begin{center}
\begin{table}[t]
\renewcommand{\arraystretch}{2}
\centering
\resizebox{0.8\columnwidth}{!}{
\begin{tabular}{ |c || c  |c |}
\hline
   & GMRA&  Empirical GMRA 
   \\
   \hline
   \hline
\parbox{3cm}{\centering Linear projection \\ on $\Cjk$ }
  & 
  $\calP_{j,k}(x) := c_{j,k} +\proj_{V_{j,k}} (x-c_{j,k})$
  &   $\hcalP_{j,k}(x) := \hcjk +\proj_{\hVjk} (x-\hcjk)$

   \\
   \hline
  \parbox{3cm}{\centering Linear projection \\ at scale $j$} & $\calP_j := \sum_{k\in \calK_j} \calP_{j,k} \mathbf{1}_{{j,k}}$  &
  $\hcalP_j := \sum_{k \in \calK_j} \hcalP_{j,k} \mathbf{1}_{{j,k}}$ 
  \\
 \hline
  Measure & $\rho(\Cjk)$  & $\hrho(\Cjk) = \hnjk /n$\\
      \hline
  Center  & $\cjk:= \EE_{j,k} x $
  & 
  $ \hc_{j,k} := \frac{1}{\widehat{n}_{j,k}}\textstyle\sum_{x_i \in \Cjk} x_i $
  \\
       \hline
    \parbox{3cm}{\centering Principal \\ subspaces}
    &\parbox{7cm}{\centering $\Vjk $ minimizes \\  $ \EE_{j,k} \|x-c_{j,k}-\proj_V(x-c_{j,k})\|^2$ \\ among $d$-dim subspaces
    }
    & \parbox{7cm}{\centering
  $\hV_{j,k} $ minimizes
  \\ { $\frac{1}{\hn_{j,k}} 
 \textstyle \sum_{x_i \in \Cjk} \|x-\hc_{j,k} -\proj_V (x-\hc_{j,k})\|^2$}
  \\
 among $d$-dim subspaces}
      \\
       \hline
  \parbox{3cm}{\centering Covariance \\ matrix } & $\Sigma_{j,k} := \EE_{j,k}  (x-c_{j,k})(x-c_{j,k})^T $
& 
$\hS_{j,k} := \frac{1}{\hnjk} \sum_{x_i \in \Cjk} (x_i-\hc_{j,k})(x_i-\hc_{j,k})^T $
  \\
      [5pt]
         \hline
$\langle \calP X , \calQ X\rangle$	&	$\int_\calM \langle \calP x , \calQ x \rangle d\rho$	& 	${1}/{n}\sum_{x_i \in \calX_n} \langle \calP x_i , \calQ x_i \rangle$ \\ [5pt] \hline
$\|\calP X\|$ & $\left(\int_\calM \|\calP x\|^2 d\rho \right)^{\frac 1 2}$ & $\left({1}/{n}\sum_{x_i\in\calX_n} \|\calP x_i\|^2  \right)^{\frac 1 2}$ \\ [5pt] \hline
\end{tabular}
}
\caption{$\chijk$ is the indicator function on $\Cjk$ (i.e.,$\chijk(x)=1$ if $x\in\Cjk$ and $0$ otherwise).
Here $\EE_{j,k}$ stands for expectation with respect to the conditional distribution $d\rho|_{\Cjk} $.  
The measure of $\Cjk$ is $\rho(\Cjk)$ and the empirical measure is $\hrho(\Cjk) =\hnjk/n$ where $\hnjk$ is the number of points in $\Cjk$. 
$V_{j,k}$ and $\hVjk$ are the eigen-spaces associated with the largest $d$ eigenvalues of $\Sjk$ and $\hSjk$ respectively.
Here $\calP,\calQ$: $\calM\to \RR^D$ are two operators. 
}
\label{TableGMRA}
\end{table}
\end{center}
\subsection{Empirical GMRA}
In practice the master tree $\calT$ is not given, nor can be constructed since $\calM$ is not known: we will construct one on samples by running a variation of the cover tree algorithm \citep[see][]{LangfordICML06-CoverTree}. 
From now on we denote the training data by $\calX_{2n}$. We randomly split the data into two disjoint groups such that $\calX_{2n} = \calX'_n \cup \calX_n$ where $\calX'_n =\{x'_1,\ldots,x'_{n}\}$ and $\calX_n =\{x_1,\ldots,x_n\}$, apply the cover tree algorithm \citep[see][]{LangfordICML06-CoverTree} on $\calX_n'$ to construct a tree satisfying  (A1-A5) (see section \ref{secCoverTree}).  After the tree is constructed, we assign points in the second half of data $\calX_n$, to the appropriate cells. 
In this way we obtain a family of multiscale partitions for the points in $\calX_n$, which we truncate to the largest subtree whose leaves contain at least $d$ points in $\calX_n$. This subtree is called the {\em{data master tree}}, denoted by $\calT^n$. 
We then use $\calX_n$ to perform local PCA to obtain the empirical mean $\hcjk$ and the empirical $d$-dimensional principal subspace $\hVjk$ on each $\Cjk$. Define the empirical projection $\hcalPjk(x) := \hcjk + \proj_{\hVjk}(x-\hcjk)$ for $x\in C_{j,k}$.
Table \ref{TableGMRA} summarizes the GMRA objects and their empirical counterparts.

\subsection{Geometric Multi-Resolution Analysis: uniform partitions}
\label{secLMR:MGM1}

GMRA with respect to the distribution $\rho$ associated with the multiscale tree $\calT$ consists a collection of piecewise affine projectors $\{\calP_j: \RR^D \rightarrow \RR^D\}_{j \ge \jmin}$ on the multiscale partitions $\{\Lam_j := \{\Cjk\}_{k \in \calK_j}\}_{j \ge \jmin}$. 
At scale $j$, $\calM$ is approximated by the piecewise linear sets $\{\calPjk(\Cjk)\}_{k \in \calK_j}$. 
In order to understand the empirical approximation error of $\calM$ by the piecewise linear sets $\{\hcalPjk(\Cjk)\}_{k \in \calK_j}$ at scale $j$, we split the error into a squared bias term and a variance term:
\beq
  \EE\|X-\hcalP_j X\|^2
= 
\underbrace{\|X-\calP_j X\|^2}_{\text{bias}^2 } + 
\underbrace{\EE\|\calP_j X-\hcalP_j X\|^2}_{\text{variance}}.
\label{bv1}
\eeq
$ \EE\|X-\hcalP_j X\|^2$ is also called the Mean Square Error (MSE) of GMRA.
To bound the bias term, we need regularity assumptions on the $\rho$, and for the variance term we prove concentration bounds of the relevant quantities around their expected values.

For a fixed distribution $\rho$, the approximation error of $\calM$ at scale $j$, measured by
$\|X-\calP_j X\|:=\|(\bbI-\calP_j)X\|$, decays at a rate dependent on the regularity of $\calM$ in the $\rho$-measure  \citep[see][]{CM:MGM2}. We quantify the regularity of $\rho$ as follows:

\begin{definition}[Model class $\AS$]
\label{defAs}
A probability measure $\rho$ supported on $\calM$ is in $\AS$ if
\beq
\label{eqAs}
|\rho|_{\AS}=\sup_{\calT}\ \inf\{A_0\,:\,\|X-\calP_j X\| \le A_0 2^{-js}, \forall\, j\ge \jmin\}<\infty\,,
\eeq
where $\calT$ varies over the set, assumed non-empty, of multiscale tree decompositions satisfying Assumption (A1-A5).
\end{definition}

We capture the case where the $L^2$ approximation error is roughly the same on every cell with the following definition:
\begin{definition}[Model class $\AS^\infty$]
\label{defAsInf}
A probability measure $\rho$ supported on $\calM$ is in $\AS^\infty$ if
\beq
\label{eqAsInf}
|\rho|_{\AS^{\infty}}=\sup_{\calT}\ \inf\{A_0\,:\,\|(X-\calPjk  X)\chijk\| \le A_0 2^{-js} \sqrt{\rho(\Cjk)}, \ \forall\, k \in \calK_j, j \ge \jmin\}<\infty\
\eeq
where $\calT$ varies over the set, assumed non-empty, of multiscale tree decompositions satisfying Assumption (A1-A5).

\end{definition}

Clearly $\AS^\infty \subset \AS$.
Also, since $\diam(\Cjk) \le 2\theta_2 2^{-j}$, necessarily
$\|(\bbI - \calPjk)\chijk X\| \le \theta_2 2^{-j} \sqrt{\rho(\Cjk)}, \ \forall\, k \in \calK_j, j \ge \jmin$, and therefore $\rho \in \calA_1^\infty$ in any case. 
Finally, these classes contain suitable measures supported on manifolds:

\begin{proposition}
\label{propsmanifold}

Let $\mathcal{M}$ be a closed manifold of class $\mathcal{C}^{s}$, $s\in (1,2]$ isometrically embedded in $\RR^D$, and $\rho$ be a doubing probability measure on $\mathcal{M}$ with the doubling constant $C_1$. Then our construction of $\{\Cjk\}_{k\in\calK_j,j\ge\jmin}$ in Section \ref{secCoverTree} satisfies (A1-A5), and $\rho \in \calA_s^\infty$.
%
%
%



\end{proposition}

Proposition \ref{propsmanifold} is proved in Appendix \ref{s:AppRegSpaces}. 


\begin{example}
We consider the $d$-dim S manifold whose $x_1$ and $x_2$ coordinates are on the $S$ curve and $x_i \in [0,1], i=3,\ldots,d+1$. As stated above, the volume measure on the S manifold is in $\mathcal{A}_2^\infty$. 
Numerically one can identify $s$ from data sampled from $\rho\in\AS$ as the slope of the line approximating $\log_{10} \|X-\calP_j X\|$ as a function of $\log_{10}r_j$ where $r_j$ is the average diameter of $\Cjk$'s at scale $j$.
Our numerical experiments in Figure \ref{FigASBS} (b) give rise to $s \approx 2.0239, 2.1372, 2.173$ when $d=3,4,5$ respectively.

\end{example}

\begin{example}

For comparison we consider the $d$-dimensional Z manifold whose $x_1$ and $x_2$ coordinates are on the $Z$ curve and $x_i \in [0,1], i=3,\ldots,d+1$. Volume measure on the Z manifold is in $\mathcal{A}_{1.5}$ (see appendix \ref{AppZ}).
Our numerical experiments in Figure \ref{FigASBS} (c) give rise to $s \approx 1.5367, 1.6595, 1.6204$ when $d=3,4,5$ respectively.
\end{example}

\commentout{
 In practice, $\rho$ is unknown and we only have access to the data $\calX_n$ based on which we build the empirical GMRA:
\beq
\label{eqGMRAe}
\widehat{\calP}_j := \sum_{k\in \calK_j}\hcalP_{j,k} \mathbf{1}_{j,k}
\eeq
where $\calP_{j,k}$ is the empirical affine projection
\beq
\label{eqGMRAe1}
\hcalP_{j,k}(x) := \hc_{j,k}+\proj_{\hV_{j,k}}(x-\hc_{j,k})
\eeq
with $\hc_{j,k}$ and $\hV_{j,k}$ given by
\begin{align}
\hc_{j,k} & := \frac{1}{\widehat{n}_{j,k}}\sum_{i=1}^n x_i \mathbf{1}_{j,k}(x_i),
\ \widehat{n}_{j,k}=\sum_{i=1}^n \mathbf{1}_{j,k}(x_i),
\label{eqGMRAe2}
\\
\hV_{j,k} & := \underset{\dim V = d}{\argmin} \ \frac{1}{\hn_{j,k}} \sum_{i=1}^n \|x-\hc_{j,k} -\proj_V (x-\hc_{j,k})\|^2\mathbf{1}_{j,k}(x_i).
\label{eqGMRAe3}
\end{align}
In other words, $\hc_{j,k}$ is the empirical conditional mean on $\Cjk$ and $\hV_{j,k}$ is the subspace spanned by eigenvectors corresponding to the $d$ largest eigenvalues of the empirical conditional covariance matrix 
\beq
\hS_{j,k} := \frac{1}{\hnjk} \sum_{i=1}^{n} (x_i-\hc_{j,k})(x_i-\hc_{j,k})^T \mathbf{1}_{j,k}(x_i).
\label{eqGMRAe4}
\eeq
}

The squared bias in \eqref{bv1} satisfies $\|X-\calP_j X\|^2 \le  |\rho|_{\AS}^22^{-2js}$ whenever $\rho \in \AS$. In Proposition \ref{lemma0} we will show that the variance term is estimated in terms of the sample size $n$ and the scale $j$:
$$
\EE \|\calP_j X - \hcalP_j X\|^2 
\le
\frac{d^2 \#\Lam_j \log[\alpha d  \#\Lam_j]}{\beta 2^{2j}n}{=\mathcal{O}\left(\frac{j 2^{j(d-2)}}{n}\right),}
$$
where $\alpha, \beta$ are constants depending on $\theta_2, \theta_3$. In the case $d=1$ both the squared bias and the variance decrease as $j$ increases, so choosing the finest scale of the data tree $\calTn$ yields the best rate of convergence. When $d\ge 2$, the squared bias decreases but the variance increases as $j$ gets large as shown Figure \ref{figBV}, as a manifestation of the bias-variance tradeoff in classical statistical and machine learning, except that it arises here in a geometric setting. By choosing a proper scale $j$ to balance these two terms, we obtain the following rate of convergence for empirical GMRA:

\begin{figure}[t]
  \centering
  \subfigure[Bias and variance tradeoff]{
    \includegraphics[width=0.4\textwidth]{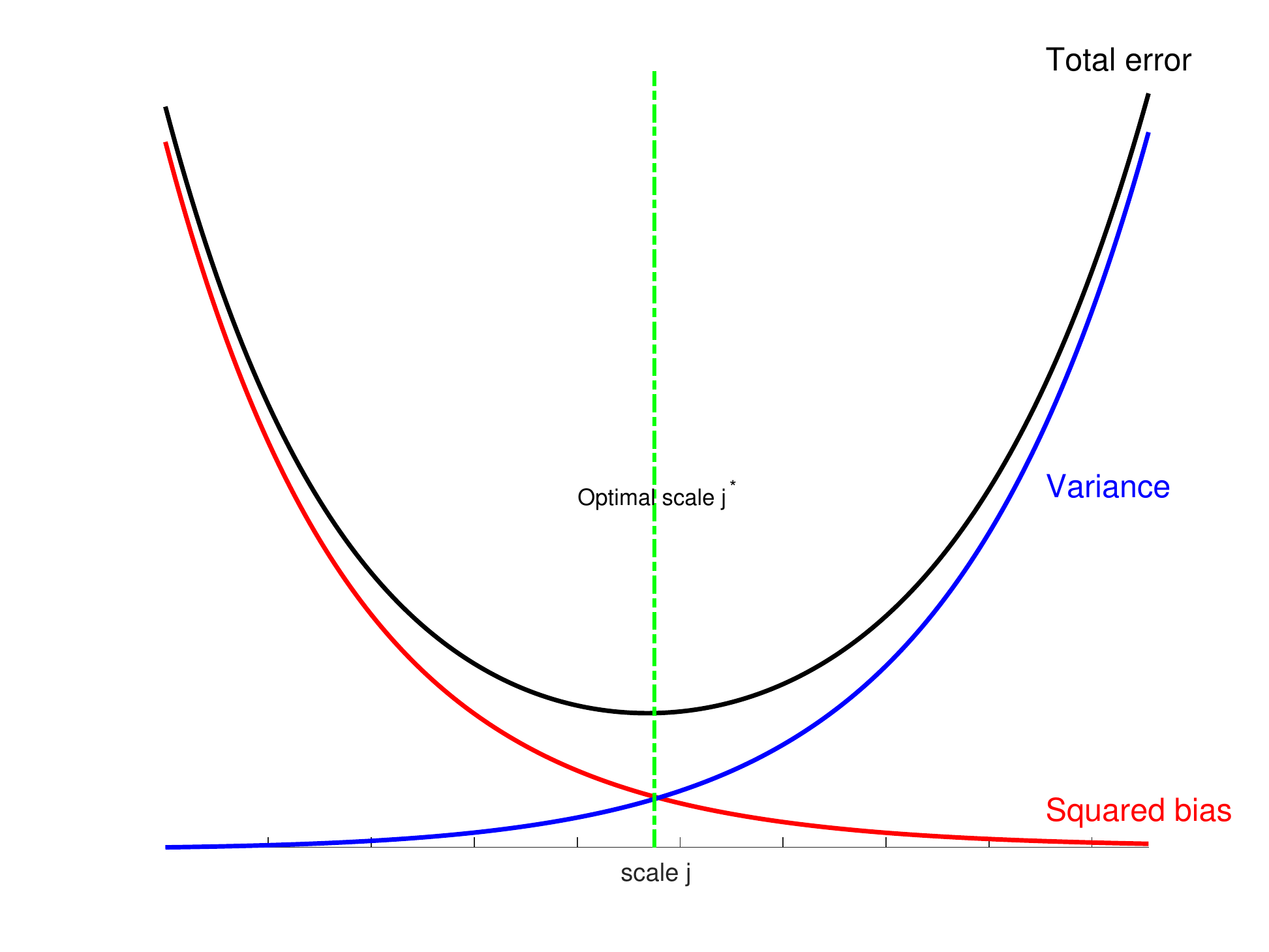}
    }
    \subfigure[Error versus the partition size]{
    \includegraphics[width=0.4\textwidth]{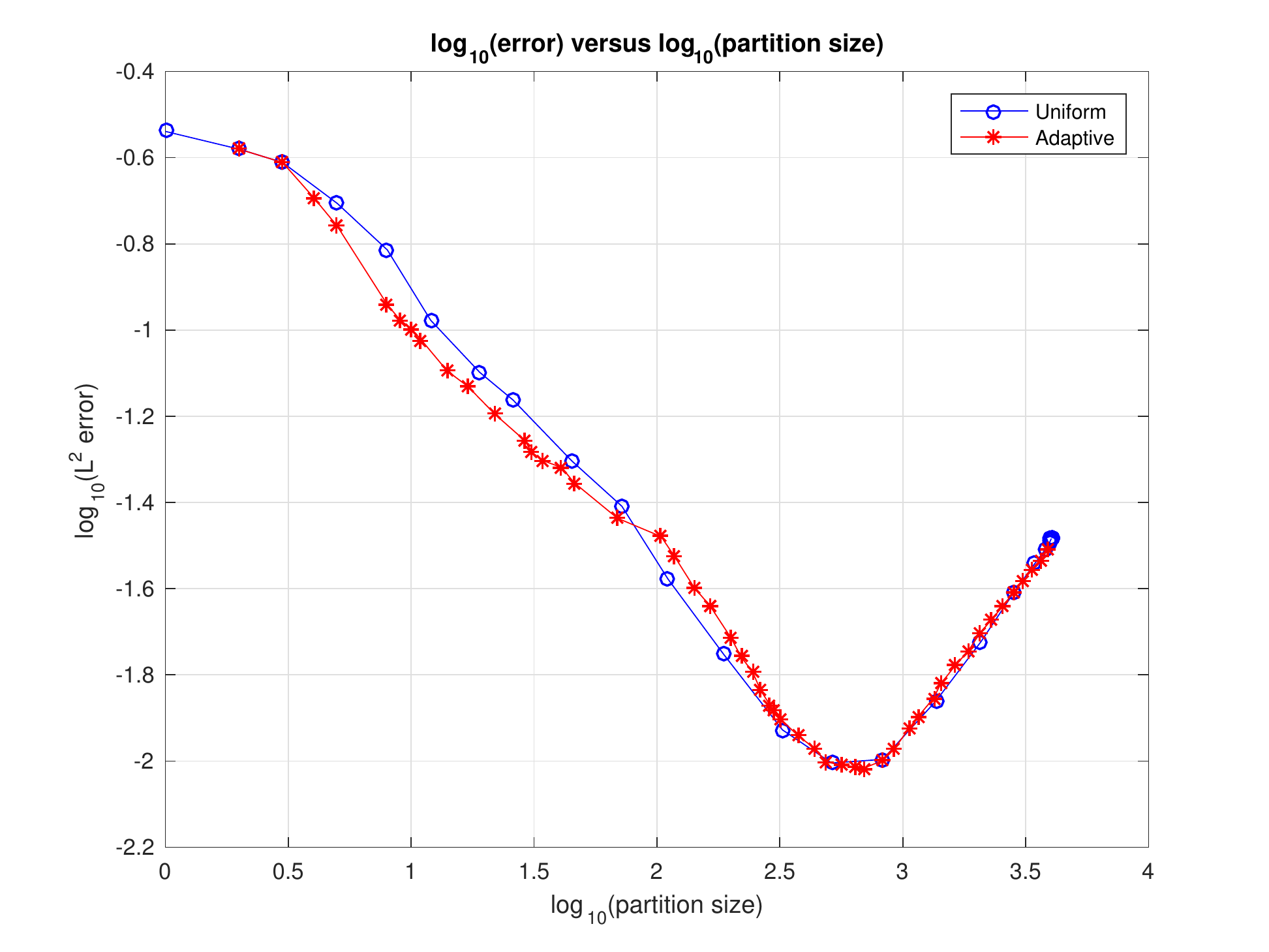}
    }
    \caption{
    %
    (a) Plot of the bias and variance estimates in Eq. \eqref{bv1}, with $s=2,d=5,n=100$. (b) shows the approximation error on test data versus the partition size in GMRA and adaptive GMRA for the 3-dim S manifold.
    When the partition size is between $1$ and $10^{2.8}$, the bias dominates the error so the error decreases; after that, the variance dominates the error, which becomes increasing.
    }
    \label{figBV}
\end{figure}

\begin{theorem}
\label{thm2}
Suppose $\rho \in \calA_s$ for $s\ge1$. Let $\nu>0$ be arbitrary and $\mu >0$.
Let $j^*$ be chosen such that
\begin{equation}
2^{-j^*} =
\begin{cases}
	\mu\lognn  &\text{   for } d=1 \\
         \mu\left( \frac{\log n}{n}\right)^{\frac{1}{2s+d-2}} , &\text{   for } d\ge 2
  \label{jd2}
\end{cases}  
\end{equation} 
then there exists $C_1 := C_1(\theta_1,\theta_2,\theta_3,\theta_4,d,\nu,\mu)$ and $C_2 :=C_2(\theta_1,\theta_2,\theta_3,\theta_4, d,\mu)$ such that:
\begin{align}
\PP\left\{\|X - \hcalP_{j^*} X \| 
\ge (|\rho|_{\calA_s} \mu^s + C_1 )\frac{\log n}{n} 
\right\}
 \le C_2  n^{-\nu}, \quad & \text{for } d=1,
\label{thm2eq01}
\\
\PP\left\{\|X - \hcalP_{j^*} X \| 
\ge (|\rho|_{\calA_s}\mu^s + C_1 )\left (\frac{\log n}{n} \right)^{\frac{s}{2s+d-2}}
\right\}
 \le C_2  n^{-\nu}, \quad
& \text{for } d\ge 2\,.
\label{thm2eq1}
\end{align}

\end{theorem}
Theorem \ref{thm2} is proved in Section \ref{s:gmra}.
In the perspective of dictionary learning, GMRA provides a dictionary $\Phi_{j^*}$ of cardinality $\asymp dn/\log n$ for $d =1$ and of cardinality $ \asymp d (n /\log n)^{\frac{d}{2s+d-2}}$ for $d \ge 2$, so that every $x$ sampled from $\rho$ (and not just samples in the training data) may be encoded with a vector with $d+1$ nonzero entries: one entry encodes the location $k$ of $x$ on the tree, e.g. $(j^*,x) = (j^*,k)$ such that $x \in C_{j^*,k}$, and the other $d$ entries are $\hV_{j^*,x}^T (x-\hc_{j^*,x})$.
We also remind the reader that GMRA automatically constructs a fast transform mapping points $x$ to the vector representing $\Phi_{j^*}$ (See \citet[][]{CM:MGM2,MMS:NoisyDictionaryLearning} for a discussion).
Note that by choosing $\nu$ large enough,
\begin{align*}
\eqref{thm2eq1}
 \implies
{\rm MSE} = \EE\|X-\hcalP_{j^*}X\|^2 \lesssim \left(\lognn\right)^{\frac{2s}{2s+d-2}}\,,
\end{align*}
and \eqref{thm2eq01} implies MSE $\lesssim (\lognn)^2$ for $d=1$.
Clearly, one could fix a desired MSE of size $\ep^2$, and obtain from the above a dictionary of size dependent only of $\ep$ and independent of $n$, for $n$ sufficiently large, thereby obtaining a way of compressing data (see \citet{MMS:NoisyDictionaryLearning} for further discussion on this point).
A special case of Theorem \ref{thm2} with $s=2$ was proved in  \citet{MMS:NoisyDictionaryLearning}. 


\subsection{Geometric Multi-Resolution Analysis: Adaptive  Partitions}
\label{secALMR:MGM1}

\begin{table}[b]
\renewcommand{\arraystretch}{1.4}
\centering
\resizebox{0.8\columnwidth}{!}{
\begin{tabular}{ |c | c  |c| }
\hline
   & Definition (infinite sample) &  Empirical version \\
   [5pt]
   \hline
  Difference operator
  & $\calQjk :=  (\calP_j - \calP_{j+1}) \chijk  $
  & $\hcalQjk :=  (\hcalP_j - \hcalP_{j+1}) \chijk $ 
  \\
     [5pt]
  \hline
  Norm of difference	  & 
  $\Deltajk^2 := \int_{\Cjk}\|\calQjk x\|^2 d\rho$
  & 
  $\hDeltajk^2 := \frac 1 n \sum_{x_i \in \Cjk}
\|\hcalQjk x_i\|^2$
  \\
       [10pt]
       \hline
    \end{tabular}
    }
\caption{Refinement criterion and the empirical version}
\label{TableRefinement}
\end{table}

The performance guarantee in Theorem \ref{thm2} is not fully satisfactory for two reasons: (i) the regularity parameter $s$ is required to be known to choose the optimal scale $j^*$, and this parameter is typically unknown in any practical setting, and (ii) none of the uniform partitions $\{\Cjk\}_{k\in\calK_j}$ will be optimal if the regularity of $\rho$ (and/or $\mathcal{M}$) varies at different locations and scales. This lack of uniformity in regularity appears in a wide variety of data sets: when clusters exist that have cores denser than the remaining regions of space, when trajectories of a dynamical system are sampled that linger in certain regions of space for much longer time intervals than others (e.g. metastable states in molecular dynamics \citep{RZMC:ReactionCoordinatesLocalScaling,ZRMC:PolymerReversal}), in data sets of images where details exist at different level of resolutions, affecting regularity at different scales in the ambient space, and so on. To fix the ideas we consider again one simplest manifestations of this phenomenon in the examples considered above: uniform partitions work well for the volume measure on the S manifold but are not optimal for the volume measure on the Z manifold, for which the ideal partition is coarse on flat regions but finer at and near the corners (see Figure \ref{FigDelta}). In applications, for example to mesh approximation, it is often the case that the point clouds to be approximated are not uniformly smooth and include different levels of details at different locations and scales (see Figure \ref{FigShapePartition}).
Thus we propose an adaptive version of GMRA that will automatically adapts to the regularity of the data and choose a near-optimal adaptive partition.

 \begin{figure}[t]
\centering
 \includegraphics[width=0.33\textwidth]{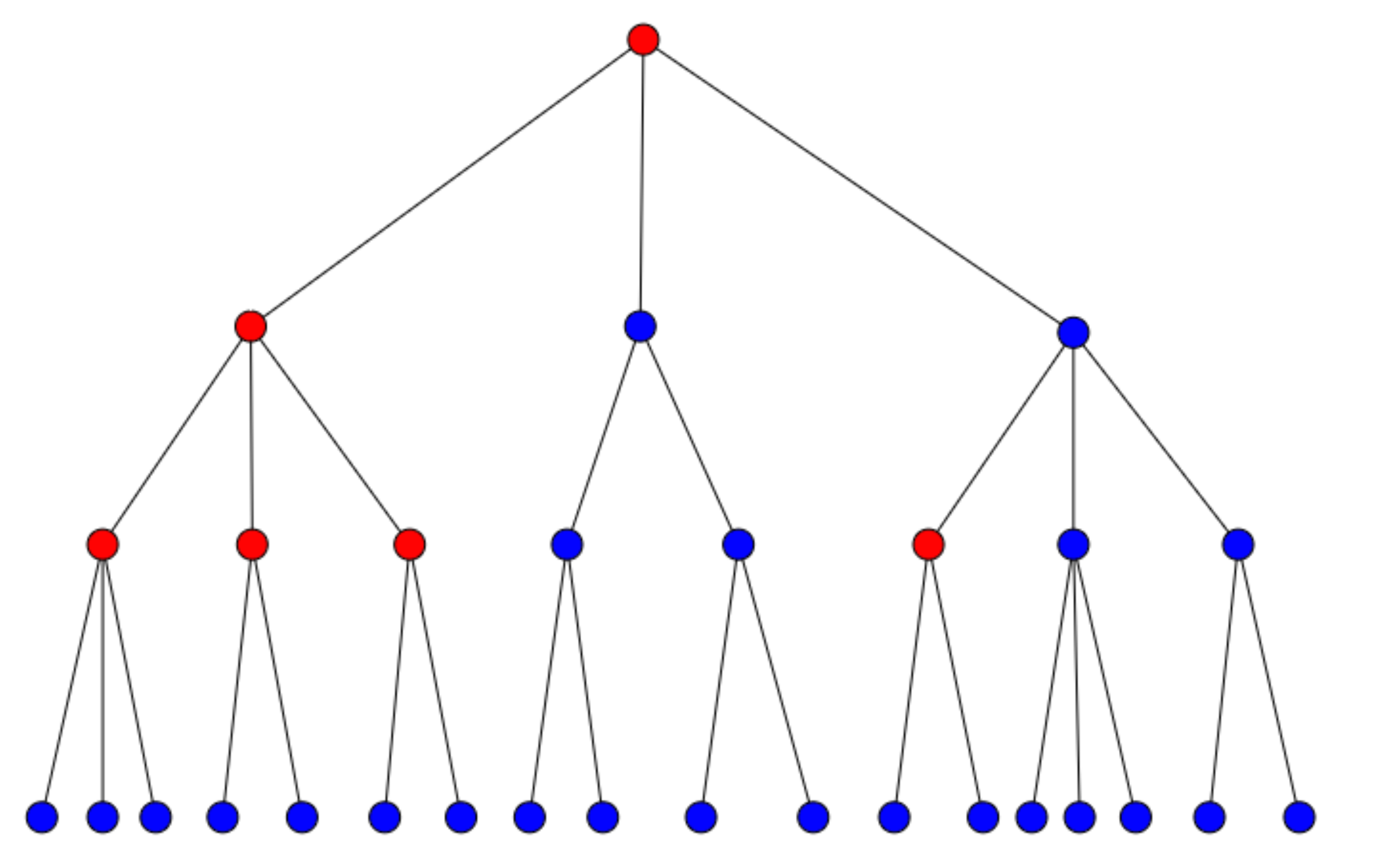}
 \includegraphics[width=0.33\textwidth]{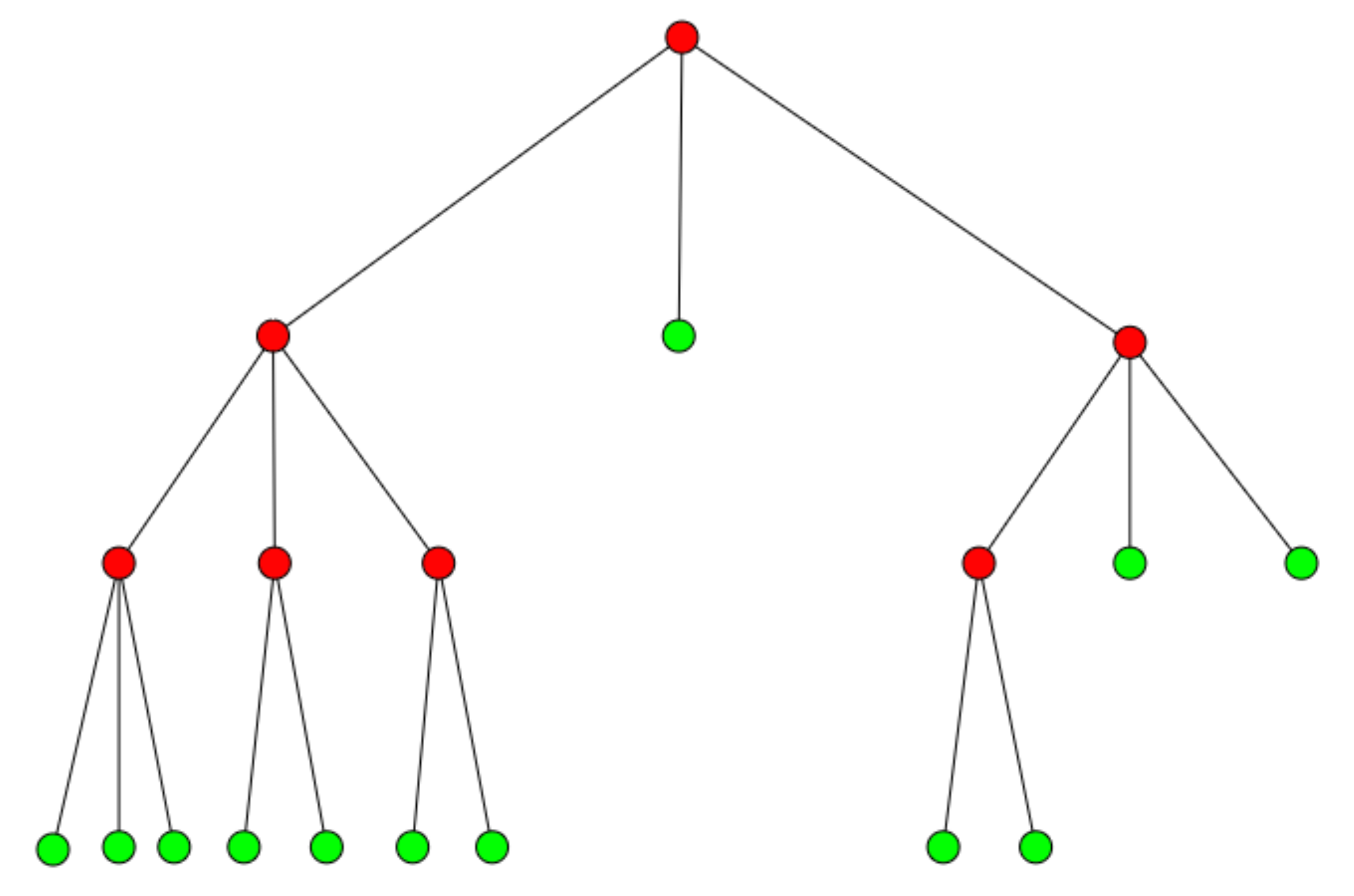}
\caption{Left: a master tree in which red nodes satisfy $\Deltajk \ge 2^{-j} \tau_n$ but blue nodes do not. Right: the subtree of the red nodes is the smallest proper subtree that contains all the nodes satisfying $\Deltajk \ge 2^{-j} \tau_n$, i.e. were red in the figure on the left. Green nodes form the adaptive partition.}
\label{figadaptivetree}
\end{figure}

\begin{algorithm}[t]                      	
\caption{Empirical adaptive GMRA}          	
\label{TableAdaptiveGMRA}		
\begin{algorithmic}[1]                    	
    \REQUIRE  data $\calX_{2n} = \calX_n'\cup\calX_n$, intrinsic dimension $d$, threshold $\kappa$ 
    \ENSURE $\calTn$, $\{C_{j,k}\}$, $\hcalP_{\hLamtaun}$ : multiscale tree, corresponding cells and adaptive piecewise linear projectors on adaptive partition.
    \STATE Construct $\calTn$ and $\{C_{j,k}\}$ from $\calX'_n$
    \STATE Now use $\calX_n$. Compute $\hcalPjk $ and $\hDeltajk$ on every node $\Cjk \in \calTn$.
    \STATE $\hcalTtaun\leftarrow$ smallest proper subtree of $\calTn$ containing all $\Cjk \in \calTn$ : $\hDeltajk \ge 2^{-j}\tau_n$  where $\tau_n = \kappa\sqrt{(\log n)/n}$.
     \STATE $\hLamtaun\leftarrow$ partition associated with outer leaves of $\hcalTtaun$
    \STATE $\hcalP_{\hLamtaun}\leftarrow\sum_{\Cjk \in \hLamtaun} \hcalPjk \chijk.$
\end{algorithmic}
\end{algorithm}
Adaptive partitions may be effectively selected with a criterion that determines whether or not a cell should participate to the adaptive partition.
The quantities involved in the selection and their empirical version are summarized in Table \ref{TableRefinement}. 

We expect $\Deltajk$ to be small on approximately flat regions, and large $\Deltajk$ at many scales at irregular locations. We also expect $\hDeltajk$ to have the same behavior, at least when $\hDeltajk$ is with high confidence close to $\Deltajk$.
We see this phenomenon represented in Figure \ref{FigDelta} (a,b): as $j$ increases, for the S manifold $\|\hcalP_{j+1}x_i- \hcalP_j x_i\|$ decays uniformly at all points, while for the Z manifold, the same quantity decays rapidly on flat regions but remains large at fine scales around the corners.
We wish to include in our approximation the nodes where this quantity is large, since we may expect a large improvement in approximation by including such nodes. However if too few samples exist in a node, then this quantity is not to be trusted, because its variance is large.
It turns out that it is enough to consider the following criterion:
let $\hcalTtaun$ be the smallest proper subtree of $\calTn$ that contains all  $C^{j,k} \in \calTn$ for which $\hDeltajk \ge 2^{-j}\tau_n$ 
where $\tau_n = \kappa\sqrt{(\log n)/n}$.
Crucially, $\kappa$ may be chosen independently of the regularity index (see Theorem \ref{thm3}).
Empirical adaptive GMRA returns piecewise affine projectors on $\hLamtaun$, the partition associated with the outer leaves of $\hcalTtaun$.
Our algorithm is summarized in Algorithm 1.

\begin{figure}[t]
 \centering
 \subfigure[S manifold]{\includegraphics[width=.32\textwidth]{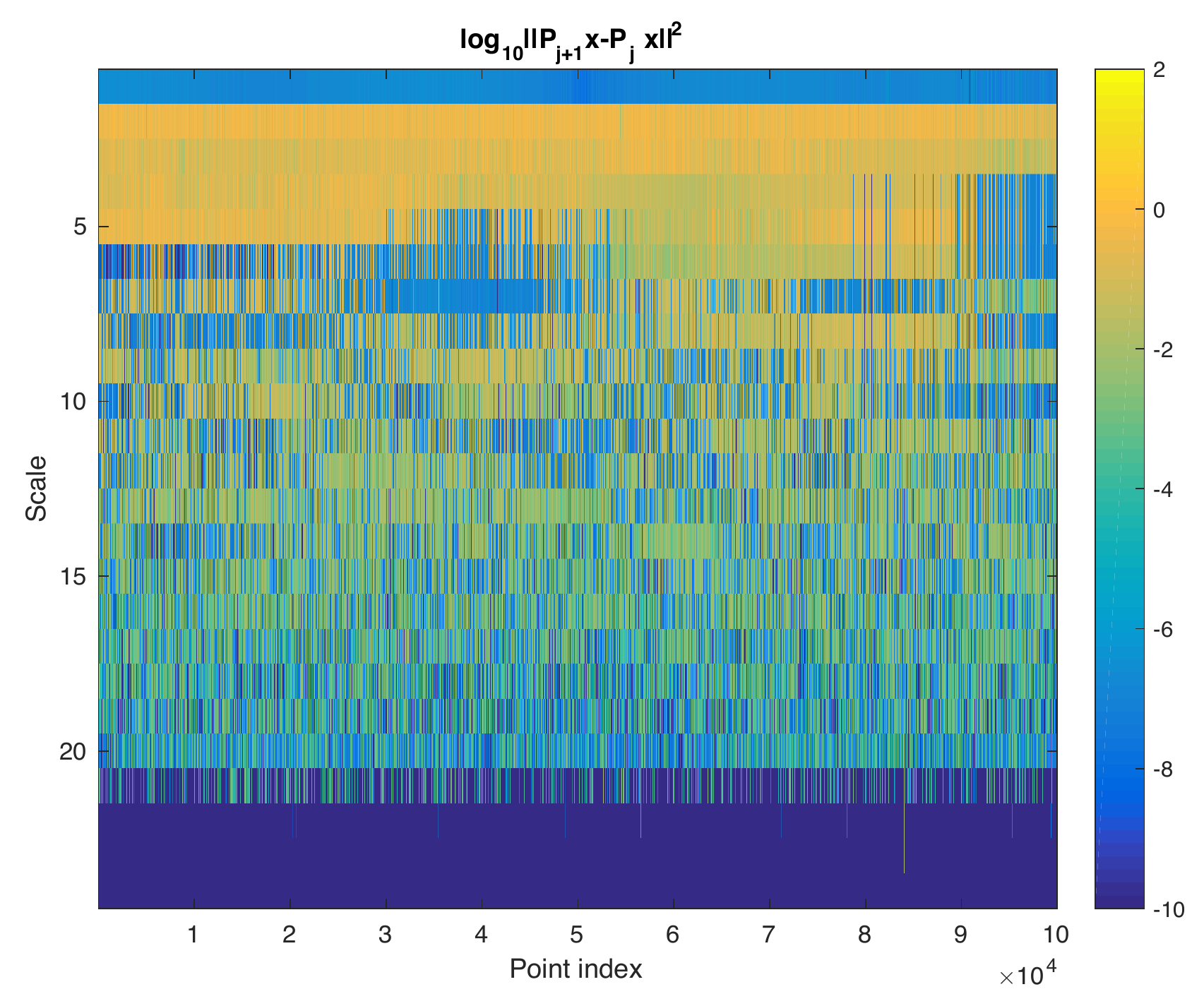}} 
 \subfigure[Z manifold]{\includegraphics[width=.32\textwidth]{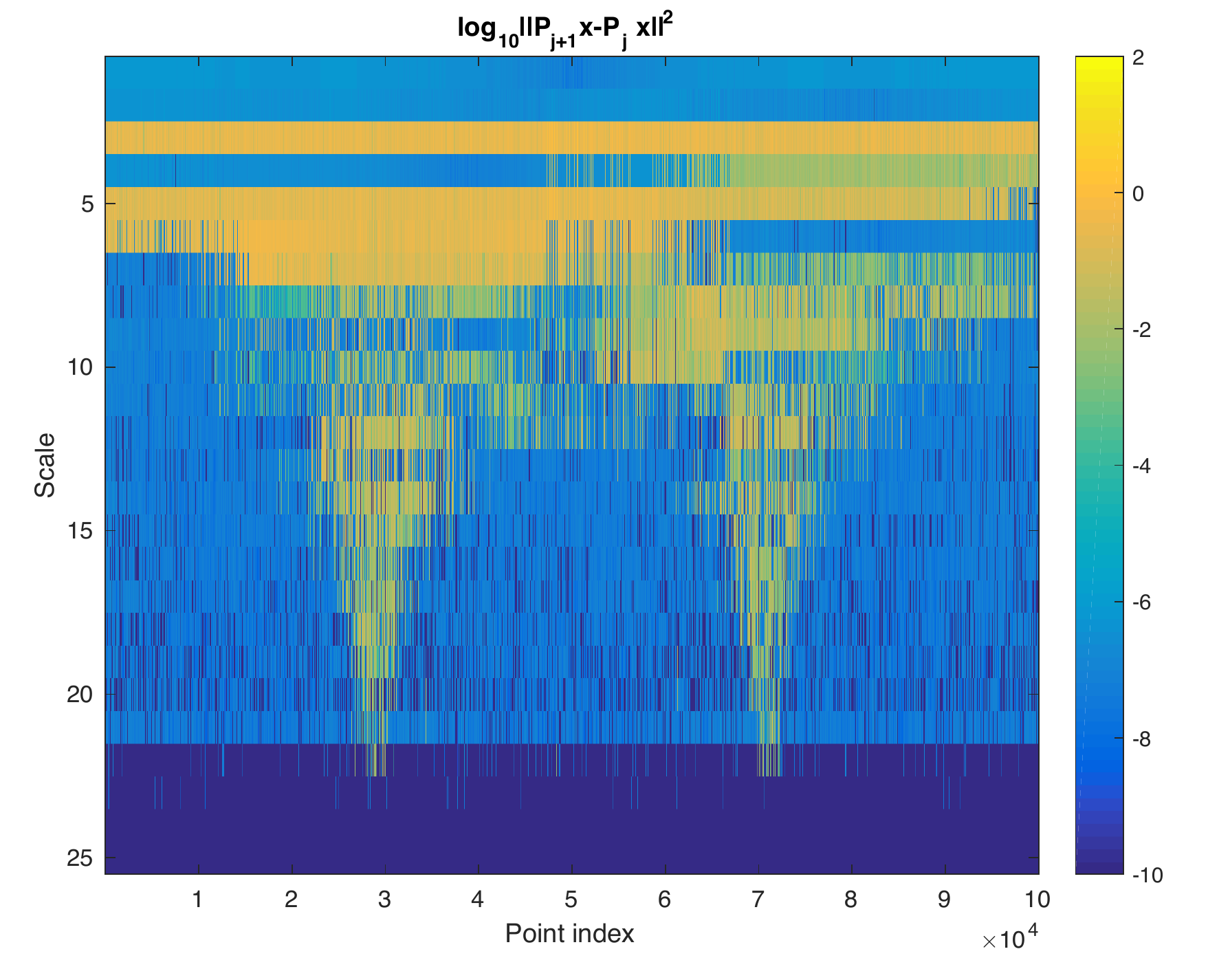}} 
 \subfigure[Error vs partition size, S manifold]{\includegraphics[width=0.345\textwidth]{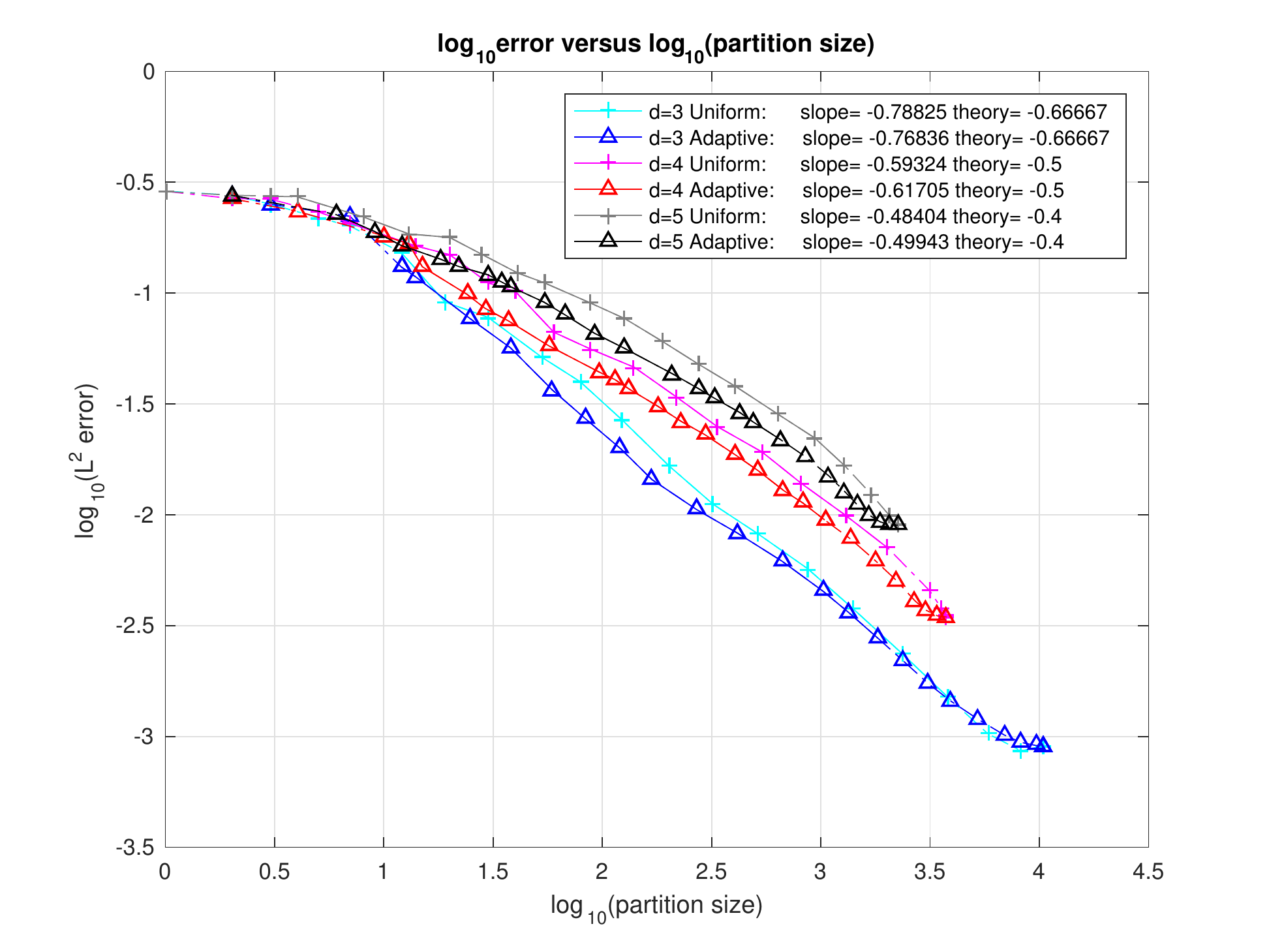}}
 \\
 \subfigure[Adaptive approx. of S]{\includegraphics[width=.317\textwidth]{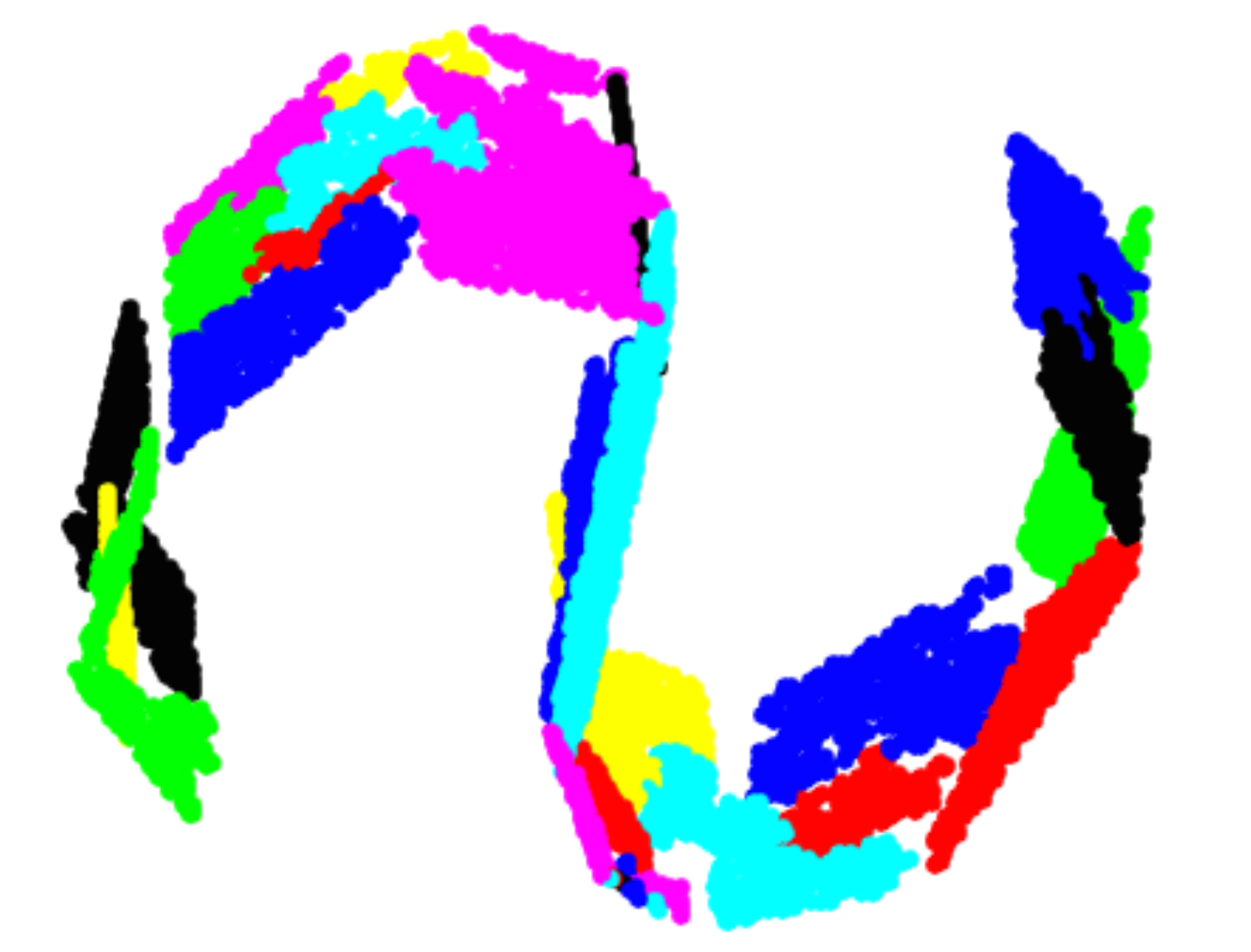}} 
 \subfigure[Adaptive approx. of Z]{\includegraphics[width=.317\textwidth]{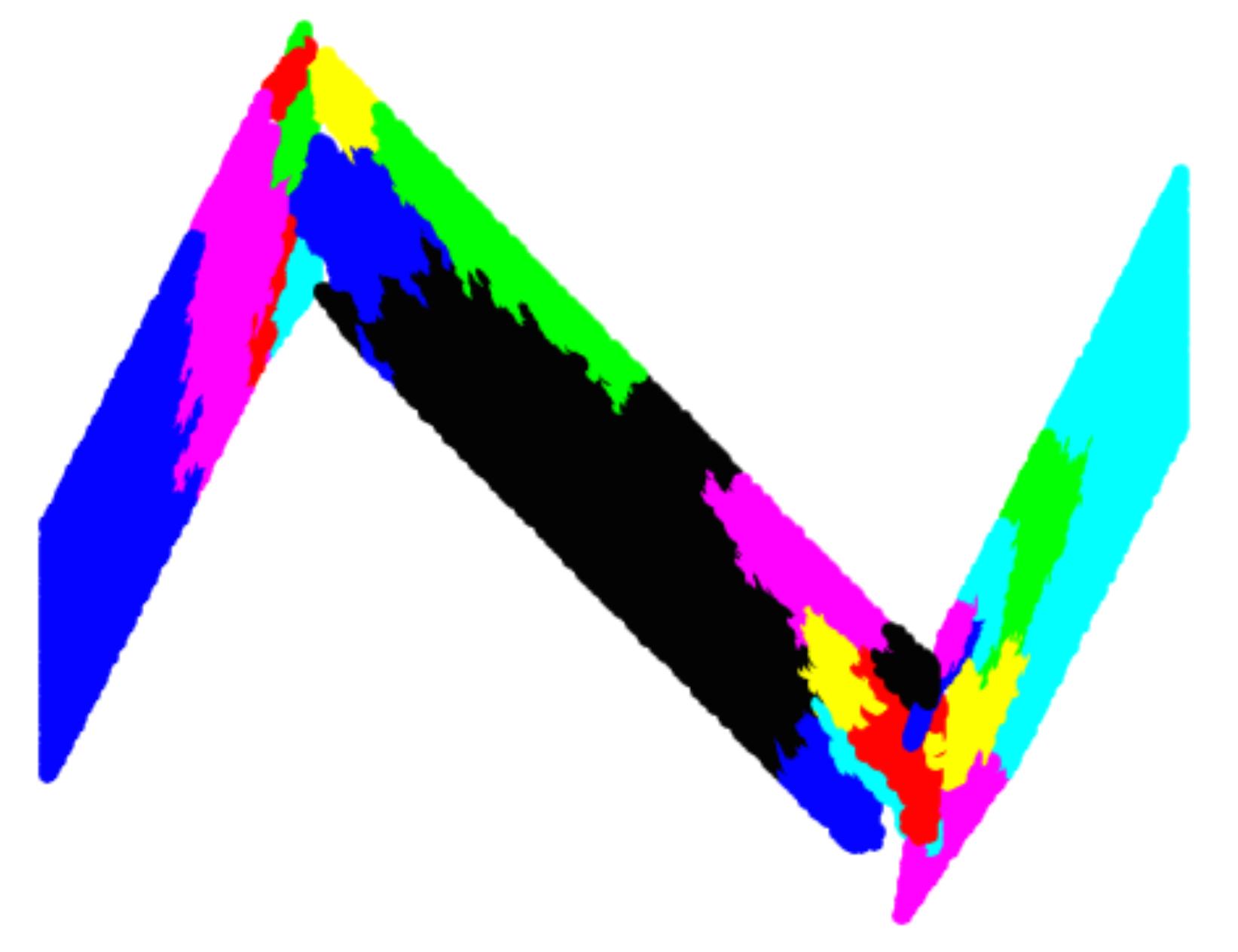}}
\subfigure[Error vs partition size, Z manifold]{\includegraphics[width=0.345\textwidth]{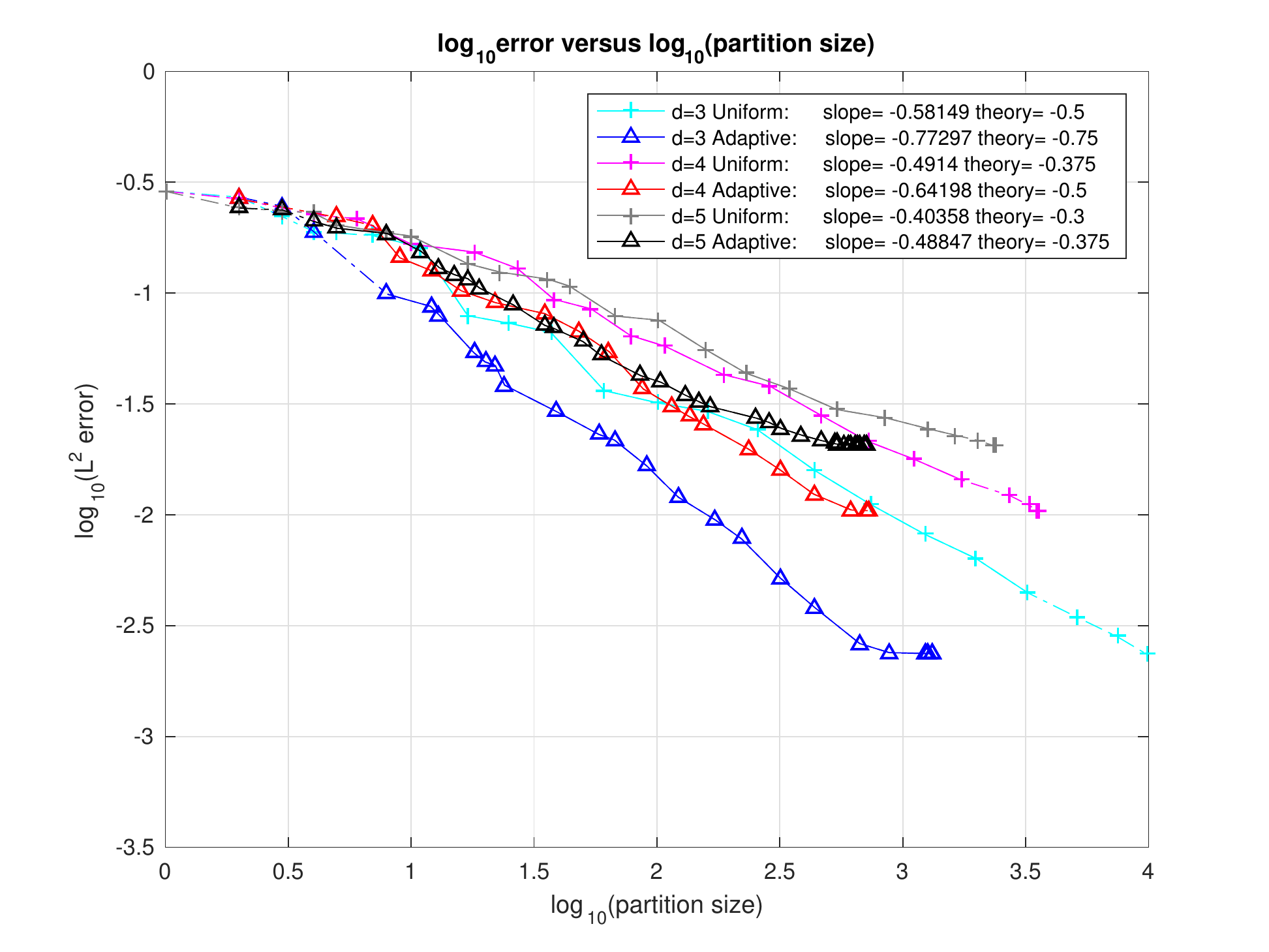}}
\caption{(a,b): $\log_{10}||\hcalP_{j}(x_i)-\hcalP_{j+1}(x_i)||$ from the coarsest scale (top) to the finest scale (bottom), with columns indexed by points, sorted from left to right on the manifold. 
            (d,e): adaptive approximations.  
            (c,f): log-log plot of the approximation error versus the partition size in GMRA and adaptive GMRA respectively.  
            Theoretically, the slope is $-2/d$ in both GMRA and adaptive GMRA for the S manifold.
            For the Z manifold, the slope is $-1.5/d$ in GMRA and $-{1.5/(d-1)}$ in adaptive GMRA (see appendix \ref{AppSZ}). 
            }
             \label{FigDelta}
\end{figure}

We will provide a finite sample performance guarantee of the empirical adaptive GMRA for a model class that is more general than $\AS^\infty$. Given any fixed threshold $\eta > 0$, we let $
\calTrhoeta$ be the smallest proper tree of $\calT$ that contains all  $C_{j,k} \in \calT$ for which  $\Delta_{j,k} \ge 2^{-j}\eta$. The corresponding adaptive partition $\Lamrhoeta$ consists of the outer leaves of $\calTrhoeta$. We let $\#_j \calTrhoeta$ be the number of cells in $\calTrhoeta$ at scale $j$.
\begin{definition}[Model class $\BS$]
\label{defBs}
{In the case $d \ge 3$, }
given $s>0$, a probability measure $\rho$ supported on $\calM$ is in $\BS$ if $\rho$ satisfies the following regularity condition:
\beq
\label{eqBs}
|\rho|_{\BS}^p:= \sup_{{\calT}}\ \sup_{\eta>0} \eta^p \sum_{j \ge \jmin }2^{-2j}\#_j \calTrhoeta <\infty ,\,\, \text{ with } 
p= \frac{2(d-2)}{2s+d-2}
\eeq
where $\calT$ varies over the set, assumed nonempty, of multiscale tree decompositions satisfying Assumption (A1-A5).
\end{definition}
For elements in the model class $\BS$ we have control on the growth rate of the truncated tree $\calTrhoeta$ as $\eta$ decreases, namely it is $\mathcal{O}(\eta^{-p})$. 
The key estimate on variance and sample complexity in Lemma \ref{thm1} indicates that the natural measure of the complexity of $\calTrhoeta$ is the weighted tree complexity measure $\sum_{j \ge \jmin} 2^{-2j}\#_j \calTrhoeta$ in the definition above.
First of all, the class $\BS$ is indeed larger than $\AS^\infty$ (see appendix \ref{appasinf} for a proof):

\commentout{
\begin{figure}[h]
\centering
S manifold \\
     \includegraphics[width=5cm]{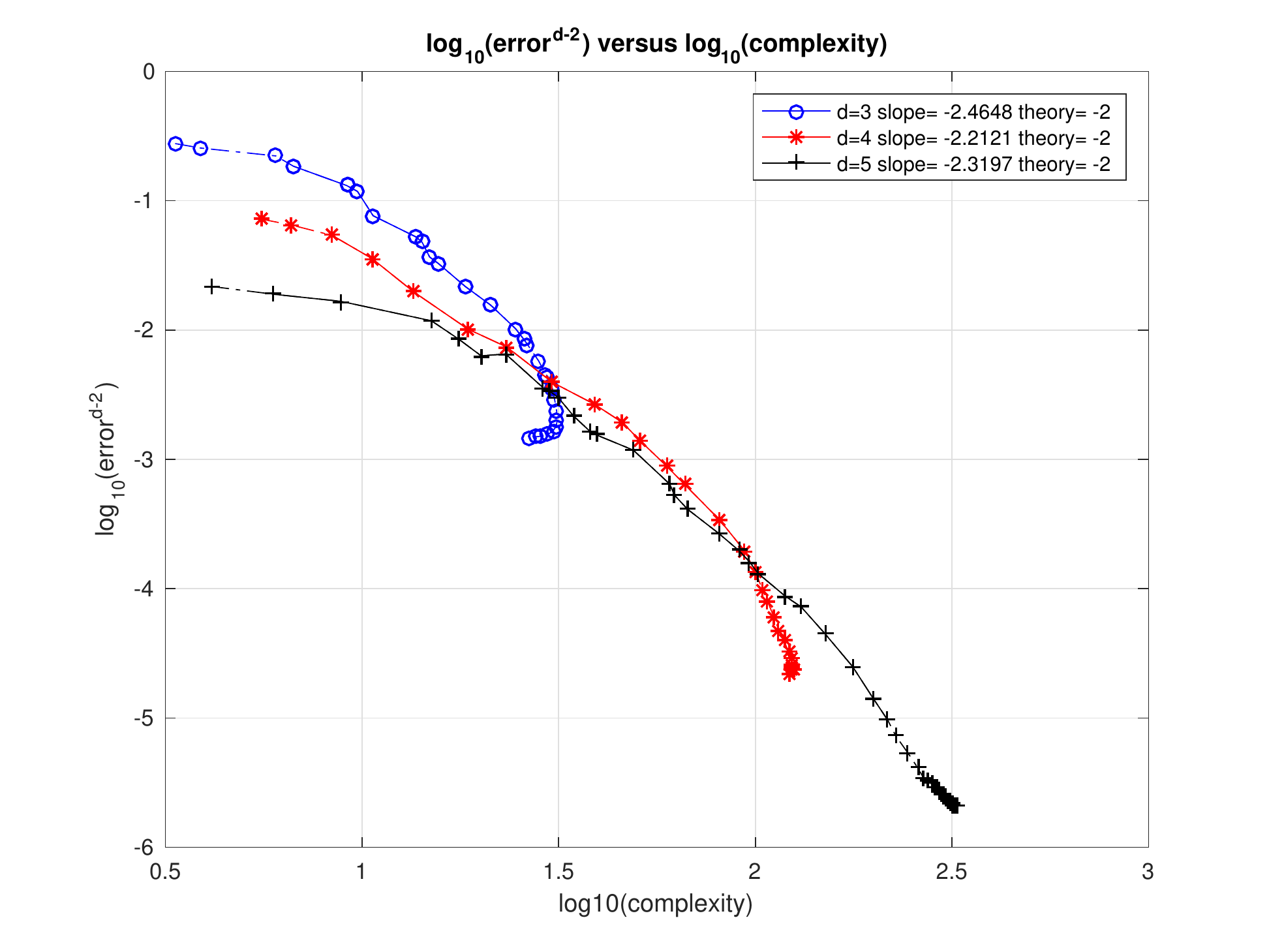}
          \includegraphics[width=5cm]{Fig23_SandZ/Sample100000/V2/MoreExperiments/S_Bs_2-eps-converted-to.pdf}
          \includegraphics[width=5cm]{Fig23_SandZ/Sample100000/V2/MoreExperiments/S_Bs_3-eps-converted-to.pdf}
           \includegraphics[width=5cm]{Fig23_SandZ/Sample100000/V2/MoreExperiments/S_Bs_4-eps-converted-to.pdf}
          \includegraphics[width=5cm]{Fig23_SandZ/Sample100000/V2/MoreExperiments/S_Bs_5-eps-converted-to.pdf}
          \\
Z manifold \\
     \includegraphics[width=5cm]{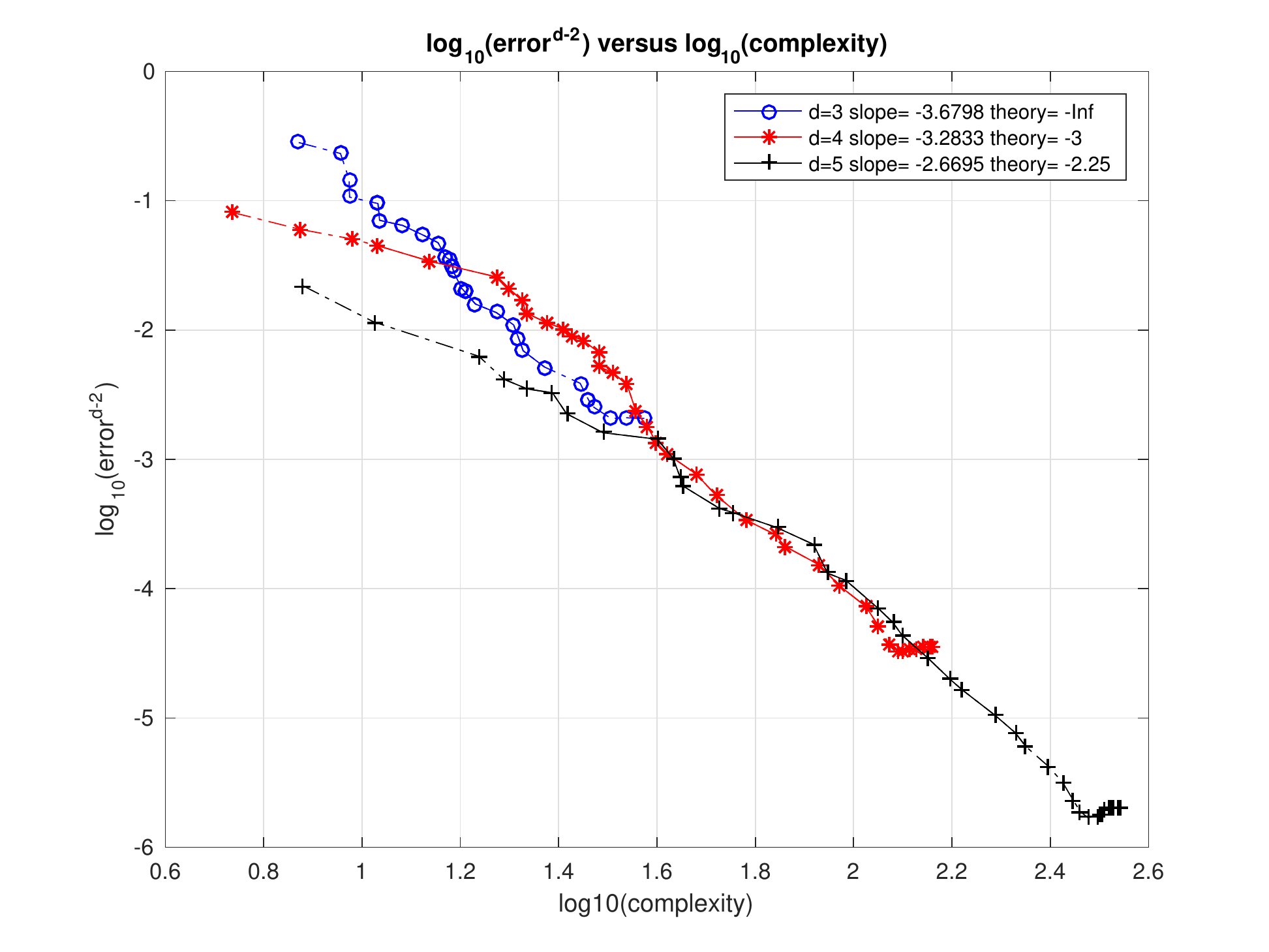}
          \includegraphics[width=5cm]{Fig23_SandZ/Sample100000/V2/MoreExperiments/Z_Bs_2-eps-converted-to.pdf}
          \includegraphics[width=5cm]{Fig23_SandZ/Sample100000/V2/MoreExperiments/Z_Bs_3-eps-converted-to.pdf}
           \includegraphics[width=5cm]{Fig23_SandZ/Sample100000/V2/MoreExperiments/Z_Bs_4-eps-converted-to.pdf}
          \includegraphics[width=5cm]{Fig23_SandZ/Sample100000/V2/MoreExperiments/Z_Bs_5-eps-converted-to.pdf}

\end{figure}
}

\begin{lemma}
\label{l:AsinBs}
$\BS$ is a more general model class than $\AS^\infty$. If $\rho\in\AS^\infty$, 
then $\rho\in\calB_{s}$ and $|\rho|_{\BS} \lesssim |\rho|_{\AS^\infty}$.
\end{lemma}

\begin{example}
The volume measures on the $d$-dim $(d\ge 3)$ S manifold and Z manifold are in $\calB_2$ and $\calB_{{1.5(d-2)}/{(d-3)}}$ respectively (see appendix \ref{AppSZ}). 
Numerically $s$ is approximated by the negative of the slope in the log-log plot of $\|X-\hcalP_{\hLameta} X\|^{d-2}$ versus the weighted complexity of the truncated tree 
according to Eq. \eqref{Bs1}.
See numerical examples in Figure \ref{FigASBS}.
\end{example}



We also need a quasi-orthogonality condition which says that the operators $\{\calQjk\}_{k\in\calK_j, j\ge \jmin}$ applied on $\calM$ are mostly orthogonal across scales and/or $\|\calQjk X\|$ quickly decays.

\begin{definition}[Quasi-orthogonality]
There exists a constant $B_0>0$ such that for any proper subtree $\tilde\calT$ of any mater tree $\calT$ satisfying Assumption (A1-A5),
\beq
\label{quasiortho}
\|\sum_{\Cjk \notin \tcalT} \calQjk X \|^2
 \le B_0 \!\!\sum_{\Cjk \notin \tcalT}\!\! \|\calQjk X\|^2.
\eeq
\end{definition}
 We postpone further discussion of this condition to Section \ref{secQO}.
One can show (see appendix \ref{proofBs1}) that in the case $d\ge 3$, $\rho\in\BS$ along with quasi-orthogonality implies a certain rate of approximation of $X$ by $\calP_{\Lamrhoeta} X$, as $\eta\rightarrow0^+$:
\beq
\label{Bs1}
\|X-\calP_{\Lamrhoeta} X\|^2
\le
B_{s,d} |\rho|_{\BS}^p \eta^{2-p} \le B_{s,d} |\rho|_{\BS}^2 \left(\sum_{j \ge \jmin }2^{-2j} \#_j \calTrhoeta \right)^{-\frac{2s}{d-2}},
 \eeq
where $s = \frac{(d-2)(2-p)}{2p}$ and $B_{s,d}:= B_0 2^p /(1-2^{p-2})$.

\commentout{
\begin{definition}[Model class $\BS$]
\label{defBs}
{In the case $d \ge 3$, }
given $s>0$, a probability measure $\rho$ supported on $\calM$ is in $\BS$ if the two following conditions hold. 
\begin{enumerate}
\item Regularity: 
\beq
\label{eqBs}
|\rho|_{\BS}^p:= \sup_{{\calT}}\ \sup_{\eta>0} \eta^p \sum_{j\ge 0}2^{-2j}\#_j \calTrhoeta <\infty ,\,\, \text{ with } 
p= \frac{2(d-2)}{2s+d-2}
\eeq
where $\calT$ varies over the set, assumed nonempty, of multiscale tree decompositions satisfying Assumption (A1-A5).

\item Quasi-orthogonality: There exists a constant $B_0>0$ such that for any proper subtree $\tilde\calT$ of $\calT$
\beq
\label{quasiortho}
\|\sum_{\Cjk \notin \calTrhoeta} \calQjk X \|^2
 \le B_0 \!\!\sum_{\Cjk \notin \calTrhoeta}\!\! \|\calQjk X\|^2.
\eeq
\end{enumerate} 
\end{definition}
The regularity condition requires that the tree entropy grows at a rate that is well-balanced by $\eta^p$, and the quasi-orthogonality condition imposes that the operators $\{\calQjk\}_{k\in\calK_j, j\ge 0}$ applied on $x \in \calM$ are mostly orthogonal and/or $\|\calQjk X\|$ quickly decays across scales. We postpone further discussion of this condition in Section \ref{secQO}.
One can show that in the case $d\ge 3$, $\rho\in\BS$ along with quasi-orthogonality implies a certain rate of approximation of $X$ by $\calP_{\Lamrhoeta} X$, as $\eta\rightarrow0^+$:
\beq
\label{Bs1}
\|X-\calP_{\Lamrhoeta} X\|^2
\le
B_{s,d} |\rho|_{\BS}^p \eta^{2-p} \le B_{s,d} |\rho|_{\BS}^2 \left(\sum_{j \ge 0}2^{-2j} \#_j \calTrhoeta \right)^{-\frac{2s}{d-2}},
 \eeq
where $s = {(d-2)(2-p)}/(2p)$ and $B_{s,d}:= B_0 2^p /(1-2^{p-2})$. This is proven in appendix \ref{proofBs1}.  
}

\commentout{
In summary, our algorithm of constructing the empirical adaptive GMRA include the following steps:
\begin{enumerate}
\item Compute the empirical affine projector $\hcalPjk$ and the empirical amount of decreased error $\hDeltajk$ for every node $\Cjk \in \calTn$.

\item Obtain $\hcalTtaun$, the smallest proper subtree of $\calTn$ which contains all $\Cjk \in \calTn$ such that $\hDeltajk \ge 2^{-j}\tau_n$.

\item Return piecewise affine projectors on $\hLamtaun$ 
$$\hcalP_{\hLamtaun} = \sum_{\Cjk \in \hLamtaun} \hcalPjk \chijk.$$
\end{enumerate}
}
The main result of this paper is the following performance analysis of empirical adaptive GMRA (see the proof in Section \ref{s:adaptivegmra}). 

\begin{theorem}
\label{thm3}
Suppose $\rho$ satisfies quasi-orthogonality and $\calM$ is bounded: $\calM \subset B_{M}(0)$. 
Let $\nu>0$. 
There exists $\kappa_0(\theta_2,\theta_3,\theta_4,\amax,d,\nu)$ such that if $\tau_n = \kappa \sqrt{(\log n)/n}$ with $\kappa \ge \kappa_0$, the following holds:
\begin{itemize}
\item[(i)] 
if ${d\ge3}$ and $\rho \in \BS$ for some $s>0$, there are $c_1$ and $ c_2$ such that
\begin{align}
&\PP\left\{ 
\|X-\hcalP_{\hLamtaun}X\| \ge c_1\left(\lognn \right)^{\frac{s}{2s+d-2}}
\right\}
 \le 
c_2 n^{-\nu}\,.
\label{thm3eq0}
\end{align}

\item[(ii)] if ${d=1}$, there exist $c_1$ and $c_2$ such that
\begin{align}
&\PP\left\{ 
\|X-\hcalP_{\hLamtaun}X\| \ge c_1\left(\lognn \right)^{\frac 1 2}
\right\}
 \le 
c_2 n^{-\nu}\,.
\label{thm3eq10}
\end{align}

\item[(iii)] if ${d=2}$ and
$$|\rho|: = \sup_\calT\sup_{\eta>0} [-\log \eta]^{-1}{\sum_{j\ge \jmin } 2^{-2j} \#_j \calTrhoeta } < +\infty\,,$$
then there exist $c_1$ and $c_2$ such that
\begin{align}
&\PP\left\{ 
\|X-\hcalP_{\hLamtaun}X\| \ge c_1\left(\frac{\log^2 n}{ n}\right)^\frac12
\right\}
 \le 
c_2 n^{-\nu}\,.
\label{thm3eq20}
\end{align}
\end{itemize}
\end{theorem}
The dependencies of the constants in Theorem \ref{thm3} on the geometric constants are as follows:
\begin{center}
    \begin{tabular}{ c c  c c c }
  $d\ge 3:$ & $c_1=c_1(\theta_2,\theta_3,\theta_4,\amax,d,s,\kappa,|\rho|_{\BS},B_0,\nu)$,& $c_2 = c_2(\theta_2,\theta_3,\theta_4,\amin,\amax,d,s,\kappa,|\rho|_{\BS},B_0)$
    \\
    $d=1:$ & $c_1 = c_1(\theta_1,\theta_2,\theta_3,\theta_4,\amax,d,\kappa,B_0,\nu)$, & $c_2 =c_2(\theta_1,\theta_2,\theta_3,\theta_4,\amin,\amax,d,\kappa,B_0)$
    \\
    $d=2:$ &$c_1 = c_1(\theta_2,\theta_3,\theta_4,\amax,d,\kappa,|\rho|, B_0,\nu)$, & $c_2=c_2(\theta_2,\theta_3,\theta_4,\amin,\amax,d,\kappa,|\rho|,B_0)$
    \\
   \end{tabular}
\end{center}
Notice that by choosing $\nu$ large enough, we have 
$$\PP\left\{ 
\|X-\hcalP_{\hLamtaun}X\| \ge c_1\left(\frac{\log^\alpha n}{ n}\right)^\beta
\right\}
 \le 
c_2 n^{-\nu}
\Rightarrow {\rm MSE } \le c_1 \left(\frac{\log^\alpha n}{n}\right)^{2\beta},$$
so we also have the MSE bound: ${\rm MSE} \lesssim (\log n /n)^{\frac{2s}{2s+d-2}}$ for $d\ge 3$ and ${\rm MSE} \lesssim \log^d n /n$ for $d=1,2$.

In comparison with Theorem \ref{thm2}, Theorem \ref{thm3} is more satisfactory for two reasons: (i) when $d \ge 3$, the same rate $(\log n /n)^{\frac{2s}{2s+d-2}}$ is proved for the model class $\BS$ which is larger than $\AS^\infty$. (ii) our algorithm is universal: it does not require a priori knowledge of the regularity $s$, since the choice of $\kappa$ is independent of $s$, yet it achieves the rate as if it knew the optimal regularity parameter $s$.

In the perspective of dictionary learning, when $d \ge 3$, adaptive GMRA provides a dictionary $\Phi_{\hLamtaun}$ associated with a tree of weighted complexity $(n/\log n)^{\frac{d-2}{2s+d-2}}$, so that every $x$ sampled from $\rho$ may be encoded by a vector with $d+1$ nonzero entries, among which one encodes the location of $x$ in the adaptive partition and the other $d$ entries are the local principal coefficients of $x$. 

%
For a given accuracy $\ep$, in order to achieve ${\rm MSE} \lesssim \ep^2$, the number of samples we need is $n_\ep \gtrsim (1/\ep)^{\frac{2s+d-2}{s}}\log (1/\ep)$. When $s$ is unknown, we can determine $s$ as follows: 
%
we fix a small $n_0$ and run adaptive GMRA with $n_0, 2n_0, 4 n_0, \ldots,10 n_0$ samples. For each sample size, we evenly split data into the training set to construct adaptive GMRA and the test set to evaluate the MSE. According to Theorem \ref{thm3}, the MSE scales like $ [(\log n)/n]^{\frac{2s}{2s+d-2}}$ where $n$ is the sample size. Therefore, the slope in the log-log plot of the MSE versus $n$ gives an approximation of $-2s/(2s+d-2)$.

The threshold $\tau_n$ in our adaptive algorithm is independent of $s$ as $\kappa_0$ does not depend on $s$, which means our adaptive algorithm does not require $s$ as a priori information but rather will learn it from data. It would also be natural to consider another stopping criterion: $\calE^2_{j,k} := \frac{1}{\rho(\Cjk)}\int_{\Cjk} \|\calP_j x - x\|^2 d\rho \le \eta^2$ which suggests stopping refinement to finer scales if the approximation error is below certain threshold. The reason why we do not  adopt this stopping criterion is that in this case the threshold $\eta$ would have to depend on $s$ in order to guarantee the (adaptive) rate ${\rm MSE} \lesssim \left(\log n /n\right)^{\frac{2s}{2s+d-2}}$ for $d \ge 3$.
More precisely, for any threshold $\eta>0$, let $\calTrhoeta^{\calE}$ be the smallest proper subtree of $\calT$ whose leaves satisfy $\calE_{j,k}^2 \le \eta^2$. The corresponding adaptive partition $\Lamrhoeta^{\calE}$ consists of the leaves of $\calTrhoeta^{\calE}$. This stopping criterion guarantees $\|X-\calP_{\Lamrhoeta^{\calE}}X\| \le \eta$. It is natural to define the model class $\calF_s$ in the case $d\ge 3$ to be the set of probability measures $\rho$ supported on $\calM$ such that $\sup_{\calT}\sup_{\eta>0} \eta^{\frac{d-2}{s}} \sum_{j \ge \jmin} 2^{-2j}\#_j\Lamrhoeta^{\calE} < \infty$ where $\calT$ varies over the set of multiscale tree decompositions satisfying (A1-A5). One can show that $\calA_s^\infty \subsetneq \calF_s$. As an analogue of Theorem \ref{thm3}, we can prove that, there exists $\kappa_0>0$ such that if our adaptive algorithm adopts the stopping criterion $\widehat{\calE}_{j,k} \le \tau^{\calE}_n$ where the threshold is chosen as $\tau^{\calE}_n = \kappa \left(\log n /n\right)^{\frac{s}{2s+d-2}}$ with $\kappa \ge \kappa_0$, then the empirical approximation on the adaptive partition  $\widehat\Lam_{\tau_n^{\calE}}$ satisfies ${\rm MSE} = \|X-\hcalP_{\widehat\Lam_{\tau_n^{\calE}}}X\|^2 \lesssim \left(\log n /n\right)^{\frac{2s}{2s+d-2}}.$ With this stopping criterion, the threshold $\tau_n^{\calE}$ would depend on $s$, forcing us to know $s$ as a priori information, unlike in Theorem \ref{thm3}.

Theorem \ref{thm3} is stated when $\calM$ is bounded. The assumption of the boundedness of $\calM$ is largely irrelevant, and may be replaced by a weaker assumption on the decay of $\rho$. 

\begin{theorem}
\label{thmunbounded}
Let $d \ge 3$, $s,\delta,\lambda,\mu>0$.
Assume that there exists $C_1$ such that $$\int_{B_R(0)^c} ||x||^{2}d\rho \le C_1 R^{-\delta}, \ \forall R \ge R_0.$$
Suppose $\rho$ satisfies quasi-orthogonality. 
If $\rho$ restricted on $B_R(0)$, denoted by $\rho|_{B_R(0)}$, is in $\BS$ for every $R \ge R_0$ and $\Large|\rho|_{B_R(0)} \Large|_{\BS}^{p} \le C_2 R^\lambda$ for some $C_2>0$, where $p = \frac{2(d-2)}{2s+d-2}$.
Then there exists $ \kappa_0(\theta_2,\theta_3,\theta_4,\amax,d,\nu)$ such that if $\tau_n = \kappa \sqrt{\log n /n}$ with $\kappa \ge \kappa_0$, we have
 \begin{align}
 \label{equnbounded}
&\PP\left\{ 
\|X-\hcalP_{\hLamtaun}X\| \ge c_1\left(\lognn \right)^{\frac{s}{2s+d-2}\frac{\delta}{\delta+\max(\lambda,2)}}
\right\}
 \le 
c_2 n^{-\nu}\,
\end{align}
for 
some $c_1,c_2$ independent of $n$,
where the estimator $\hcalP_{\hLamtaun}X$ is obtained by adaptive GMRA within $B_{R_n}(0)$ where $R_n =\max(R_0, \mu(n/\log n)^{\frac{2s}{(2s+d-2)(\delta+\max(\lambda,2))}} )$, and is equal to $0$ for the points outside $B_{R_n}(0)$.
\end{theorem}

Theorem \ref{thmunbounded} is proved at the end of Section \ref{s:adaptivegmra}. It implies ${\rm MSE} \lesssim (\log n /n)^{\frac{2s}{2s+d-2} \cdot \frac{\delta}{\delta+\max(\lambda,2)}}$
As $\delta$ increases, i.e., $\delta \rightarrow +\infty$, the MSE approaches $(\log n /n)^{\frac{2s}{2s+d-2}}$, which is consistent with Theorem \ref{thm3} for bounded $\calM$.
Similar results, with similar proofs, would hold under different assumptions on the decay of $\rho$; for example for $\rho$ decaying exponentially, or faster, only $\log n$ terms in the rate would be lost compared to the rate in Theorem \eqref{thm3}.

\begin{remark}
\label{remarkunbounded}
We claim that $\lambda$ is not large in simple cases. For example, if $\rho \in \AS^\infty$ and $\rho$ decays in the radial direction in such a way that $\rho(\Cjk) \le C 2^{-jd}\|\cjk\|^{-(d+1+\delta)}$, it is easy to show that $\rho|_{B_R(0)} \in \BS$ for all $R>0$ and $\Large | \rho|_{B_R(0)}\Large|_{\BS}^p \le R^\lambda$ with $\lambda = d-\frac{(d+1+\delta)(d-2)}{2s+d-2}$ (see the end of Section \ref{s:adaptivegmra}).
\end{remark}

\begin{remark}
Suppose that $\rho$ was supported in a tube of radius $\sigma$ around a $d$-dimensional manifold $\calM$, a model that can account both for (bounded) noise and situations where data is not exactly on a manifold, but close to it, as in \citet{MMS:NoisyDictionaryLearning}. Then Theorem \ref{thm3} and Theorem \ref{thmunbounded} apply in this case, provided one stops at scale $j$ such that $2^{-j}\gtrsim\sigma$.
\end{remark}

\begin{remark}
In these Theorems we are assuming that $d$ is given because it can be estimated using existing techniques, see for example \citet{LMR:MGM1} and references therein.
\end{remark}


\commentout{
\subsubsection{Scale-independent threshold and model class $\Gammas$}
We also consider scale-independent threshold where we run Algorithm \ref{TableAdaptiveGMRA} with $\hDeltajk \ge \tau_n$ in Step 3.
Given any fixed threshold $\eta > 0$, we let $
\calTrhoeta$ \textcolor{green}{MM: we can't use the same notation as for the scale-dependent threshold...} be the smallest proper tree of $\calT$ that contains all  $C_{j,k} \in \calT$ for which  $\Delta_{j,k} \ge \eta$. We define model class $\Gammas$ consisting of the elements for which we can control
the growth rate of $\#\calTrhoeta$ as $\eta$ decreases.

\begin{definition}[Model class $\Gammas$]
\label{defGammas}
Given $s>0$, a probability measure $\rho$ supported on $\calM$ is in $\Gammas$ if $\rho$ satisfies the following regularity condition:
\beq
\label{eqGammas}
|\rho|_{\Gammas}^p:= \sup_{{\calT}}\ \sup_{\eta>0} \eta^p \# \calTrhoeta <\infty ,\,\, \text{ with } 
p= \frac{2d}{2s+d}
\eeq
where $\calT$ varies over the set, assumed nonempty, of multiscale tree decompositions satisfying Assumption (A1-A5).
\end{definition}

One can show that $\rho\in\Gammas$ along with quasi-orthogonality (eqn. \eqref{quasiortho}) implies
$$
\|X-\calP_{\Lamrhoeta} X\|^2
\le
B_0 2^p /(1-2^{p-2}) |\rho|_{\Gammas}^p \eta^{2-p} \le  B_0 2^p /(1-2^{p-2})  |\rho|_{\Gammas}^2  \#\calTrhoeta^{-\frac{2s}{d}},
$$
where $s = {d(2-p)}/(2p)$. 

\begin{example}
$\Gammas$ is also more general than $\AS^\infty$. If $\rho\in\AS^\infty$, 
then $\rho\in\Gammas$ and $|\rho|_{\Gammas} \lesssim |\rho|_{\AS^\infty}$ (see appendix \ref{appAsInf2}).
\end{example}

\begin{lemma}
\label{l:spacerelationships}
\begin{equation*}
\AS^\infty\subsetneq\Gammas\subsetneq\BS
\end{equation*}
\end{lemma}
\begin{proof}
\textcolor{green}{MM: put the proof in the appendix. Summarize all containments in one place, for example this Lemma. Would be also nice to show that the containments are strict}.

\end{proof}

\begin{example}
Volume measure on the $d$-dimensional $(d\ge 3)$ S manifold and Z manifold are in $\Gamma_2$ and $\Gamma_{{1.5d}/{(d-1)}}$ respectively (see appendix \ref{AppSZ}). 
Numerically $s$ is approximated by the negative of the slope of the log-log plot of $\|X-\hcalP_{\hLameta} X\|^{d}$ versus cardinality of the truncated tree.
See numerical experiments in Figure \ref{FigBS} (b,d).
\end{example}

\commentout{
In summary, our algorithm of constructing the empirical adaptive GMRA include the following steps:
\begin{enumerate}
\item Compute the empirical affine projector $\hcalPjk$ and the empirical amount of decreased error $\hDeltajk$ for every node $\Cjk \in \calTn$.

\item Obtain $\hcalTtaun$, the smallest proper subtree of $\calTn$ which contains all $\Cjk \in \calTn$ such that $\hDeltajk \ge 2^{-j}\tau_n$.

\item Return piecewise affine projectors on $\hLamtaun$ 
$$\hcalP_{\hLamtaun} = \sum_{\Cjk \in \hLamtaun} \hcalPjk \chijk.$$
\end{enumerate}
}

The second main result of this paper is the performance analysis of empirical adaptive GMRA with scale-independent threshold for the model class $\Gammas$. 
\begin{theorem}
\label{thm31}
Suppose $\rho$ satisfies quasi-orthogonality and $\calM$ is bounded: $\|x\| \le M$ for all $x \in \calM$. 
Let $\nu>0$. 
There exists $\kappa_0 = \kappa_0(\theta_2,\theta_3,\amax,d,\nu)$ such that if $\tau_n = \kappa \sqrt{(\log n)/n}$ with $\kappa \ge \kappa_0$, then if $\rho \in \Gammas$ for some $s>0$, there are $c_1:=c_1(\theta_2,\theta_3,\amax,d,s,\kappa,|\rho|_{\Gammas},B_0,\nu)$ and $c_2:= c_2(\theta_2,\theta_3,\amin,\amax,d,s,\kappa,|\rho|_{\Gammas},B_0)$ such that
$$
\PP\left\{ 
\|X-\hcalP_{\hLamtaun}X\| \ge c_1\left(\lognn \right)^{\frac{s}{2s+d}}
\right\}
 \le 
c_2 n^{-\nu}\,.
$$

\end{theorem}

We also have the MSE bound: 
$$
\EE \|X-\hcalP_{\hLamtaun}X\|^2 
\le c_3(\theta_2,\theta_3,\amin,\amax,d,M,s,\kappa,|\rho|_{\Gammas},B_0)\left(\lognn\right)^{\frac{2s}{2s+d}}.
$$
In the perspective of dictionary learning, adaptive GMRA provides a dictionary $\Phi_{\hLamtaun}$ with cardinality $\asymp d(n/\log n)^{\frac{d}{2s+d}}$, so that every $x$ sampled from $\rho$ may be encoded with a vector with $d+1$ nonzero entries.

\begin{figure}
\centering
\begin{tikzpicture}
\draw[color=blue, very thick](-1,0) circle (2);
\node [blue] at (-1.7,0) {$\BS$};
\draw[color=black, very thick](1,0) circle (2);
\node [black] at (1.7,0) {$\Gammas$};
\draw[color=red!60, fill=red!5, very thick](0,0) circle (0.8);
\node [red] at (0,0) {$\AS^\infty$};
\end{tikzpicture}
\caption{Relation of $\AS^\infty$, $\BS$ and $\Gammas$}
\end{figure}
}

\subsection{Connection to previous works}

The works by \cite{CM:MGM2} and \cite{MMS:NoisyDictionaryLearning} are natural predecessors to this work.
In \cite{CM:MGM2}, GMRA and orthogonal GMRA were proposed as data-driven dictionary learning tools to analyze intrinsically low-dimensional point clouds in a high dimensions. The bias $\|X-\calP_j X\|$ were estimated for volume measures on $\calC^{s}, s \in (1,2]$ manifolds .
The performance of GMRA,  including sparsity guarantees and computational costs, were systematically studied and tested on both simulated and real data. 
In \cite{MMS:NoisyDictionaryLearning} the finite sample behavior of empirical GMRA was studied. A non-asymptotic probabilistic bound on the approximation error $\|X-\hcalP_j X\|$ for the model class $\calA_2$ (a special case of Theorem \ref{thm2} with $s=2$) was established. It was further proved that if the measure $\rho$ is absolutely continuous with respect to the volume measure on a tube of a bounded $\mathcal{C}^2$ manifold with a finite reach, then $\rho$ is in $\calA_2$. Running the cover tree algorithm on data gives rise to a family of multiscale partitions satisfying Assumption (A3-A5).
The analysis in \cite{MMS:NoisyDictionaryLearning} robustly accounts for noise and modeling errors as the probability measure is concentrated \lq\lq near\rq\rq\ a manifold. This work extends GMRA by introducing adaptive GMRA, where low-dimensional linear approximations of $\calM$ are built on adaptive partitions at different scales. The finite sample performance of adaptive GMRA is proved for a large model class. Adaptive GMRA takes full advantage of the multiscale structure of GMRA in order to model data sets of varying complexity across locations and scales.
We also generalize the finite sample analysis of empirical GMRA from $\calA_2$ to $\calA_s$, and analyze the finite sample behavior of orthogonal GMRA and adaptive orthogonal GMRA.

In a different direction, a popular learning algorithm for fitting low-dimensional planes to data is $k$-flats: let $\mathcal{F}_k$ be the collections of $k$ flats (affine spaces) of dimension $d$. Given data $\calX_n =\{x_1,\ldots,x_n\}$, $k$-flats solves the optimization problem
\beq
\label{eqkflats}
\displaystyle\min_{S \in \mathcal{F}_k} 
\frac{1}{n} \sum_{i=1}^n {{\rm dist}^2(x_i,S)}
\eeq
where ${\rm dist}(x,S) = \inf_{y \in S} \|x-y\|$. Even though a global minimizer of \eqref{eqkflats} exists, it is hard to attain due to the non-convexity of the model class $\mathcal{F}_k$, and practitioners are aware that many local minima that are significantly worse than the global minimum exist. While often $k$ is considered given, it may be in fact chosen from the data: for example Theorem 4 in \cite{Kflats} implies that, given $n$ samples from a probability measure that is absolutely continuous with respect to the volume measure on a smooth $d$-dimensional manifold $\calM$, the expected (out-of-sample) $L^2$ approximation error of $\calM$ by $k_n = C_1(\calM,\rho)n^{\frac{d}{2(d+4)}}$ planes is of order $\mathcal{O}( n^{-\frac{2}{d+4}})$. This result is comparable with our Theorem \ref{thm2} in the case $s=2$ which says that the $L^2$ error by empirical GMRA at the scale $j$ such that $2^j \asymp \left ( n/\log n\right)^{\frac{1}{d+2}}$ achieves a faster rate $\mathcal{O}(n^{-\frac{2}{d+2}})$. So we not only achieve a better rate, but we do so with provable and fast algorithms, that are nonlinear but do not require non-convex optimization.

Multiscale adaptive estimation has been an intensive research area for decades. 
In the pioneering works by Donoho and Johnstone \citep[see][]{DJ1,DJ2}, soft thresholding of wavelet coefficients was proposed as a spatially adaptive method to denoise a function. 
In machine learning, Binev et al. addressed the regression problem with piecewise constant approximations \citep[see][]{DeVore:UniversalAlgorithmsLearningTheoryI} and piecewise polynomial approximations \citep[see][]{DeVore:UniversalAlgorithmsLearningTheoryII} supported on an adaptive subpartition chosen as the union of data-{\em{independent}} cells (e.g. dyadic cubes or recursively split samples).
While the works above are in the context of function approximation/learning/denoising, a whole branch of geometric measure theory (following the seminal work by \cite{Jones-TSP,DS}) quantifies via multiscale least squares fits the rectifiability of sets and their approximability by multiple images of bi-Lipschitz maps of, say, a $d$-dimensional square. We can the view the current work as extending those ideas to the setting where data is random, possibly noisy, and guarantees on error on future data become one of the fundamental questions.

Theorem \ref{thm3} can be viewed as a geometric counterpart of the adaptive function approximation in \citet{DeVore:UniversalAlgorithmsLearningTheoryI,DeVore:UniversalAlgorithmsLearningTheoryII}. 
Our results are a ``geometric counterpart'' of sorts. 
We would like to point out two main differences between Theorem \ref{thm3} and Theorem 3 in \cite{DeVore:UniversalAlgorithmsLearningTheoryI}: 
(i) In \citet[Theorem 3]{DeVore:UniversalAlgorithmsLearningTheoryI}, there is an extra assumption that the function is in $\calA_{\gamma}$ with $\gamma$ arbitrarily small. This assumption takes care of the error at the nodes in $\calT\setminus \calTn$ where the thresholding criteria would succeed: these nodes should be added to the adaptive partition but have not been explored by our data. 
This assumption is removed in our Theorem \ref{thm3} by observing that the nodes below the data master tree have small measure  
so their refinement criterion is smaller than $2^{-j}\tau_n$ with high probability.
(ii) we consider scale-dependent thresholding criterion $\hDeltajk \ge 2^{-j} \tau_n$ 
unlike the criterion in \cite{DeVore:UniversalAlgorithmsLearningTheoryI,DeVore:UniversalAlgorithmsLearningTheoryII} that is scale-independent. This difference arises 
because at scale $j$ our linear approximation is built on data within  a ball of radius $\lesssim 2^{-j}$ and so the variance of PCA on a fixed cell at scale $j$ is proportional to $2^{-2j}$.  
For the same reason, we measure the complexity of $\calTrhoeta$ in terms of the weighted tree complexity instead of the cardinality since the former one gives an upper bound of the variance in piecewise linear approximation on partition via PCA (see Lemma \ref{thm1}). Using scale-dependent threshold and measuring tree complexity in this way give rise to the best rate of convergence. In contrast, if we use scale-independent threshold and define a model class $\Gammas$ for whose elements $\#\calTrhoeta=\mathcal{O}(\eta^{-\frac{2d}{2s+d}})$ (analogous to the function class in \cite{DeVore:UniversalAlgorithmsLearningTheoryI,DeVore:UniversalAlgorithmsLearningTheoryII}), we can still show that $\AS^\infty \subset \Gammas$, but the estimator only achieves ${\rm MSE} \lesssim ((\log n)/n)^{\frac{2s}{2s+d}}$. However many elements\footnote{For these elements, the average cell-wise refinement is monotone such that: for every $\Cjk$ and $C_{j+1,k'} \subset \Cjk$, $\frac{\Delta_{j+1,k'}}{\sqrt{\rho(C_{j+1,k'})}} \le \frac{\Deltajk}{\sqrt{\rho(\Cjk)}}$.} of $\Gamma_s$ not in $\AS^\infty$ are in $\mathcal{B}^{s'}$ with $\frac{2(d-2)}{2s'+d-2} = \frac{2d}{2s+d}$, and in Theorem \ref{thm3} the estimator based on scaled thresholding achieves a better rate, which we believe is optimal.

We refer the reader to \cite{MMS:NoisyDictionaryLearning} for a thorough discussion of further related work related to manifold and dictionary learning.

\subsection{Construction of a multiscale tree decomposition}
\label{secCoverTree}

Our multiscale tree decomposition is constructed from a variation of the cover tree algorithm \citep[see][]{LangfordICML06-CoverTree} applied on half of the data denoted by $\calX_n' $. In brief the cover tree $T(\calX_n')$ on $\calX_n'$ is a leveled tree where each level is a \lq\lq cover\rq\rq\ for the level beneath it. Each level is indexed $j$ and each node in $T(\calX_n')$ is associated with a point in $\calX'_n$. A point can be associated with multiple nodes in the tree but it can appear at most once at every level. Let $T_j(\calX_n') \subset \calX_n'$ be the set of nodes of $T$ at level $j$. The cover tree obeys the following properties for all $j \in [\jmin,\jmax]$:

\begin{enumerate}
\item Nesting: $T_j(\calX_n')  \subset T_{j+1}(\calX_n') $;
\item Separation: for all distinct $p,q \in T_j(\calX_n') $, $\|p-q\|> 2^{-j}$;
\item Covering: for all $q \in T_{j+1}(\calX_n') $, there is $p \in T_j(\calX_n') $ such that $\|p-q\|<2^{-j}$. The node at level $j$ associated with $p$ is a parent of the node at level $j+1$ associated with $q$.
\end{enumerate}

In the third property, {$q$ is called a child of $p$. Each node can have multiple parents but is only assigned to one of them in the tree.} The properties above imply that for any $q \in \calX_n'$, there exists $p \in T_j$ such that $\|p-q\|<2^{-j+1}$. The authors  in \cite{LangfordICML06-CoverTree} showed that cover tree always exists and the construction takes time $\mathcal{O}(C^d D n\log n)$ .

We know show that from a set of nets $\{T_j(\calX_n')\}_{j \in [\jmin,\jmax]}$ as above we can construct a set of $C_{j,k}$ with desired properties. (see Appendix \ref{appa} for the construction of $\Cjk$'s and the proof of Proposition \ref{propcovertree2}). $\hcalM$ defined in \eqref{e:tildeM} is equal to the union of the $\Cjk's$ constructed above up to a set whose empirical measure is $0$.

\begin{proposition}
\label{propcovertree2}

Assume $\rho$ is a doubling probability measure on $\calM$ with the doubling constant $C_1$. Then $\{\Cjk \}_{k \in \calK_j, \jmin \le j \le\jmax}$ constructed above satisfies the Assumptions 
\begin{enumerate}
\item (A1) with $\amax = C_1^2 (24)^d$ and $\amin=1$.

\item For any $\nu>0$, 
\beq
\label{eqct1}
\PP\left\{\rho(\calM\setminus \hcalM) > \frac{28\nu\log n}{3n}\right\} \le 2n^{-\nu};
\eeq

\item (A3) with $\theta_1 = C_1^{-1} 4^{-d}$;

\item (A4) with $\theta_2 = 3$.

\item If additionally
\begin{itemize}
\item[5a.] if $\rho$ satisfies the conditions in (A5) with $B_r(z)$,  $z \in \calM$, replacing $\Cjk$ with constants $\ttheta_3,\ttheta_4$ such that $\lambda_d({\rm Cov}(\rho|_{B_r(z)})) \ge \ttheta_3 r^2/d$ and $\lambda_{d+1}({\rm Cov}(\rho|_{B_r(z)})) \le \ttheta_4 \lambda_d({\rm Cov}(\rho|_{B_r(z)}))$, then the conditions in (A5) are satisfied by the $\Cjk$'s we construct with $\theta_3:=\ttheta_3(4C_1)^{-2} 12^{-d}$ and $\theta_4:=\ttheta_4/\ttheta_3 12^{2d+2} C_1^4$.
\item[5b.] if $\rho$ is the volume measure on a closed $\calC^s$ Riemannian manifold isometrically embedded in $\RR^D$, then the conditions in (A5) are satisfied by the $\Cjk$'s when $j$ is sufficiently large.
\end{itemize}

\end{enumerate}
\end{proposition}

Even though the $\{\Cjk\}$ does not exactly satisfy Assumption (A2), we claim that \eqref{eqct1} is sufficient for our performance guarantees in the case that $\calM$ is bounded by $M$ and $d\ge 3$, since simply approximating points on $\calM \setminus \hcalM$ by $0$ gives the error: 
\beq
\label{eqct2}
\PP\left\{ \int_{\calM \setminus \hcalM} \|x\|^2 d\rho
\ge \frac{28M^2\log n}{3n}
\right\}  \le 2n^{-\nu}. 
\eeq

The constants in Proposition \ref{propcovertree2} are extremely pessimistic, due to the generality of the assumptions on the space $\mathcal{M}$. Indeed when $\mathcal{M}$ is a nice manifold as in case (5b), the statement in the Proposition says that the constants for the $C_{j,k}$'s we construct are similar to those of the ideal $C_{j,k}$'s. In practice we use a much simpler and more efficient tree construction method and we experimentally obtain the properties above with $\amax = C_1^2 4^d $ and $\amin =1$, at least for the vast majority of the points, and $\theta_{\{3,4\}}\approxeq\ttheta_{\{3,4\}}$.
We describe this simpler construction for the multiscale partitions in Appendix \ref{appatree}, together with experiments suggesting that 
$\theta_{\{3,4\}}\approxeq\ttheta_{\{3,4\}}$.

Besides cover tree, there are other methods that can be used in practice for the multiscale partition, such as METIS by \cite{KarypisSIAM99-METIS} that is used in \citet{CM:MGM2}, iterated PCA (see some analysis in \cite{Szlam:iteratedpartitioning}) or iterated $k$-means. These can be computationally more efficient than cover trees, with the downside being that they may lead to partitions not satisfying our usual assumptions.

\commentout{
\subsection{Construction of multiscale tree decompositions}
\label{secCoverTree}

Our multiscale tree decomposition is constructed from the cover tree algorithm \citep[see][]{LangfordICML06-CoverTree} applied on half of the data denoted by $\calX_n' $. 
In brief the cover tree $T(\calX_n')$ on $\calX_n'$ is a leveled tree where each level is a \lq\lq cover\rq\rq\ for the level beneath it. Each level is indexed $j$ and each node in $T(\calX_n')$ is associated with a point in $\calX'_n$. A point can be associated with multiple nodes in the tree but it can appear at most once at every level. Let $T_j(\calX_n') \subset \calX_n'$ be the set of nodes of $T$ at level $j$. The cover tree obeys the following properties for all $j$:

\begin{enumerate}
\item Nesting: $T_j(\calX_n')  \subset T_{j+1}(\calX_n') $;

\item Separation: for all distinct $p,q \in T_j(\calX_n') $, $\|p-q\|> 2^{-j}$;

\item Covering: for all $q \in T_{j+1}(\calX_n') $, there exists $p \in T_j(\calX_n') $ such that $\|p-q\|<2^{-j}$. The node in level $j$ associated with $p$ is a parent of the node in level $j+1$ associated with $q$.
\end{enumerate}

In the third property, {$q$ is called a child of $p$. Each node can have multiple parents but is only assigned to one of them in the tree.}
The properties above imply that for any $q \in \calX_n'$, there exists $p \in T_j$ such that $\|p-q\|<2^{-j+1}$. The authors  in \cite{LangfordICML06-CoverTree} showed that cover tree always exists and the construction takes time $\mathcal{O}(C^d D n\log n)$ .

Assume that $T_{j}(\calX_n') = \{a_{{j},k}\}_{k=1}^{N(j)}$. Let $j_{\max}$ be the maximal scale of $T(\calX_n')$.  
Define the indexing map
$$k_{j_{\max}}(x):= \operatornamewithlimits{argmin}\limits_{1\le k \le N(j_{\max})} \|x-a_{j_{\max},k}\|,$$
and let $C_{j_{\max},k}$ be the intersection of the Voronoi region and a ball of radius $2^{-j_{\max}}$:
$$C_{j_{\max},k} =  \{x\in \RR^D \ : \  k_{j_{\max}}(x)=k\} \cap B_{2^{-j_{\max}}}(a_{j_{\max},k}).$$
For any $j<j_{\max}$, we define 
$$\Cjk = \bigcup\limits_{\substack{a_{j-1,k'} \text{ is a} \\ \text{child of } a_{j,k}}} C_{j-1,k'}.$$
Let $\hcalM = \cup_{k=1}^{N(j_{\max})} C_{j_{\max},k} $. We will prove that $\calM \setminus \hcalM$ has a small measure (see Appendix \ref{appcovertree} for a proof).

\begin{proposition}
\label{propcovertree1}
The $\{\Cjk\}$ constructed above form a tree, satisfy the Assumption (A4) with $\theta_2 = 2$ and for any $\nu>0$, 
\beq
\label{eqct1}
\PP\left\{\rho(\calM\setminus \hcalM) > \frac{28\nu\log n}{3n}\right\} \le 2n^{-\nu}.
\eeq

\end{proposition}

Even though the $\{\Cjk\}$ does not exactly satisfy the Assumption (A2), we claim that, instead of having Assumption (A2),  Eq. \eqref{eqct1} is sufficient for our performance guarantees in the case that $\calM$ is bounded by $M$ and $d\ge 3$. Simply approximating points on $\calM \setminus \hcalM$ by $0$ gives the error: 
\beq
\label{eqct2}
\PP\left\{ \int_{\calM \setminus \hcalM} \|x\|^2 d\rho
\ge \frac{28M^2\log n}{3n}
\right\}  \le 2n^{-\nu}. 
\eeq

Additionally, we observe that, if $\rho$ is doubling, there exits $\theta_0 >0$ such that every $\Cjk$ constructed above contains a ball of radius $\theta_0 2^{-j}$ (see numerical experiments in Appendix \ref{appcovertree}). Under this assumption, the $\{\Cjk\}$ satisfies (A1-A5) for doubling probability measures on $\calC^s, s\in (1,2]$ manifolds.
\begin{proposition}
\label{propcovertree2}
Assume that $\rho$ is a doubling probability measure on $\calM$ with the doubling constant $C_1$, and Assumption (A2) is satisfied.
Suppose there exits $\theta_0 >0$ such that every $\Cjk$ contains a ball of radius $\theta_0 2^{-j}$. Then the $\{\Cjk\}$ satisfies 
\begin{enumerate}
\item Assumption (A1) with $\amax = C_1^2 (\theta_2/\theta_0)^d 2^d$ and $\amin = C_1^{-2}(\theta_0/\theta_2)^d2^d$;

\item Assumption (A3) with $\theta_3 = \theta_0^d/C_1$;
\end{enumerate}

Furthermore, let $\calM$ be  a compact manifold of class $\calC^s, s\in (1,2]$ without boundaries isometrically embedded in $\RR^D$.
Let $\epsilon := \epsilon(C_1,\theta_0,\theta_2,d) = C_1^2(\theta_2/\theta_0)^2-1$. If $4\epsilon (4+\epsilon) < 1/d$, then the $\{\Cjk\}$ satisfies (A5) when $j$ is sufficiently large.

\end{proposition}

Besides cover tree, there are other methods that can be used for the multiscale partition, such as METIS by \cite{KarypisSIAM99-METIS} that is used in \citet{CM:MGM2}, iterated PCA (see some analysis in the work by \cite{Szlam:iteratedpartitioning}) or iterated $k$-means.
}

\section{Numerical experiments}
\label{secnum}

We conduct numerical experiments on both synthetic and real data to demonstrate the performance of our algorithms. Given $\{x_i\}_{i=1}^n$, we split them to training data for the constructions of empirical GMRA and adaptive GMRA and test data for the evaluation of the approximation errors:
\begin{center}
\renewcommand{\arraystretch}{1.4}
\resizebox{0.8\columnwidth}{!}{
\begin{tabular}{|c | c | c  |}
\hline
    & $L^2$ error & $L^\infty$ error  \\
        \hline
    Absolute error & 
    $\bigg( \frac{1}{n^{\rm test}} \displaystyle\sum_{x_i \in {\rm test\ set}}
{\|x_i - \hcalP x_i\|^2}\bigg)^{\frac 1 2}$
    & $\displaystyle \max_{x_i \in {\rm test\ set}}
{\|x_i - \hcalP x_i\|}$
    \\
    \hline  
    Relative error & 
    $\bigg(\frac{1}{n^{\rm test}} \displaystyle\sum_{x_i \in {\rm test\ set}}
{\|x_i - \hcalP x_i\|^2}/{\|x_i\|^2} \bigg)^{\frac 1 2}$
    &
    $\displaystyle\max_{x_i \in {\rm test\ set}}
{\|x_i - \hcalP x_i\|}/{\|x_i\|} $
    \\
    \hline  
\end{tabular}}
\end{center}
where $n^{\rm test}$ is the cardinality of the test set and $\hcalP$ is the piecewise linear projection given by empirical GMRA or adaptive GMRA. In our experiments we use absolute error for synthetic data, 3D shape and relative error for the MNIST digit data, natural image patches.

\begin{figure}[th]
  \centering
  %
\begin{minipage}{0.2\columnwidth}
   \subfigure[Projection of S and Z manifolds]{
    \includegraphics[width=2.5cm,height=4.5cm]{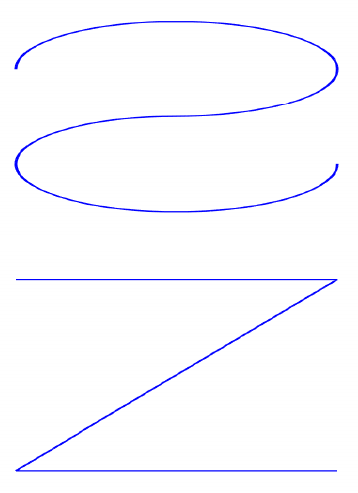}
    }
\end{minipage}
\begin{minipage}{0.79\columnwidth}
     \subfigure[S: error vs. scale]{
    \includegraphics[width=0.49\columnwidth]{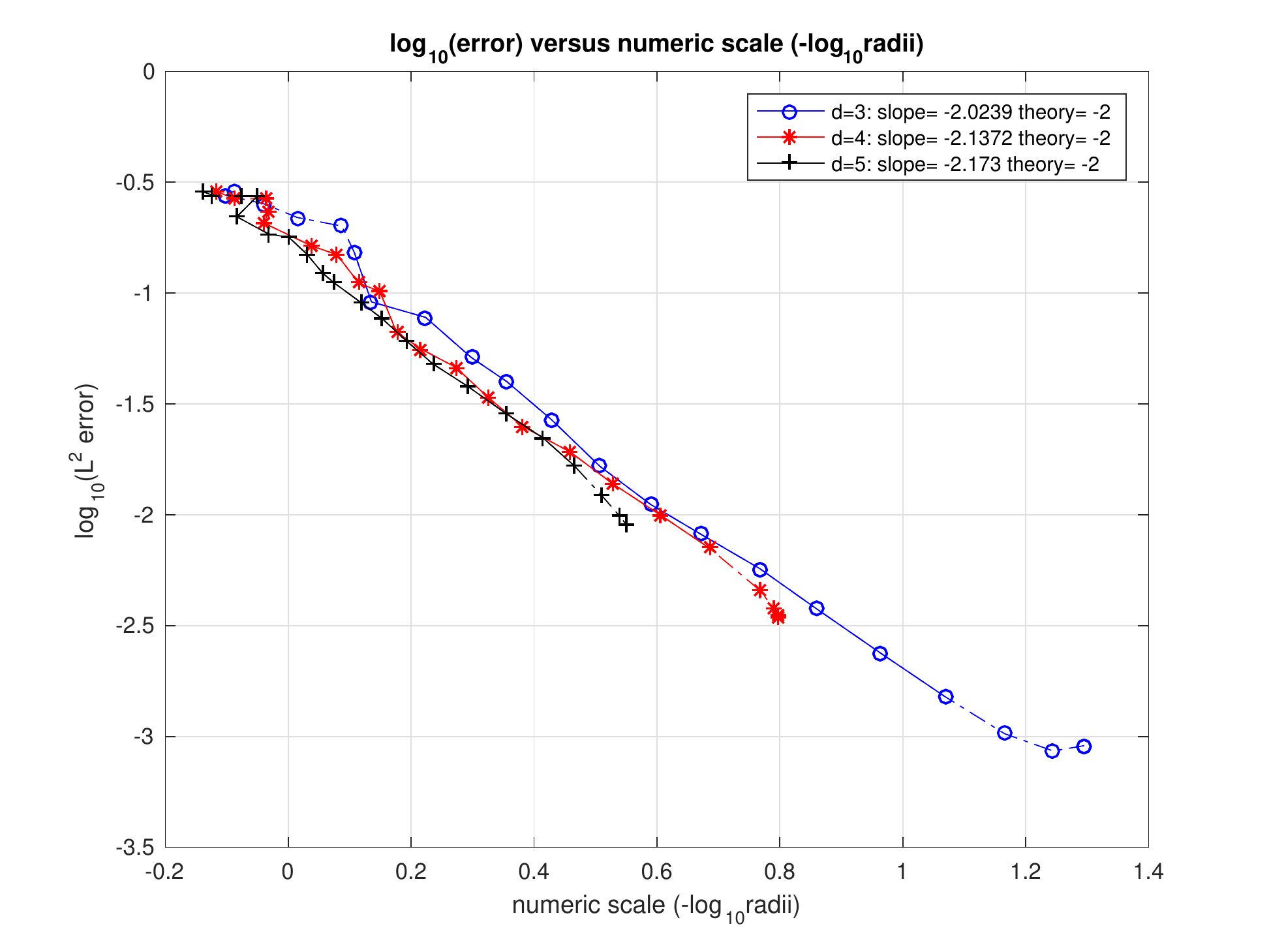}
    }
        \hspace{-0.6cm}
     \subfigure[Z: error vs. scale]{
    \includegraphics[width=0.49\columnwidth]{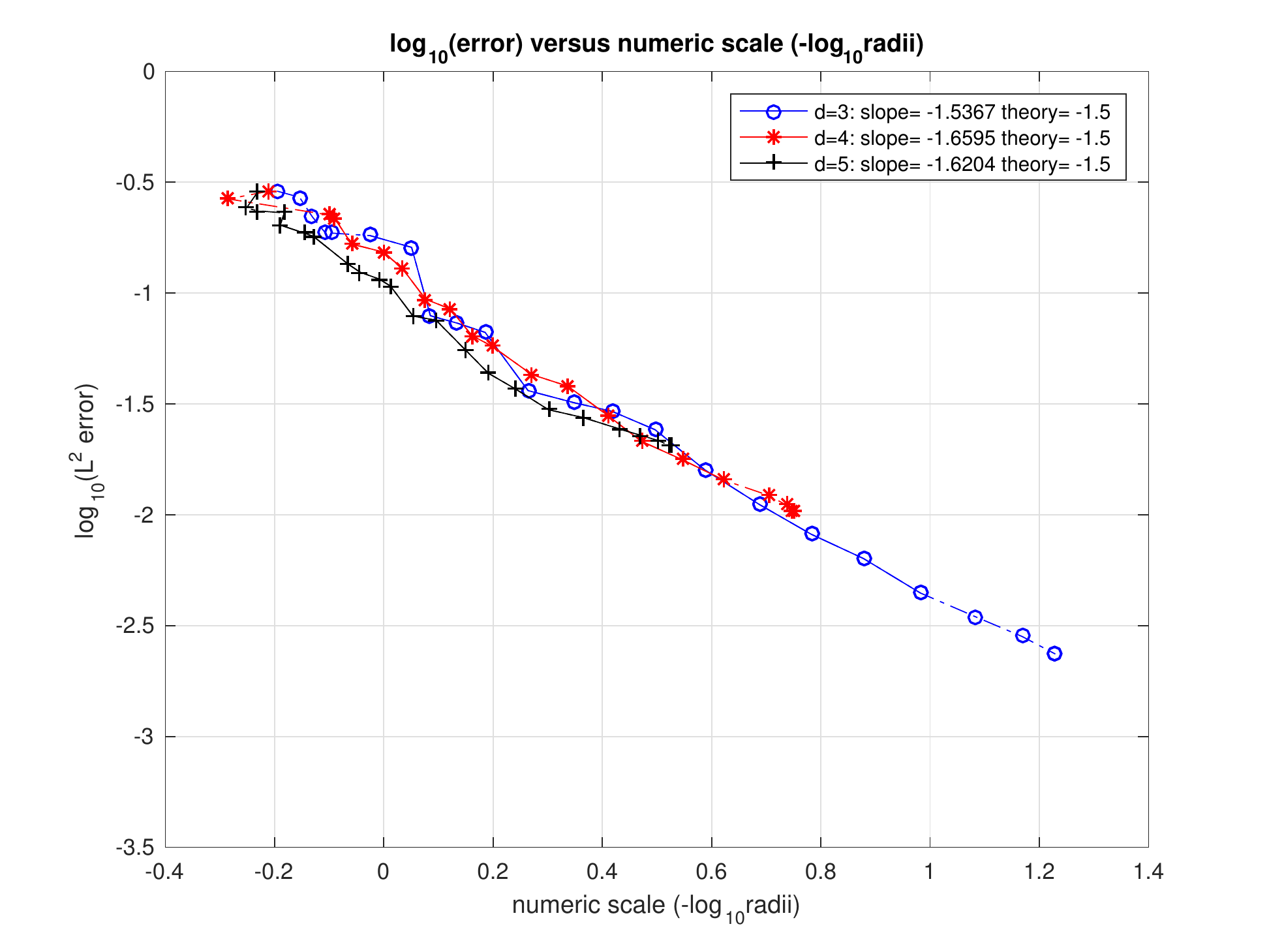}
    }
    \end{minipage}
    \\
\begin{minipage}{0.2\columnwidth}
\vskip-0.5cm
\center
    \resizebox{1\columnwidth}{!}{
      \begin{tabular}{| c  |c | c|}
    \hline
    & \multicolumn{2}{|c|}{the S manifold}	\\
     	\hline  	& theoretical $s$  & numerical $s$ 	\\
  	\hline    	$d=3$ & $2$&  $2.1\pm0.2$ \\
  	\hline	$d=4$ & $2$& $2.1\pm0.2$ \\
  	\hline	$d=5$ & $2$& $2.3\pm0.1$\\
    	\hline	
    & \multicolumn{2}{|c|}{ the Z manifold}  	\\
     	\hline  	& theoretical  $s$  & numerical $s$ 	\\
  	\hline	$d=3$ & $+\infty$& $3.4\pm0.3$\\
	\hline	$d=4$ & $3$& $3.1\pm0.2$\\
  	\hline 	$d=5$ & $2.25$&$2.5\pm0.2$\\
    	\hline	
    \end{tabular}	
    }
    \end{minipage}
\begin{minipage}{0.79\columnwidth}
\subfigure[S: error vs. tree complexity]{
{\includegraphics[width=0.49\columnwidth]{Fig23_SandZ/Sample100000/V2/MoreExperiments/S_Bs_1-eps-converted-to.pdf} }}
        \hspace{-0.6cm}
\subfigure[Z: error vs. tree complexity]{
{ \includegraphics[width=0.49\columnwidth]{Fig23_SandZ/Sample100000/V2/MoreExperiments/Z_Bs_1-eps-converted-to.pdf} }}
\end{minipage}
            \caption{
   $10^5$ training points are sampled on the $d$-dimensional S or Z manifold $(d=3,4,5)$. 
  In (b) and (c), we display $\log_{10} \|X-\hcalP_j X\|_n$, versus scale $j$.
  The negative of the slope on the solid portion of the line approximates the regularity parameter $s$ in the $\AS$ model. 
    In (d) and (e), we display the log-log plot of $\|X-\calP_{\hLameta} X\|_n^{d-2}$ versus the weighted complexity of the adaptive partition for the $d$-dimensional S and Z manifold. The negative of the slope on the solid portion of the line approximates the regularity parameter $s$ in the $\BS$ model. 
    Our five experiments give the $s$ in the table. For the $3$-dim Z manifold, while $s=+\infty$ in the case of infinite samples, we do obtain a large $s$ with $10^5$ samples. 
      }
    \label{FigASBS}
\end{figure}

\subsection{Synthetic data}
We take samples $\{x_i\}_{i=1}^n$ on the $d$-dim S and Z manifold 
whose $x_1,x_2$ coordinates are on the S and Z curve and $x_i\in [0,1], i=3,4,5$
and evenly split them to the training set and the test set.
In the noisy case, training data are corrupted by Gaussian noise: $\tilde{x}^{\rm train}_i = x^{\rm train}_i + \frac{\sigma}{\sqrt{D}} \xi_i, i=1,\ldots,\frac n 2$ where $\xi_i \sim \mathcal{N}(0,I_{D\times D})$, but test data are noise-free. Test data error below the noise level imply that we are denoising the data.

\subsubsection{Regularity parameter $s$ in the $\AS$ and $\BS$ model}
We sample $10^5$ training points on the $d$-dim S or Z manifold $(d=3,4,5)$. The measure on the S manifold is not exactly the volume measure but is comparable with the volume measure. 

The log-log plot of the approximation error versus scale in Figure \ref{FigASBS} (b) shows that volume measures on the $d$-dim S manifold are in $\AS$ with $s \approx 2.0239, 2.1372, 2.173$ when $d=3,4,5$, consistent with our theory which gives $s=2$. Figure \ref{FigASBS} (c) shows that volume measures on the $d$-dim Z manifold are in $\AS$ with $s \approx 1.5367, 1.6595, 1.6204$ when $d=3,4,5$, consistent with our theory which gives $s=1.5$.

The log-log plot of the approximation error versus the weighted complexity of the adaptive partition in Figure \ref{FigASBS} (d) and (e) gives rises to an approximation of the regularity parameter $s$ in the $\BS$ model in the table.

\subsubsection{Error versus sample size $n$} 
We take $n$ samples on the $4$-dim S and Z manifold. 
In Figure \ref{FigSZ}, we set the noise level $\sigma =0$ (a) and $\sigma = 0.05$ (b), display the log-log plot of the average  approximation error over 5 trails with respect to the sample size $n$ for empirical GMRA at scale $j^*$ which is chosen as per Theorem \ref{thm2}: $2^{-j^*} = [(\log n)/n]^{\frac{1}{2s+d-2}}$ with $d=4$ and $s=2$ for the S manifold and $s=1.5$ for the Z manifold. For adaptive GMRA, the ideal $\kappa$ increases as $\sigma$ increases.  We let $\kappa\in\{0.05,0.1\}$ when $\sigma =0$ and $\kappa\in\{1,2\}$ when $\sigma = 0.05$. We also test the Nearest Neighbor (NN) approximation. 
The negative of the slope, determined by least squared fit, gives rise to the rate of convergence: $L^2 \ {\rm error} \sim (1/n)^{-\rm slope}$. 
When $\sigma = 0$, the convergence rate for the nearest neighbor  approximation should be $1/d =0.25$. GMRA gives rise to a smaller error and a faster rate of convergence than the nearest neighbor approximation.
For the Z manifold, Adaptive GMRA yields a faster rate of convergence than GMRA. When $\sigma = 0.05$, adaptive GMRA with $\kappa = 0.5$ and $1$ gives rise to the fastest rate of convergence. Adaptive GMRA with $\kappa = 0.05$ has similar rate of convergence as the nearest neighbor approximation since the tree is almost truncated at the finest scales.
We note a de-noising effect when the approximation error falls below $\sigma$ as $n$ increases. In adaptive GMRA, when $\kappa$ is sufficiently large, i.e., $\kappa = 0.5,1$ in this example, different values of $\kappa$ do yield different errors up to a constant, but the rate of convergence is independent of $\kappa$, as predicted by Theorem \ref{thm3}.

\commentout{
\begin{figure}[h]
\begin{minipage}{0.35\columnwidth}
{\includegraphics[width=\columnwidth]{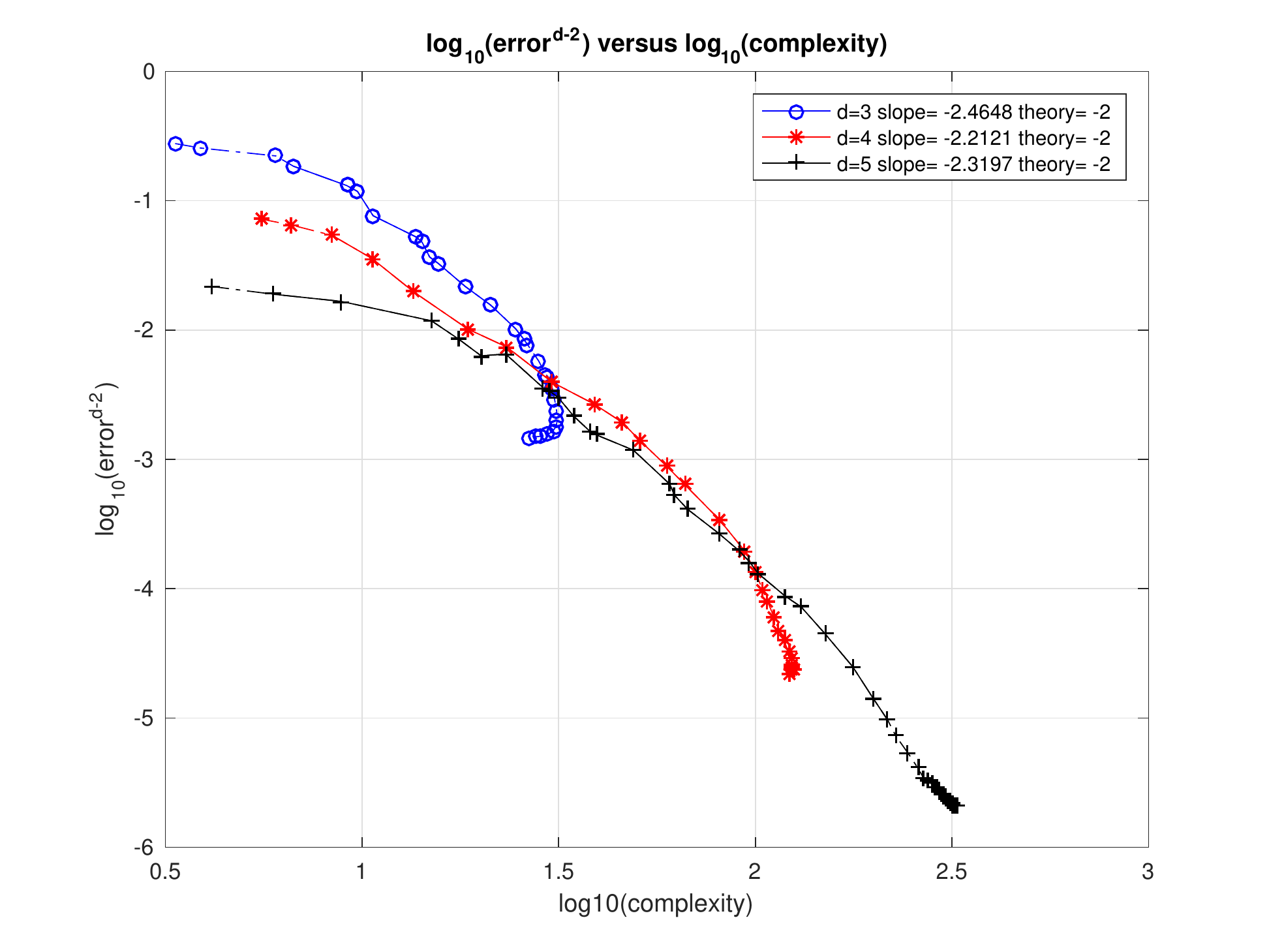} }
\end{minipage}
\begin{minipage}{0.35\columnwidth}
{ \includegraphics[width=\columnwidth]{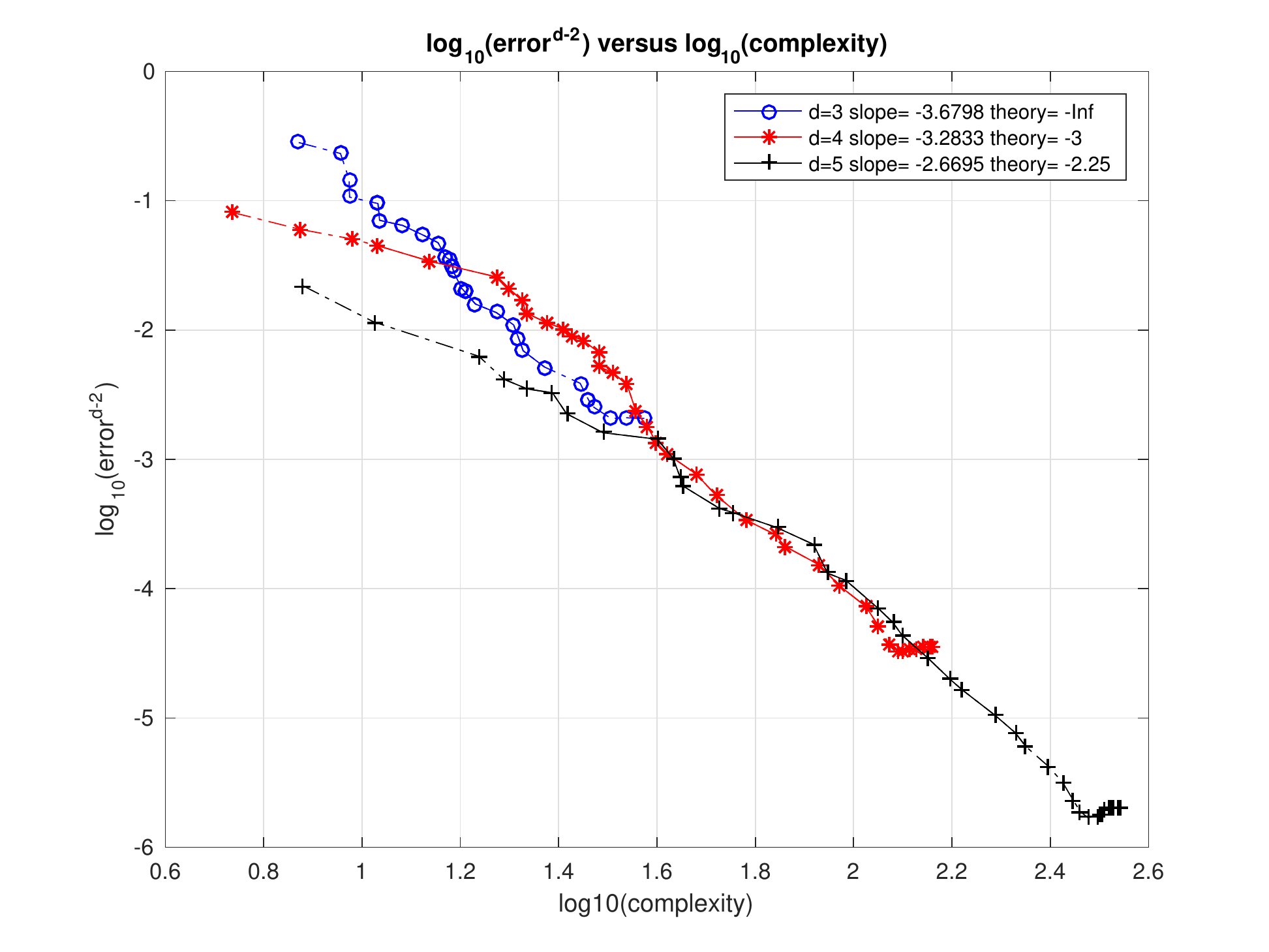} }
\end{minipage}
\begin{minipage}{0.2\columnwidth}
\vskip-0.5cm
\center
    \resizebox{1.25\columnwidth}{!}{
      \begin{tabular}{| c  |c | c|}
    \hline
    & \multicolumn{2}{|c|}{the S manifold}	\\
     	\hline  	$s$ & theoretical & numerical	\\
  	\hline    	$d=3$ & $2$&  $2.1\pm0.2$ \\
  	\hline	$d=4$ & $2$& $2.1\pm0.2$ \\
  	\hline	$d=5$ & $2$& $2.3\pm0.1$\\
    	\hline	
    & \multicolumn{2}{|c|}{ the Z manifold}  	\\
  	\hline	$d=3$ & $+\infty$& $3.4\pm0.3$\\
	\hline	$d=4$ & $3$& $3.1\pm0.2$\\
  	\hline 	$d=5$ & $2.25$&$2.5\pm0.2$\\
    	\hline	
    \end{tabular}	
    }
    \end{minipage}
    \caption{Log-log plot of $\|X-\calP_{\hLameta} X\|_n^{d-2}$ versus the weighted complexity of the adaptive partition on training data for the $d$-dimensional S manifold (a) and the $d$-dimensional Z manifold (b). The negative of the slope on the solid portion of the line approximates the regularity parameter $s$ in the $\BS$ model. 
    Our five experiments give the $s$ in the following table. For the $3$-dim Z manifold, while $s=+\infty$ in the case of infinite samples, we do obtain a large $s$ with $10^5$ samples. }
    \label{FigBS}
\end{figure}
}

\subsubsection{Robustness of GMRA and adaptive GMRA}
The robustness of the empirical GMRA and adaptive GMRA is tested on the $4$-dim S and Z manifold while $\sigma$ varies but $n$ is fixed to be $10^5$.
Figure \ref{FigSZ2} shows that the average $L^2$ approximation error in $5$ trails increases linearly with respect to $\sigma$ for both uniform and adaptive GMRA with $\kappa \in \{0.05,0.5,1\}$.


\begin{figure}[t]
 \centering
 \subfigure[the $4$-dim S manifold, $\sigma =0$]{
    \includegraphics[clip,trim=20 5 25 30,width=.45\textwidth]{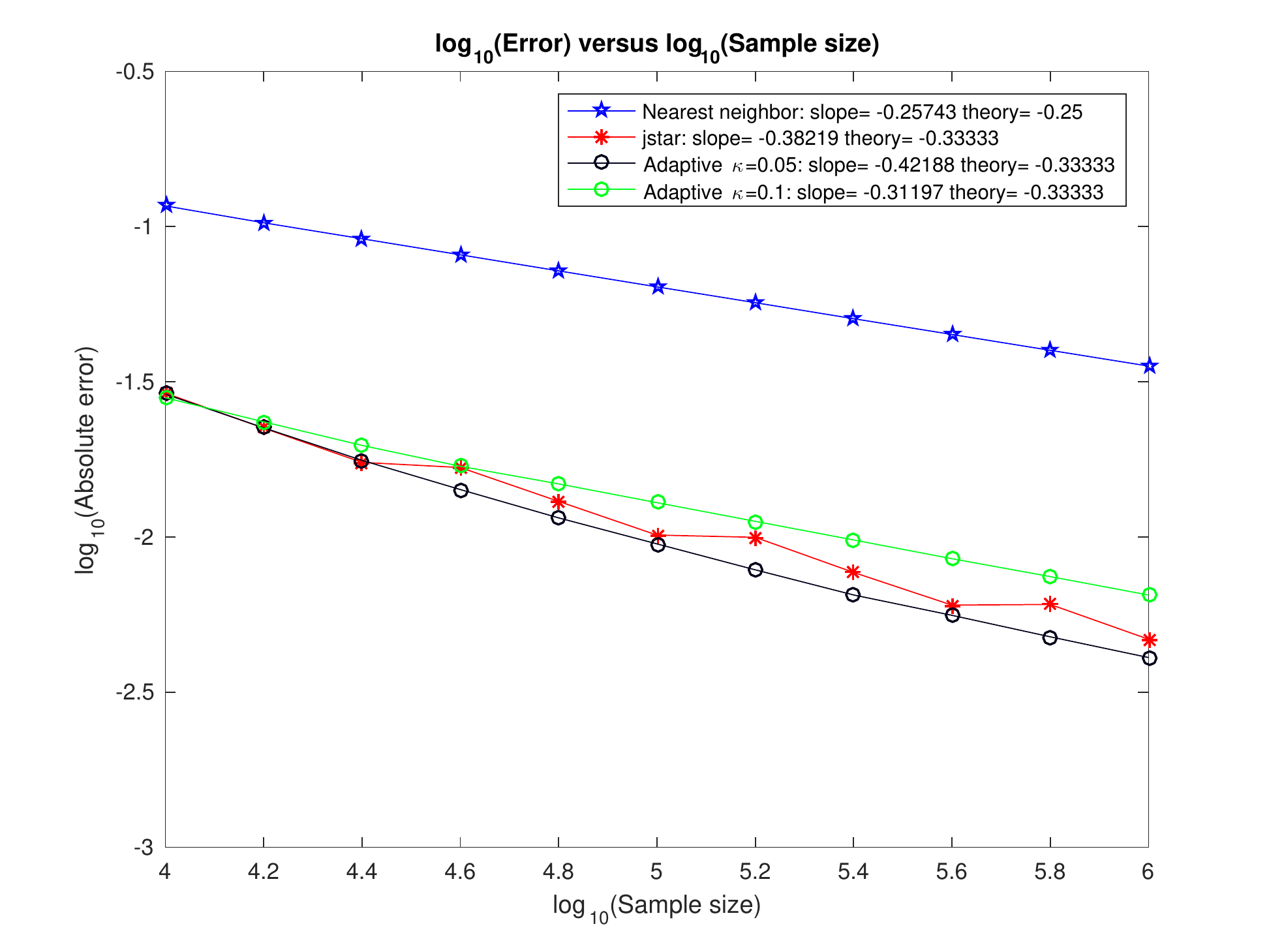}
    }
    \hspace{-0.8cm}
    \subfigure[the $4$-dim S manifold, $\sigma =0.05$]{
    \includegraphics[clip,trim=20 5 25 25,width=.45\textwidth]{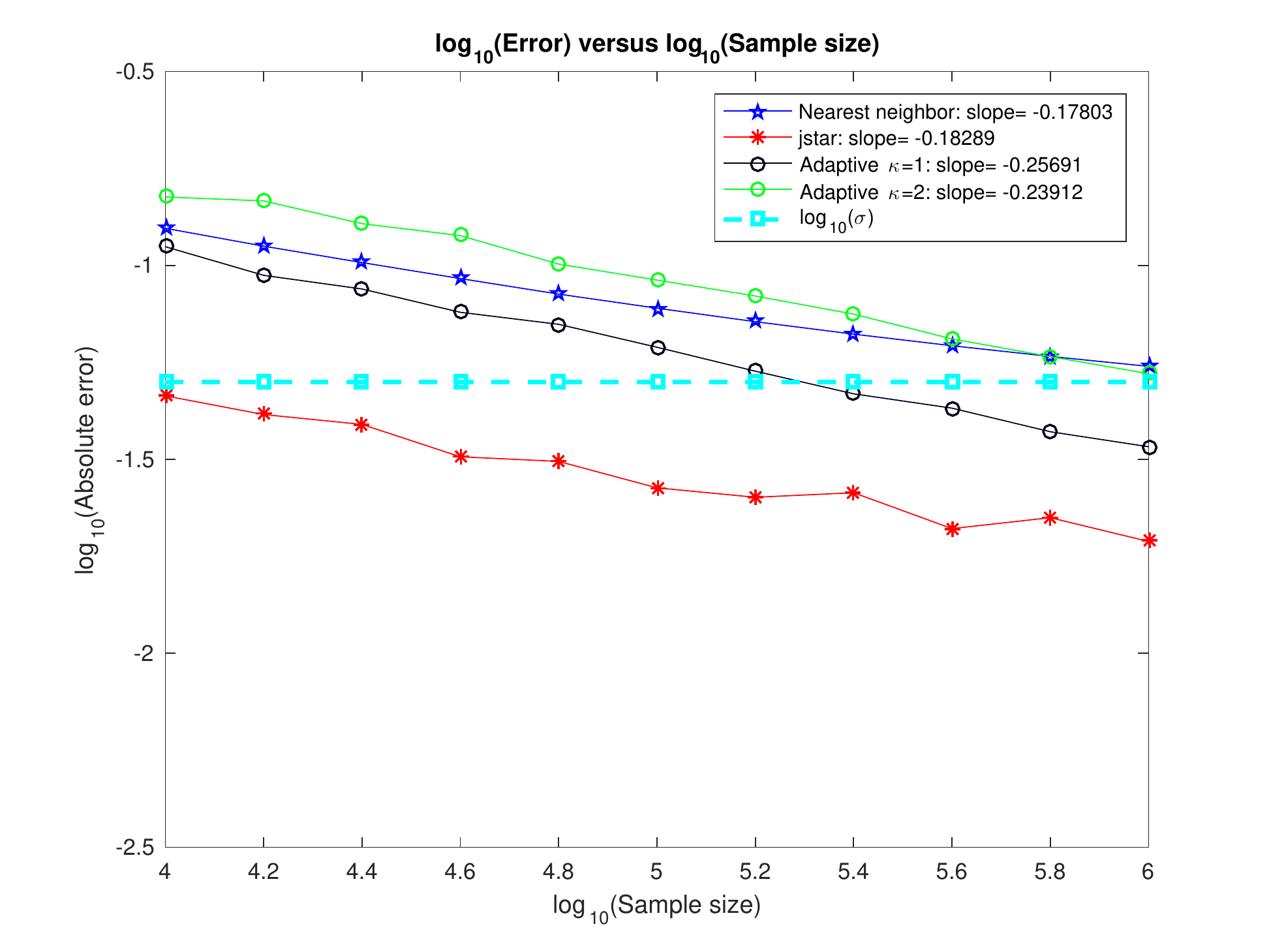}
    }
    \\
 \subfigure[the $4$-dim Z manifold, $\sigma =0$]{
     \includegraphics[clip,trim=20 5 25 30,width=.45\textwidth]{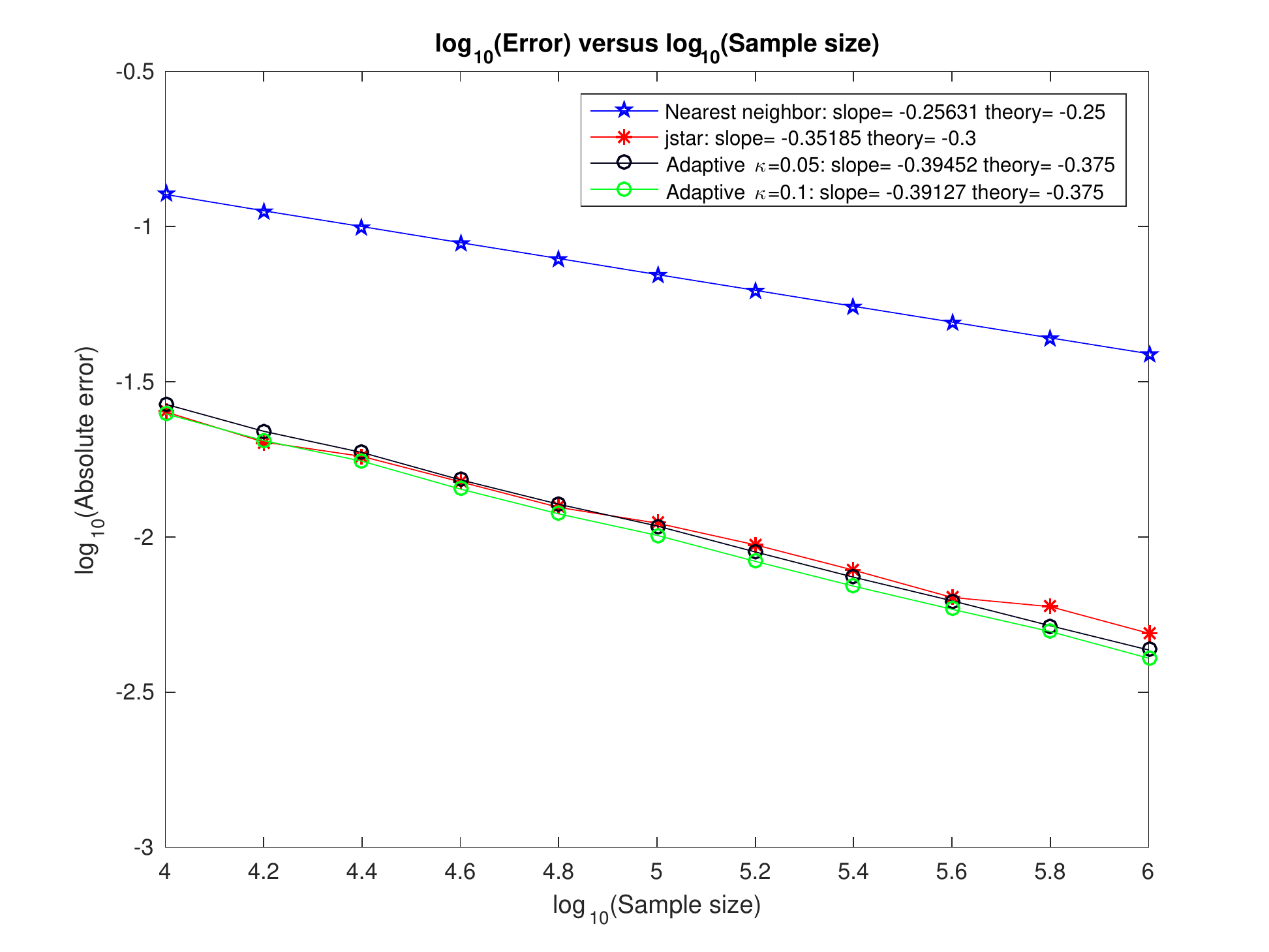}
    }
    \hspace{-0.8cm}
    \subfigure[the $4$-dim Z manifold, $\sigma =0.05$]{
    \includegraphics[clip,trim=20 5 25 25,width=.45\textwidth]{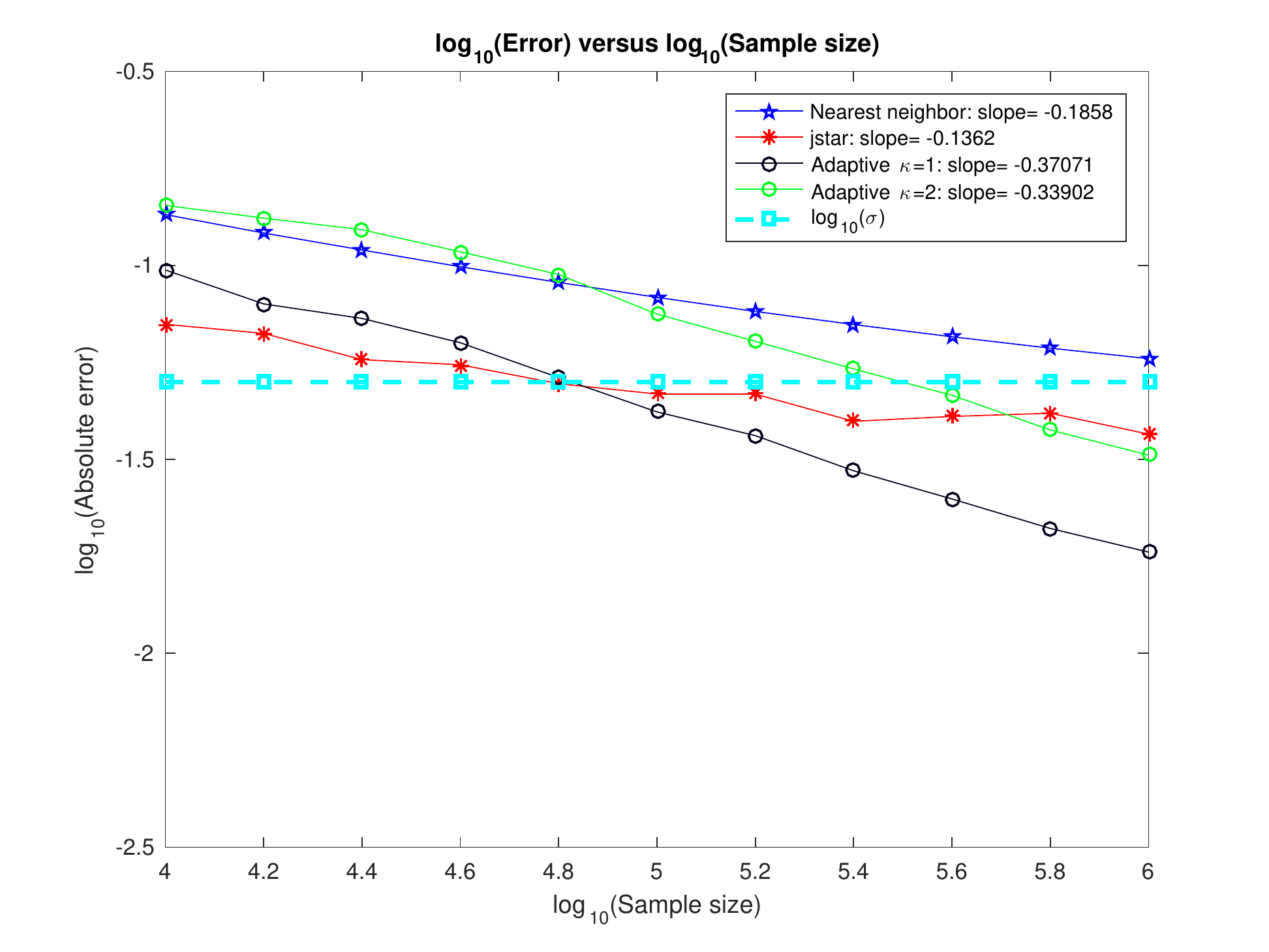}
    }
\\
               \caption{$L^2$ error versus the sample size $n$, for the $4$-dim S and Z manifolds $(d=4)$, of GMRA at the scale $j^*$ chosen as per Theorem \ref{thm2} and adaptive GMRA with varied $\kappa$. We let $\kappa\in\{0.05,0.1\}$ when $\sigma =0$ and $\kappa\in\{1,2\}$ when $\sigma = 0.05$.
               }    
  \label{FigSZ}
\end{figure}

\begin{figure}[ht]
\centering   
  \begin{minipage}[c]{0.59\textwidth}
  \vspace{-20pt}
       \caption{The average $L^2$ approximation error in $5$ trails versus $\sigma$ for GMRA and adaptive GMRA with $\kappa \in \{0.05,0.5,1\}$ on data sampled on the 4-dim S and Z manifolds. This shows the error of approximation grows linearly with the noise size, suggesting robustness in the construction.}
  \end{minipage}
 \begin{minipage}[c]{0.4\textwidth}
   \includegraphics[clip,trim=20 5 25 30,width=\textwidth]{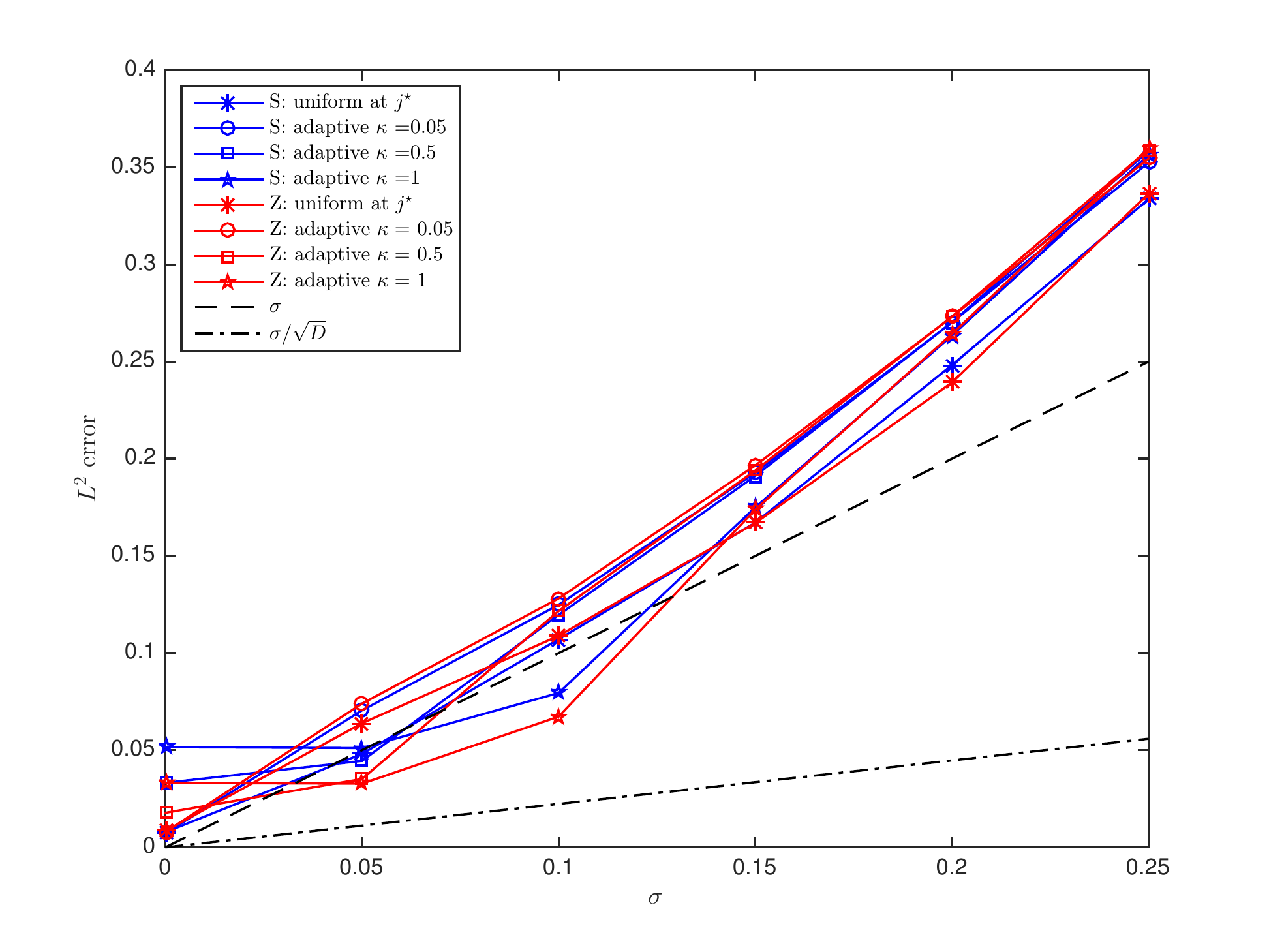}
  \end{minipage}\hfill
       \label{FigSZ2}
\end{figure}

\subsection{3D shapes}

\begin{figure}[ht]
\centering
\hspace{-0.5cm}
    \subfigure[41,472 points]{ 
    \includegraphics[clip,trim=120 10 20 5,width=4.5cm]{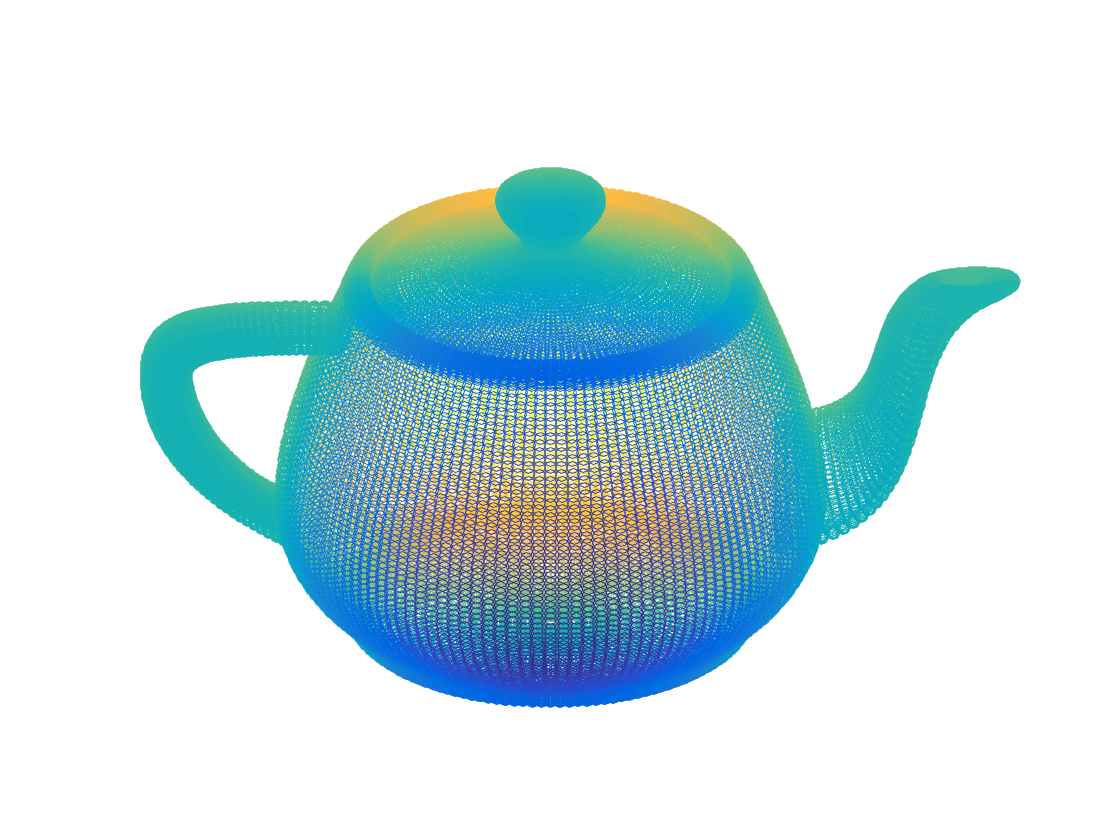}
     }
     \hspace{-0.3cm}
      \subfigure[165,954 points]{
    \includegraphics[clip,trim=10 10 20 5,width=4.4cm,height=3.8cm]{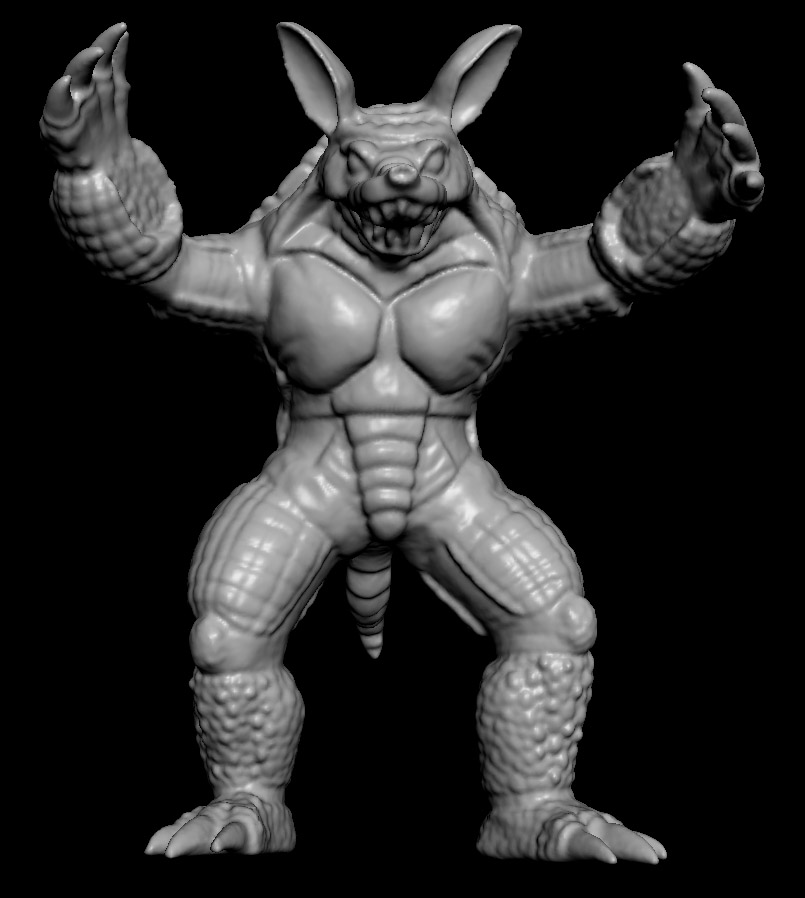}
    }
    \hspace{0.3cm}
  \subfigure[437,645 points]
    {    \includegraphics[clip,trim=10 10 20 5,width=4.4cm,height=3.8cm]{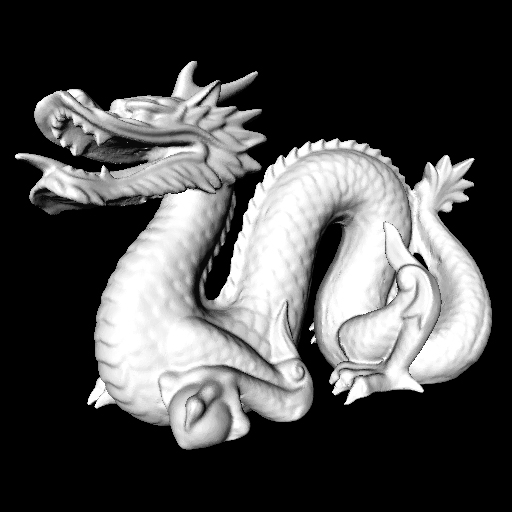}
}
      \subfigure[$\kappa \approx 0.18$, partition size $= 338$]{ 
    \includegraphics[clip,trim=125 10 20 15,width=5.05cm]{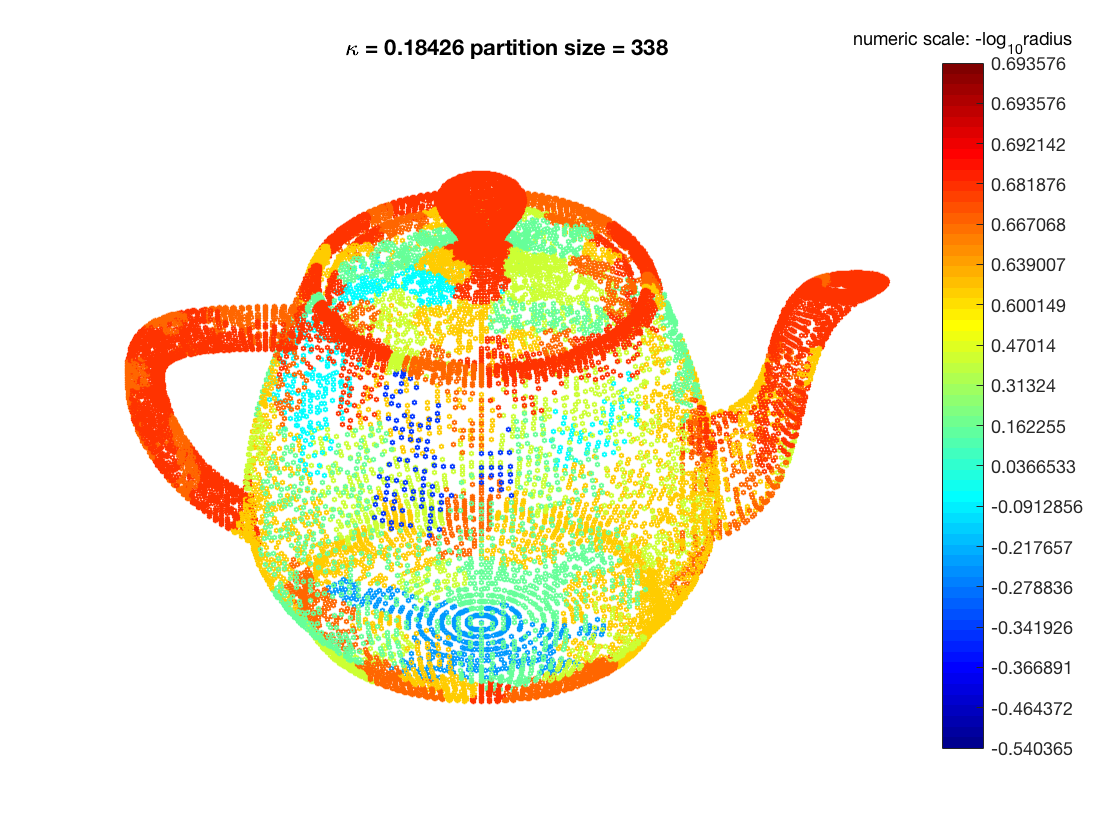}
     }      
     \hspace{-0.6cm}
      \subfigure[$\kappa \approx 0.41$, partition size $= 749$]{ 
    \includegraphics[clip,trim=120 10 20 10,width=5.1cm]{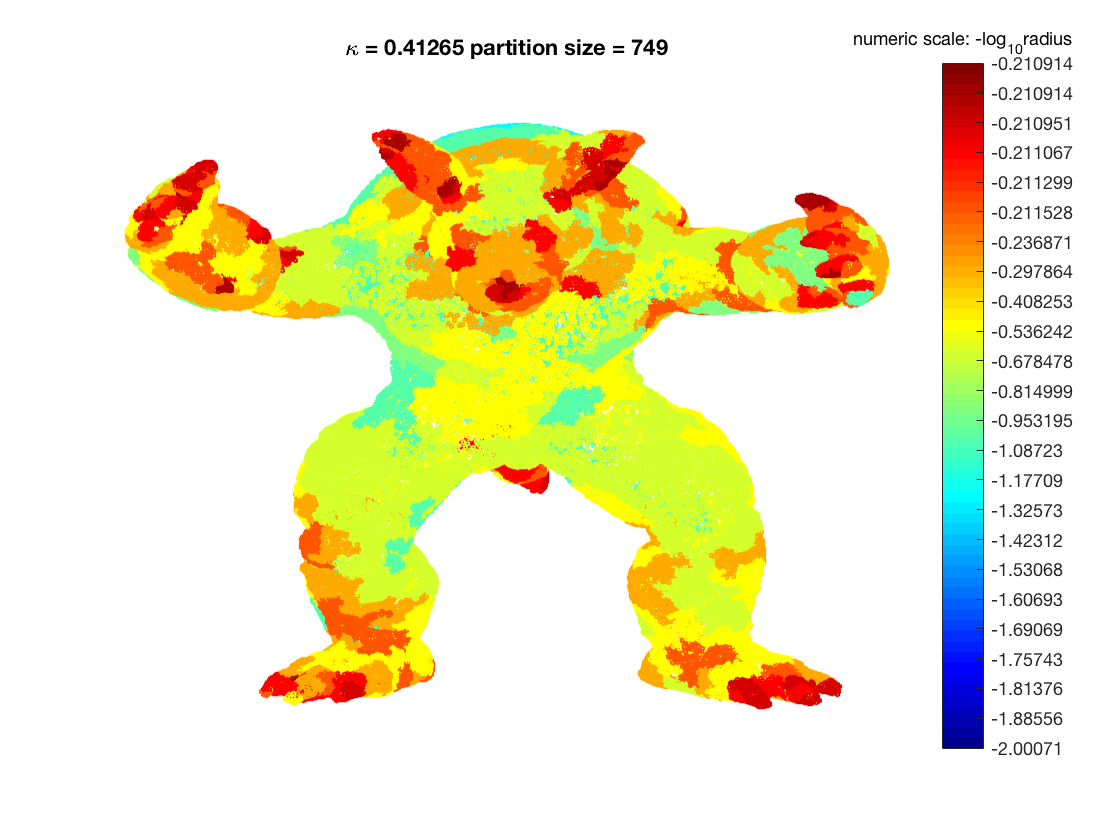}
     }     
     \hspace{-0.7cm}
      \subfigure[$\kappa \approx 0.63$, partition size $= 1141$]{ 
    \includegraphics[clip,trim=120 10 20 5,width=5.1cm]{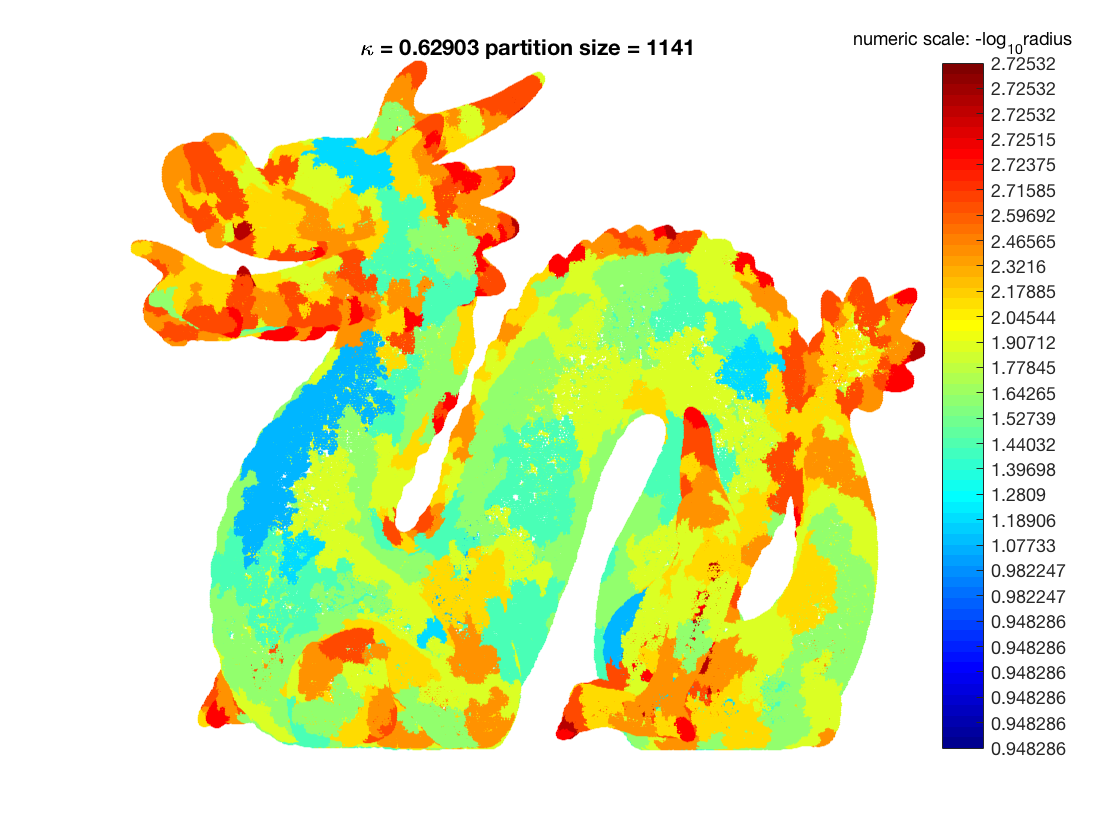}
     }      
\caption{Top line: 3D shapes; bottom line: adaptive partitions selected with refinement criterion $\hDeltajk \ge 2^{-j} \kappa\sqrt{(\log n)/n}$. Every cell is colored by 
scale. In the adaptive partition, at irregular locations cells are selected at finer scales than at \lq\lq flat\rq\rq\ locations.} 
\label{FigShapePartition}
\end{figure}

We run GMRA and adaptive GMRA on 3D points clouds on the teapot, armadillo and dragon in Figure \ref{FigShapePartition}. The teapot data are from the matlab toolbox and others are from the Stanford 3D Scanning Repository \url{http://graphics.stanford.edu/data/3Dscanrep/}.

Figure \ref{FigShapePartition} shows that the adaptive partitions chosen by adaptive GMRA matches our expectation that, at irregular locations, cells are selected at finer scales than at \lq\lq flat\rq\rq\ locations.

In Figure \ref{FigShape}, we display the absolute $L^2/L^\infty$ approximation error on test data versus scale and partition size. The left column shows the  $L^2$ approximation error versus scale for GMRA and the center approximation. While the GMRA approximation is piecewise linear, the center approximation is piecewise constant. Both approximation errors decay from coarse to fine scales, but GMRA yields a  smaller error than the approximation by local centers. In the middle column, we run GMRA and adaptive GMRA with the $L^2$ refinement criterion defined in Table \ref{TableRefinement} with scale-dependent ($\Deltajk \ge 2^{-j}\tau_n$) and scale-independent ($\Deltajk \ge \tau_n$) threshold respectively, and display the log-log plot of the $L^2$ approximation error versus the partition size. Overall adaptive GMRA yields the same $L^2$ approximation error as GMRA with a smaller partition size, but the difference is insignificant in the armadillo and dragon, as these 3D shapes are complicated and the $L^2$ error simply averages the error at all locations. Then we implement adaptive GMRA with the $L^\infty$ refinement criterion: $\hDeltajk^\infty  = \max_{x_i \in \Cjk} \|\hcalP_{j+1}x_i -\hcalP_j x_i\|$ and display the log-log plot of the $L^\infty$ approximation error versus the partition size in the right column. In the $L^\infty$ error, adaptive GMRA saves a considerable number (about half) of cells in order to achieve the same approximation error as GMRA. In this experiment, scale-independent threshold is slightly better than scale-dependent threshold in terms of saving the partition size.

\begin{figure}[th]
 \centering
Teapot
\\
   \includegraphics[clip,trim=10 10 20 5,width=5cm]{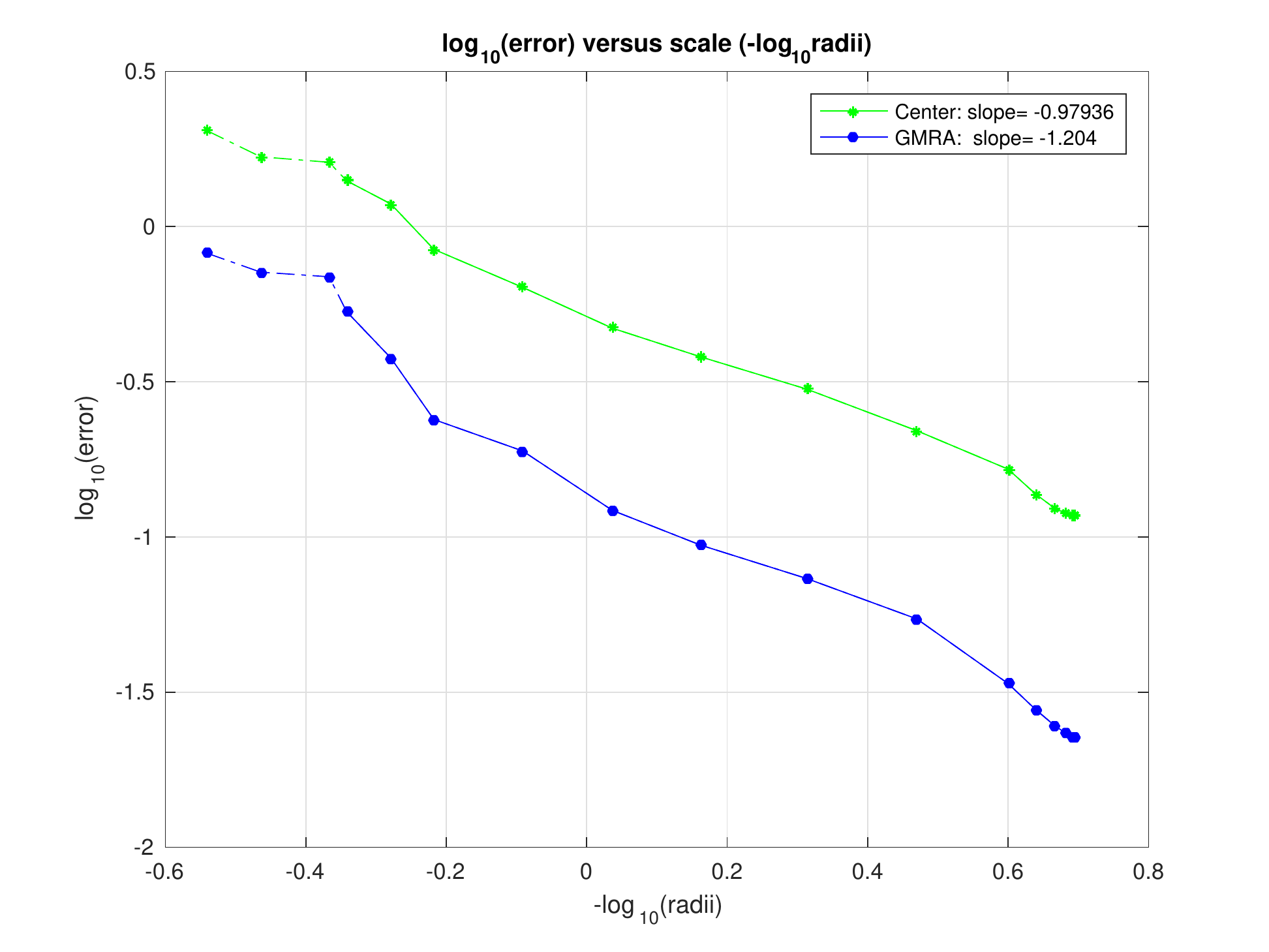}
      \hspace{-0.5cm}
      \includegraphics[clip,trim=10 10 20 5,width=5cm]{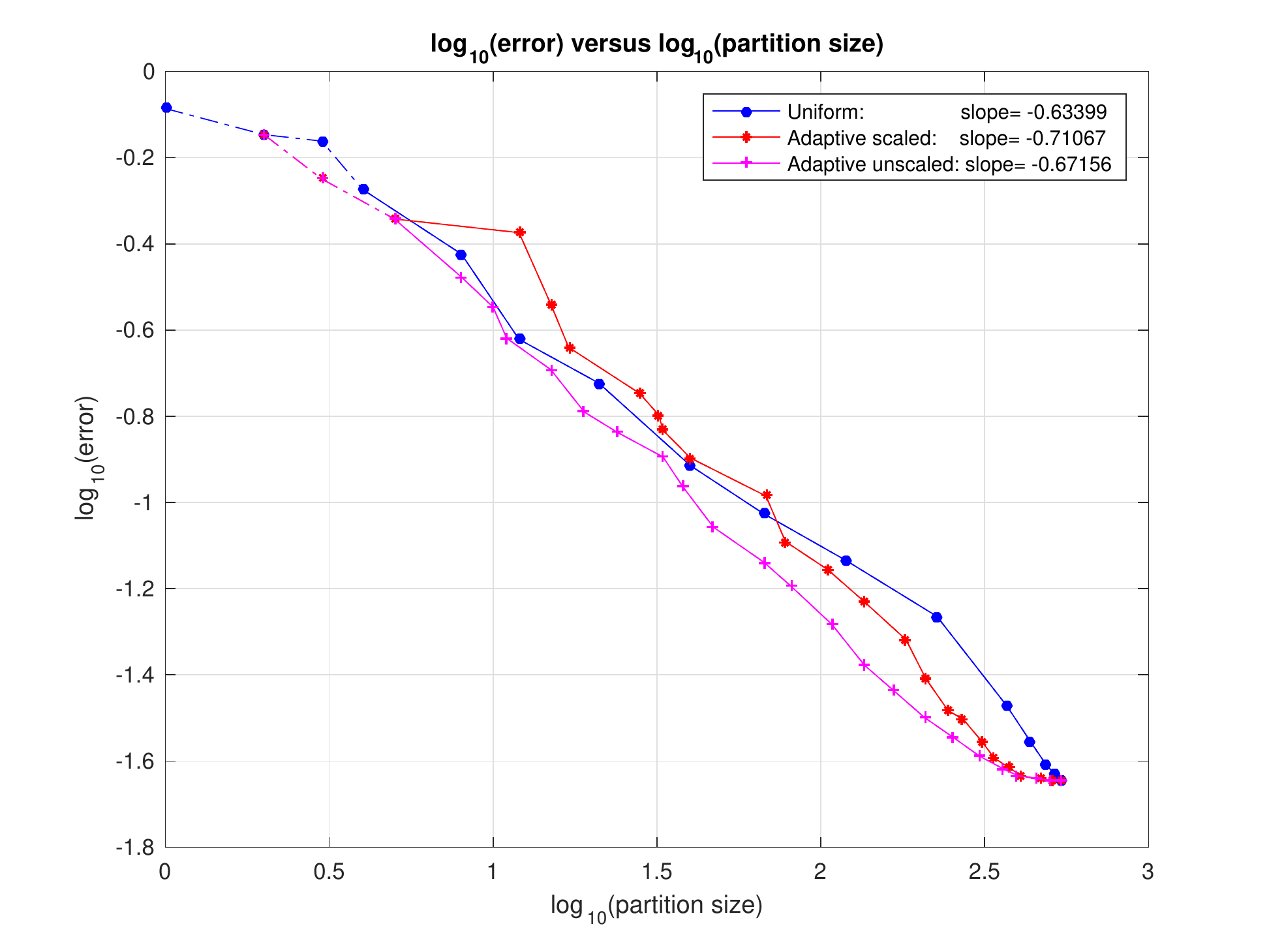}
      \hspace{-0.5cm}
   \includegraphics[clip,trim=10 10 20 5,width=5cm]{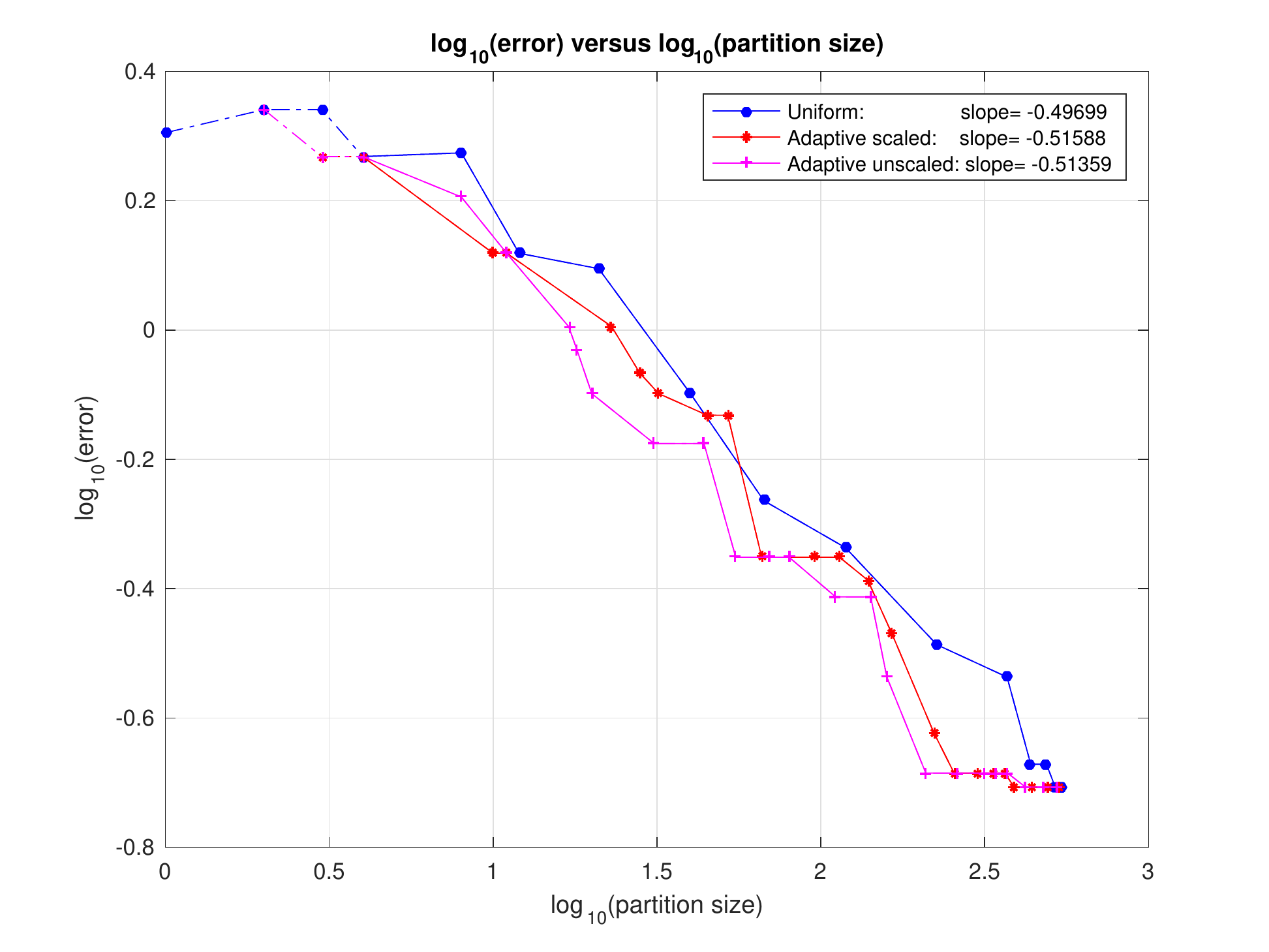}
Armadillo
\\
   \includegraphics[clip,trim=10 10 20 5,width=5cm]{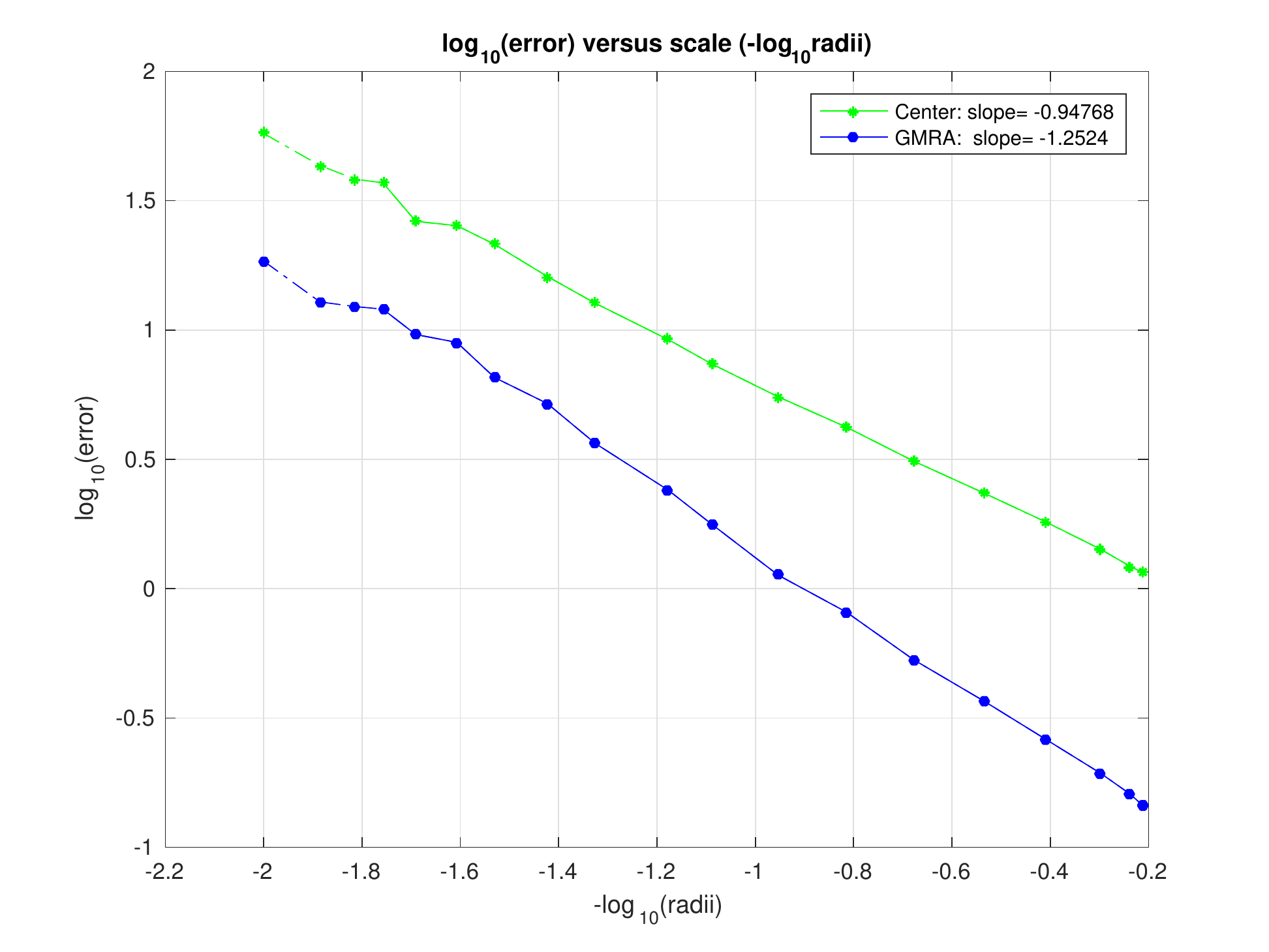}
      \hspace{-0.5cm}
      \includegraphics[clip,trim=10 10 20 5,width=5cm]{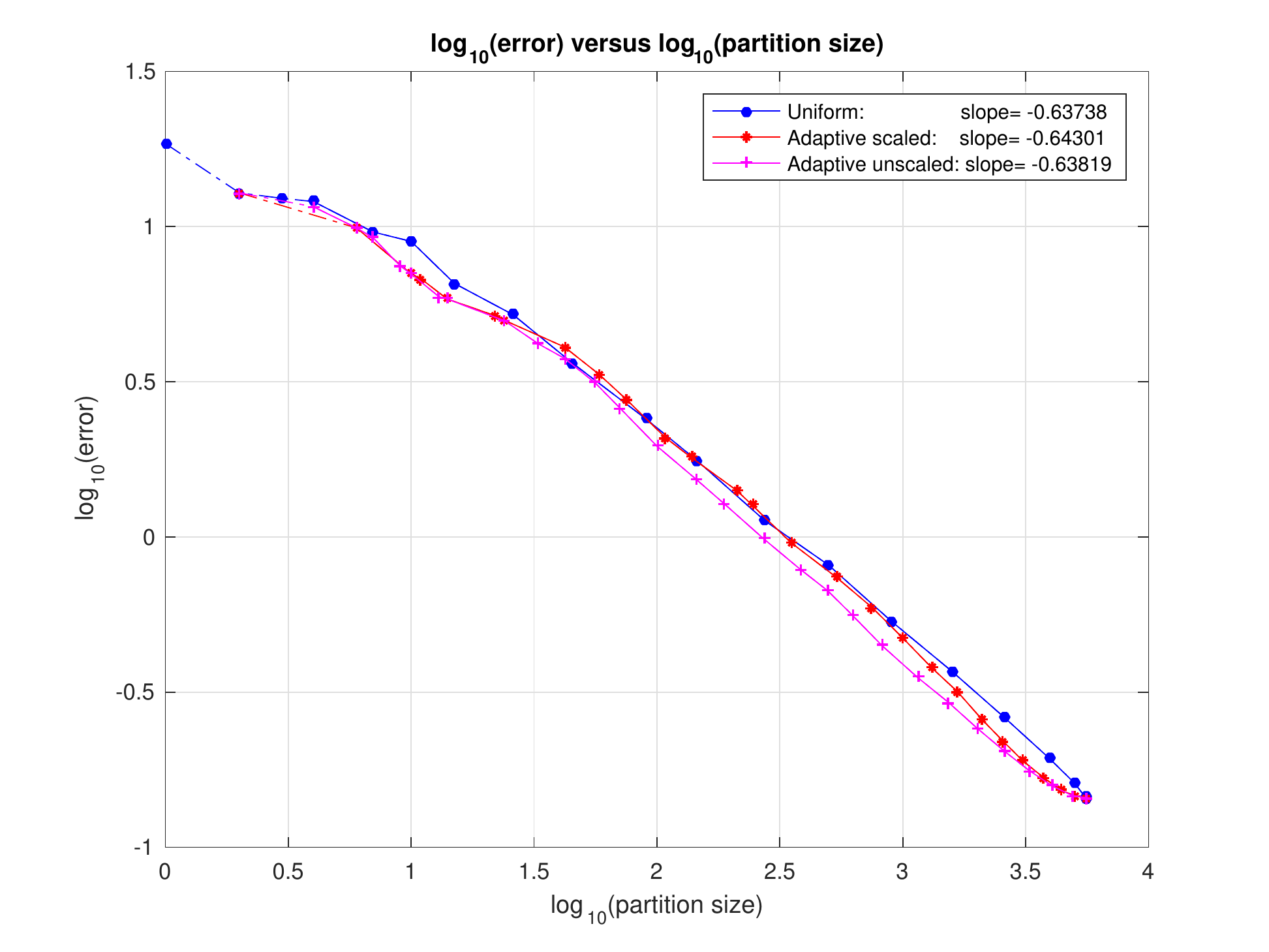}
      \hspace{-0.5cm}
   \includegraphics[clip,trim=10 10 20 5,width=5cm]{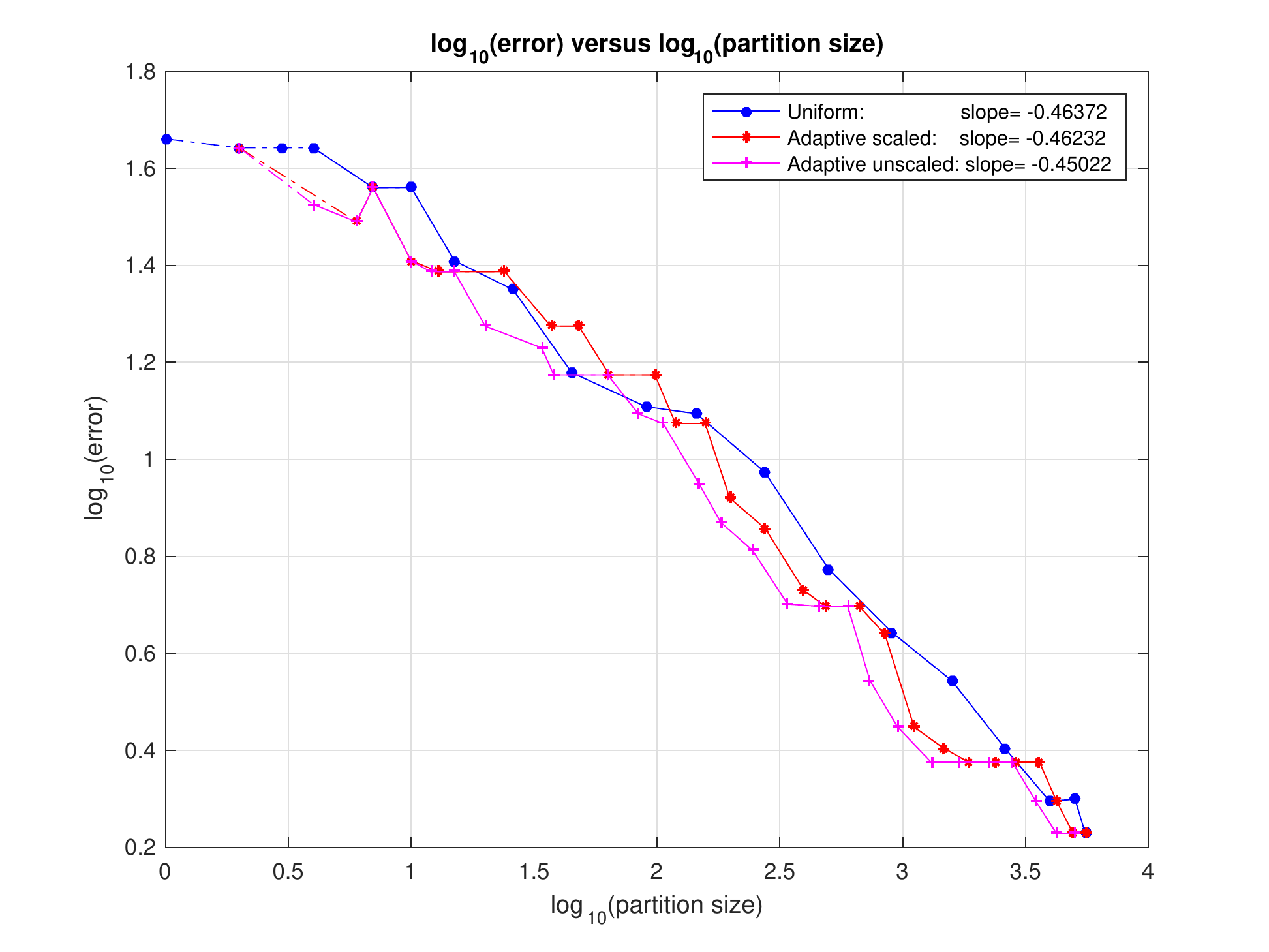}
Dragon
\\
   \includegraphics[clip,trim=10 10 20 5,width=5cm]{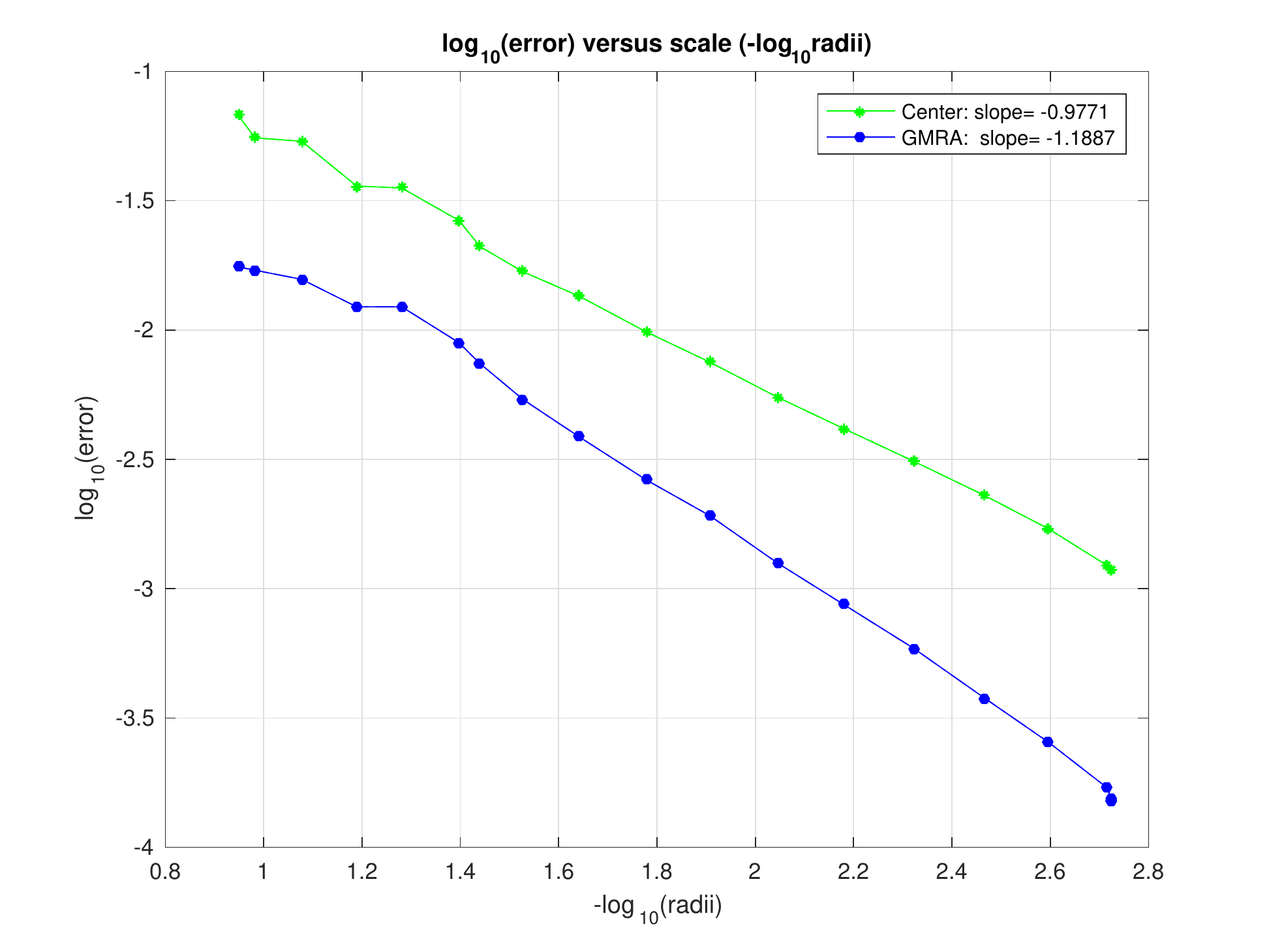}
      \hspace{-0.5cm}
      \includegraphics[clip,trim=10 10 20 5,width=5cm]{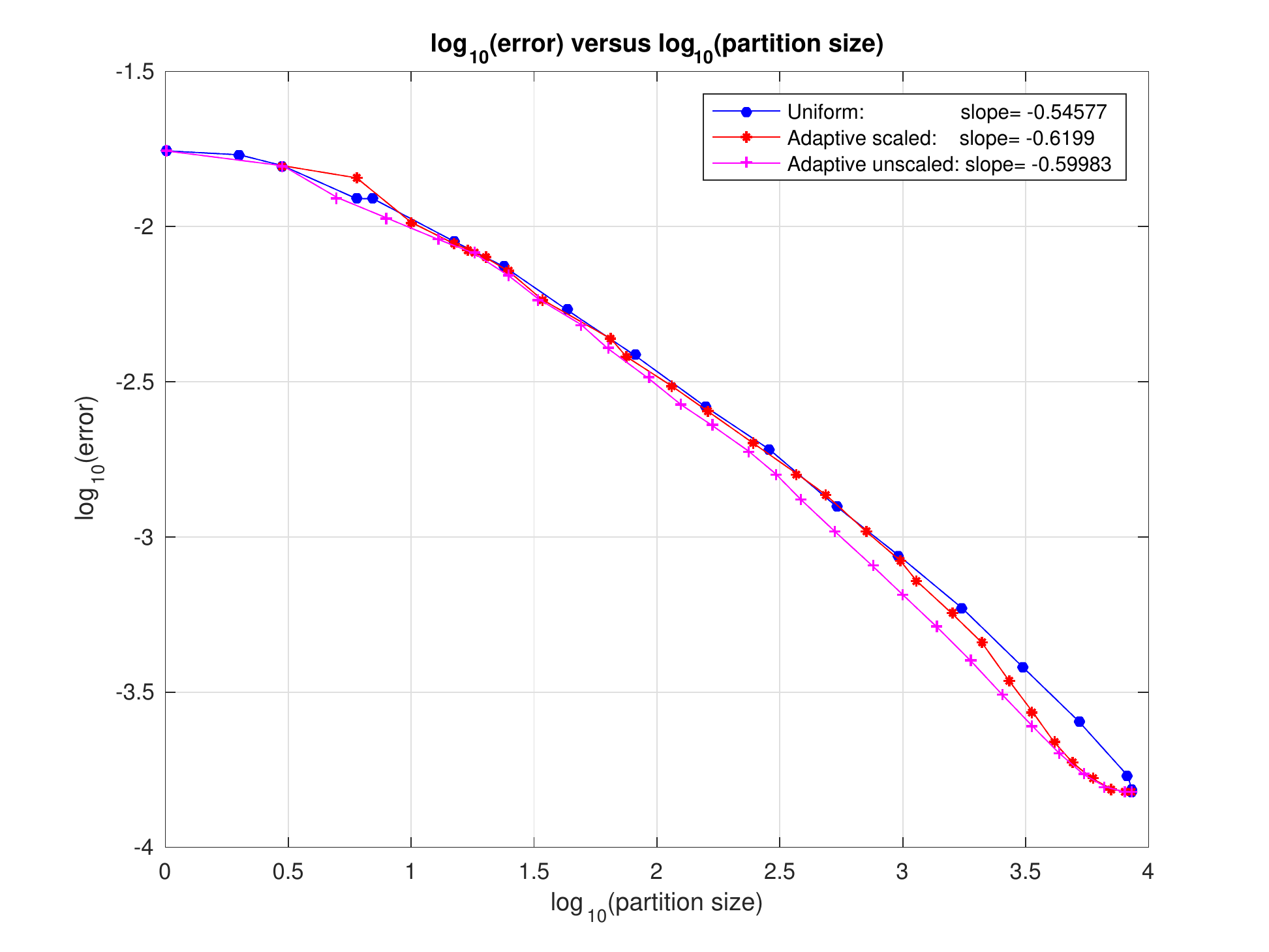}
      \hspace{-0.5cm}
   \includegraphics[clip,trim=10 10 20 5,width=5cm]{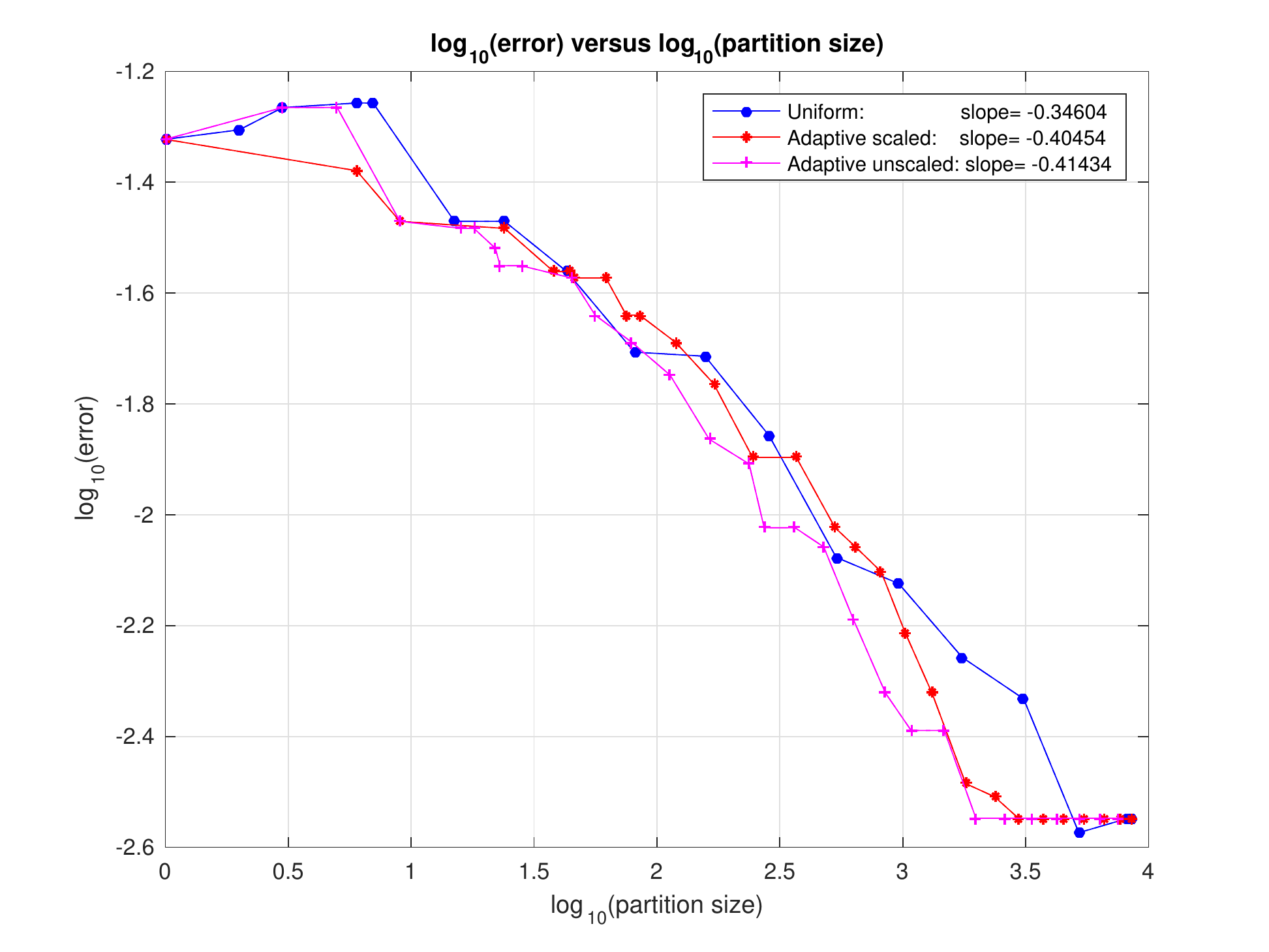}
       \caption{Left column: $\log_{10}(L^2$ error) versus scale
       for GMRA and center approximation; Middle column: log-log plot of the $L^2$ error versus partition size for GMRA and adaptive GMRA with scale-dependent and scale-independent threshold under the $L^2$ refinement defined in Table \ref{TableRefinement}; Right column: log-log plot of $L^\infty$ error versus partition size for GMRA and adaptive GMRA with scale-dependent and scale-independent threshold under the $L^\infty$ refinement.}
    \label{FigShape}
\end{figure}

\begin{figure}[thp]
 \centering
 \includegraphics[width=13cm,height=1.2cm]{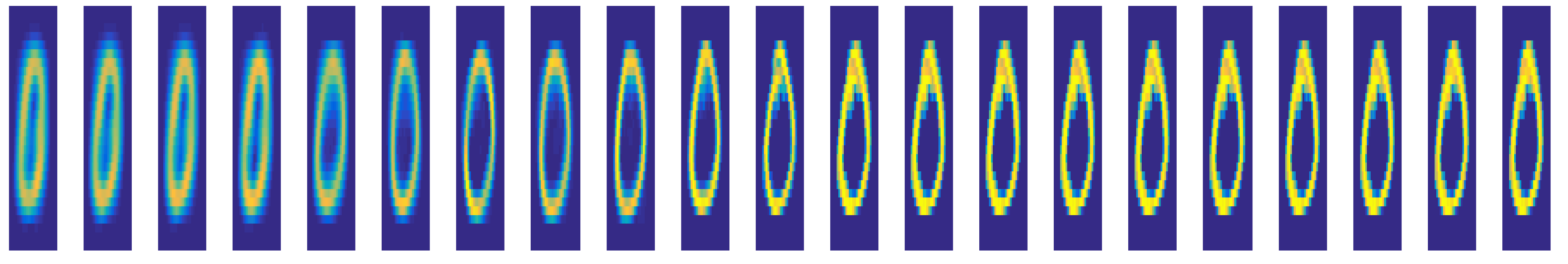}
\\
\includegraphics[width=13cm,height=1.2cm]{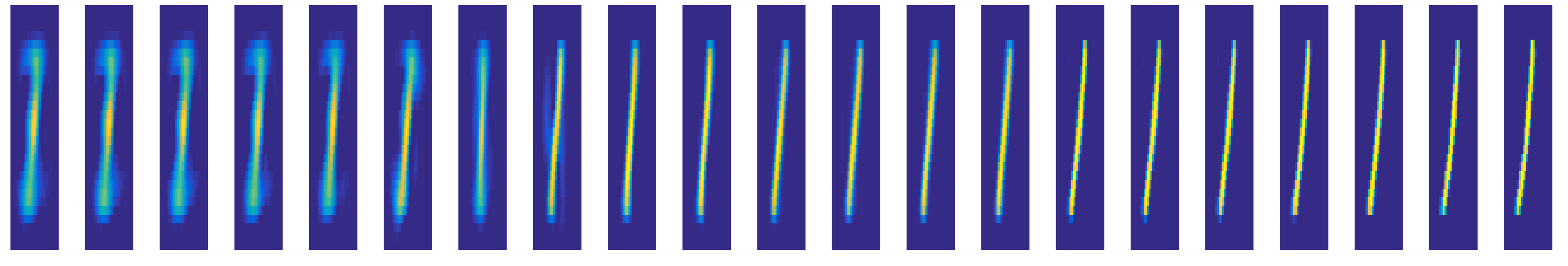}
\\
\includegraphics[width=13cm,,height=1.2cm]{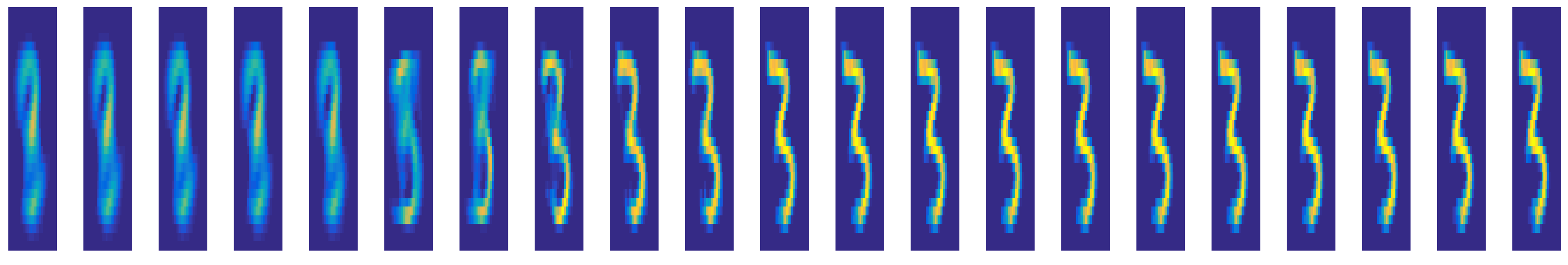}
 \subfigure[Scaling subspace]{ \includegraphics[clip,trim=10 10 20 5,width=6cm]{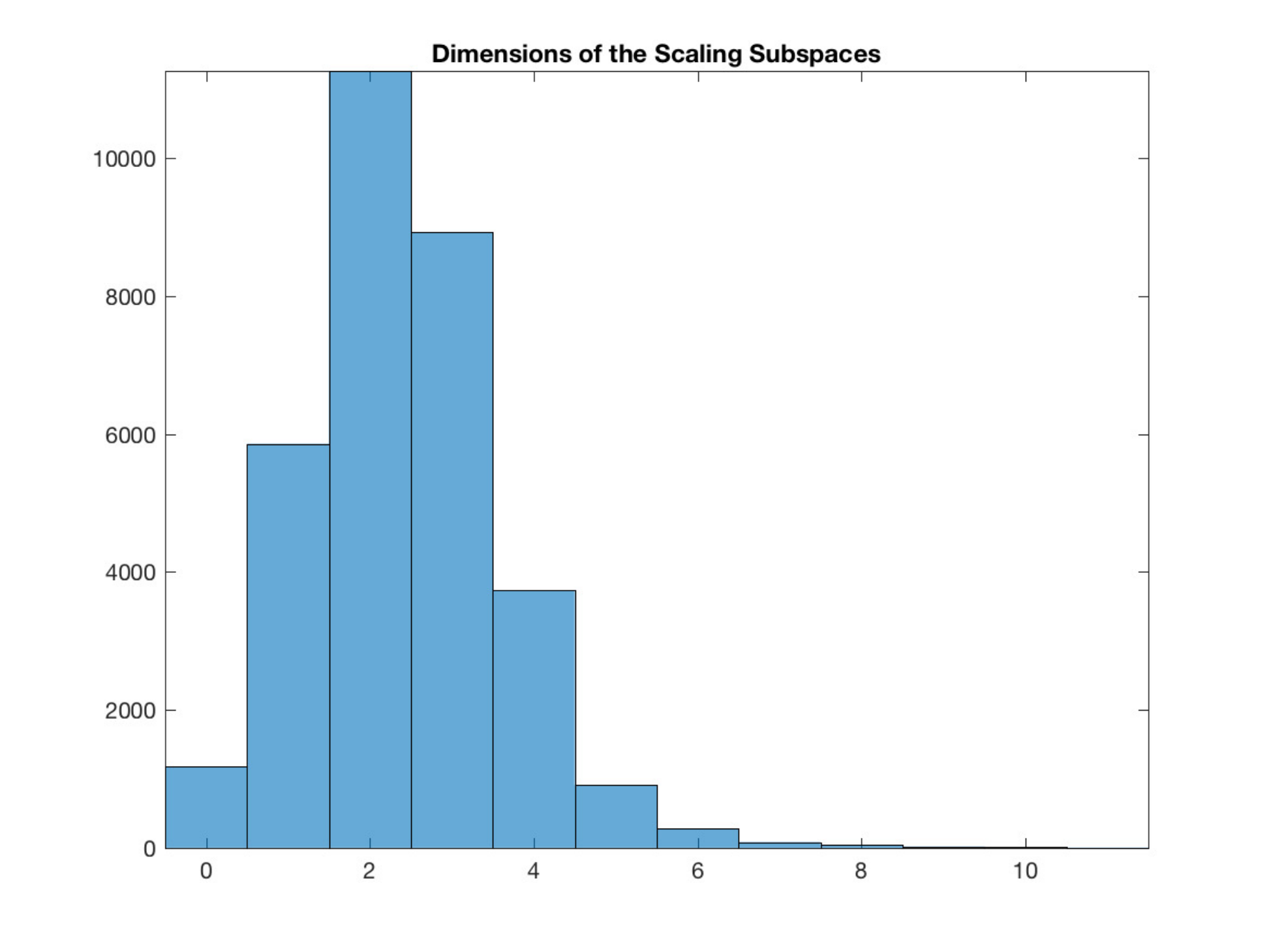}}    \hspace{-0.5cm}
    \subfigure[$\log_{10} \|\hcalP_{j+1} x_i - \hcalP_j x_i\|$]{
    \includegraphics[clip,trim=10 10 20 5,width=6cm]{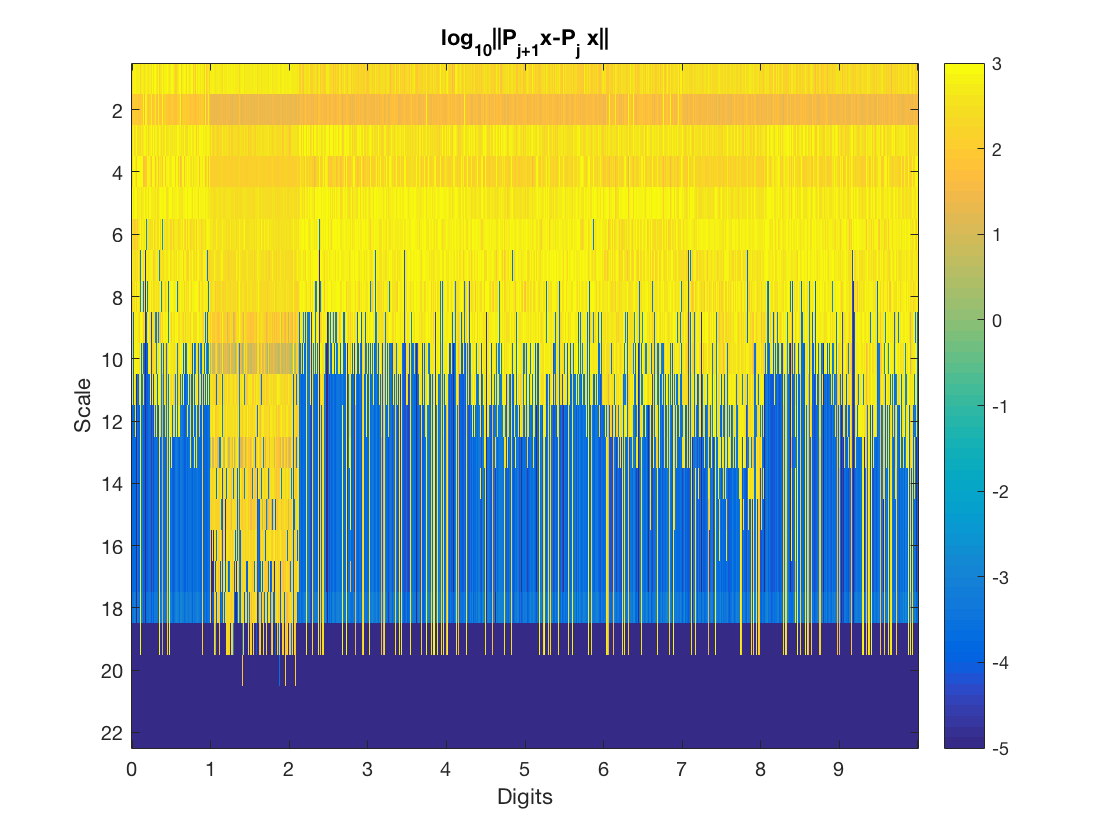}
    }
\\
        \subfigure[Error versus scale]{
     \includegraphics[clip,trim=10 0 20 5,width=6cm]{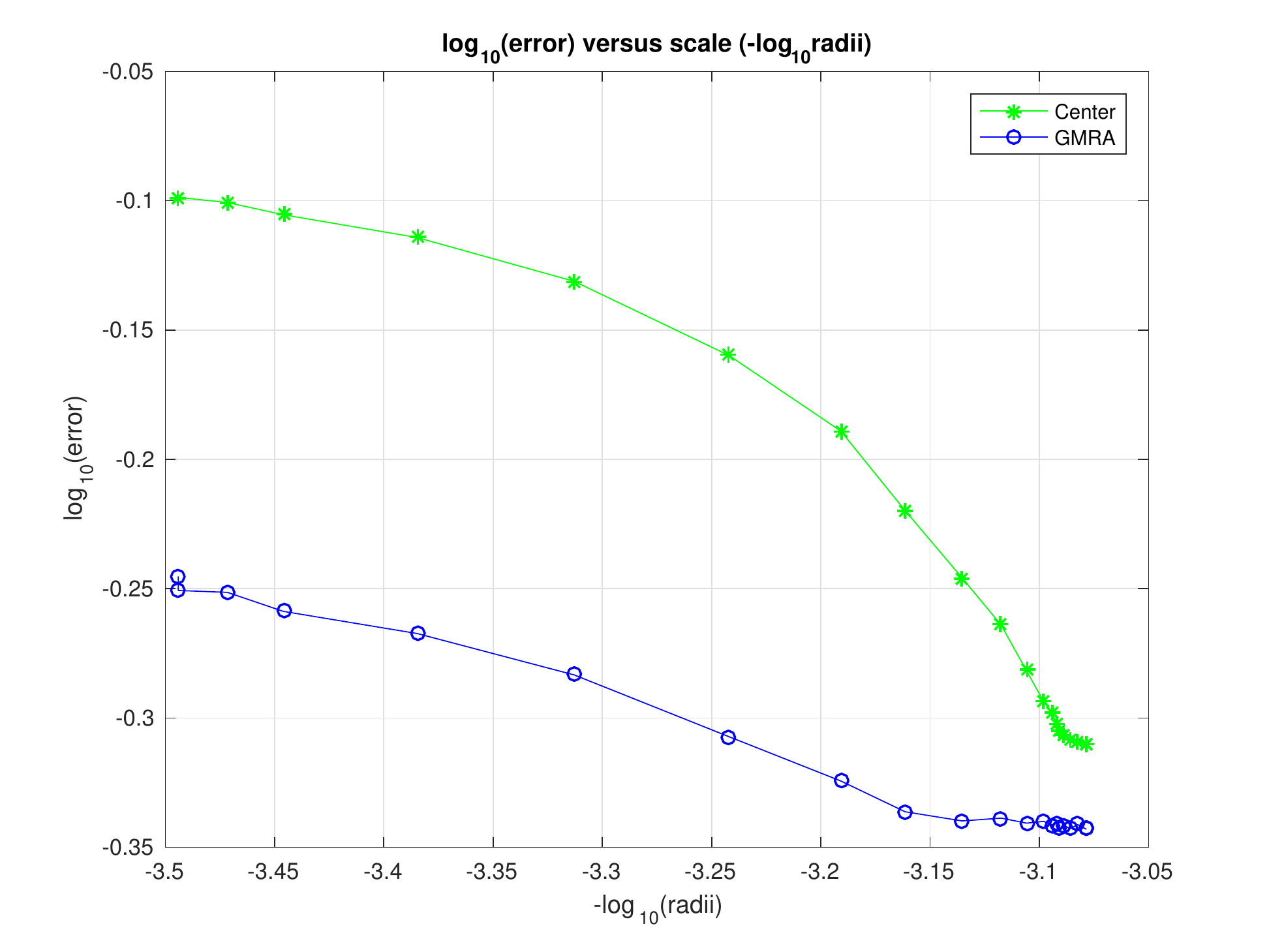}
     }
        \hspace{-0.6cm}
        \subfigure[Error versus partition size]{
     \includegraphics[clip,trim=10 0 20 5,width=6cm]{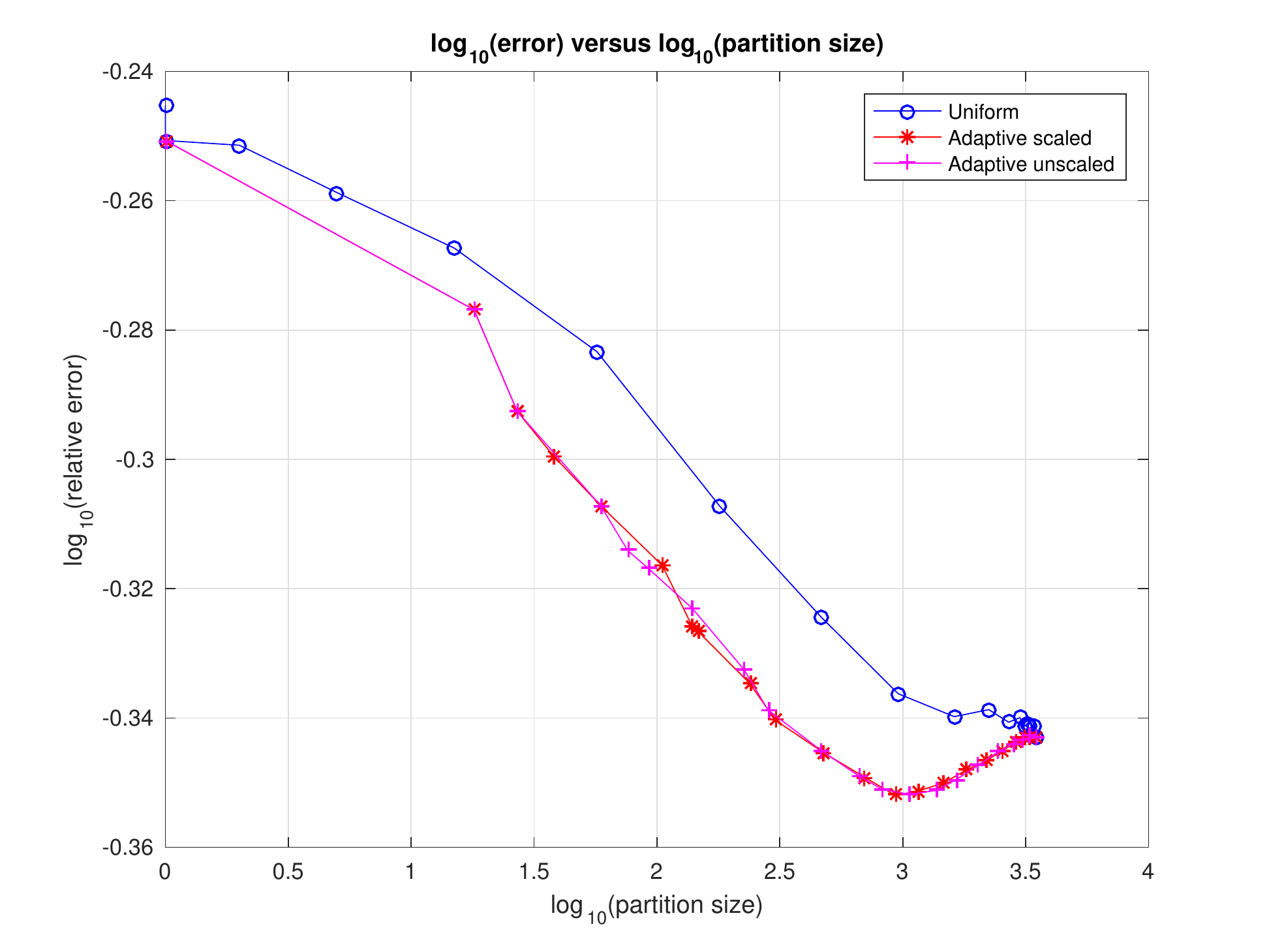}
     }
\\
               \caption{The top three rows: multiscale approximations of the digit $0,1,2$ in the MNIST data set, from the coarsest scale (left) to the finest scale (right).              (a) the histogram of dimensions of the subspaces $\hVjk$;              
                 (b) $\log_{10} \|\hcalP_{j+1} x_i - \hcalP_j x_i\|$ from the coarsest scale (top) to the finest scale (bottom), with columns indexed by the digits, sorted from $0$ to $9$;                 
                 (c) log-log plot of the relative $L^2$ error versus scale in GMRA and the center approximation;            
                    (d) log-log plot of the relative $L^2$ error versus partition size for GMRA, adaptive GMRA with scale-dependent and scale-independent threshold.              }
  \label{FigMNIST}
\end{figure}

\subsection{MNIST digit data}

We consider the MNIST data set from \url{http://yann.lecun.com/exdb/mnist/}, which contains images of $60,000$ handwritten digits, each of size $28 \times 28$,
grayscale. 
The intrinsic dimension of this data set varies for different digits and across scales, as it was observed in \cite{LMR:MGM1}. We run GMRA by setting the diameter of cells at scale $j$ to be $\mathcal{O}(0.9^j)$ in order to slowly zoom into the data at multiple scales. 

We evenly split the digits to the training set and the test set.
As the intrinsic dimension
is not well-defined, we set GMRA to pick the dimension of $\hVjk$ adaptively, as the smallest dimension needed to capture $50\%$ of the energy of the data in $\Cjk$. 
As an example, we display the GMRA approximations of the digit $0,1,2$ from coarse scales to fine scales in Figure \ref{FigMNIST}.
The histogram 
of the dimensions of the subspaces $\hVjk$ is displayed in (a). (b) represents $\log_{10} \|\hcalP_{j+1} x_i - \hcalP_j x_i\|$
from the coarsest scale (top) to the finest scale (bottom), with columns indexed by the digits, sorted from $0$ to $9$. We observe that $1$ has more fine scale information than the other digits.
In (c), we display the log-log plot of the relative $L^2$ error versus scale in GMRA and the center approximation. The improvement of GMRA over center approximation is noticeable. 
Then we compute the relative $L^2$ error for GMRA and adaptive GMRA when the partition size varies. Figure \ref{FigMNIST} (d) shows that adaptive GMRA achieves the same accuracy as GMRA with fewer cells in the partition. Errors increase when the partition size exceeds $10^3$ due to a large variance at fine scales.
In this experiment, scale-dependent threshold and scale-independent threshold yield similar performances.


\begin{figure}[th]
 \centering
 \commentout{
  \includegraphics[clip,trim=10 10 20 5,width=2cm,height=2cm]{Fig_Caltech101/Patches/h1.png}
  \hspace{-0.2cm}
    \includegraphics[clip,trim=10 10 20 5,width=2cm,height=2cm]{Fig_Caltech101/Patches/h5.png}
      \hspace{-0.2cm}
        \includegraphics[clip,trim=10 10 20 5,width=2cm,height=2cm]{Fig_Caltech101/Patches/v4.png}
          \hspace{-0.2cm}
                \includegraphics[clip,trim=10 10 20 5,width=2cm,height=2cm]{Fig_Caltech101/Patches/v5.png}
                  \hspace{-0.2cm}
            \includegraphics[clip,trim=10 10 20 5,width=2cm,height=2cm]{Fig_Caltech101/Patches/v7.png}
               \hspace{-0.2cm}
     \includegraphics[clip,trim=10 10 20 5,width=2cm,height=2cm]{Fig_Caltech101/Patches/d4.png}
       \hspace{-0.2cm}
            \includegraphics[clip,trim=10 10 20 5,width=2cm,height=2cm]{Fig_Caltech101/Patches/d5.png}
              \\
                    \includegraphics[clip,trim=10 10 20 5,width=2cm,height=2cm]{Fig_Caltech101/Patches/sd3.png}
        \hspace{-0.2cm}
         \includegraphics[clip,trim=10 10 20 5,width=2cm,height=2cm]{Fig_Caltech101/Patches/sd10.png}
           \hspace{-0.2cm}
                \includegraphics[clip,trim=10 10 20 5,width=2cm,height=2cm]{Fig_Caltech101/Patches/sd7.png}
                  \hspace{-0.2cm}
    \includegraphics[clip,trim=10 10 20 5,width=2cm,height=2cm]{Fig_Caltech101/Patches/c4.png}
      \hspace{-0.2cm}
            \includegraphics[clip,trim=10 10 20 5,width=2cm,height=2cm]{Fig_Caltech101/Patches/dsd1.png}
              \hspace{-0.2cm}
        \includegraphics[clip,trim=10 10 20 5,width=2cm,height=2cm]{Fig_Caltech101/Patches/b1.png}
          \hspace{-0.2cm}
                \includegraphics[clip,trim=10 10 20 5,width=2cm,height=2cm]{Fig_Caltech101/Patches/v2.png}
                  \hspace{-0.2cm}
 \\
 }
 Learning image patches
 \\
  \subfigure[Scaling subspace]{
    \includegraphics[clip,trim=10 10 20 5,width=5cm]{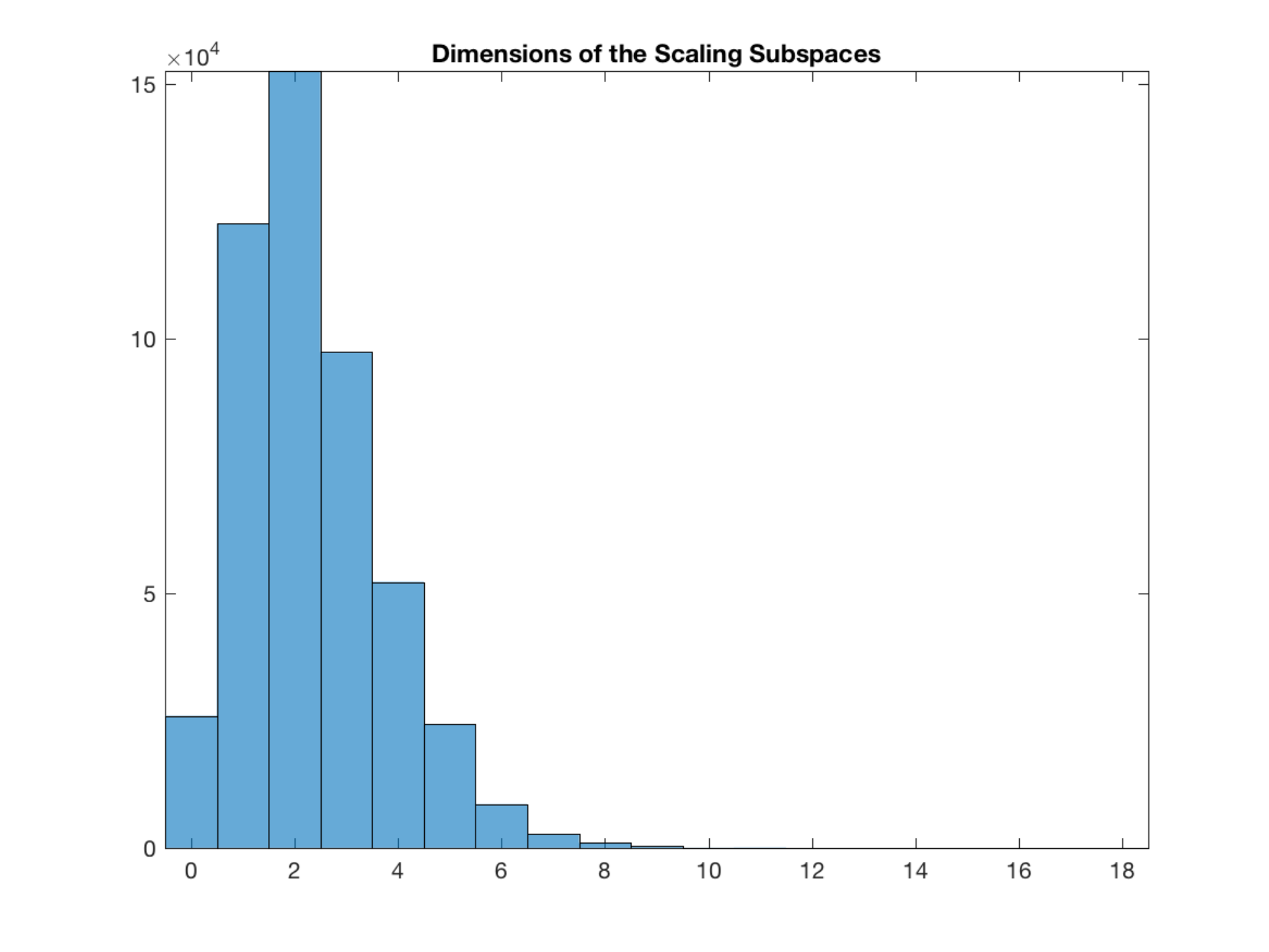}
    }
      \hspace{-0.6cm}
      \subfigure[Error versus scale]{
     \includegraphics[clip,trim=10 10 20 5,width=5cm]{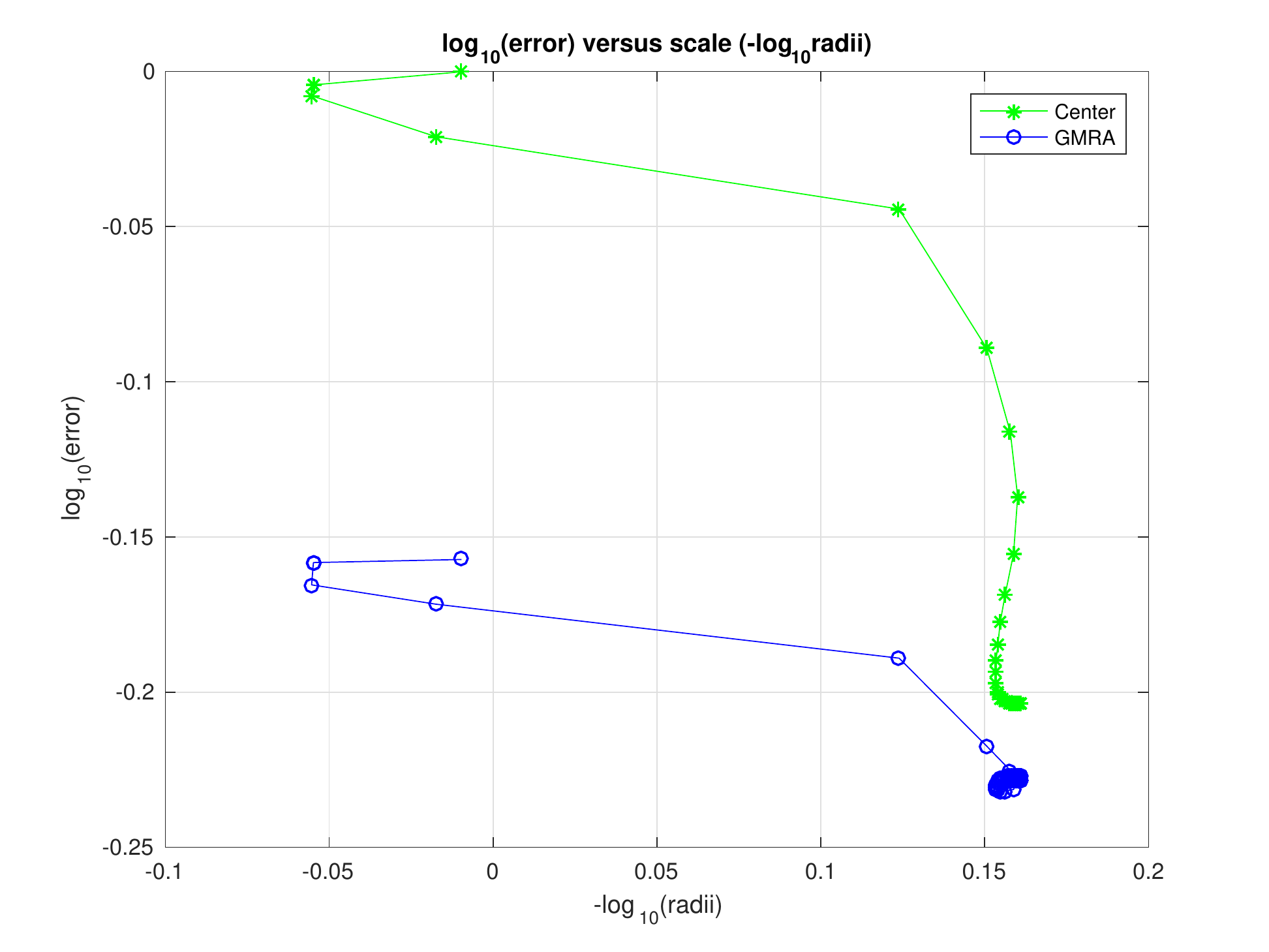} 
     }
        \hspace{-0.6cm}
        \subfigure[Error versus partition size]{
     \includegraphics[clip,trim=10 10 20 5,width=5cm]{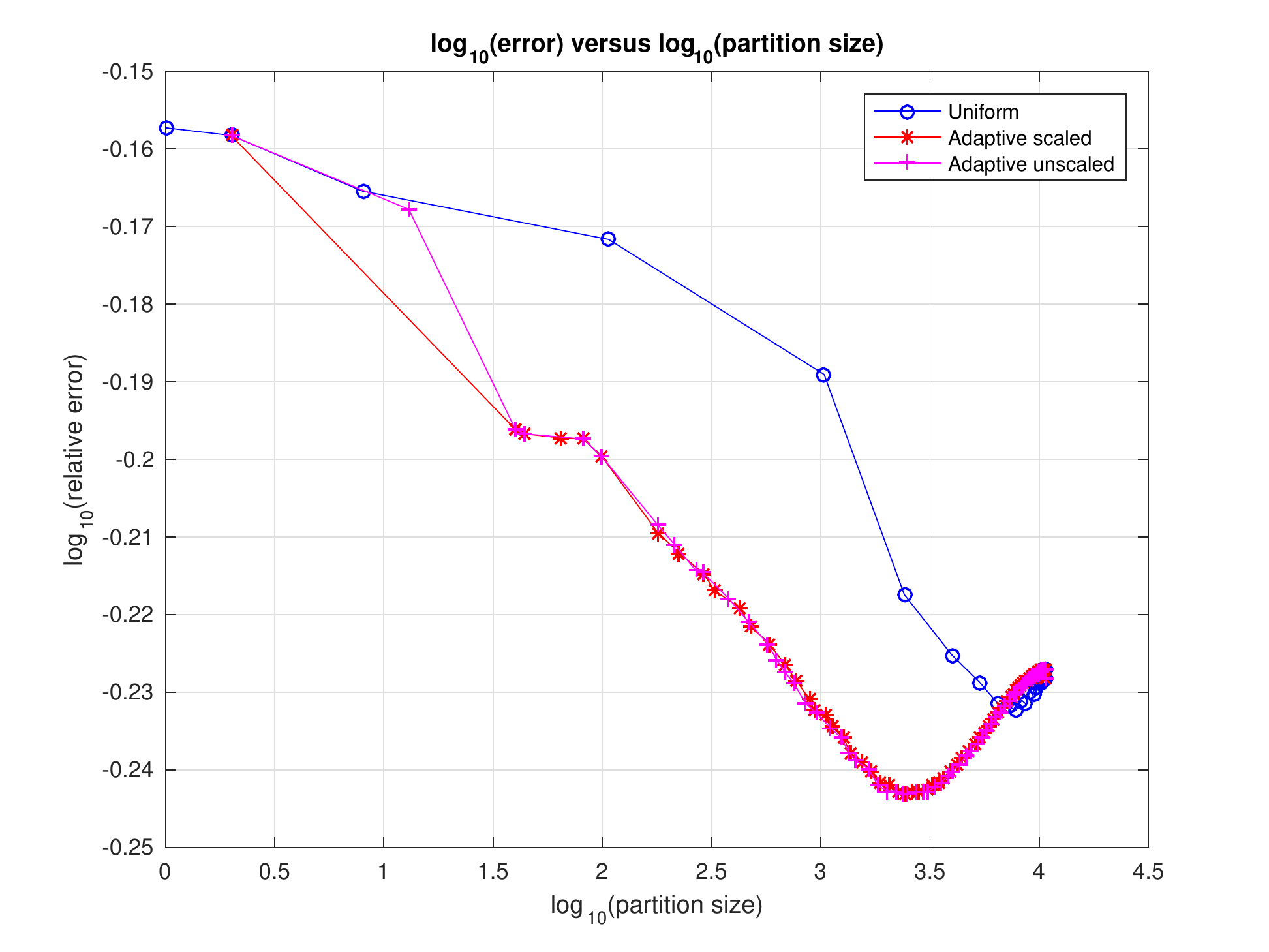} 
     }
     \\
      Learning the Fourier magnitudes of image patches
      \\
 \subfigure[Scaling subspace]{
    \includegraphics[clip,trim=10 10 20 5,width=5cm]{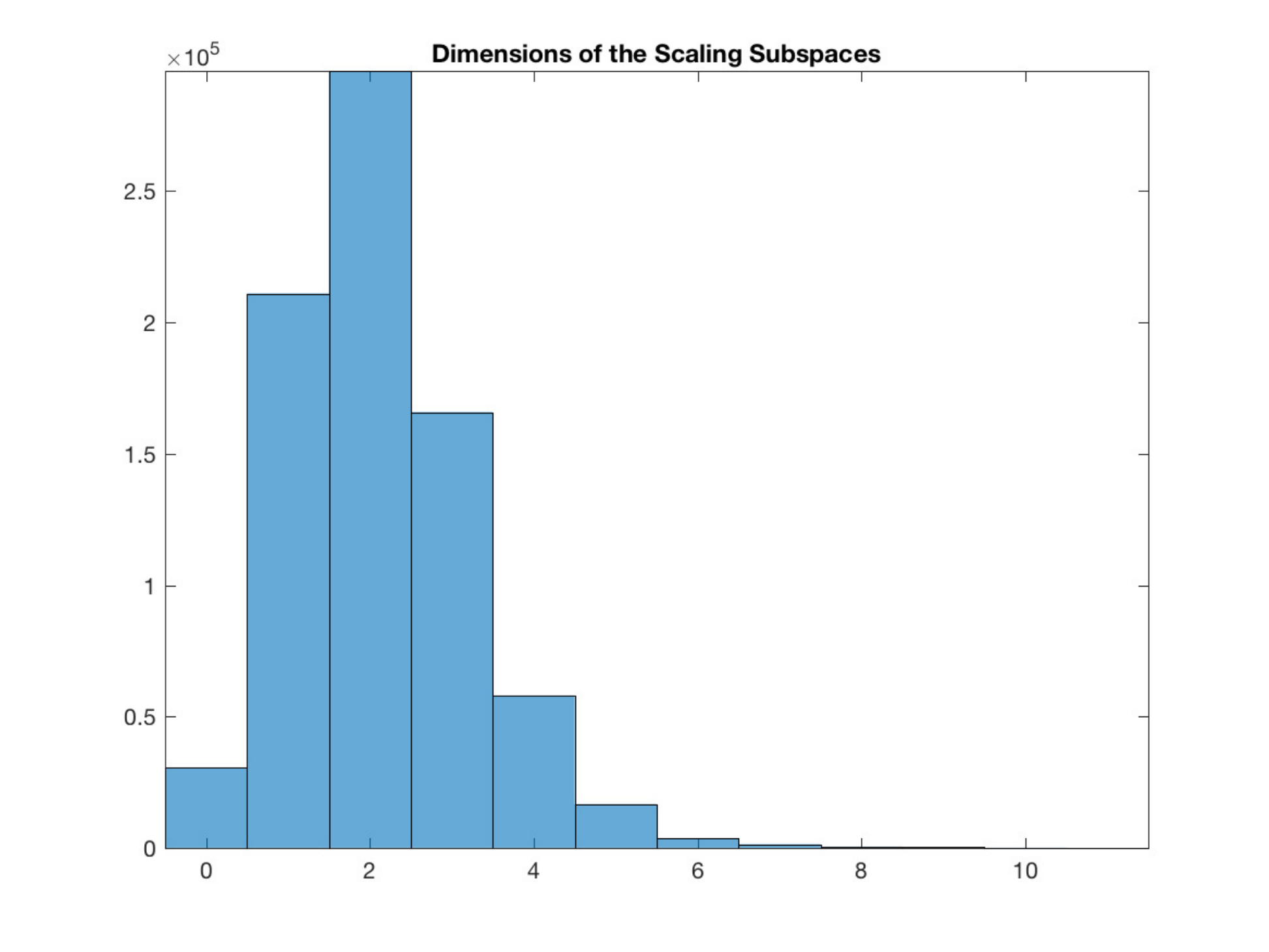}
    }
      \hspace{-0.6cm}
      \subfigure[Error versus scale]{
     \includegraphics[clip,trim=10 10 20 5,width=5cm]{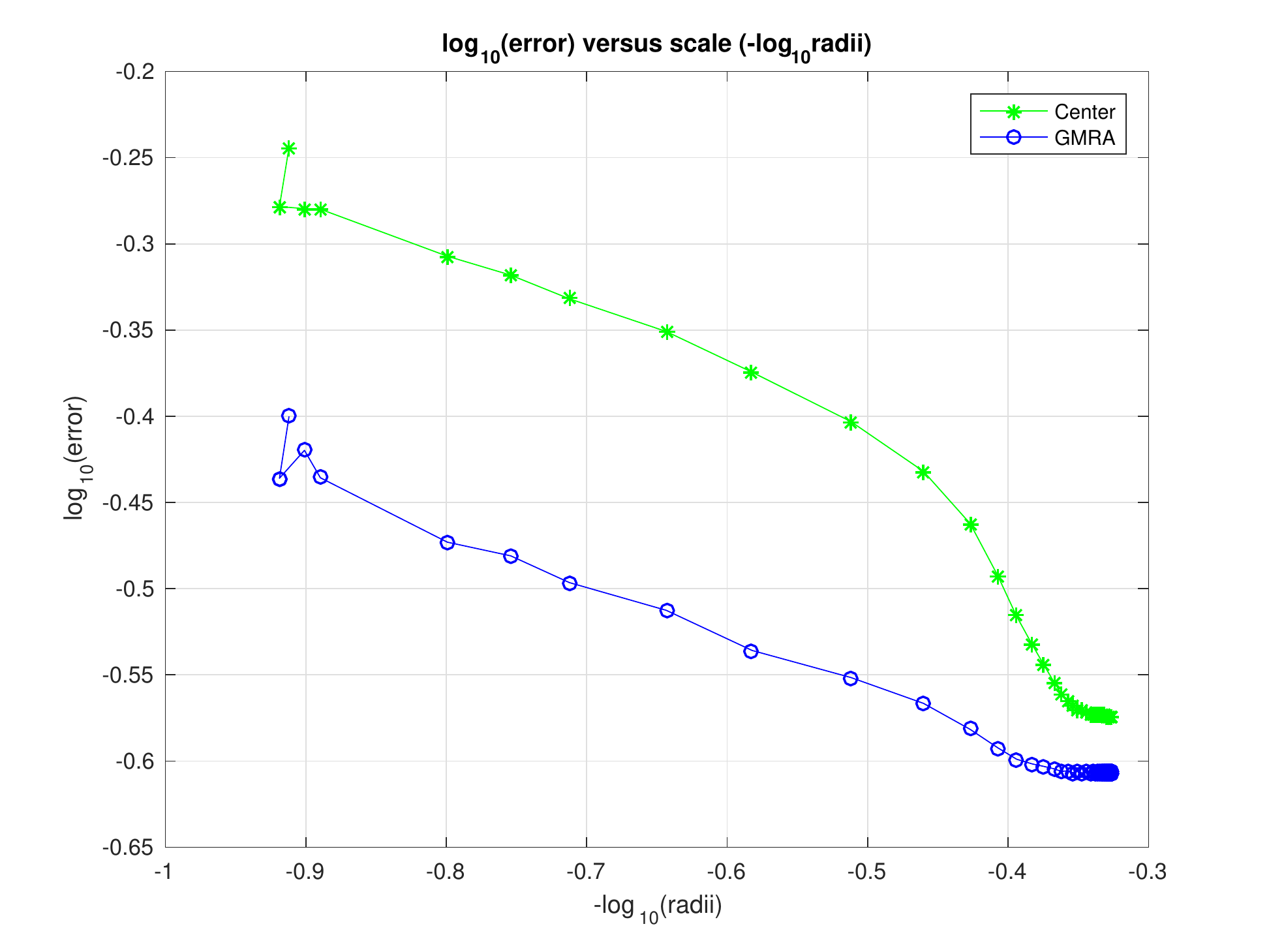} 
     }
        \hspace{-0.6cm}
        \subfigure[Error versus partition size]{
     \includegraphics[clip,trim=10 10 20 5,width=5cm]{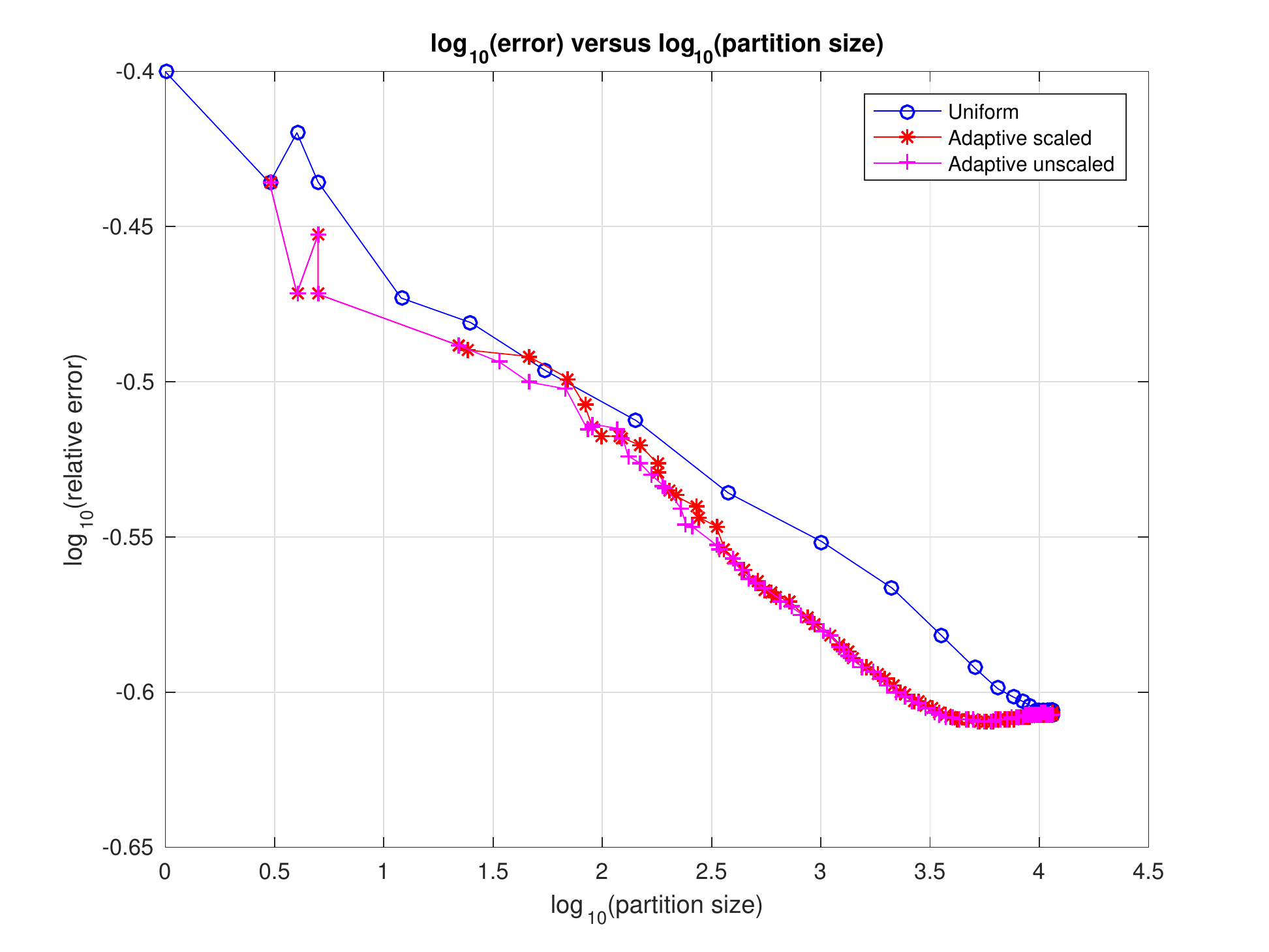} 
     }
     \caption{Caltech 101 image patches}
              \caption{             
             Top line: learning on $200,000$ image patches; bottom line: results of learning the Fourier magnitudes of the same image patches. 
             (a,d) histograms of the dimensions of the subspaces $\hVjk$;               
              (b,e) relative $L^2$ error versus scale for GMRA and the center approximation;             
              (c,f) relative $L^2$ error versus the partition size for GMRA, adaptive GMRA with scale-dependent and scale-independent threshold.                    }    
  \label{FigCaltech101}
\end{figure}

\subsection{Natural image patches}
It was argued in \cite{Peyre} that many sets of patches extracted from natural images can be modeled a low-dimensional manifold. We use the Caltech 101 dataset from \url{https://www.vision.caltech.edu/Image_Datasets/Caltech101/} \citep[see][]{FFLCaltech101}, take 40 images from four categories: accordion, airplanes, hedgehog and scissors and extract multiscale patches of size $8\times 8$ 
from these images. Specifically, if the image is of size $m \times m$, for $\ell = 1,\ldots,\log_2(m/8)$, we collect patches of size $2^\ell 8$, low-pass filter them and downsample them to become patches of size $8 \times 8$ (see \cite{GM_MultiscaleDictionariesSPIE} for a discussion about dictionary learning on patches of multiple sizes using multiscale ideas). 
Then we randomly pick $200,000$ patches, evenly split them to the training set and the test set.
In the construction of GMRA, we set  the diameter of cells at scale $j$ to be $\mathcal{O}(0.9^j)$ and the dimension of $\hVjk$ to be the smallest dimension needed to capture $50\%$ of the energy of the data in $\Cjk$. 
We also run GMRA and adaptive GMRA on the Fourier magnitudes of these image patches to take advantage of translation-invariance of the Fourier magnitudes. The results are shown in Figure \ref{FigCaltech101}.
The histograms 
of the dimensions of the subspaces $\hVjk$ are displayed in (a,e).
%
Figure \ref{FigCaltech101} (c) and (g) show the relative $L^2$ error versus 
scale for GMRA and the center approximation.
We then compute the relative $L^2$ error for GMRA and adaptive GMRA when the partition size varies and display the log-log plot in (d) and (h). It is noticeable that adaptive GMRA achieves the same accuracy as GMRA with a smaller partition size.
We conducted similar experiments on $200,000$ multiscale patches from CIFAR 10 from \url{https://www.cs.toronto.edu/~kriz/cifar.html} \citep[see][]{CIFAR10} with extremely similar results (not shown).

\section{Performance analysis of GMRA and adaptive GMRA}
\label{secGMRA}

This section is devoted to the performance analysis of empirical GMRA and adaptive GMRA. We will start with the following stochastic error estimate on any  partition.

\subsection{Stochastic error on a fixed partition}
Suppose $\tcalT$ is a finite proper subtree of the data master tree $\calTn$.
Let $\Lam$ be the partition consisting the outer leaves of $\tcalT$. 
The piecewise affine projector on $\Lam$ and its empirical version are 
$$\calP_\Lam = \sum_{\Cjk \in \Lam} \calPjk \mathbf{1}_{j,k}
\quad \text{and}
\quad
\hcalP_\Lam = \sum_{\Cjk \in \Lam} \hcalPjk \mathbf{1}_{j,k}.$$
A non-asymptotic probability bound on the stochastic error $\|\calP_\Lam X -\hcalP_\Lam X\|$ is given by:

\begin{lemma}
\label{thm1}
Let $\Lam$ be the partition associated a finite proper subtree $\tcalT$ of the data master tree $\calTn$. 
Suppose $\Lam$ contains $\#_j \Lam $ cells at scale $j$. 
Then for any $\eta > 0$,
\beq
\label{thm1eq}
\PP\{\|\calP_\Lam X - \hcalP_\Lam X\| \ge \eta \}
 \le 
\alpha d 
\cdot \# \Lam \cdot  e^{-
\frac{\beta  n \eta^2}{d^2\sum_{j} 2^{-2j} \#_j \Lam }}
\eeq
$$\EE \|\calP_\Lam X - \hcalP_\Lam X\|^2
\le 
\frac{d^2  \log ( \alpha d \#\Lam) \sum_{j} 2^{-2j}\#_j \Lam}{\beta n }$$
where $\alpha=\alpha(\theta_2,\theta_3)$
and $\beta=\beta(\theta_2,\theta_3,\theta_4)$.
\end{lemma}
Lemma \ref{thm1} and Proposition \ref{lemma0} below are proved in appendix \ref{app1} .

\subsection{Performance analysis of empirical GMRA}
\label{s:gmra}

According to Eq. \eqref{bv1}, the approximation error of empirical GMRA is split into the squared bias and the variance. A corollary of Lemma \ref{thm1} with $\Lam = \Lam_j$ results in an estimate of the variance term. 


\begin{proposition}
\label{lemma0}
For any $\eta \ge 0$,
\begin{eqnarray}
\PP\{\|\calP_j X - \hcalP_j X\| \ge \eta \}
& \le & 
\alpha d
  \#\Lam_j  e^{-
\frac{\beta 2^{2j} n \eta^2}{d^2 \# \Lam_j}}
\label{lemma0eq1}
\\
\EE \|\calP_j X - \hcalP_j X\|^2 
& \le &
\frac{d^2 \# \Lam_j \log[\alpha d \# \Lam_j ]}{\beta 2^{2j}n}\,.
\label{lemma0eq2}
\end{eqnarray}
\end{proposition}

In Eq. \eqref{bv1}, the squared bias decays like $\mathcal{O}(2^{-2js})$ whenever $\rho \in \AS$ and the variance scales like  $\mathcal{O}(j 2^{j(d-2)}/n)$. A proper choice of the scale $j$ gives rise to Theorem \ref{thm2} whose proof is given below.

\subsubsection*{Proof of Theorem \ref{thm2}}

\begin{proof}[Proof of Theorem \ref{thm2}]
\begin{align*}
&\EE \|X-\hcalP_j X\|^2
\le 
 \|X-\calP_j X\|^2
+
\EE \|\calP_j X-\hcalP_j X\|^2
\\
& 
\le |\rho|_{\calA_s}^2 2^{-2sj} + 
\frac{ d^2 \#\Lam_j \log[\alpha d  \#\Lam_j]}{\beta 2^{2j}n}
\le 
|\rho|_{\calA_s}^2 2^{-2sj} + 
\frac{ d^2 2^{j(d-2)} }{\theta_1 \beta n}\log\frac{\alpha   d 2^{jd}}{\theta_1}
\end{align*}
as $\#\Lam_j \le 2^{jd}/\theta_1$ due to Assumption (A3).

\noindent{\bf{Intrinsic dimension $d=1$:}} In this case, both the squared bias and the variance decrease as $j$ increases, so we should choose the scale $j^*$ as large as possible as long as most cells at scale $j^*$ have $d$ points. We will choose $j^*$ such that $2^{-j^*}  = \mu \lognn$ for some $\mu >0$. After grouping $\Lam_{j^*}$ into light and heavy cells whose measure is below or above $\frac{28(\nu+1)\log n}{3n}$, we can show that the error on light cells is upper bounded by $C(\lognn)^{2}$ and all heavy cells have at least $d$ points with high probability.


\begin{lemma}
\label{lemmathm2_1}
Suppose $j^*$ is chosen such that $2^{-j^* }  = \mu \lognn$ with some $\mu>0$.
Then 
\begin{align*}
&\|(X-\calP_{\jstar} X)\mathbf{1}_{\{C_{\jstar,k}: \rho(C_{\jstar,k}) \le \frac{28(\nu+1)\log n}{3n}\}}\|^2 
\le
 \frac{28(\nu+1) \theta_2^2 \mu}{3\theta_1} \left( \lognn\right)^{2}, 
\\
&\PP\left\{ \text{each } C_{j^*,k} \text{ satisfying } \rho(C_{\jstar,k}) >  \tfrac{28(\nu+1)\log n}{3n} \text{ has at least $d$ points} \right\} \ge 1- n^{-\nu}.
\end{align*}

\end{lemma}
Lemma \ref{lemmathm2_1} is proved in appendix \ref{app12}.
If $j^*$ is chosen as above,
The probability estimate in \eqref{thm2eq01} follows from
$$\|X-\calP_{j^*} X\| \le |\rho|_{\calA_s}2^{-sj^*} \le 
|\rho|_{\calA_s} \mu^s\left(\lognn \right)^{s  } \le |\rho|_{\calA_s} \mu^{s}\lognn,$$
\begin{align*}
&\PP\left\{\|\calP_{j^*} X - \hcalP_{j^*} X \| 
\ge C_1 \frac{\log n}{n} 
\right\}
\le \tfrac{\alpha  }{\theta_1\mu} \left( (\log n)/n\right)^{-1} e^{-\frac{{\mu\beta\theta_1
C_1^2}\log n}{d^2}}
\le \tfrac{\alpha  }{\theta_1\mu}\frac{n\,n^{-\frac{\mu\beta\theta_1
C_1^2}{d^2}} }{\log n} 
\le C_2  n^{-\nu}\,
\end{align*}
provided that ${\mu\beta\theta_1 C_1^2/d^2}-1 > \nu$.

\noindent{\bf{Intrinsic dimension $d\ge 2$:}} When $d\ge 2$, the squared bias decreases but the variance increases as $j$ gets large.
We choose $j^*$ such that $2^{-j^*}  = \mu \left( (\log n)/n\right)^{\frac{1}{2s+d-2}}$ to balance these two terms. 
We use the same technique as $d=1$ to group $\Lam_{j^*}$ into light and heavy cells whose measure is below or above $28/3\cdot(\nu+1)(\log n)/n$, we can show that the error on light cells is upper bounded by $C((\log n)/n)^{\frac{2s}{2s+d-2}}$ and all heavy cells have at least $d$ points with high probability.

\begin{lemma}
\label{lemmathm2_2}
Let $j^*$ be chosen such that $2^{-j^*}  = \mu \left( (\log n)/n\right)^{\frac{1}{2s+d-2}}$ with some $\mu>0$.
Then 
\begin{align*}
&\|(X-\calP_{\jstar} X)\mathbf{1}_{\{C_{\jstar,k}: \rho(C_{\jstar,k}) 
\le \frac{28(\nu+1)\log n}{3n}\}}\|^2 
\le
 \frac{28(\nu+1) \theta_2^2 \mu^{2-d}}{3\theta_1} \left( \lognn\right)^{\frac{2s}{2s+d-2}}, 
\\
&\PP\left\{ \forall\, C_{j^*,k} \,:\, \rho(C_{\jstar,k}) >  \frac{28(\nu+1)\log n}{3n},\,\,  C_{j^*,k} \text{ has at least $d$ points} \right\} \ge 1- n^{-\nu}.
\end{align*}

\end{lemma}
Proof of Lemma \ref{lemmathm2_2} is omitted since it is the same as the proof of Lemma \ref{lemmathm2_1}.
The probability estimate in \eqref{thm2eq1} follows from
$$\|X-\calP_{j^*} X\| \le |\rho|_{\calA_s}2^{-s{j^*}} \le 
|\rho|_{\calA_s} \mu^s \left(\lognn \right)^{\frac{s}{2s+d-2}},$$
\begin{align*}
&\PP\left\{\|\calP_{j^*} X - \hcalP_{j^*} X \| 
\ge C_1 \left (\frac{\log n}{n} \right)^{\frac{s}{2s+d-2}}
\right\}
\le \frac{\alpha d \mu^{-d} }{\theta_1} \left( \lognn\right)^{-\frac{d}{2s+d-2}} e^{-\frac{\beta\theta_1
C_1^2\mu^{d-2}\log n}{d^2}}
\le C_2  n^{-\nu}
\end{align*}
provided that ${\beta\theta_1 C_1^2 \mu^{d-2}}/d^2-1 > \nu$.


\end{proof}

\subsection{Performance analysis of Adaptive GMRA}
\label{s:adaptivegmra}


\begin{proof}[Proof of Theorem \ref{thm3}] In the case that $\calM$ is bounded by $M$, the minimum scale $\jmin = \log_2\frac{\theta_2}{M}$.
We first consider the case $d \ge 3$. 
In our proof $C$ stands for constants that may vary at different locations, but it is independent of $n$ and $D$. We will begin by defining several objects of interest:
\begin{itemize}
\item $\calTn$: the data master tree whose leaf contains at least $d$ points in $\calX_n$. It can be viewed as the part of a multiscale tree that our data have explored.

\item $\calT$: a complete multiscale tree containing $\calT^n$. 
$\calT$ can be viewed as the union $\calT^n$ and some empty cells, mostly at fine scales with high probability, that our data have not explored.  

\item $\calTrhoeta$: the smallest subtree of $\calT$ which contains $\{\Cjk \in \calT \,:\,\Deltajk \ge 2^{-j}\eta\}$.

\item $\calTeta = \calTrhoeta \cap \calTn$.

\item $\hcalTeta$: the smallest subtree of $\calTn$ which contains $\{\Cjk \in \calTn\,:\,\hDeltajk \ge 2^{-j}\eta\}$.

\item $\Lamrhoeta$: the partition associated with $\calTrhoeta$.

\item $\Lameta:$ the partition associated with $\calTeta$.



 
\item $\hLameta:$ the partition associated with $\hcalTeta$. 

\item Suppose $\calT^0$ and $\calT^1$ are two subtrees of $\calT$. If $\Lam^0$ and ${\Lam}^1$ are two adaptive partitions associated with $\calT^0$ and $\calT^1$ respectively, we denote by $\Lam^0 \vee \Lam^1$ and $\Lam^0 \wedge \Lam^1$ the partitions associated to the trees $\calT^0 \cup \calT^1$ and $\calT^0 \cap \calT^1$ respectively. 

\end{itemize}
We also let $b=2\amax+5$ where $\amax$ is the maximal number of children that a node has in $\calT$; 
$\kappa_0 = \max(\kappa_1,\kappa_2)$ where ${b^2\kappa_1^2}/(21\theta_2^2) =\nu + 1$ and $\alpha_2 \kappa_2^2/b^2  = \nu+1$ with $\alpha_2$ defined in Lemma \ref{lemma4}. In order the obtain the MSE bound, one can simply set $\nu = 1$.


The empirical adaptive GMRA projection is given by 
$\hcalP_{\hLamtaun} = \sum_{\Cjk \in \hLamtaun}
\hcalPjk \chijk.$
Using the triangle inequality, we split the error as follows:
$$
\| X - \hcalP_{\hLamtaun} X \| \le 
e_1 + e_2 + e_3 +e_4$$
where
\begin{align*}
e_1 & := 
\| X -\calP_{\hLamtaun\vee \Lambtaun} X\|\,,
&
e_2 &:=
\| \calP_{\hLamtaun\vee \Lambtaun} X
-
\calP_{\hLamtaun\wedge \Lamtaunb} X
\|
\\
e_3 &:=
\| 
\calP_{\hLamtaun\wedge \Lamtaunb} X
-
\hcalP_{\hLamtaun\wedge \Lamtaunb} X
\|\,,
&
e_4 &:=
\| 
\hcalP_{\hLamtaun\wedge \Lamtaunb} X
-
\hcalP_{\hLamtaun} X
\|.
\end{align*}
A similar split appears in the works of \citet{DeVore:UniversalAlgorithmsLearningTheoryI,DeVore:UniversalAlgorithmsLearningTheoryII}. The partition built from those $\Cjk$'s satisfying $\hDeltajk \ge 2^{-j}\tau_n$ does not exactly coincide with the partition chosen based on those $\Cjk$ satisfying $\Deltajk \ge 2^{-j}\tau_n$. This is accounted by $e_2$ and $e_4$, corresponding to those $\Cjk$'s whose $\hDeltajk$ is significantly larger or smaller than $\Deltajk$, which we will  prove to be small with high probability. The remaining terms $e_1$ and $e_3$ correspond to the bias and variance of the approximations on the partition obtained by thresholding $\Deltajk$.

\noindent{\bf{Term $e_1$:}} The first term $e_1$ is essentially the bias term. Since $\hLamtaun\vee \Lambtaun\supseteq\Lambtaun$,
\begin{align*}
e_1^2
&  = 
\| X -\calP_{\hLamtaun\vee \Lambtaun} X\|^2
 \le 
\| X -\calP_{ \Lambtaun} X\|^2
\le 
\underbrace{\| X -\calP_{ \Lamrhobtaun} X\|^2}_{e_{11}^2}
+
\underbrace{\|\calP_{ \Lamrhobtaun} X -\calP_{ \Lambtaun} X\|^2}_{e_{12}^2}\,.
\end{align*}
$e_{11}^2$ may be upper bounded deterministically from Eq. \eqref{Bs1}:
\beq
e_{11}^2
\le 
B_{s,d}|\rho|_{\BS}^p(b\tau_n)^{2-p}
\le 
B_{s,d} |\rho|_{\BS}^{\frac{2(d-2)}{2s+d-2}} (b\kappa)^{\frac{4s}{2s+d-2}} \left(\lognn\right)^{\frac{2s}{2s+d-2}}.
\label{thm3eq2}
\eeq
$e_{12}$ encodes the difference between thresholding $\calT$ and $\calTn$, but it is $0$ with high probability:

\begin{lemma}
\label{thm3_lemma1}
For any $\nu>0$, $\kappa$ such that $\kappa > \kappa_1$, where $b^2\kappa_1^2/(21\theta_2^2) = \nu+1$, 
\beq
\label{thm3eq21}
\PP\{e_{12}>0\} \le C(\theta_2,\amax,\amin,\kappa) n^{-\nu}
\eeq 
\end{lemma}
The proof is postponed, together with those of the Lemmata that follow, to appendix \ref{app12}).
If $\calM$ is bounded by $M$, then $e_{12}^2 \le 4M^2$ and 
\beq
\label{thm3eq3}
\EE e_{12}^2 \le 4M^2  \PP\{e_{12}>0\} 
\le 
4M^2 C n^{-{\nu}}
\le   4M^2 C\left(\lognn\right)^{\frac{2s}{2s+d-2}} 
\eeq
if ${\nu}>{2s}/(2s+d-2)$, for example ${\nu} = 1$. 

\noindent{\bf{Term $e_3$:}}
$e_3$ corresponds to the variance on the partition $\hLamtaun\wedge \Lamtaunb$. For any $\eta>0$,
$$\PP\{e_3 > \eta\} 
\le 
\alpha d \#(\hLamtaun\wedge \Lamtaunb) e^{
-\frac{\beta n \eta^2}{d^2 \sum_{j\ge \jmin}2^{-2j}\#_j( \hLamtaun\wedge \Lamtaunb )}
}$$
according to Lemma \ref{thm1}. Since $\hLamtaun\wedge \Lamtaunb \subset \calTtaunb$, for any $j\ge0$, regardless of $\hLamtaun$, we have $\#_j(\hLamtaun\wedge \Lamtaunb) \le \#_j \calTtaunb\le \# \calTtaunb$. Therefore
\begin{align}
\PP\{e_3 > \eta\} 
&\le \alpha d \# \calTtaunb  e^{-\frac{\beta n \eta^2}{d^2 \sum_{j\ge \jmin}2^{-2j}\#_j \calTtaunb }}
\le
\alpha d \#\calTtaunb e^{-\frac{\beta n \eta^2}{d^2 |\rho|_{\BS}^p (\tau_n/b)^{-p} }}\,,
\label{thm3eq4}
\end{align}
which implies
\begin{align*}
&\EE e_3^2 
= \int_{0}^{+\infty} \eta \PP\left\{e_3 > \eta \right\} d\eta
= \int_{0}^{+\infty} \eta \min\left(1, \alpha d \# \calTtaunb e^{
-\frac{\beta n \eta^2}{d^2 \sum_{j\ge \jmin}2^{-2j}\#_j \calTtaunb }
} \right) d\eta
\\
&
\le \frac{d^2\log \alpha d \# \calTtaunb}{\beta n} \!\sum_{j\ge \jmin} 2^{-2j} \#_j \calTtaunb
\le C \lognn \left(\frac{\tau_n}b\right)^{-p}\!\!
\le
C(\theta_2,\theta_3,d,\kappa,s,|\rho|_{\BS}) \left(\lognn \right)^{\frac{2s}{2s+d-2}}\!\!\!\!\!\!\!\!\!\!.
\end{align*}

\commentout{
Since $\hLamtaun\wedge \Lamtaunb$ is a more refined partition than $\Lamtaunb$, one have 
\begin{align*}
\#(\hLamtaun\wedge \Lamtaunb)
& \le 
\#\Lamtaunb,
\\
\sum_{j \ge 0}{2^{-2j}\#_j(\hLamtaun\wedge \Lamtaunb)}
&\le
\sum_{j \ge 0}{2^{-2j}\#_j \Lamtaunb}
\end{align*}
when $d \ge 2$, and therefore
$$\EE e_3^2 \le \frac{d^2\log[\alpha d \# \Lamtaunb] }{\beta n} \sum_{j \ge 0}{2^{-2j}\#_j \Lamtaunb}.$$
}

\noindent{\bf{Term $e_2$ and $e_4$:}} 
These terms account for the difference of truncating the master tree based on $\Deltajk$'s and its empirical counterparts $\hDeltajk$'s. We prove that $\hDeltajk$'s concentrate near $\Deltajk$'s with high probability if there are sufficient samples.

\begin{lemma}
\label{lemma4}
For any $\eta >0$ and any $\Cjk \in \calT$
\begin{align}
\max\left\{\PP\left\{\hDeltajk \le \eta \ \text{ and } \
\Deltajk \ge b\eta  \right\},
\PP\left\{\Deltajk \le \eta \ \text{ and } \
\hDeltajk \ge b\eta  \right\}\right\}
& \le 
\alpha_1 e^{-\alpha_2 2^{2j} n \eta^2}
\label{lemma4eq1}
\end{align}
for some constants $\alpha_1 := \alpha_1(\theta_2,\theta_3,\amax,d)$ and $\alpha_2 := \alpha_2(\theta_2,\theta_3,\theta_4,\amax,d)$.
\end{lemma}
This Lemma enables one to show that $e_2 = 0$ and $e_4 =0$ with high probability:
\begin{lemma}
\label{thm3_lemma2}
Let $\alpha_1$ and $\alpha_2$ be the constants in Lemma \ref{lemma4}.
For any fixed ${\nu} > 0$, 
\beq
\label{hbx}
\PP\{e_2 > 0\} + \PP\{e_4>0\} \le \alpha_1 \amin n^{-{\nu}} 
\eeq
when $\kappa$ is chosen such that $\kappa > \kappa_2$, with $\alpha_2 \kappa_2^2 /b^2 = {\nu}+1$. 

\end{lemma}
Since $\calM$ is bounded by $M$, we have $e_2^2 \le 4M^2$ so
$$\EE e_2^2 \le 4M^2 \PP\{e_2>0\} 
\le 4M^2 \alpha_1 \amin n^{-{\nu}}
\le 4M^2 \alpha_1 \amin \left(\lognn\right)^{\frac{2s}{2s+d-2}}
$$
if ${\nu}>{2s}/(2s+d-2)$, for example ${\nu} = 1$.
The same bound holds for $e_4$. 

Finally, we complete the probability estimate \eqref{thm3eq0}: let $c_0^2 = B_{s,d} |\rho|_{\BS}^{\frac{2(d-2)}{2s+d-2}} (b\kappa)^{\frac{4s}{2s+d-2}}$  such that $e_{11} \le c_0 \left((\log n)/n\right)^{\frac{s}{2s+d-2}}$. We have
\begin{align*}
&
\PP\left\{ \| X - \hcalP_{\hLamtaun} X\|  \ge c_1 \left( (\log n)/n \right)^{\frac{s}{2s+d-2}}\right\}
\\
&\le 
\PP\left\{e_3>(c_1 -c_0)\left((\log n)/n \right)^{\frac{s}{2s+d-2}}\right\} 
+
\PP\{e_{12}>0 \}+\PP\{e_{2}>0 \}+\PP\{e_{4}>0 \}
\\
&
\le 
\PP\left\{e_3>(c_1 -c_0)\left((\log n)/n \right)^{\frac{s}{2s+d-2}}\right\} 
+
C n^{-{\nu}},
\end{align*}
as long as $\kappa$ is chosen such that $\kappa > \max(\kappa_1,\kappa_2)$ where
${b^2\kappa_1^2}/(21\theta_2^2) = \nu + 1$ and $\alpha_2 \kappa_2^2/b^2  = \nu+1$ according to \eqref{thm3eq21} and \eqref{hbx}.
Applying \eqref{thm3eq4} gives rise to 
\begin{align*}
&\PP\left\{e_3>(c_1 -c_0)\left((\log n)/n \right)^{\frac{s}{2s+d-2}}\right\} 
\le
\alpha d\# \calTtaunb e^{
-\frac{\beta n }{ |\rho|_{\BS}^p (\tau_n/b)^{-p} }
(c_1-c_0)^2 \left( (\log n)/n \right)^{\frac{2s}{2s+d-2}}
}
\\
& 
\le 
\alpha d\# \calTtaunb n^{-\frac{\beta (c_1 -c_0)^2\kappa^p}{b^p |\rho|_{\BS}^p}}
\le
\alpha d\amin  n^{-\left(\frac{\beta (c_1 -c_0)^2\kappa^p}{b^p |\rho|_{\BS}^p}-1\right)}
\le
\alpha d\amin n^{-{\nu}} 
\end{align*}
if $c_1$ is taken large enough such that $\frac{\beta (c_1 -c_0)^2\kappa^p}{b^p |\rho|_{\BS}^p} \ge {\nu}+1$.

We are left with the cases $d=1,2$. When $d=1$, for any distribution $\rho$ satisfying quasi-orthogonality \eqref{quasiortho} and any $\eta>0$, the tree complexity may be bounded as follows: 
$$\sum_{j\ge \jmin} 2^{-2j} \#_j\calTrhoeta \le \sum_{j \ge \jmin} 2^{-2j}2^j /\theta_1 = 2/\theta_1 2^{-\jmin} = 2M/(\theta_1\theta_2)\,,$$ 
so $\|X-\calP_{\Lamrhoeta} X\|^2  \le 8 M B_0 \eta^2 /(3\theta_1\theta_2)$. Hence 
$$e_{11}^2 \le \tfrac{8 M B_0}{3\theta_1\theta_2}(b\tau_n)^2 \le \tfrac{8M B_0 b^2 \kappa^2}{3\theta_1\theta_2} (\log n)/n
,\quad
\PP\{e_3>\eta\} \le \alpha d \# \calTtaunb e^{-\frac{\theta_1 \theta_2\beta n \eta^2}{2M d^2} },$$
which yield $\EE e_3^2 \le 2 M d^2 \log \alpha d \#\calTtaunb/(\theta_1 \theta_2 \beta n) \le C (\log n)/n$
and estimate \eqref{thm3eq10}.

When $d=2$, for any distribution satisfying quasi-orthogonality and given any $\eta>0$, we have $\sum_{j \ge \jmin} 2^{-2j} \#_j \calTrhoeta  \le -|\rho|^{-1}\log \eta$, whence $\|X- \calP_{\Lamrhoeta} X\|^2 \le -\frac{4}{3}B_0|\rho| \eta^2 \log \eta$.
Therefore
$$e_{11}^2 \le -\tfrac{4}{3} B_0 |\rho|(b\tau_n)^2\log(b\tau_n) \le C (\log^2 n)/n
\quad,\quad
\PP\{e_3 > \eta\} \le \alpha d \# \calTtaunb e^{-\frac{2\beta n \eta^2}{d^2 |\rho| \log n}}\,,$$
which yield $\EE e_3^2 \le d^2 |\rho| \log\alpha d \# \calTtaunb(\log n)/(2\beta n) \le C(\log^2 n)/n$
and the probability estimate \eqref{thm3eq20}.
\end{proof}

\begin{proof}[Proof of  Theorem  \ref{thmunbounded}]
Let $R>0$.
If we run adaptive GMRA on $B_R(0)$, and approximate points outside $B_R(0)$ by $0$, the MSE of the adaptive GMRA in $B_R(0)$ is
$$\|(\bbI - \hcalP_{\hLamtaun})\mathbf{1}_{\{ \|x\| \le R\}} X\|^2 \lesssim
(\Large|\rho|_{B_\mathbf{0}(R)}\Large|^p + R^2) \left( (\log n)/n \right)^{\frac{2s}{2s+d-2}}
\lesssim
R^{\max(\lambda,2)} \left( (\log n)/n \right)^{\frac{2s}{2s+d-2}}.$$
The squared error outside $B_R(0)$ is 
\begin{align}
\|\mathbf{1}_{\{\|x\| \ge R\}}X\|^2 
= \int_{B_R(0)^c} ||x||^{2}d\rho\le C R^{-\delta}.
\label{equnbounded1}
\end{align}
The total MSE is 
$$
{\rm MSE} \lesssim R^{\max(\lambda,2)} \left( (\log n)/n \right)^{\frac{2s}{2s+d-2}} + R^{-\delta}.
$$
Minimizing over $R$ suggests taking $R = R_n =\max(R_0, \mu (\log n/n)^{-\frac{2s}{(2s+d-2)(\delta+\max(2,\lambda))}} )$,
yielding $
{\rm MSE} \lesssim
\left( (\log n)/n \right)^{\frac{2s}{2s+d-2} \cdot \frac{\delta}{\delta+\max(\lambda,2)}}.$
The probability estimate \eqref{equnbounded} follows from Eq. \eqref{equnbounded1} and Eq. \eqref{thm3eq0} in Theorem \ref{thm3}.
\end{proof}

In Remark \ref{remarkunbounded}, we claim that $\lambda$ is not large in simple cases.
If $\rho \in \AS^\infty$ and $\rho$ decays such that $\rho(\Cjk) \le  2^{-jd}\|\cjk\|^{-(d+1+\delta)}$, we have $\Deltajk \le 2^{-js} 2^{-jd/2}\|\cjk\|^{-(d+1+\delta)/2}$. Roughly speaking, for any $\eta>0$, the cells of distance $r$ to $0$ satisfying $\Deltajk \ge 2^{-j}\eta$ will satisfy $2^{-j} \ge (\eta r^{\frac{d+1+\delta}{2}})^{\frac{2}{2s+d-2}}$. In other words, the cells of distance $r$ to $0$ are truncated at scale $\jmax$ such that $2^{-\jmax} = (\eta r^{\frac{d+1+\delta}{2}})^{\frac{2}{2s+d-2}}$, which gives rise to complexity $\le 2^{-2 \jmax} r^{d-1} 2^{\jmax d} \le \eta^{-\frac{2(d-2)}{2s+d-2}} r^{d-1-\frac{(d+1+\delta)(d-2)}{2s+d-2}}$. If we run adaptive GMRA with threshold $\eta$ on $B_R(0)$, the weighted complexity of the truncated tree is upper bounded by $\eta^{-\frac{2(d-2)}{2s+d-2}} r^{d-\frac{(d+1+\delta)(d-2)}{2s+d-2}}$. Therefore, $\rho|_{B_R(0)} \in \BS$ for all $R>0$ and $\Large | \rho|_{B_R(0)}\Large|_{\BS}^p \le R^\lambda$ with $\lambda = d-\frac{(d+1+\delta)(d-2)}{2s+d-2}$.

\section{Discussions and extensions}
\label{secdisex}

\subsection{Computational complexity}

The computational costs of GMRA and adaptive GMRA are summarized in Table \ref{TableComputation}.

\begin{table}[th]
\renewcommand{\arraystretch}{1.5}
\begin{center}
\resizebox{0.8\columnwidth}{!}{
\begin{tabular}{ |c | c |}
\hline
  Operations & Computational cost 
  \\  \hline\hline
  Multiscale tree construction  &  $C^d D n \log n$
  \\ \hline
 Randomized PCA at scale $j$ &  $\underbrace{D d n 2^{-jd}}_{\text{PCA cost at one node}} \cdot \underbrace{ 2^{jd}}_{\text{number of nodes}} = D  d n$ 
 \\ \hline
 Randomized PCA at all nodes & $\underbrace{Dd n}_{\text{cost at a fixed scale} } \cdot \underbrace{{1/d\log n}}_{\text{number of scales}} = D n \log n$
 \\ \hline
 Computing $\Deltajk$'s &  $\underbrace{D d n}_{\text{cost for single scale}} \cdot  \underbrace{1/d\log n}_{\text{number of scales}} = D n \log n$
 \\ \hline  
 Compute $\calP_j(x)$ for a new sample $x$& $\underbrace{D\log n}_{\text{find } \Cjk \text{ containing } x} \ + \ \underbrace{Dd}_{\text{compute } \calPjk (x)} = D(\log n + d)$
 \\ \hline
  \end{tabular}
}
\end{center}
\caption{Computational cost}
\label{TableComputation}
\end{table}

\subsection{Quasi-orthogonality}
\label{secQO}
A main difference between GMRA and orthonormal wavelet bases \citep[see][]{Dau,Mallat_book} is that $V_{j,x} \nsubseteq V_{j+1,x}$ where $(j,x)=(j,k)$ such that $x \in \Cjk$.
Therefore the geometric wavelet subspace $\proj_{\Vjx^\perp}\Vjox$ which encodes the difference between $\Vjox$ and $\Vjx$ is in general not orthogonal across scales.

Theorem \ref{thm3} involves a quasi-orthogonality condition \eqref{quasiortho}, which is satisfied if the operators $\{\calQjk\}$ applied on $ \calM$ are rapidly decreasing in norm or are orthogonal. When $\rho \in \calA_1^\infty$ such that $\|\calQjk X\| \sim 2^{-j}\sqrt{\rho(\Cjk)}$, quasi-orthogonality is guaranteed. In this case, for any node $\Cjk$ and $C_{j',k'} \subset \Cjk$, we have $\|\calQ_{j',k'} X\| /\sqrt{\rho(C_{j',k'})}\lesssim 2^{-(j'-j)} \|\calQ_{j,k} X\| /\sqrt{\rho(\Cjk) }$, which implies $\sum_{C_{j',k'} \subset \Cjk} \langle \calQjk X , \calQ_{j',k'}X\rangle \lesssim 2 \|\calQjk X\|^2$. Therefore $B_0 \lesssim 2$. Another setting is when $\calQ_{j',k'}$ and $\calQ_{j,k}$ are orthogonal whenever $C_{j',k'} \subset \Cjk$, as guaranteed in orthogonal GMRA in Section \ref{secOGMRA}, in which case exact orthogonality is automatically satisfied.

 \commentout{
 \textcolor{blue}{WL: I am not very confident about the statement and the experiments below.}
In numerical experiments, the regularity parameter $s$ in the $\BS$ model can be approximately identified through Eq. \eqref{eqBs} or Eq. \eqref{Bs1}. Recall that $\sum_{j} 2^{-2j} \#_j \calTrhoeta$ is called the weighted complexity of $\calTrhoeta$. 
%
(i) If we apply Eq. \eqref{eqBs}, the slope in the log-log plot of $\eta^{d-2} [\text{entropy of } \calTrhoeta]^{\frac{d-2}{2}}$ versus the weighted complexity of $\calTrhoeta$ gives an approximation of $-s$, denoted by $-s_1$; 
(ii) if we use Eq. \eqref{Bs1}, the slope in the log-log plot of $\|X-\calP_{\Lamrhoeta}X\|^{d-2}$ versus the weighted complexity of $\calTrhoeta$ gives another approximation of $s$, denoted by $-s_2$. If quasi-orthogonality holds, $s_2 = s_1$; otherwise $s_2 < s_1$. 
%
In practice, the weighted tree complexity is dominated by the sum on the leaves so we use weighted complexity of the adaptive partition in numerical experiments.
%
We conduct experiments on the $5$-dimensional S manifold and Z manifold in Figure \ref{FigQO}. As for GMRA on uniform partitions, we plot $\|X-\calP_j X\|_n^{d-2}$ versus the weighted complexity of $\Lam_j$. 
For the S manifold, we obtain $s_1 \approx 2.89$ and $s_2 \approx 2.49$. 
For the Z manifold, we obtain $s_1 \approx 3.66$ and $s_2 \approx 3.23$. 
We observe that $s_2 < s_1$ but their difference is small. Quasi-orthogonality holds for the S manifold because at every point $x$, the principal subspaces $\{\Vjx\}$ become more and more parallel as $j$ increases. For the Z manifold, quasi-orthogonality may be violated at the corners but $s_2$ is not too far from $s_1$ since the corners have a small measure.

\begin{figure}[hthp]
     \centering    
           \hspace{-1cm}
 \subfigure[5-dim S manifold]{
    \includegraphics[width=7cm]{Fig23_SandZ/Sample100000/V2/SDim5_Entropy-eps-converted-to.pdf}
    }
    \hspace{-0.6cm}
     \subfigure[5-dim Z manifold]{
    \includegraphics[width=7cm]{Fig23_SandZ/Sample100000/V2/ZDim5_Entropy-eps-converted-to.pdf}
    }
          \caption{Log-log plot of $\eta^{d-2} [\text{\textcolor{red}{MM: all words ``entropy'' should become ``complexity''! we should have a symbol for it as well perhaps} entropy of }\hLameta]^{\frac{d-2}{2}} $ and $\|X-\calP_{\hLameta} X\|_n^{d-2}$  versus the entropy of $\hLameta $ for the $5$-dimensional S manifold (a) and the $5$-dimensional Z manifold (b).}
   \label{FigQO}
 \end{figure}  
}

Quasi-orthogonality enters in the proof of Eq. \eqref{Bs1}. If quasi-orthogonality is violated, we still have a convergence result in Theorem \ref{thm3} but the convergence rate will be worse: MSE $\lesssim [(\log n)/n]^{\frac{s}{2s+d-2}}$ when $d \ge 3$ and ${\rm MSE} \lesssim [(\log^d n) /n]^{\frac 1 2}$ when $d=1,2$.

\subsection{Orthogonal GMRA and adaptive orthogonal GMRA}
\label{secOGMRA}


A different construction, called orthogonal geometric multi-resolution analysis in Section 5 of \citet{CM:MGM2}, follows the classical wavelet theory by constructing a sequence of increasing subspaces and then the corresponding wavelet subspaces exactly encode the orthogonal complement across scales. Exact orthogonality is therefore satisfied.

\subsubsection{Orthogonal GMRA}
\label{secOLMR:MGM1}

In the construction, we build the sequence of subspaces $\{\widehat{S}_{j,k}\}_{k \in \calK_j , j\ge \jmin}$ with a coarse-to-fine algorithm in Table \ref{TableOGMRA}. For fixed $x$ and $j$, $(j,x)$ denotes $(j,k)$ such that $x\in \Cjk$.
In orthogonal GMRA the sequence of subspaces $\Sjx$ is increasing such that $S_{0,x} \subset S_{1,x} \subset \cdots \Sjx \subset \Sjox \cdots$ and the subspace $\Ujox$ exactly encodes the orthogonal complement of $\Sjx$ in $\Sjox$. 
Orthogonal GMRA with respect to the distribution $\rho$ corresponds to affine projectors onto the subspaces $\{S_{j,k}\}_{k \in \calK_j , j\ge \jmin}$.

\begin{table}[h]
\renewcommand{\arraystretch}{1.4}
\centering
\resizebox{0.8\columnwidth}{!}{
\begin{tabular}{ |c || c  |c |}
\hline
   & Orthogonal GMRA&  Empirical orthogonal GMRA \\
   [5pt]
   \hline
   \hline
       \multirow{6}{*}{Subpaces} &$S_{0,x} = V_{0,x}$ & 
  $\widehat{S}_{0,x} = \hV_{0,x}$
  \\
             [5pt]
  &
  $U_{1,x}   = \proj_{S_{0,x}^\perp} V_{1,x}, \ S_{1,x} = S_{0,x} \oplus U_{1,x}$
  & $\widehat{U}_{1,x}   = \proj_{\widehat{S}_{0,x}^\perp} \hV_{1,x}, \
  \widehat{S}_{1,x} = \widehat{S}_{0,x} \oplus \widehat{U}_{1,x}$
  \\
  &$\ldots$& $\ldots$
  \\
  &
  $U_{j+1,x} =\proj_{\Sjxp} \Vjox$
  & $\widehat{U}_{j+1,x} =\proj_{\hSjxp} \hVjox$  
  \\
       [5pt]
  & $\Sjox = \Sjx \oplus \Ujox$ & $\hSjox = \hSjx \oplus \hUjox$
  \\     [5pt]
  \hline
  Affine 
  & $\calS_j := \sum_{k\in \calK_j} \calS_{j,k} \mathbf{1}_{{j,k}}$ 
  & $\hcalS_j := \sum_{k\in \calK_j} \hcalS_{j,k} \mathbf{1}_{{j,k}}$ 
   \\
  projectors  & $\calS_{j,k}(x) := c_{j,k} +\proj_{S_{j,k}} (x-c_{j,k})$ &
  $\hcalS_{j,k}(x) := \hcjk +\proj_{\widehat{S}_{j,k}} (x-\hcjk)$
  \\
           [5pt]
  \hline
\end{tabular}
}
\caption{Orthogonal GMRA}
\label{TableOGMRA}
\end{table}

\commentout{
 For each scale  $j\ge 0$, $\calP_j$ corresponds to piecewise affine projectors on the partition $\Lam_j$, i.e.
\beq
\label{eqGMRA}
\calP_j := \sum_{k\in \calK_j} \calP_{j,k} \mathbf{1}_{{j,k}}
\eeq
 where $\calPjk$ is the affine projection 
\beq
\label{eqCM:MGM2}
\calP_{j,k}(x) := c_{j,k} +\proj_{V_{j,k}} (x-c_{j,k})
\eeq
with $c_{j,k}$ and $V_{j,k}$ given by
\begin{align}
c_{j,k} & := \EE_{j,k} x = { \EE [x\mathbf{1}_{j,k}(x)]}/{\rho(C_{j,k})},  
\label{eqMMS:NoisyDictionaryLearning}
\\
V_{j,k} & := \underset{\dim V = d}{\argmin}\ \EE_{j,k} \|x-c_{j,k}-\proj_V(x-c_{j,k})\|^2,
\label{eqGMRA3}
\end{align}
where the minimum is taken over all subspaces of dimension $d$. In other words, $c_{j,k}$ is the conditional mean on $\Cjk$ and $V_{j,k}$ is the subspace spanned by the eigenvectors corresponding to the $d$ largest eigenvalues of the conditional covariance matrix 
\beq
\label{eqGMRA4}
\Sigma_{j,k} = \EE_{j,k} \left[ (x-c_{j,k})(x-c_{j,k})^T \right].
\eeq
}

For a fixed distribution $\rho$, the approximation error $\|X-\calS_j X\|$ decays as $j$ increases. We will consider the model class $\AS^{\rm o}$ where $\|X-\calS_j X\|$ decays like $\mathcal{O}(2^{-js})$.

\begin{definition}
A probability measure $\rho$ supported on $\calM$ is in $\AS^{\rm o}$ if
\beq
\label{eqoAs}
|\rho|_{\AS^{\rm o}}=\sup_{\calT} \ \inf\{A_0^o\,:\,\|X-\calS_j X\| \le A_0^o 2^{-js}, \forall\, j\ge \jmin\}<\infty\,,
\eeq
where $\calT$ varies over the set, assumed non-empty, of multiscale tree decompositions satisfying Assumption (A1-A5).
\end{definition}


Notice that $\AS \subset \AS^{\rm o}$. 
We split the MSE into the squared bias and the variance as:
$
\EE\|X-\hcalS_j X\|^2
= 
{\|X-\calS_j X\|^2} + 
{\EE\|\calS_j X-\hcalS_j X\|^2}.
$
The squared bias  $\|X-\calS_j X\|^2 \le |\rho|^2_{\AS^{\rm o}} 2^{-2js}$ whenever $\rho \in \AS^{\rm o}$.
In Lemma \ref{lemmao0} we show 
$
\EE \|\calS_j X - \hcalS_j X\|^2 
\le
\frac{d^2 j^4 \#\Lam_j \log[\alpha   j \#\Lam_j]}{\beta 2^{2j}n}
= \mathcal{O}\left(\frac{j^5 2^{j(d-2)}}{n}\right)
$
where $\alpha$ and $\beta$ are the constants in Lemma \ref{thm1}. A proper choice of the scale yields the following result:

\begin{theorem}
\label{thmo2}
Assume that $\rho \in \AS^{\rm o},\ s \ge 1$. Let $\nu>0$ be arbitrary and $\mu>0$. 
If $j^*$ is properly chosen such that
$$
2^{-j^*}  =
\begin{cases}
\mu\lognn 
    & \text{ for } d=1
    \\ 
    \mu \left( \frac{\log^5 n}{n}\right)^{\frac{1}{2s+d-2}}
    & \text{ for } d \ge 2
   \end{cases},
$$
then there exists a constant $C_1(\theta_1,\theta_2,\theta_3,\theta_4,d,\nu,\mu,s)$ such that 
\begin{align}
\PP\left\{\|X - \hcalS_{j^*} X \| 
\ge (|\rho|_{\AS^{\rm o}} \mu^{s} + C_1 )\frac{\log^5 n}{n}  
\right\}
 \le C_2(\theta_1,\theta_2,\theta_3,\theta_4,d,\mu)  n^{-\nu} 
 &\quad  \text{for } d=1,
\nonumber
\\
 \PP\left\{\|X - \hcalS_{j^*} X \| 
\ge (|\rho|_{\AS^{\rm o}}\mu^s + C_1 )\left (\frac{\log^5 n}{n} \right)^{\frac{s}{2s+d-2}} 
\right\}
 \le C_2(\theta_1,\theta_2,\theta_3,\theta_4,d,\mu,s)  n^{-\nu}
&\quad \text{for } d\ge 2. 
\label{thmo2eq21}
\end{align}

\end{theorem}

Theorem \ref{thmo2} is proved in appendix \ref{appo1}.

\subsubsection{Adaptive Orthogonal GMRA}
\label{secAOLMR:MGM1}

\begin{table}[h]
\renewcommand{\arraystretch}{1.4}
\centering
\resizebox{0.8\columnwidth}{!}{
\begin{tabular}{ | c | c  |}
\hline
    Definition (infinite sample) & Empirical version  \\
      [5pt]
         \hline
 $\hspace{-3.8cm}
{\Deltajk^o} :=
 \lcol  (\calS_j -\calS_{j+1} ) \chijk X  \rcol$
 & $\hspace{-3.8cm}
{\hDeltajk^o} :=
 \lcol  (\hcalS_j -\hcalS_{j+1} ) \chijk X  \rcol$
     \\
   \hspace{1cm}$= 
\left( \lcol (\bbI -  \calS_j ) \mathbf{1}_{j,k} X \rcol^2
 -
 \lcol  (\bbI - \calS_{j+1} ) \mathbf{1}_{j,k} X  \rcol^2 
 \right)^{\frac 1 2}$
        &
   \hspace{1.1cm}$= 
\left( \lcol (\bbI -  \hcalS_j ) \mathbf{1}_{j,k} X \rcol^2
 -
 \lcol  (\bbI - \hcalS_{j+1} ) \mathbf{1}_{j,k} X  \rcol^2 
 \right)^{\frac 1 2}$
   \\
   [5pt]
   \hline
\end{tabular}
}
\caption{Refinement criterion in adaptive orthogonal GMRA}
\label{TableORefinement}
\end{table}

Orthogonal GMRA can be constructed adaptively to the data with the refinement criterion defined in Table \ref{TableORefinement}. 
We let $\tau_n^o := \kappa (\log^5 n/{n})^{\frac 1 2}$ where $\kappa$ is a constant, truncate the data master tree $\calTn$ to the smallest proper subtree that contains all  $C^{j,k} \in \calTn$ satisfying $\hDeltajk^o \ge 2^{-j}\tau_n^o$, denoted by $\hcalTtauno$. Empirical adaptive orthogonal GMRA returns piecewise affine projectors on the adaptive partition $\hLamtauno$ consisting of the outer leaves of $ \hcalTtauno$.
Our algorithm is summarized in Algorithm 2.

\commentout{
For each cell $C_{j,k}$ in the master tree $\calT$ and any distribution $\rho$ on $\calM$, we define the refinement criterion as
\begin{align}
{\Deltajk^o}
& :=
 \lcol  \calSjk \mathbf{1}_{j,k} X - \sum_{C_{j+1,k'} \in \Child(\Cjk)}\calS_{j+1,k'}\mathbf{1}_{j+1,k'} X  \rcol
 \label{ejk1}
 \\
 & = 
\left( \lcol (\bbI -  \calSjk ) \mathbf{1}_{j,k} X \rcol^2
 - \sum_{C_{j+1,k'} \in \Child(\Cjk)}
 \lcol  (\bbI - \calS_{j+1,k'} ) \mathbf{1}_{j+1,k'} X  \rcol^2 
 \right)^{\frac 1 2}
\label{ejk2}
 \end{align}
which measures the amount of decreased error when $C_{j,k}$ is refined to its children $\Child(\Cjk)$.
For orthogonal GMRA, the equality from Eq. \eqref{ejk1} to Eq. \eqref{ejk2} always holds because
the sequence of subspaces $\{\Sjx\}_{j \in \ZZ}$ is increasing such that $\Sjx \subset \Sjox $ and therefore
 $\calSjk\calS_{j+1,k'}  =\calSjk$ whenever $C_{j+1,k'} \in \Child(\Cjk)$.
}

If $\rho$ is known, 
given any fixed threshold $\eta > 0$, we let 
$\calTrhoeta$ be the smallest proper tree of $\calT$ that contains all $C_{j,k} \in \calT$ for which $\Deltajk^o \ge   2^{-j}\eta.$
This gives rise to an adaptive partition $\Lamrhoeta$ consisting the outer leaves of $\calTrhoeta$. 
We introduce a model class $\BS^o$ for whose elements we can control
the growth rate of the truncated tree $\calTrhoeta$ as $\eta$ decreases. 

\begin{algorithm}[t!]                      	
\caption{Empirical Adaptive Orthogonal GMRA}          	
\label{TableAdaptiveOGMRA}		
\begin{algorithmic}[1]                    	
    \REQUIRE data $\calX_{2n} = \calX'_n \cup \calX_n$, intrinsic dimension $d$, threshold $\kappa$ 
    \ENSURE $\hcalS_{\hLamtauno}$ : adaptive piecewise linear projectors
    \STATE Construct $\calTn$ and $\{C_{j,k}\}$ from $\calX_n'$
    \STATE Compute $\hcalSjk $ and $\hDeltajk^o$ on every node $\Cjk \in \calTn$.
    \STATE $\hcalTtauno\leftarrow$ smallest proper subtree of $\calTn$ containing all $\Cjk \in \calTn$ : $\hDeltajk^o \ge 2^{-j}\tau_n^o$ where $\tau_n^o =\kappa\sqrt{(\log^5 n)/n}$.
     \STATE $\hLamtauno\leftarrow$ partition associated with the outer leaves of $\hcalTtauno$
    \STATE $\hcalS_{\hLamtauno}\leftarrow\sum_{\Cjk \in \hLamtauno} \hcalSjk \chijk.$
\end{algorithmic}
\end{algorithm}

\begin{definition}
In the case $d\ge 3$, given $s>0$, a probability measure $\rho$ supported on $\calM$ is in $\BS^o$ if
 the following quantity is finite 
\beq
\label{eqoBs}
|\rho|_{\BS^o}^p:= \sup_{\calT}\ \sup_{\eta>0} \eta^p \sum_{j\ge \jmin}  2^{-2j}\#_j \calTrhoeta  \text{ with } 
p= \frac{2(d-2)}{2s+d-2}
\eeq
where $\calT$ varies over the set, assumed non-empty, of multiscale tree decompositions satisfying Assumption (A1-A5).
\label{defoBs}
\end{definition}

Notice that exact orthogonality is satisfied for orthogonal GMRA.
One can show that, as long as $\rho \in \BS^o$, 
$$
\|X-\calS_{\Lamrhoeta} X\|^2
\le
B^o_{s,d} |\rho|_{\BS^o}^p \eta^{2-p}
\le
B^o_{s,d} |\rho|_{\BS^o}^2 \left(\sum_{j\ge \jmin}2^{-2j} \#_j \calTrhoeta\right)^{-\frac{2s}{d-2}},
 $$
where $B^o_{s,d}:= 2^p /(1-2^{p-2})$.  
We can prove the following performance guarantee of the empirical adaptive orthogonal GMRA (see Appendix \ref{appo2}):

\begin{theorem}
\label{thmo3}
Suppose $\calM$ is bounded: $\calM \subset B_M(0)$ and the multiscale tree satisfies $\rho(\Cjk) \le \theta_0 2^{-jd}$ for some $\theta_0>0$.
Let $d\ge 3$ and $\nu>0$. There exists $ \kappa_0(\theta_0,\theta_2,\theta_3,\theta_4,\amax,d,\nu)$ such that if $\rho \in  \BS^o$ for some $s>0$ and $\tau_n^o = \kappa\left[(\log^5 n)/{n}\right]^{\frac 1 2}$ with $\kappa \ge \kappa_0$, then there is a $c_1$ and $c_2$ such that
\begin{align}
&\PP\left\{ 
\|X-\hcalS_{\hLamtauno}X\| \ge c_1\left(\frac{\log^5 n}{n} \right)^{\frac{s}{2s+d-2}}
\right\}
\le 
c_2 n^{-\nu}.
\label{thmo3eq0}
\end{align}


\end{theorem}
 
 In Theorem \ref{thmo3}, the constants are $c_1 := c_1(\theta_0,\theta_2,\theta_3,\theta_4,\amax,d,s,\kappa,|\rho|_{\BS^o},\nu)$ and $c_2 := c_2(\theta_0,\theta_2,\theta_3,\theta_4,\amin,\amax,d,s,\kappa,|\rho|_{\BS^o})$.
Eq. \eqref{thmo3eq0} implies that ${\rm MSE} \lesssim (\frac{\log^5 n}{n})^{\frac{2s}{2s+d-2}} $ for orthogonal adaptive GMRA when $d\ge 3$. In the case of $d=1,2$, we can prove that ${\rm MSE} \lesssim \frac{\log^{4+d} n}{n}$.

\acks{This research was partially funded by ONR N00014-12-1-0601, NSF-DMS-ATD-1222567, 1708553, 1724979  and AFOSR FA9550-14-1-0033.}

\appendix 

\section{Tree construction, regularity of geometric spaces}
\label{appa}
\subsection{Tree construction}
\label{apptree}


We now show that from a set of nets $\{T_j(\calX_n')\}_{j \in [\jmin,\jmax]}$ from the cover tree algorithm we can construct a set of $C_{j,k}$ with desired properties.
Let $\{a_{{j},k}\}_{k=1}^{N(j)}$ be the set of points in $T_{j}(\calX_n')$. Given a set of points $\{z_1,\ldots,z_m\} \subset \RR^D$, the Voronoi cell of $z_\ell$ with respect to $\{z_1,\ldots,z_m\} $ is defined as 
$${\rm Voronoi}(z_\ell,\{z_1,\ldots,z_m\}) = \{x\in \RR^D: \|x-z_\ell\| \le \|x-z_i\| \text{ for all } i \neq \ell  \}.$$
Let 
\begin{equation}
\hcalM =\bigcup_{j=\jmin}^{\jmax} \bigcup_{a_{j,k} \in T_j(\calX_n')} B_{\frac 1 4 2^{-j}} (a_{j,k})\,.
\label{e:tildeM}
\end{equation}
Our $\Cjk$'s are constructed in Algorithm \ref{TableCoverTree}. These $\Cjk$'s form a multiscale tree decomposition of $\hcalM$.
We will prove that $\calM \setminus \hcalM$ has a negligible measure and $\{\Cjk \}_{k \in \calK_j, j \in [\jmin,\jmax]}$ satisfies Assumptions (A1-A5). The key is that every $\Cjk$ is contained in a ball of radius $3\cdot 2^{-j}$ and also contains a ball of radius $ 2^{-j}/4$.

\begin{algorithm}[th]                      	
\caption{Construction of a multiscale tree decomposition $\{\Cjk\}$}         	
\label{TableCoverTree}		
\begin{algorithmic}[1]                    	
    \REQUIRE  data $ \calX_n'$
    \ENSURE A multiscale tree decomposition $\{\Cjk\}$
    \STATE Run cover tree on $\calX_n'$ to obtain a set of nets $\{T_j(\calX_n')\}_{j \in [\jmin,\jmax]}$
    \STATE $j=\jmin$: $C_{\jmin,0} = \hcalM$ defined in \eqref{e:tildeM}
    \STATE for $j=\jmin+1,\dots,\jmax$ : For every $C_{j-1,k_0}$ at scale $j-1$, $C_{j-1,k_0}$ has $\#\left( T_j(\calX_n') \cap C_{j-1,k_0}\right)$ children indexed by $a_{j,k} \in T_j(\calX_n') \cap C_{j-1,k_0}$ with corresponding $C_{j,k}$'s constructed as follows:
\begin{align*}
C_{j,k}^{(j)} =\hcalM \ \bigcap \ \Voronoi(a_{j,k}, T_j(\calX_n') \cap C_{j-1,k_0}) 
\end{align*}
and for $i=j+1,\dots,\jmax$
$$C_{j,k}^{(i)} =\Big( \bigcup_{a_{i,k'} \in C_{j,k}^{(i-1)}} B_{\frac 1 4 2^{-i}}(a_{i,k'}) \Big) \bigcup C_{j,k}^{(i-1)}$$
Finally, let $C_{j,k} = C_{j,k}^{(\jmax)}$.
\end{algorithmic}
\end{algorithm}

\begin{lemma}
\label{propcovertree1}

Every $\Cjk$ constructed in Algorithm \ref{TableCoverTree} 
satisfies $B_{\frac{2^{-j}}4}(a_{j,k})\subseteq C_{j,k}\subseteq B_{3\cdot 2^{-j}}(a_{j,k})$
\end{lemma}

\begin{proof}
For any $x \in \RR^D$ and any set $C \in \RR^D$, the diameter of $C$ with respect to $x$ is defined as $\diam(C,x) : = \sup_{z\in C} \|z-x\|$.
First, we prove that, for every $j$, $C_{j,k_1} \cap C_{j,k_2} = \emptyset$ whenever $k_1 \neq k_2$. Take any $a_{j+1,k_1'} \in C_{j,k_1}$ and $a_{j+1,k_2'} \in C_{j,k_2}$. Our construction guarantees that 
$$\diam(C_{j+1,k_1'},a_{j+1,k_1'}) \le \tfrac 1 4 2^{-(j+1)} + \tfrac 1 4 2^{-(j+2)}+\ldots < \tfrac 1 2 2^{-(j+1)}$$
and similarly for $\diam(C_{j+1,k_2'},a_{j+1,k_2'})$.
Since $\|a_{j+1,k_1'}-a_{j+1,k_2'}\| \ge 2^{-(j+1)}$, this implies that $C_{j+1,k_1'} \cap C_{j+1,k_2'} = \emptyset$. In our construction, 
$$C_{j,k_1} =\Big( \bigcup_{a_{j+1,k_1'} \in C_{j,k_1}} C_{j+1,k_1'} \Big) \bigcup B_{\frac{2^{-j}}{4}}(a_{j,k_1}), \ 
C_{j,k_2} =\Big( \bigcup_{a_{j+1,k_2'} \in C_{j,k_2}} C_{j+1,k_2'} \Big) \bigcup B_{\frac{2^{-j}}{4}}(a_{j,k_2}).$$ Since $\|a_{j,k_1} - a_{j,k_2}\| \ge 2^{-j}$, we observe that $B_{\frac 1 4 2^{-j}}(a_{j,k_1}) \cap B_{\frac 1 4 2^{-j}}(a_{j,k_2}) = \emptyset$, 
$C_{j+1,k_1'} \cap B_{\frac 1 4 2^{-j}}(a_{j,k_2}) = \emptyset$ for every $a_{j+1,k_1'} \in C_{j,k_1}$,
and
$C_{j+1,k_2'} \cap B_{\frac 1 4 2^{-j}}(a_{j,k_1}) = \emptyset$ for every $a_{j+1,k_2'} \in C_{j,k_2}$.
Therefore $C_{j,k_1}  \cap C_{j,k_2} = \emptyset$. 

Our construction of $C_{j,k}$'s guarantees that every $\Cjk $ contains a ball of radius $\frac 1 4 \cdot 2^{-j}$. Next we prove that every $\Cjk$ is contained in a ball of radius $3\cdot 2^{-j}$. The cover tree structure guarantees that $\calX_n' \subset \cup_{a_{j,k}\in T_j(\calX_n')} B_{2\cdot 2^{-j}}(a_{j,k})$ for every $j$. Hence, for every $a_{j,k}$ and every $a_{j+1,k'}\in \Cjk$, we obtain $\|a_{j+1,k'} -a_{j,k}\| \le 2\cdot 2^{-j}$ and the computation above yields $\diam(C_{j+1,k'},a_{j+1,k'}) \le 2^{-j}/4$, and therefore $\diam(\Cjk,a_{j,k}) \le 2\cdot 2^{-j} +2^{-j}/4 \le  3\cdot 2^{-j}$. In summary $\Cjk$ is contained in the ball of radius $3\cdot 2^{-j}$ centered at $a_{j,k}$.
\end{proof}

The following Lemma will be useful when comparing comparing covariances of sets:
\begin{lemma} 
\label{l:covcompare}
If $B\subseteq A$, then we have $\lambda_{d}(\mathrm{cov}(\rho|_A)) \ge \frac{\rho(B)}{\rho(A)} \lambda_{d}(\mathrm{cov}(\rho|_B)).$
\end{lemma}
\begin{proof}
Without loss of generality, we assume both $A$ and $B$ are centered at $x_0$.
Let $V$ be the eigenspace associated with the largest $d$ eigenvalues of $\mathrm{cov}(\rho|_B)$. Then
\begin{align*}
\lambda_{d}(\mathrm{cov}(\rho|_A))
& = \max_{{\rm dim}U = d} \min_{u \in U} \frac{u^T \mathrm{cov}(\rho|_A) u}{u^T u}
\ge \min_{v \in V} \frac{v^T \mathrm{cov}(\rho|_A) v}{v^T v}
\\
& \ge  \min_{v \in V} \frac{v^T \left( \int_{A} (x-x_0)(x-x_0)^T d\rho\right) v}{\rho(A)v^T v} 
\\
& =  \min_{v \in V}  \left(\frac{v^T \left( \int_{B} (x-x_0)(x-x_0)^T d\rho\right) v}{\rho(A)v^T v} +  \frac{v^T \left( \int_{A\setminus B} (x-x_0)(x-x_0)^T d\rho\right) v}{\rho(A)v^T v} \right)
\\
& \ge \min_{v \in V}  \frac{v^T \left( \int_{B} (x-x_0)(x-x_0)^T d\rho\right) v}{\rho(A)v^T v} = \frac{\rho(B)}{\rho(A)} \lambda_{d}(\mathrm{cov}(\rho|_B)).
\end{align*}
\end{proof}

\commentout{
\begin{proof}[proof of Proposition \ref{propcovertree2}]
Claim (A1) follows by a simple volume argument: $\Cjk$ is contained in a ball of radius $3\cdot2^{-j}$, and therefore has volume at most $C_1(3\cdot2^{-j})^d$, and each child contains a ball of radius $2^{-(j+1)}/4$, and therefore volume at least $C_1^{-1}(2^{-(j+1)}/4)^d$. It follows that $\amax$ cannot be larger than $C_1^2(3\cdot2^{-j}/2^{-(j+1)}\cdot4)^d$. Clearly $\amin\ge1$ since every $a_{j,k}$ belongs to both $T_j(\calX_n')$ and $T_{j'}(\calX_n')$ with $j'\ge j$.
The other statements, except the last one, are straightforward consequences of the doubling assumption and Lemma \ref{propcovertree1}.

In order to prove the last statement about property (A5), observe that $B_{\frac{2^{-j}}4}(a_{j,k})\subseteq\rho(C_{j,k})\subseteq B_{3\cdot 2^{-j}}(a_{j,k})$. By Lemma \ref{l:covcompare} we have 
\begin{align*}
\frac{(2^{-j}/4)^d}{\rho(C_{j,k})}\lambda_d(\mathrm{cov}(\rho|_{B_{\frac{2^{-j}}4}(a_{j,k})})
\le 
\lambda_d(\mathrm{cov}(\rho|_{C_{j,k}})
\le 
\frac{\rho(C_{j,k})}{(3\cdot 2^{-j})^d}\lambda_d(\mathrm{cov}(\rho|_{B_{3\cdot 2^{-j}}(a_{j,k})})
\end{align*}
and therefore $\lambda_d(\mathrm{cov}(\rho|_{C_{j,k}})\ge (1/12)^d\lambda_d(\mathrm{cov}(\rho|_{B_{\frac{2^{-j}}4}(a_{j,k})})\ge (1/12)^d \theta_3(2^{-j}/4)^2/d$, so that (A5)-(i) holds with $\tilde\theta_3=\theta_3(1/12)^d/16$.
Proceeding similarly for $\lambda_{d+1}$, we obtain from the upper bound above that $\lambda_{d+1}(\mathrm{cov}(\rho|_{C_{j,k}})\le\frac{(3\cdot 2^{-j})^d}{(2^{-j}/4)^d}\lambda_{d+1}(\mathrm{cov}(\rho|_{B_{3\cdot 2^{-j}}(a_{j,k})})\le(12^d)^2\cdot144\theta_4\lambda_d(\mathrm{cov}(\rho|_{C_{j,k}})$ so that (A5)-(ii) holds with $\tilde\theta_4=(12^d)^2\cdot144\theta_4$, which is less than $1/2$ for $\theta_4$ small enough.
\end{proof}
 }

\subsection{Regularity of geometric spaces}
\label{s:AppRegSpaces}

To fix the ideas, consider the case where $\mathcal{M}$ is a manifold of class $\mathcal{C}^{s}$, $s \in \RR^+ \setminus \ZZ$, i.e. around every point $x_0$ there is a neighborhood $U_{x_0}$ that is parametrized by a function $f: V\rightarrow U_{x_0}$, where $V$ is an open connected set of $\mathbb{R}^d$, and $f\in\mathcal{C}^{s}$, i.e. $f$ is $\fs$ times continuously differentiable and the $\fs$-th derivative $f^{\fs}$ is H\"older continuous of order $s-\fs$, i.e. $||f^{\fs}(x)-f^{\fs}(y)||\le ||f^{\fs}||_{\calC^{s-\fs}} ||x-y||^{s-\fs}$.
In particular, for $s\in(0,1)$, $f$ is simply a H\"older function of order $s$. 
For simplicity we assume $x =f(x^d)$ where $x^d \in V$.

If $\calM$ is a manifold of class $\calC^s, s \in (0,1)$, a constant approximation of $f$ on a set $I$ by the value $x_0 :=f(x_0^d)$ on such set yields
\begin{align*}
\frac{1}{\rho(I)}\int_I | f(x^d)- f(x_0^d) |^2d\rho(x)
\le \frac{1}{\rho(I)}\int_I || x^d-x^d_0 ||^{2 s} ||f||_{\mathcal{C}^s}^2 d\rho(x)
\le ||f||_{\mathcal{C}^s}^2 \diam(I)^{2s}
\end{align*}
where we used continuity of $f$.
If $I$ was a ball, we would obtain a bound which would be better by a multiplicative constant no larger than $1/d$.
Moreover, the left hand side is minimized by the mean $\frac{1}{\rho(I)}\int_I f(y)d\rho(y)$ of $f$ on $I$, and so the bound on the right hand side holds a fortiori by replacing $f(x_0^d)$ by the mean.


Next we consider the linear approximation of $\calM$ on $I \subset \calM$. Suppose there exits $\theta_0, \theta_2$ such that $I$ is contained in a ball of radius $\theta_2 r$ and contains a ball of radius $\theta_0 r$.
Let $x_0 \in I$ be the closest point on $I$ to the mean. Then $I$ is the graph of a $\calC^s$ function $f$: $P_{T_{x_0}(I)} \rightarrow P_{T^\perp_{x_0}(I)}$ where $T_{x_0}(I)$ is the plane tangent to $I$ at $x_0$ and $T^\perp_{x_0}(I)$ is the orthogonal complement of $T_{x_0}(I)$.
Since all the quantities involved are invariant under rotations and translations, up to a change of coordinates, we may assume $x^d = (x_1,\ldots,x_d)$ and $f=(f_1,\ldots,f_{D-d})$ where $f_i := f_i(x^d),\ i=d+1,\ldots,D.$
A linear approximation of $f=(f_{d+1},\dots,f_D)$ based on Taylor expansion and an application of the mean value theorem yields the error estimates.

\begin{itemize}
\item Case 1: $s \in (1,2)$
\begin{align*}
&\frac{1}{\rho(I)}\int_{I} \left\| f(x^d)- f(x^d_0)-\nabla f(x^d_0)\cdot(x^d-x^d_0)\right\|^2d\rho 
\\ 
=&\sum_{i=d+1}^D 
\frac{1}{\rho(I)}\sup_{\xi_i \in {\rm domain}(f_i)}\int_{I} \left | \nabla f_i(\xi_i)(x^d-x^d_0)-\nabla f_i(x^d_0)\cdot(x^d-x^d_0)\right|^2d\rho 
\\
\le& \sum_{i=d+1}^D\frac{1}{\rho(I)}\sup_{\xi_i \in {\rm domain}(f_i)}\int_{\Cjk} || x^d -x^d_0 ||^{2} \|\xi_i -x^d_0\|^{2(s-\fs)} ||\nabla f_i||_{\mathcal{C}^{s-\fs}}^2 d\rho
\\
\le & D \max_{i=1,\ldots,D-d} ||\nabla f_i||_{\mathcal{C}^{s-\fs}}^2 \diam(I)^{2s}\,.
\end{align*}

\item Case 2: $s=2$
\begin{align*}
&\frac{1}{\rho(I)}\int_{I} \left\| f(x^d)- f(x^d_0)-\nabla f(x^d_0)\cdot(x^d-x^d_0)\right\|^2d\rho 
 \\ 
=&\sum_{i=d+1}^D \frac{1}{\rho(I)} \int_{I}\left\| f_i(x^d)- f_i(x^d_0)-\nabla f_i(x^d_0)\cdot(x^d-x^d_0)\right\|^2d\rho  
 \\
\le& \sum_{i=d+1}^D\frac{1}{\rho(I)}\sup_{\xi_i \in {\rm domain}(f_i)} \int_{I} 
\left\|  \frac 1 2 (\xi_i -x_0^d)^T D^2 f_i |_{x_0^d} (\xi_i - x_0^d) + o(\|\xi_i-x_0^d\|^2)\right\|^2 d\rho
 \\
\le &  \frac D 2 \max_{i=1,\ldots,D-d} \|D^2 f_i\| \diam(I)^{4}+ o(2^{-4j}).
\end{align*}

\end{itemize}

$\calM$ does not have boundaries, so the Taylor expansion in the computations above can be performed on the convex hull of $P_{T_{x_0}(I)}$, whose diameter is no larger than $ \diam(I)$.
Note that this bound then holds for other linear approximations which are at least as good, in $L^2(\rho|_{I})$, as Taylor expansion. One such approximation is, by definition, the linear least square fit of $f$ in $L^2(\rho|_{I})$. Let $L_{I}$ be the least square fit to the function $x\mapsto f(x)$. Then
\begin{align}
&\sum_{i=d+1}^D \lambda_{i}(\mathrm{cov}(\rho|_{I}))^2
 = \frac{1}{\rho(I)}\int_{I} || f(x) -L_{I}(x)||^2d\rho(x) \nonumber
\\
& \le
\left\{\begin{array}{ll}
          D\max_{i=1,\ldots,D-d} ||\nabla f_i||_{\mathcal{C}^{s-\fs}}^2 {\diam(I)}^{2s}, & s \in (1,2) \\
        \frac D 2 \max_{i=1,\ldots,D-d} \|D^2 f_i\| {\diam(I)}^{4}, & s=2
                \end{array}   \right. . 
                \label{eqlambdad1}
\end{align}

\begin{proof}[Proof of Proposition \ref{propcovertree2}]
Claim (A1) follows by a simple volume argument: $\Cjk$ is contained in a ball of radius $3\cdot2^{-j}$, and therefore has volume at most $C_1(3\cdot2^{-j})^d$, and each child contains a ball of radius $2^{-(j+1)}/4$, and therefore volume at least $C_1^{-1}(2^{-(j+1)}/4)^d$. It follows that $\amax \le C_1^2(3\cdot2^{-j}/2^{-(j+1)}\cdot4)^d$. Clearly $\amin\ge1$ since every $a_{j,k}$ belongs to both $T_j(\calX_n')$ and $T_{j'}(\calX_n')$ with $j'\ge j$. (A1),(A3), (A4) are straightforward consequences of the doubling assumption and Lemma \ref{propcovertree1}.
As for (A2), for any $\nu>0$, we have 
\begin{align*}
&
\PP\left\{\rho(\calM\setminus \hcalM) > \frac{28\nu\log n}{3n}\right\}
= \PP\left\{\hrho(\calM\setminus \hcalM)=0 \text{ and }\rho(\calM\setminus \hcalM) > \frac{28\nu\log n}{3n}\right\}
\\
&
\le \PP\left\{ |\hrho(\calM\setminus \hcalM)-\rho(\calM\setminus \hcalM)| >\frac 1 2 \rho(\calM\setminus \hcalM)
\text{ and }
\rho(\calM\setminus \hcalM) > \frac{28\nu\log n}{3n}
\right\}
\\
&
\le 2e^{-\frac{3}{28}n\rho(\calM\setminus \hcalM) } 
\le 2 n^{-\nu}.
\end{align*}

In order to prove the last statement about property (A5) in the case of 5a, observe that $B_{{2^{-j}}/4}(a_{j,k})\subseteq \Cjk \subseteq B_{3\cdot 2^{-j}}(a_{j,k})$. By Lemma \ref{l:covcompare} we have 
\begin{align*}
\frac{C_1^{-1}(2^{-j}/4)^d}{\rho(C_{j,k})}\lambda_d(\mathrm{cov}(\rho|_{B_{{2^{-j}}/4}(a_{j,k})})
\le 
\lambda_d(\mathrm{cov}(\rho|_{C_{j,k}})
\le 
\frac{C_1(3\cdot 2^{-j})^d}{\rho(\Cjk)}\lambda_d(\mathrm{cov}(\rho|_{B_{3\cdot 2^{-j}}(a_{j,k})})
\end{align*}
and therefore $\lambda_d(\mathrm{cov}(\rho|_{C_{j,k}})\ge C_1^{-2} (1/12)^d\lambda_d(\mathrm{cov}(\rho|_{B_{{2^{-j}}/4}(a_{j,k})})\ge C_1^{-2}(1/12)^d \ttheta_3(2^{-j}/4)^2/d$, so that (A5)-(i) holds with $\theta_3=\ttheta_3 (4C_1)^{-2}(1/12)^d$.
Proceeding similarly for $\lambda_{d+1}$, we obtain from the upper bound above that $$\lambda_{d+1}(\mathrm{cov}(\rho|_{C_{j,k}})\le\frac{C_1(3\cdot 2^{-j})^d}{C_1^{-1}(2^{-j}/4)^d}\lambda_{d+1}(\mathrm{cov}(\rho|_{B_{3\cdot 2^{-j}}(a_{j,k})})\le(12^d)^2\cdot144 C_1^4\ttheta_4 /\ttheta_3\lambda_d(\mathrm{cov}(\rho|_{C_{j,k}})$$ so that (A5)-(ii) holds with $\theta_4=(12^d)^2\cdot144 C_1^4\ttheta_4/\ttheta_3$. 

In order to prove (A5) in the case of 5b, we use calculations as in \citet{LMR:MGM1,MMS:NoisyDictionaryLearning} where one obtains that the first $d$ eigenvalues of the covariance matrix of $\rho|_{B_{r}(z)}$ with $z \in \calM$, is lower bounded by $\tilde\theta_3 r^2/d$ for some $\tilde\theta_3 >0$. Then
(A5)-(i) holds for $\Cjk$ with $\theta_3=\ttheta_3 (4C_1)^{-2}(1/12)^d$.
The estimate of $\lambda_{d+1}(\mathrm{cov}(\rho|_{C_{j,k}}))$ follows from \eqref{eqlambdad1} such that 
$$
\sum_{i=d+1}^D \lambda_{i}(\mathrm{cov}(\rho|_{C_{j,k}}))^2
 \le
 \left\{\begin{array}{ll}
          D\max_{i=1,\ldots,D-d} ||\nabla f_i||_{\mathcal{C}^{s-\fs}}^2 (6 \cdot 2^{-j})^{2s}, & s \in (1,2)\\
        \frac D 2 \max_{i=1,\ldots,D-d} \|D^2 f_i\| (6 \cdot 2^{-j})^4, & s=2
                \end{array}   \right. .            %
$$
Therefore, there exists $j_0$ such that $\lambda_{d+1}(\mathrm{cov}(\rho|_{C_{j,k}})) < \theta_4 \lambda_{d}(\mathrm{cov}(\rho|_{C_{j,k}}))$ when $j \ge j_0$.
The calculation above also implies that $\rho \in \calA_{s}^\infty$ if $\max_{i=1,\ldots,D-d} ||\nabla f_i||_{\mathcal{C}^{s-\fs}}^2$ for $s\in (1,2)$ or $ \max_{i=1,\ldots,D-d} \|D^2 f_i\|$ for $s=2$ is uniformly upper bounded.

\end{proof}

\subsection{An alternative tree construction method}
\label{appatree}

The $\{\Cjk\}$ constructed by Algorithm \ref{TableCoverTree} is proved to satisfy Assumption (A1-A5). In numerical experiments, we use a much simpler algorithm to construct $\{\Cjk\}$ as follows:
$$C_{j_{\max},k} = \Voronoi(a_{j_{\max},k},T_{j_{\max}}(\calX_n')) \cap B_{2^{-j_{\max}}}(a_{j_{\max},k}),$$
and for any $j<j_{\max}$, we define 
$\Cjk = \bigcup\limits_{\substack{a_{j-1,k'} \text{ child of } a_{j,k}}} C_{j-1,k'}.$

We observe that the vast majority of $\Cjk$'s constructed above satisfy Assumption (A1-A5) in our numerical experiments. In Fig. \ref{FigTheta34},  we will show that (A5) is satisfied when we experiment on volume measures on the $3$-dim S and Z manifold. Here we sample $10^5$ training data, perform multiscale tree decomposition as stated above, and compute $\theta_3^{j,k}, \theta_4^{j,k}$ at every $\Cjk$. In Fig. \ref{FigTheta34}, we display the mean of $\{\theta_3^{j,k}\}_{k \in \calK_j}$ or $\{\theta_4^{j,k}\}_{k \in \calK_j}$ versus scale $j$, with a vertical error bar representing the standard deviation of $\{\theta_3^{j,k}\}_{k \in \calK_j}$ or $\{\theta_4^{j,k}\}_{k \in \calK_j}$ at each scale. We observe that $\theta_3 = \min_{j,k}\theta_3^{j,k} \ge 0.05$ at all scales and $\theta_4 = \max_{j,k}\theta_4^{j,k} \le 1/2$ except at very coarse scales, which  demonstrates Assumption (A5) is satisfied here. Indeed $\theta_4$ is not only bounded, but also decreases from coarse scales to fine scales.

\begin{figure}[ht]
 \centering
 \subfigure[$\{\theta_3^{j,k}\}_{k \in \calK_j}$ versus $j$ of the $3$-dim S manifold]{
    \includegraphics[width=.22\textwidth]{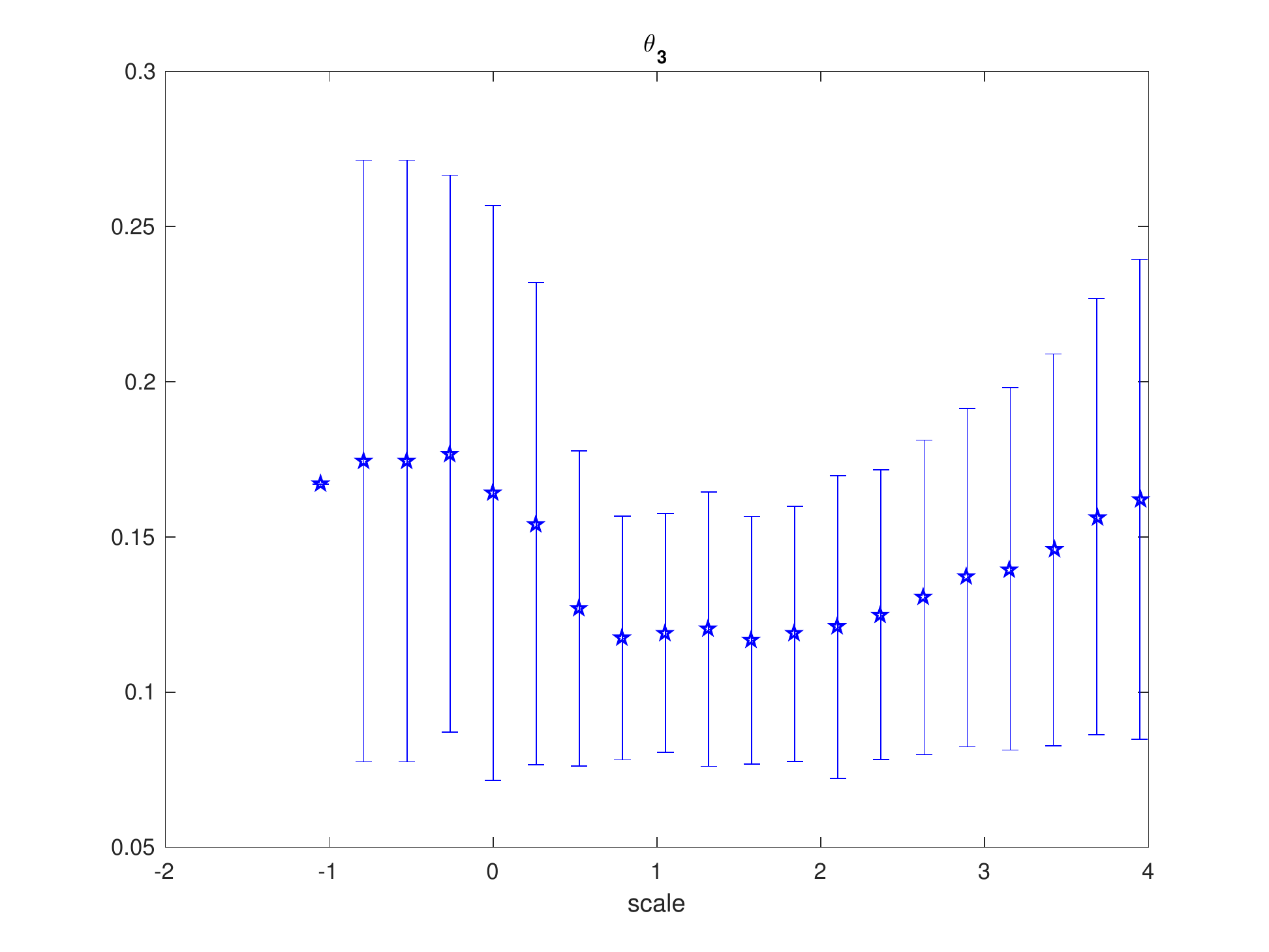}
    }
   \hspace{-0.1cm}
    \subfigure[$\{\theta_4^{j,k}\}_{k \in \calK_j}$ versus $j$ of the $3$-dim S manifold]{
    \includegraphics[width=.22\textwidth]{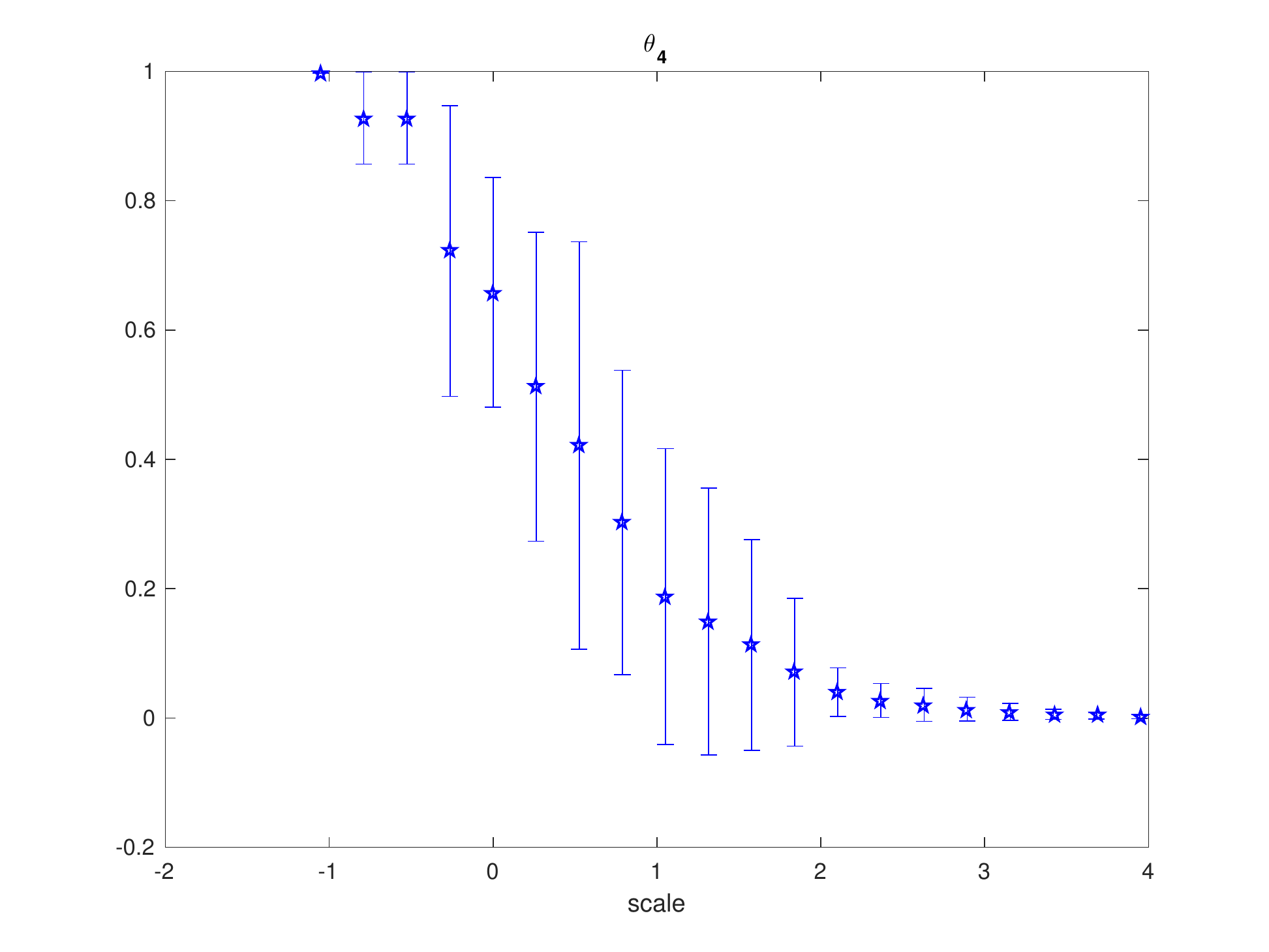}
    }
       \hspace{-0.1cm}
    \subfigure[$\{\theta_3^{j,k}\}_{k \in \calK_j}$ versus $j$ of the $3$-dim Z manifold]{
     \includegraphics[width=.22\textwidth]{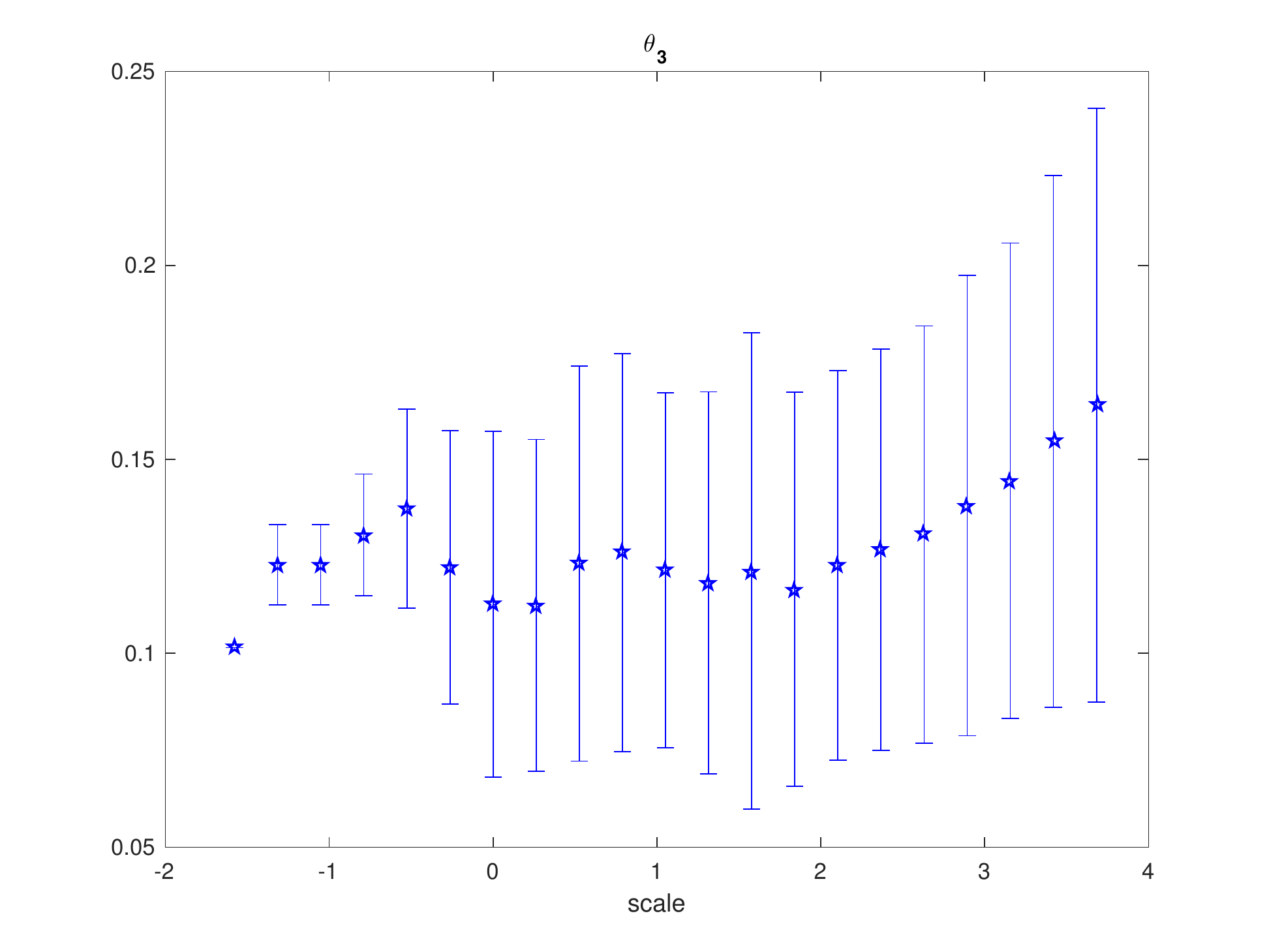}
     }
        \hspace{-0.1cm}
     \subfigure[$\{\theta_4^{j,k}\}_{k \in \calK_j}$ versus $j$ of the $3$-dim Z manifold]{
    \includegraphics[width=.22\textwidth]{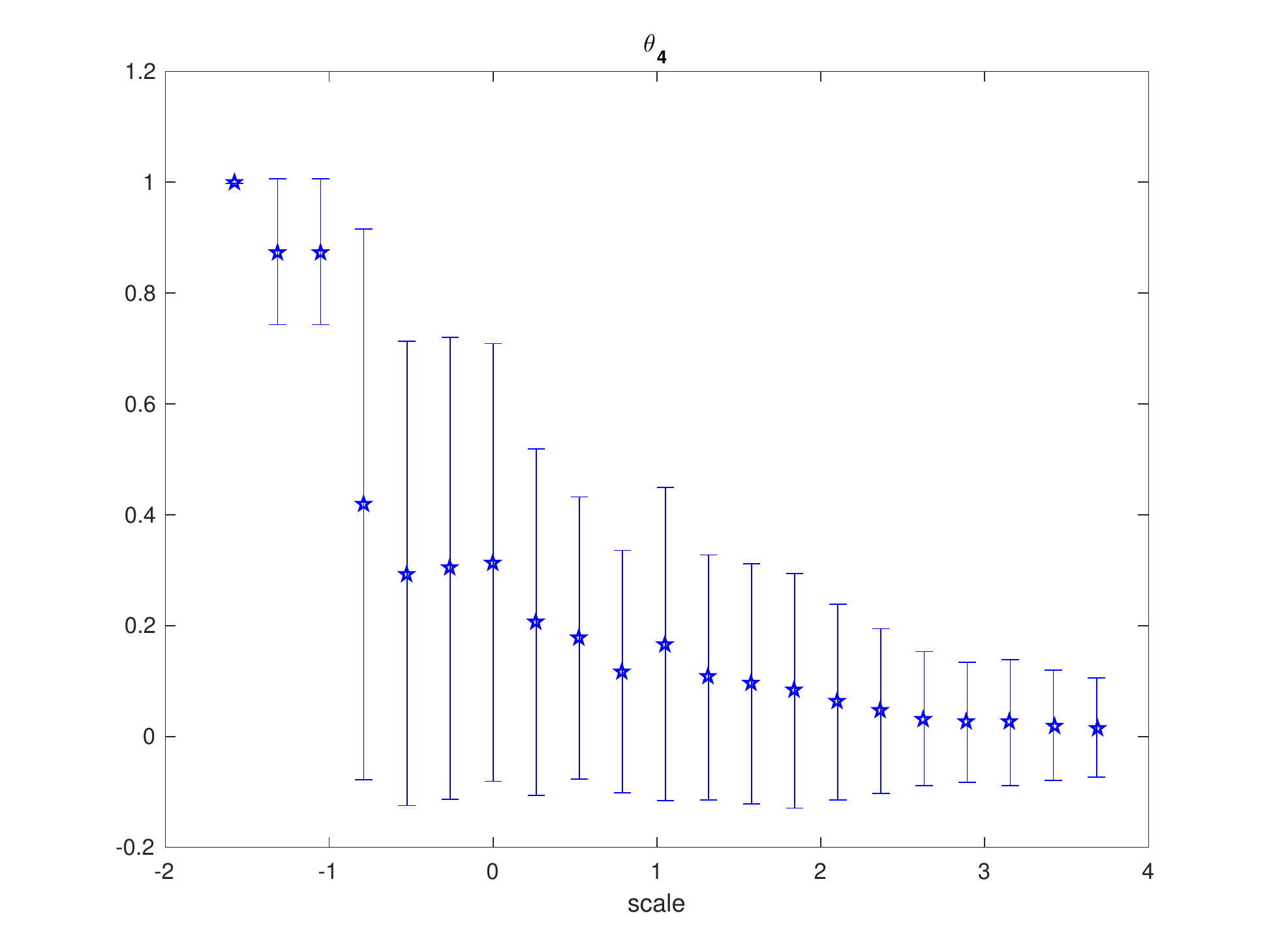}
    }
         \caption{
         The mean of $\{\theta_3^{j,k}\}_{k \in \calK_j}$ or $\{\theta_4^{j,k}\}_{k \in \calK_j}$ versus scale $j$, with a vertical error bar representing the standard deviation of $\{\theta_3^{j,k}\}_{k \in \calK_j}$ or $\{\theta_4^{j,k}\}_{k \in \calK_j}$ at each scale.
               }    
  \label{FigTheta34}
\end{figure}

We observe that, every $\Cjk$ constructed above is contained in a ball of radius $\theta_2 2^{-j}$ and contains a ball of radius $\theta_0 2^{-j}$, with $\theta_2 / \theta_0 \in [1, 2]$ for the majority of $\Cjk$'s.
 In Fig. \ref{FigCovertree}, we take the volume measures on the $3$-dim S and Z manifold, and plot $\log_2$ of the outer-radius and the statistics a lower bound for the in-radius\footnote{The in-radius of $\Cjk$ is approximately computed as follows: we randomly pick a center, and evaluate the largest radius with which the ball contains at least $95\%$ points from $\Cjk$. This procedure is repeated for two centers, and then we pick the maximal radius as an approximation of the in-radius.} versus the scale of cover tree. 
Notice that the in-radius is a fraction of the outer-radius at all scales, and $\log_2\theta_2 - \log_2 \theta_0 \le 1$ for the majority of cells.

\begin{figure}[ht]
 \centering
    \includegraphics[width=0.15\textwidth]{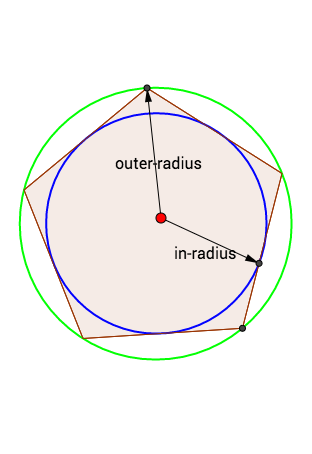}
    \hspace{-0.3cm}
    \includegraphics[width=.35\textwidth]{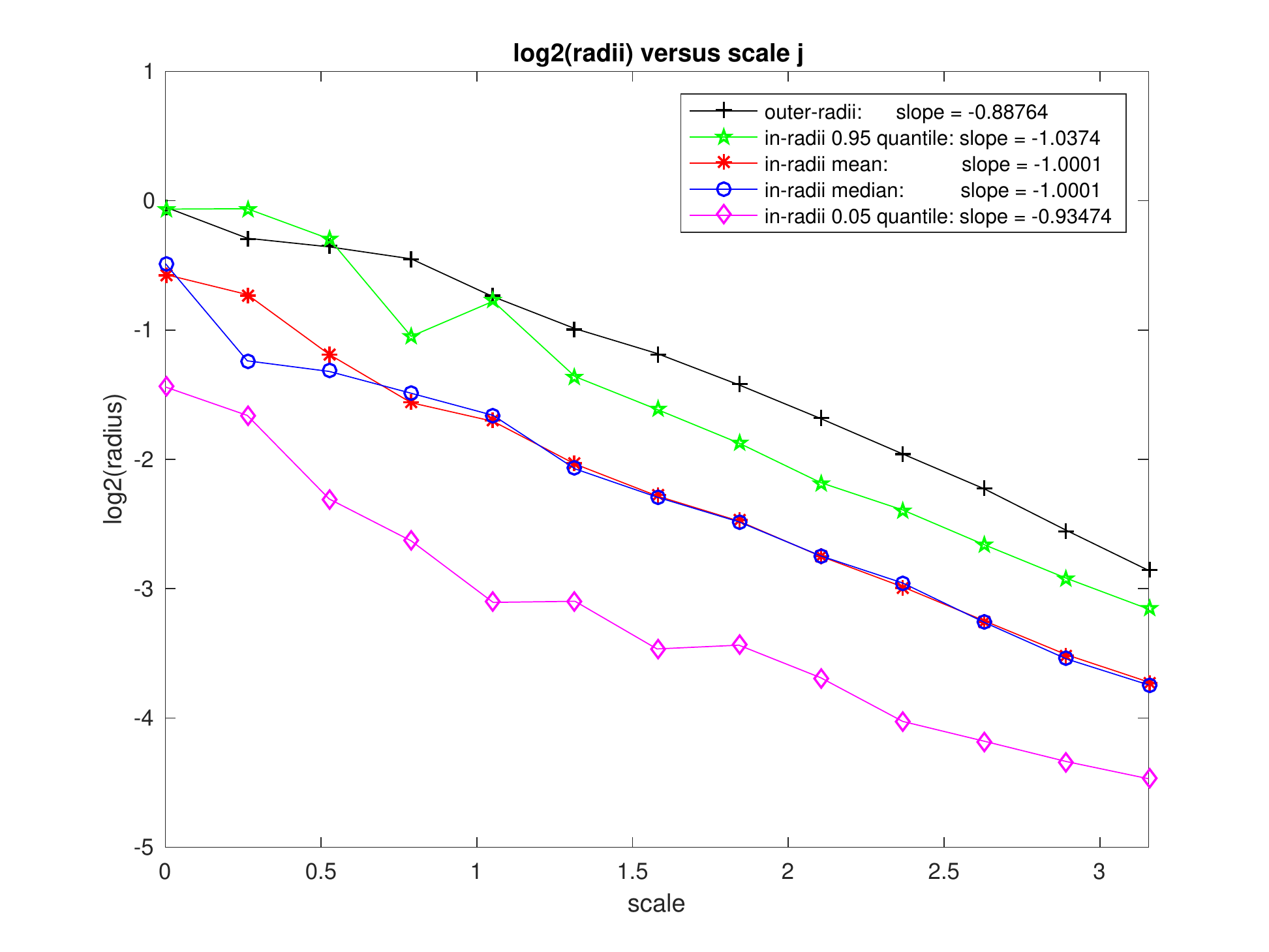}
        \hspace{-0.5cm}
    \includegraphics[width=.35\textwidth]{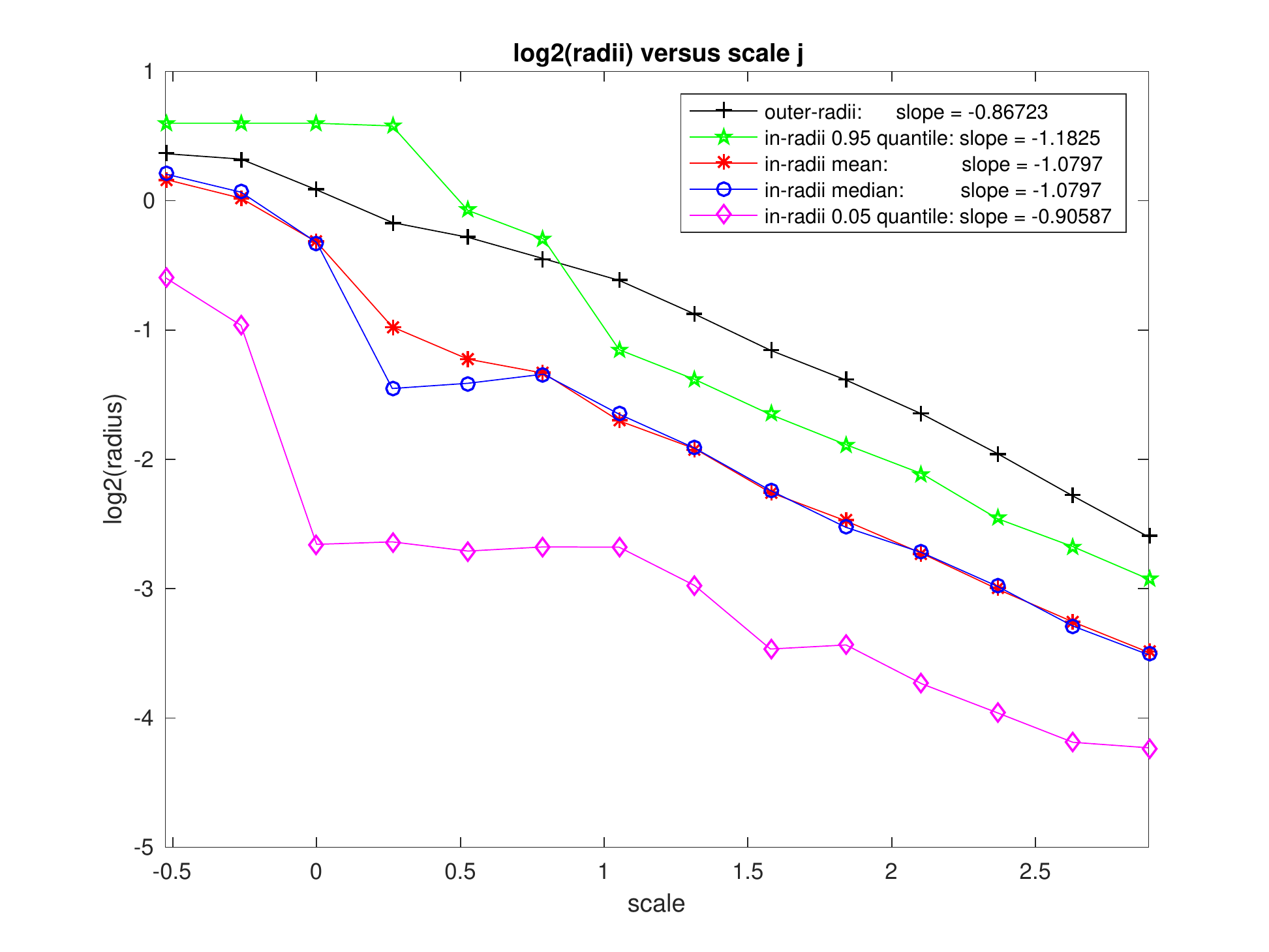} 
    \begin{minipage}{0.14\textwidth}
    \vspace{-1.8\textwidth}
     \includegraphics[height=0.8\textwidth]{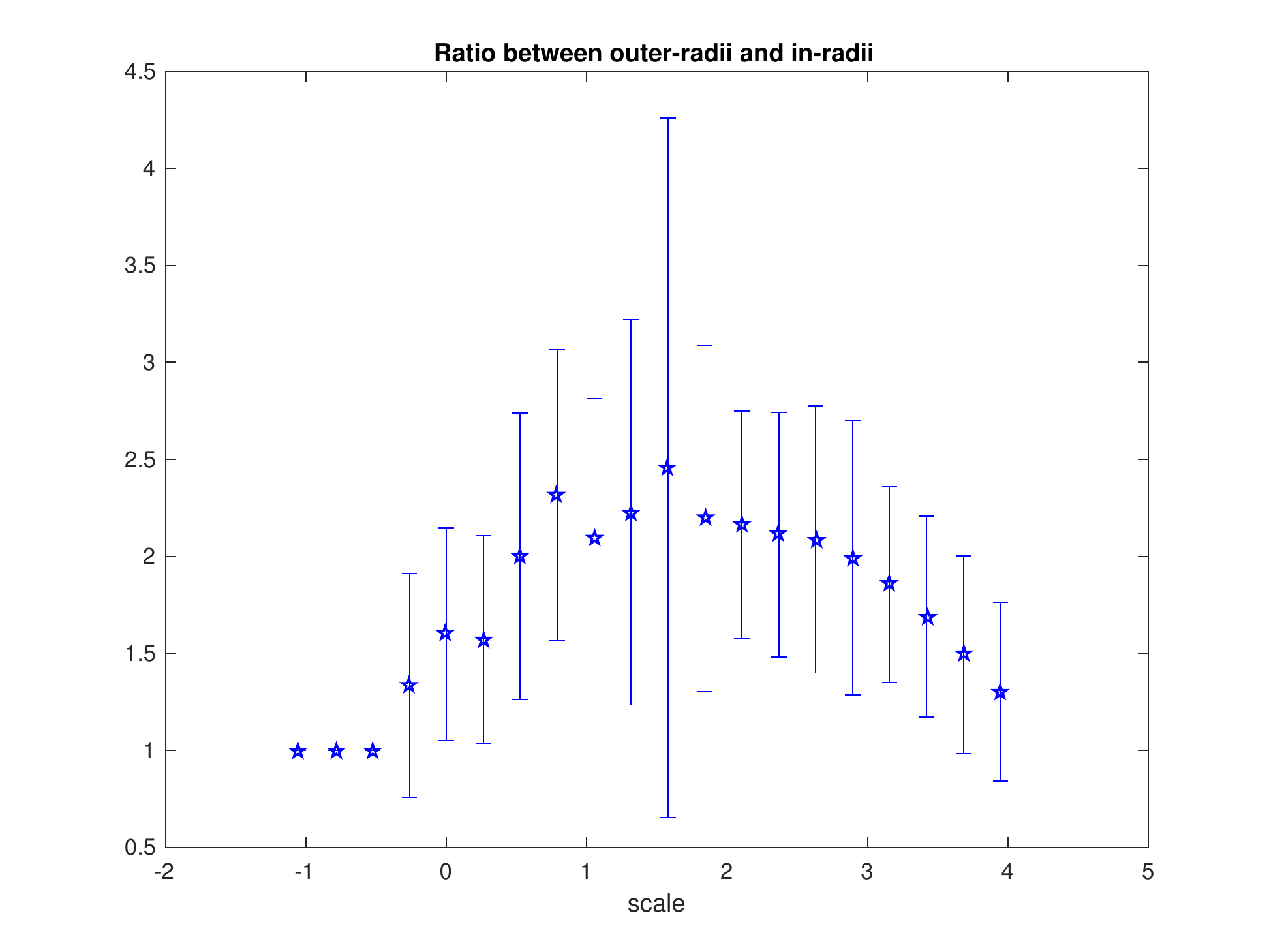}   
      \includegraphics[height=0.8\textwidth]{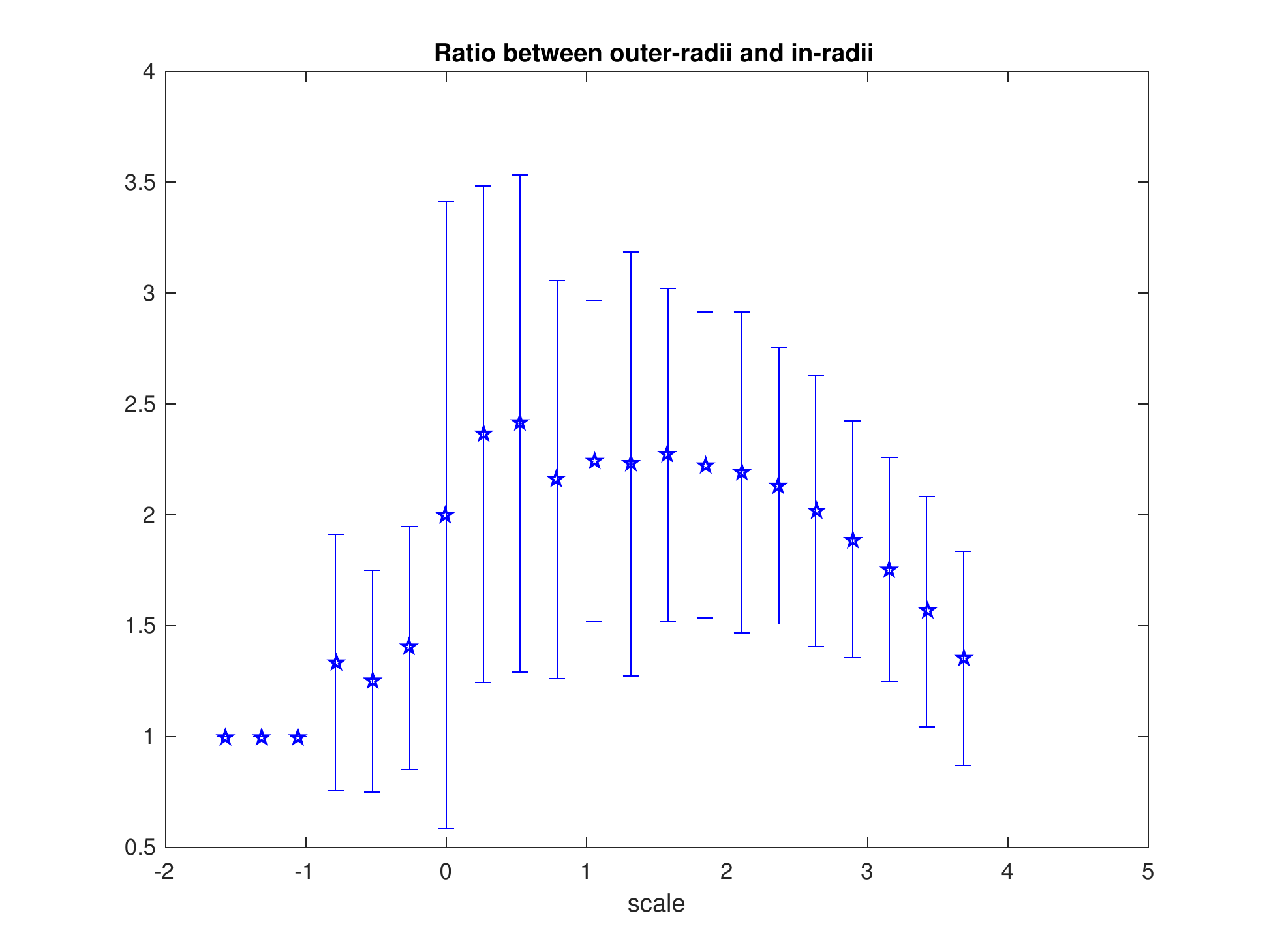}     
      \end{minipage}
      \caption{From left to right: the in-radius and outer-radius of a pentagon;  $\log_2$ of the outer-radius and the statistics of the in-radius versus the scale of cover tree for the $3$-dim S manifold, and then the same plot for the $3$-dim Z manifold; ratio between outer-radii and in-radii, for the $3$-dim S manifold top) and the $3$-dim Z manifold (bottom).
               }    
  \label{FigCovertree}
\end{figure}

\subsection{$\AS^\infty \subset \BS$}
\label{appasinf}

\begin{proof}[proof of Lemma \ref{l:AsinBs}]
Assume $\rho(\Cjk) \asymp 2^{-jd}$.
According to Definition \ref{defAsInf}, $\rho\in \AS^\infty$ if $\| (X - \calPjk  X)\chijk\| \le |\rho|_{\AS^\infty} 2^{-js} \sqrt{\rho(\Cjk)}, \ \forall k \in \calK_j, j \ge \jmin$, which implies
$\Deltajk \le 2 |\rho|_{\AS^\infty} 2^{-js} \sqrt{\rho(\Cjk)} \lesssim |\rho|_{\AS^\infty} 2^{-j(s+\frac d 2)}$. 

Let $\eta >0$ and $\calTrhoeta$ be the smallest proper subtree of $\calT$ that contains all $\Cjk$ for which $\Deltajk \ \ge 2^{-j}\eta$. All the nodes satisfying $\Deltajk \ge 2^{-j}\eta$ will satisfy $|\rho|_{\AS^\infty} 2^{-j(s+\frac d 2)} \gtrsim 2^{-j} \eta$ which implies $2^{-j} \gtrsim ({\eta}/{|\rho|_{\AS^\infty}})^{\frac{2}{2s+d-2}}$. Therefore, the truncated tree $\calTrhoeta$ is contained in $\calT_{j^*} = \cup_{j \le j^*} \Lam_j$ with $2^{-j^*} \asymp (\eta/|\rho|_{\AS^\infty})^{\frac{2}{2s+d-2}}$, so the entropy of $\calTrhoeta$ is upper bounded by the entropy of $\calT_{j^*}$, which is $\sum_{j \le j^*} 2^{-2j}\#\Lam_j \asymp 2^{j^* (d-2)} \asymp (\eta/|\rho|_{\AS^\infty})^{-\frac{2(d-2)}{2s+d-2}}$. Then $\rho\in\calB_{s}$ and $|\rho|_{\BS} \lesssim |\rho|_{\AS^\infty}$ according to Definition \ref{defBs}.
\end{proof}



\commentout{
\subsection{Regularity of $\Cjk$}

In Proposition \ref{propcovertree1}, we prove that every $\Cjk$ constructed by the cover tree algorithm is contained in a ball of radius $2\cdot 2^{-j}$ on $\calM$. The ideal $\Cjk$ should also contain a ball of radius $\theta_0 2^{-j}$ on $\calM$ for some $\theta_0 >0$.
In numerical experiments, we observe that if $\rho$ is doubling, then the majority of the $\Cjk$'s constructed by the cover tree algorithm satisfy this condition. In Fig. \ref{FigCovertree}, we take the volume measures on the $3$-dim S and Z manifold, and plot the outer-radius and the statistics of the in-radius 
\footnote{The in-radius of $\Cjk$ is approximately computed as follows: we randomly pick a center, and evaluate the largest radius with which the ball contains at least $95\%$ points from $\Cjk$. This procedure is repeated for two centers, and then we pick the maximal radius as an approximation of the in-radius.} 
in the $\log_2$ scale versus the scale of cover tree. Our results show that the in-radius is a fraction of the outer-radius at all scales.

\begin{figure}[hthp]
 \centering
    \includegraphics[width=3.4cm]{Fig23_SandZ/Sample100000/V2/InOuterRadii.png}
    \hspace{-0.4cm}
    \includegraphics[clip,trim=20 5 25 30,width=.4\textwidth]{Fig23_SandZ/Sample100000/V2/InRadii_SDim3-eps-converted-to.pdf}
        \hspace{-0.5cm}
    \includegraphics[clip,trim=20 5 25 30,width=.4\textwidth]{Fig23_SandZ/Sample100000/V2/InRadii_ZDim3-eps-converted-to.pdf}                   \caption{Left: the in-radius and outer-radius of a pentagon;  Middle: the outer-radius and the statistics of the in-radius in the $\log_2$ scale versus the scale of cover tree for the $3$-dim S manifold; Right: the same plot for the $3$-dim Z manifold.
               }    
  \label{FigCovertree}
\end{figure}
}

\section{S manifold and Z manifold}
\label{AppSZ}

We consider volume measures on the $d$ dimensional S manifold and Z manifold whose $x_1$ and $x_2$ coordinates are on the S curve and Z curve in Figure \ref{FigASBS} (a) and $x_i, i=3,\ldots,d+1$ are uniformly distributed in $[0,1]$.

\subsection{S manifold}
\label{AppS}

Since S manifold is smooth and has a bounded curvature, the volume measure on the S manifold is in $\calA_2^\infty$.
Therefore, the volume measure on the S manifold is in $\calA_2$ and $\calB_2$ when $d \ge 3$.

\subsection{Z manifold}
\label{AppZ}

\subsubsection{The volume on the Z manifold is in $\calA_{1.5}$}
\label{appZAs}

The uniform distribution on the $d$ dimensional Z manifold is in $\calA_1$ at two corners and satisfies $\|( X - \calPjk X)\chijk\| = 0$ when $\Cjk$ is away from the corners.
There exists $A_0>0$ such that $\|(X - \calPjk X)\chijk\| \le A_0 2^{-j}\sqrt{\rho(\Cjk)}$ when $\Cjk$ intersects with the corners. At scale $j$, there are about $2^{jd}$ cells away from the corners and there are about $2^{j(d-1)}$ cells which intersect with the corners. As a result,
\begin{align*}
\|X-\calP_j X\| & \le \mathcal{O}\left( \sqrt{2^{jd} \cdot 0 \cdot 2^{-jd} + 2^{j(d-1)} \cdot 2^{-2j} \cdot 2^{-jd}}\right  ) = \mathcal{O}(2^{-1.5j}),
\end{align*} 
so the volume measure on Z manifold is in $\calA_{1.5}$.

\subsubsection{Model class $\BS$}
\label{appZBs}
Assume $\rho(\Cjk) \asymp 2^{-jd}$.
We compute the regularity parameter $s$ in the $\BS$ model class when $d \ge  3$. It is easy to see that $\Deltajk = 0$ when $\Cjk$ is away from the corners and $\Deltajk \le 2 A_0 2^{-j} \sqrt{\rho(\Cjk)} \lesssim 2^{-j(\frac d 2 +1)}$ when $\Cjk$ intersects with the corners. Given any fixed threshold $\eta > 0$, in the truncated tree $\calTrhoeta$, the parent of the leaves intersecting with the corners satisfy $2^{-j(\frac d 2 +1)} \gtrsim 2^{-j} \eta $. In other words, at the corners the tree is truncated at a scale coarser than $j^*$ such that $2^{-j^*} = \mathcal{O} (\eta^{\frac 2 d})$. Since $\Deltajk = 0$ when $\Cjk$ is away from the corners, the entropy of $\calTrhoeta$ is dominated by the nodes intersecting with the corners whose cardinality is $2^{j(d-1)}$ at scale $j$. Therefore
\begin{align*}
\text{Entropy of } \calTrhoeta \lesssim 
 \sum_{j \le j^*} 2^{-2j} 2^{j(d-1)} 
 = \mathcal{O} \left( \eta^{-\frac{2(d-3)}{d}}\right),
\end{align*}
which implies that $p \le \frac{2(d-3)}{d}$ and $s \ge \frac{3(d-2)}{2(d-3)} > 1.5$. 

Then we study the relation between the error $\|X-\calP_{\Lamrhoeta}X\|$ and the partition size $\#\Lamrhoeta$, which is numerically verified in Figure \ref{FigDelta}. Since all the nodes in $\calTrhoeta$ that intersect with corners are at a scale coarser than $j^*$, $\#\Lamrhoeta \approx 2^{j^* (d-1)} \asymp \eta^{-\frac{2(d-1)}{d}}$. Therefore, $\eta \lesssim [\#\Lamrhoeta]^{-\frac{d}{2(d-1)}}$ and $$\|X-\calP_{\Lamrhoeta}X\| \lesssim \eta^{\frac{2-p}{2}} = \eta^{\frac{2s}{2s+d-2}} \lesssim [\#\Lamrhoeta]^{-\frac{2sd}{2(d-1)(2s+d-2)}} =  [\#\Lamrhoeta]^{-\frac{3}{2(d-1)}}.$$ 


\section{Proofs of Lemma \ref{thm1} and Proposition \ref{lemma0}}
\label{app1}

\commentout{
We first recall two Bernstein inequalities from \cite{Tropp:userfriendlyRMT}. The first inequality is an exponential inequality for the spectral norm of a sum of independent random vectors.

\begin{proposition}[A rank one case of Corollary 7.3.2 in \cite{Tropp:userfriendlyRMT}]
\label{prop1}
Let $\xi_1,\ldots,\xi_n$ be independent random vectors in $\RR^D$ that satisfy
$$\EE\xi_i = \mathbf{0}\ \text{and}\ \|\xi_i \|\le R, \ i=1,\ldots,n.$$
Form the mean $\xi = \frac{1}{n}\sum_{i=1}^n \xi_i,$ and introduce a variance parameter
$$\sigma^2 = \|\EE(\xi^T\xi)\|.$$
Then 
$$\PP\{\|\xi\|\ge t\} \le 8e^{-\frac{t^2/2}{\sigma^2+Rt/3}} \ \text{ for all } t\ge \sigma + R/3.$$
\end{proposition}
}

\subsection{Concentration inequalities}

We first recall a Bernstein inequality from \citet{Tropp:userfriendlyRMT}
which is an exponential inequality to estimate the spectral norm of a sum independent random Hermitian matrices of size $D\times D$. It features the dependence on an intrinsic dimension parameter which is usually much smaller than the ambient dimension $D$. For a positive-semidefinite matrix $A$, the intrinsic dimension is the quantity
$${\rm intdim}(A) = \frac{{\trace}(A)}{\|A\|}.$$

\begin{proposition}[Theorem 7.3.1 in \cite{Tropp:userfriendlyRMT}]
\label{prop2}
Let $\xi_1,\ldots,\xi_n$ be $D\times D$ independent random Hermitian matrices that satisfy
$$\EE\xi_i = 0 \ \text{and} \ \|\xi_i\|\le R, \ i=1,\ldots,n.$$
Form the mean $\xi =\frac{1}{n}\sum_{i=1}^n \xi_i$. Suppose 
$\EE(\xi^2) \preceq \Phi.$ Introduce the intrinsic dimension parameter
$\din = \intdim(\Phi).$
Then, for $n t \ge n\|\Phi\|^{1/2}+R/3$,
$$\PP\{\|\xi\|\ge t\} \le 4\din e^{-\frac{n t^2/2}{n\|\Phi\|+Rt/3}}.$$
\end{proposition}

We use the above inequalities to estimate the deviation of the empirical mean from the mean and the deviation of the empirical covariance matrix from the covariance matrix when the data $\calX_{j,k} = \{x_1,\ldots,x_n\}$ (with a slight abuse of notations) are i.i.d. samples from the distribution $\rho|_{\Cjk}$.

\begin{lemma}
\label{lemma1}
Suppose $x_1,\ldots,x_n$ are i.i.d. samples from $\rho|_{\Cjk}$. 
Let 
\begin{eqnarray*}
c_{j,k} = \int_{C_{j,k}} xd\rho|_{\Cjk} &,& \hc_{j,k} := \frac 1 n \sum_{i=1}^n x_i  \\
\Sigma_{j,k} = \int_{C_{j,k}} (x-c_{j,k})(x-c_{j,k})^T d\rho|_{\Cjk} &,&\hS_{j,k} := \frac{1}{n}\sum_{i=1}^n (x_i- \hc_{j,k})(x_i -\hc_{j,k})^T
\end{eqnarray*}
Then
\begin{eqnarray}
\label{eqc1}
\PP\{\|\hc_{j,k} -c_{j,k}\| \ge t \}  &\le& 8e^{-\frac{3n t^2}{6\theta_2^2 2^{-2j}+2\theta_2 2^{-j}t} }\,,\\
\PP\{\|\hSjk - \Sjk\| \ge t\} &\le& \left(\frac{4\theta_2^2 }{\theta_3}d+8\right)e^{-\frac{3nt^2}{24\theta_2^4 2^{-4j} +8\theta_2^2 2^{-2j}t}}\,.
\label{eqS1}
\end{eqnarray}
\end{lemma}

\begin{proof}
We start by proving \eqref{eqc1}.
We will apply Bernstein inequality with $\xi_i = x_i - c_{j,k} \in \RR^{D}$. Clearly $\EE\xi_i = 0$, and $\|\xi_i\| \le {\theta_2 2^{-j}}$
due to Assumption (A4). We form the mean $\xi = \frac{1}{n}\sum_{i=1}^n \xi_i = \hc_{j,k} - c_{j,k}$
and compute the variance
\begin{align*}
\sigma^2 
 &= n^2\|\EE\xi^T\xi\|
= \left\| \EE\left (\sum_{i=1}^{n} {x_i-c_{j,k}}\right)^T \left (\sum_{i=1}^{n} {x_i-c_{j,k}}\right)\right\|
 = \left\| \sum_{i=1}^n {\EE (x_i -c_{j,k})^T (x_i -c_{j,k})}\right\|\le n{\theta_2^2 2^{-2j}}.
\end{align*}
Then for $nt \ge \sigma+\theta_2 2^{-j}/3$,
$$\PP\{\|\hc_{j,k} -c_{j,k}\| \ge t \} \le 
8 e^{-\frac{n^2 t^2/3}{\sigma^2 + \theta_2 2^{-j}nt/3} }
\le8e^{-\frac{3n t^2}{6\theta_2^2 2^{-2j}+2\theta_2 2^{-j}t} }.$$

We now prove \eqref{eqS1}. Define the intermediate matrix $\bar\Sigma_{j,k} =\frac 1 n \sum_{i=1}^n (x_i -c_{j,k})(x_i - c_{j,k})^T.$
Since 
$\hS_{j,k} - \Sigma_{j,k} = \bar\Sigma_{j,k} -\Sjk - (\hcjk-\cjk)(\hcjk-\cjk)^T,$
we have
$$\|\hS_{j,k} - \Sigma_{j,k}\| \le
\|\bar\Sigma_{j,k} -\Sjk\| + \| \hcjk-\cjk\|^2
\le
\|\bar\Sigma_{j,k} -\Sjk\| + \theta_2 2^{-j}\| \hcjk-\cjk\|.$$
A sufficient condition for 
$\|\hS_{j,k} - \Sigma_{j,k}\| < t$ is $\|\bar\Sigma_{j,k} -\Sjk\| < t/2$ 
and 
$\| \hcjk-\cjk\|<2^j t/(2\theta_2)$. 
We apply Proposition \ref{prop2} to estimate $\PP\{\|\bar\Sigma_{j,k} -\Sjk\| \ge t/2\}$:
let 
$\xi_i =  (x_i -\cjk)(x_i -\cjk)^T - \Sjk \in \RR^{D \times D}.$ One can verify that 
$\EE\xi_i = 0
\ \text{and} \ 
\|\xi_i\| \le {2\theta_2^2 2^{-2j}}.$
We form the mean
$\xi = \frac{1}{n}\sum_{i=1}^n \xi_i = \bar\Sigma_{j,k} -\Sjk$,
and then
\begin{align*}
\EE\xi^2 
& = \EE \left(\frac{1}{n^2} \sum_{i=1}^n \xi_i  \sum_{i=1}^n \xi_i \right)
= \frac{1}{n^2}\sum_{i=1}^n \EE \xi_i^2 
\preceq
\frac{1}{n^2}\sum_{i=1}^n  {\theta_2^2 2^{-2j}}\Sjk
\preceq \frac{\theta_2^2 2^{-2j}}{n}\Sjk,
\end{align*}
which satisfies 
$\left\|\frac{\theta_2^2 2^{-2j}}{n}\Sjk\right\| \le {\theta_2^4 2^{-4j}}/{n}.$
Meanwhile
$$\din = \intdim(\Sjk) = \frac{\trace(\Sjk)}{\|\Sjk\|} 
\le
\frac{\theta_2^2 2^{-2j}}{\theta_3 2^{-2j}/d} = \frac{\theta_2^2}{\theta_3} d.$$
Then, Proposition \ref{prop2} implies
$$
\PP\{\| \bar\Sigma_{j,k} -\Sjk\| \ge t/2 \}
\le
  \frac{4\theta_2^2}{\theta_3} d 
e^{
\frac{-n t^2/8} {\theta_2^4 2^{-4j} + \frac{\theta_2^2 2^{-2j}  t}{3}}
}
=
 \frac{4\theta_2^2}{\theta_3} d 
e^{
\frac{-3nt^2}{24\theta_2^4 2^{-4j} +8\theta_2^2 2^{-2j}t}
}\,.
$$
Combining with \eqref{eqc1}, we obtain
\begin{align*}
\PP\{\|\hSjk - \Sjk\| \ge t\}
&\le
\PP\{\| \bar\Sigma_{j,k} -\Sjk\| \ge t/2 \} 
+
\PP\left\{\|\hcjk-\cjk \| \ge \frac{2^j t}{2\theta_2}\right\} 
\\
& \le \left(\frac{4\theta_2^2 }{\theta_3}d+8\right)
e^{-
\frac{3nt^2}{24\theta_2^4 2^{-4j} +8\theta_2^2 2^{-2j}t}}\,.
\end{align*}

\end{proof}

In Lemma \ref{lemma1} data are assumed to be i.i.d. samples from the conditional distribution $\rho|_{\Cjk}$. Given $\calX_n = \{x_1,\ldots,x_n\}$ which contains i.i.d. samples from $\rho$, we will show that the empirical measure $\hrho(C_{j,k}) = \hnjk/n$ is close to $\rho(C_{j,k})$ with high probability.

\begin{lemma}
\label{lemma2}
Suppose $x_1,\ldots,x_n$ are i.i.d. samples from $\rho$. Let $\rho(\Cjk) = \int_{\Cjk} 1 d\rho$ and $\hrho(\Cjk) = \hnjk/n$ where $\hnjk$ is the number of points in $\Cjk$. Then
\beq
\label{eqrho1}
\PP\{\left|\hrho(\Cjk)-\rho(\Cjk)\right| \ge t \} 
\le 
2e^{-\frac{3nt^2}{6\rho(C_{j,k}) +2t}}
\eeq
for all $t \ge 0$. Setting $t = \frac 1 2 \rho(\Cjk)$ gives rise to
\beq
\label{eqrho2}
\PP\left\{|\hrho(\Cjk)-\rho(\Cjk)| \ge \frac{1}{2}\rho(\Cjk) \right\} 
\le 
2e^{-\frac{3}{28}n\rho(\Cjk)}.
\eeq
\end{lemma}


Combining Lemma \ref{lemma1} and Lemma \ref{lemma2} gives rise to probability bounds on $\|\hcjk-\cjk\|$ and $\|\hSjk-\Sjk\|$ where $\cjk$, $\hcjk$, $\Sjk$ and $\hSjk$ are the conditional mean, empirical conditional mean, conditional covariance matrix and empirical conditional covariance matrix on $\Cjk$, respectively.

\begin{lemma}
\label{lemma3}
Suppose $x_1,\ldots,x_n$ are i.i.d. samples from $\rho$. Define $\cjk,\Sjk$ and $\hcjk,\hSjk$ as Table \ref{TableGMRA}. Then given any $t>0$,
\begin{eqnarray}
\label{eqc2}
\PP\left\{\|\hcjk-\cjk\|\ge t \right\}
&\le&
2e^{-\frac{3}{28}n\rho(\Cjk)}
+ 
8e^{-\frac{3 n\rho(\Cjk) t^2}{12\theta_2^2 2^{-2j}+4\theta_2 2^{-j}t} }\,,\\
\label{eqS2}
\PP\left\{\|\hSjk-\Sjk\|\ge t \right\}
&\le& 
2e^{-\frac{3}{28}n\rho(\Cjk)}
+
\left(\frac{4\theta_2^2 }{\theta_3}d+8\right)
e^{-
\frac{3 n\rho(\Cjk) t^2}{96\theta_2^4 2^{-4j} +16\theta_2^2 2^{-2j}t}
}\,.
\end{eqnarray}

\end{lemma}

\begin{proof}
The number of samples on $\Cjk$ is $\hnjk = \sum_{i=1}^n \chijk(x_i)$. Clearly $\EE [\hnjk] = n\rho(\Cjk)$. Let $\calI \subset \{1,\ldots,n\}$ and $|\calI| = s$. Conditionally on the event $A_\calI := \{x_i \in \Cjk \text{ for } i \in \calI \text{ and } x_i \notin \Cjk \text{ for } i \notin \calI\}$, the random variables $\{x_i , i \in \calI\}$ are i.i.d. samples from $\rho|_{\Cjk}$. According to Lemma \ref{lemma2},
\begin{align*}
& \PP\{\|\hcjk-\cjk\| \ge t \ | \ \hn_{j,k} = s\}  
= \sum_{\substack{\calI \subset \{1,\ldots,n\} \\
|\calI| = s}} \PP\{\|\hcjk-\cjk\| \ge t \ | \ A_\calI\} \frac{1}{\binom{n}{s}}
\\
 & = 
 \PP\{\|\hcjk - \cjk\| \ge t \ |\ A_{\{1,\ldots,s\}} \}
 \le
8e^{-\frac{3s t^2}{6\theta_2^2 2^{-2j}+2\theta_2 2^{-j}t} },
\end{align*}
and
$$\PP\{\|\hSjk -\Sjk\| \ge t \ | \ \hn_{j,k} = s \} \le
\left(\frac{4\theta_2^2 }{\theta_3}d+8\right)
e^{-
\frac{3 s t^2}{24\theta_2^4 2^{-4j} +8\theta_2^2 2^{-2j}t}
}\,.$$
Furthermore $|\hrho(\Cjk) -\rho(\Cjk)| \le \frac{1}{2}\rho(\Cjk)$ yields $\hn_{j,k} \ge \frac 1 2 n \rho(\Cjk)$ and then
\beq
\label{lemma3eq1}
\PP\left\{\|\hcjk-\cjk\|\ge t\ \Big|  \ |\hrho(\Cjk) -\rho(\Cjk)| \le \frac{1}{2}\rho(\Cjk)\right\}
\le 
8e^{-\frac{3 n\rho(\Cjk) t^2}{12\theta_2^2 2^{-2j}+4\theta_2 2^{-j}t}}\,,
\eeq
\beq
\label{lemma3eq2}
\PP\left\{\|\hSjk-\Sjk\|\ge t\ \Big|  \ |\hrho(\Cjk) -\rho(\Cjk)| \le \frac{1}{2}\rho(\Cjk)\right\}
\le 
\left(\frac{4\theta_2^2 }{\theta_3}d+8\right)
e^{-
\frac{3 n\rho(\Cjk) t^2}{48\theta_2^4 2^{-4j} +16\theta_2^2 2^{-2j}t}}\,.
\eeq
Eq. \eqref{lemma3eq1} \eqref{lemma3eq2} along with Lemma \ref{lemma2} gives rise to
$$
\PP\left\{\|\hcjk-\cjk\|\ge t \right\}
\le
2e^{-\frac{3}{28}n\rho(\Cjk)}
+ 
8e^{-\frac{3 n\rho(\Cjk) t^2}{12\theta_2^2 2^{-2j}+4\theta_2 2^{-j}t}}\,,
$$
$$
\PP\left\{\|\hSjk-\Sjk\|\ge t \right\}
\le 
2e^{-\frac{3}{28}n\rho(\Cjk)}
+
\left(\frac{4\theta_2^2 }{\theta_3}d+8\right)
e^{-\frac{3 n\rho(\Cjk) t^2}{48\theta_2^4 2^{-4j} +16\theta_2^2 2^{-2j}t}}\,.
$$
\end{proof}

Given $\|\hSjk-\Sjk\|$, we can estimate the angle between the eigenspace of $\hSjk$ and $\Sjk$ with the following proposition.
\begin{proposition}[\citet{DKPerturbation} or Theorem 3 in \citet{ZB_eigenspace}]
\label{prop3}
Let $\delta_d(\Sjk) = \frac 1 2 (\lambda_d^{j,k}-\lambda_{d+1}^{j,k})$. If $\|\hSjk - \Sjk\| \le \frac{1}{2} {\delta_d(\Sjk)} $, then
$$\left\|\proj_{\Vjk} - \proj_{\hVjk}\right\|
\le 
\frac{\|\hSjk - \Sjk\|}{\delta_d(\Sjk)}.$$
\end{proposition}
According to Assumption (A4) and (A5), $\delta_d(\Sjk) \ge \theta_3 2^{-2j}/(4d)$.
An application of Proposition \ref{prop3} yields
\begin{align}
&\PP\left\{ \|\proj_{\Vjk} - \proj_{\hVjk}
\right\| \ge t\}
 \le
\PP\left\{\|\hSjk-\Sjk\|\ge \frac{\theta_3(1-\theta_4) t}{2d2^{2j}} \right\}
\nonumber
\\ &
\le 
2e^{-\frac{3}{28}n\rho(\Cjk)}
+
\left(\frac{4\theta_2^2 }{\theta_3}d+8\right)
e^{-\frac{3\theta_3^2 (1-\theta_4)^2 n \rho(\Cjk)t^2}{384\theta_2^4 d^2 + 32 \theta_2^2 \theta_3(1-\theta_4) d t}}\,.
\label{prop3eq1}
\end{align}

\begin{proof}[Proof of Lemma \ref{thm1}]
Since
$$\|\calP_\Lam X - \hcalP_\Lam X\|^2
= \sum_{\Cjk \in \Lam}\int_{\Cjk} \|\calP_{j,k} x -\hcalPjk x \|^2 d\rho
= 
\sum_{j}
\sum_{k:\Cjk \in \Lam} \int_{\Cjk} 
\|\calP_{j,k} x -\hcalPjk x \|^2 d\rho,
$$
we obtain the estimate 
\beq
\label{thm1p1}
\PP
\left\{
\|\calP_\Lam X - \hcalP_\Lam X\| \ge \eta
\right\} 
\le
\sum_{j} 
\PP\left\{
\sum_{k:\Cjk \in \Lam} \int_{\Cjk} 
\|\calP_{j,k} x -\hcalPjk x \|^2 d\rho
\ge \frac{2^{-2j} \#_j\Lam\eta^2}{\sum_{j\ge \jmin}2^{-2j}\#_j\Lam}
\right\}.
\eeq
Next we prove that, for any fixed scale $j$, 
\beq
\label{thm1p2}
\PP\left\{
\sum_{k:\Cjk \in \Lam} \int_{\Cjk} 
\|\calP_{j,k} x -\hcalPjk x \|^2 d\rho
\ge t^2
\right\} 
\le
\alpha \#_j\Lam e^{-\frac{\beta 2^{2j} n t^2}{\#_j\Lam}}.
\eeq
Then Lemma \ref{thm1} is proved by setting $t^2 = 2^{-2j} \#_j\Lam \eta^2/(\sum_{j\ge 0} 2^{-2j} \#_j\Lam)$.

The proof of \eqref{thm1p2} starts with the following calculation:
\begin{align*}
&\sum_{k:\Cjk \in \Lam} \int_{\Cjk} 
\|\calP_{j,k} x -\hcalPjk x \|^2 d\rho
\\
& = \sum_{k:\Cjk \in \Lam}
\int_{\Cjk} \|\cjk+\proj_{\Vjk} (x-\cjk)-\hcjk -\proj_{\hVjk}(x-\hcjk)\|^2 d\rho
\\
& 
\le 
\sum_{k:\Cjk \in \Lam}
 \int_{\Cjk} \| (\bbI - \proj_{\hVjk})(\cjk-\hcjk)
+
(\proj_{\Vjk} -\proj_{\hVjk})(x-\cjk)\|^2 d\rho
\\
&
\le 
2\sum_{k:\Cjk \in \Lam}
\int_{\Cjk} 
\left[
\|\cjk-\hcjk\|^2
+
\|(\proj_{\Vjk} -\proj_{\hVjk})(x-\cjk)\|^2 \right] d\rho
\\
& 
\le 
2\sum_{k:\Cjk \in \Lam}
\int_{\Cjk} 
\left[
\|\cjk-\hcjk\|^2
+
\theta_2^2 2^{-2j}\|\proj_{\Vjk} -\proj_{\hVjk}\|^2 \right] d\rho
\end{align*}
For any fixed $j$ and given $t>0$, we divide $\Lam$ into light cells $\Lam_{j,t}^{-}$ and heavy cells $\Lam_{j,t}^{+}$, where
$$\Lam^{-}_{j,t} := \left\{ \Cjk \in \Lam : \rho(\Cjk) \le \frac{t^2}{20\theta_2^2 2^{-2j} \#_j\Lam  }\right\}
\ \text{and} \
\Lam^{+}_{j,t} := \Lam\setminus\Lam^-_{j,t}\,.$$ 
Since $ \int_{\Cjk} 
\left[
\|\cjk-\hcjk\|^2
+
\theta_2^2 2^{-2j}\|\proj_{\Vjk} -\proj_{\hVjk}\|^2 \right] d\rho \le 5\theta_2^2  2^{-2j}\rho(\Cjk)$, for light sets we have
\beq
\label{thm1eq1}
2\sum_{k:\Cjk \in \Lam_{j,t}^{-}} 
\int_{\Cjk}
\left[
\|\cjk-\hcjk\|^2
+
\theta_2^2 2^{-2j}\|\proj_{\Vjk} -\proj_{\hVjk}\|^2 \right] d\rho
\le \frac{t^2}{2}.
\eeq
\commentout{
Notice that Lemma \ref{thm1} holds if $\tilde\calT$ is a proper subtree of the complete master tree $\calT$
since
all cells in $\calT\setminus \calTn$ belong to $\cup_{j\ge 0}\Lam_{j,t}^{-}$ with high probability. Every $\Cjk \in \calT\setminus \calTn$ satisfies $\hrho(\Cjk) = 0$ and we can show that
$$\left\{\hrho(\Cjk)=0 \text{ and } \rho(\Cjk) > \frac{t^2}{20\theta_2^2 2^{-2j} \#_j\Lam}\right\} \le 2e^{-\frac{3}{560\theta_2^2} \cdot \frac{2^{2j}n t^2}{\#_j\Lam}}$$
by applying Lemma \ref{lemma2}. This implies that we will obtain the same probability estimate no matter whether we work on the complete master tree $\calT$ or the \textcolor{red}{data} master tree $\calTn$. 
}
Next we consider 
$\Cjk\in \Lam_{j,t}^{+}$. We have
 \begin{align} 
&\PP\left\{\|\hcjk-\cjk\|\ge \frac{t}{\sqrt{8\#_j\Lam\rho(\Cjk)}} \right\}
\nonumber
\\
&\le
2\exp \left(-\frac{3}{28}n\rho(\Cjk)\right)
+ 
8e^{-\frac{3 n\rho(\Cjk) \frac{t^2}{8\#_j\Lam \rho(\Cjk)}}{12\theta_2^2 2^{-2j}+4\theta_2 2^{-j}\frac{t}{\sqrt{8\#_j\Lam\rho(\Cjk)}}}}
 \le 
C_1 e^{-C_2 \frac{ 2^{2j} n t^2}{\#_j\Lam}},
\label{thm1eq2}
\end{align}
and
\begin{align}
&\PP\left\{ \|\proj_{\Vjk}-\proj_{\hVjk}\| \ge 
\frac{2^j t}{\theta_2 \sqrt{8\#_j\Lam\rho(\Cjk)}}\right\}
\nonumber
\\
& \le 
2e^{-\frac{3}{28}n\rho(\Cjk)}
+
\left(\frac{4\theta_2^2 }{\theta_3}d+8\right)
e^{
-\frac{3\theta_3^2(1-\theta_4)^2 n \rho(\Cjk)\frac{2^{2j}t^2}{8\theta_2^2 \#_j\Lam\rho(\Cjk)}}{384\theta_2^4 d^2 + 32 \theta_2^2 \theta_3(1-\theta_4) d \frac{2^{j}t}{\theta_2 \sqrt{8 \#_j\Lam\rho(\Cjk)}}}
}
 \le 
C_3 de^{
- C_4\frac{ 2^{2j}nt^2}{d^2 \#_j\Lam}
}
\label{thm1eq3}
\end{align}
where positive constants $C_1,C_2,C_3,C_4$ depend on $\theta_2$ and $\theta_3$.
Combining \eqref{thm1eq1}, \eqref{thm1eq2} and \eqref{thm1eq3} gives rise to
\eqref{thm1p2} with $\alpha = \max(C_1,C_3)$ and $\beta = \min(C_2,C_4).$

\end{proof}

\begin{proof}[Proof of Proposition \ref{lemma0}]
The bound \eqref{lemma0eq1} follows directly from Lemma \ref{thm1} applied to $\Lam = \Lam_j$; 
\eqref{lemma0eq2} follows from \eqref{lemma0eq1} by integrating the probability over $\eta$:
\begin{align*}
& \EE \|\calP_j X -\hcalP_j X\|^2
 = \int_{0}^{+\infty}  \eta \PP\left\{ 
\|\calP_j X - \hcalP_j X\| \ge \eta
\right\} d\eta
\\
& 
\le \int_0^{+\infty} \eta \min\left\{1, \alpha d  \#\Lam_j e^{-\frac{\beta 2^{2j} n \eta^2}{d^2 \#\Lam_j}}\right\} d\eta
= \int_{0}^{\eta_0} \eta d \eta
+\int_{\eta_0}^{+\infty}  \alpha d \eta \#\Lam_j e^{-\frac{\beta 2^{2j} n \eta^2}{d^2 \#\Lam_j}} d\eta 
\end{align*}
where $\alpha  d \#\Lam_j e^{-\frac{\beta 2^{2j}n\eta_0^2}{ d^2 \#\Lam_j}}= 1$. Then
$$
 \EE \|\calP_j X -\hcalP_j X\|^2
 = \frac 1 2 \eta_0^2 +\frac{\alpha}{2\beta} \cdot \frac{ \#\Lam_j^2 }{2^{2j}n} e^{-\beta\frac{2^{2j} n \eta_0^2}{ \#\Lam_j}}
\le 
\frac{ d^2 \#\Lam_j \log[\alpha d  \#\Lam_j]}{\beta 2^{2j}n}.
$$
\end{proof}

\section{Proof of Eq. \eqref{Bs1}, Lemma \ref{lemmathm2_1}, \ref{thm3_lemma1},  \ref{lemma4}, \ref{thm3_lemma2}}
\label{app12}

\subsection{Proof of Eq. \eqref{Bs1}}
\label{proofBs1}
Let $\Lam^{+0}_{(\rho,\eta)} = \Lamrhoeta$ and 
$\Lam^{+n}_{(\rho,\eta)}$ be the partition consisting of the children of $\Lam^{+(n-1)}_{(\rho,\eta)}$ for $n=1,2,\ldots$. Then 
\begin{align*}
\|X-\calP_{\Lamrhoeta} X\| 
&=  \|\sum_{\ell=0}^{n-1} (\calP_{\Lam^{+\ell}_{(\rho,\eta)}} X - \calP_{\Lam^{+(\ell+1)}_{(\rho,\eta)}} X )
+
\calP_{\Lam^{+n}(\rho,\eta)} X  - X
\|
\\
&=  \|\sum_{\ell=0}^{\infty} (\calP_{\Lam^{+\ell}_{(\rho,\eta)}} X - \calP_{\Lam^{+(\ell+1)}_{(\rho,\eta)}} X )
+
\lim_{n \rightarrow \infty}\calP_{\Lam^{+n}_{(\rho,\eta)}} X  - X
\|
\\
& \le 
\|\sum_{\Cjk \notin \calTrhoeta} \calQjk X \|
 +\|\lim_{n \rightarrow \infty} \calP_{\Lam^{+n}_{(\rho,\eta)}} X  - X\|.
\end{align*}
We have $\|\lim_{n \rightarrow \infty} \calP_{\Lam^{+n}_{(\rho,\eta)}} X - X\| = 0$ due to Assumption (A4). Therefore,
\begin{align*}
&\|X-\calP_{\Lamrhoeta} X\|^2
\le 
\|\sum_{\Cjk \notin \calTrhoeta} \calQjk X \|^2
\le
 \sum_{\Cjk \notin \calTrhoeta} B_0 \|\calQjk X\|^2
=B_0 \sum_{\Cjk \notin \calTrhoeta}  \Deltajk^2
\\
&\le 
B_0 \sum_{\ell \ge 0} \sum_{\Cjk \in \calT_{(\rho,2^{-(\ell+1)}\eta) }\setminus \calT_{(\rho,2^{-\ell}\eta)} } \Deltajk^2 
\le 
B_0 {\sum_{\ell \ge 0}\sum_{j\ge \jmin}
  (2^{-j}2^{-\ell}\eta)^2
\#_j \calT_{(\rho,2^{-(\ell+1)}\eta)}
}
\\
& \le
B_0 \sum_{\ell \ge 0} 2^{-2\ell} \eta^2
\sum_{j\ge \jmin}
  2^{-2j}
\#_j \calT_{(\rho,2^{-(\ell+1)}\eta)}
 \le
B_0 \sum_{\ell \ge 0} 2^{-2\ell} \eta^2
|\rho|_{\BS}^p [2^{-(\ell+1)}\eta]^{-p}
\\
& \le
B_0 2^p \left(\sum_{\ell \ge 0} 2^{-\ell(2-p)}\right)
|\rho|_{\BS}^p \eta^{2-p}
 \le
B_{s,d} |\rho|_{\BS}^p \eta^{2-p}.
\end{align*}

\subsection{Proof of Lemma \ref{lemmathm2_1}}

\begin{align*}
&\left\|(X-\calP_{\jstar} X)\mathbf{1}_{\left\{C_{\jstar,k}: \rho(C_{{\jstar},k}) \le \frac{28(\nu+1)\log n}{3n}\right\}}\right\|^2
 \le \sum_{\left\{C_{\jstar,k}: \rho(C_{\jstar,k}) \le  \frac{28(\nu+1)\log n}{3n}\right\}} \int_{C_{\jstar,k}} \|x-\calP_{j^*,k}\|^2 d\rho
\\ 
&
\le \# {\left\{C_{\jstar,k}: \rho(C_{\jstar,k}) \le  \frac{28(\nu+1)\log n}{3n}\right\}} \theta_2^2 2^{-2j^*} \frac{28(\nu+1)\log n}{3n}
\\
&
\le  \tfrac{28(\nu+1)\theta_2^2}{3\theta_1} 2^{j^* (d-2)} (\log n)/n
\le \tfrac{28(\nu+1)\theta_2^2 \mu}{3\theta_1} \left( (\log n)/n\right)^{2}.
\end{align*}
For every $C_{j^*,k}$, we have 
\begin{align*}
& \PP\left\{ \rho(C_{\jstar,k}) > \tfrac{28}3(\nu+1)(\log n)/n \text{ and } \hrho(C_{\jstar,k}) < d/n \right\}
\\
&
\le \PP\left\{ |\hrho(C_{\jstar,k}) -\rho(C_{\jstar,k})| > \rho(C_{\jstar,k})/2 \text{ and }  \rho(C_{\jstar,k}) > \tfrac{28}3(\nu+1)(\log n)/n
\right\}
\\
& \qquad \text{ for $n$ so large that } 14(\nu+1)\log n > 3d
\\
& \le 2 e^{-\frac{3}{28} n \rho( C_{\jstar,k})} \le 2n^{-\nu-1}.
\end{align*}
Then
\begin{align*}
& \PP\left\{ \text{each } C_{j^*,k} \text{ satisfying } \rho(C_{\jstar,k}) >  \tfrac{28}3(\nu+1)(\log n)/n \text{ has at most $d$ points} \right\}
\\
&\le \# {\left\{C_{\jstar,k}: \rho(C_{\jstar,k}) <  \tfrac{28}3(\nu+1)(\log n)/n\right\}} 
2n^{-\nu-1}
\le \#\Lam_{\jstar} 2n^{-\nu-1} 
\le {2n^{-\nu}}/{(\theta_1\mu \log n)} <  n^{-\nu}, 
\end{align*}
when $n$ is so large that $\theta_1 \mu \log n > 2.$

\subsection{Proof of Lemma \ref{thm3_lemma1}}

Since $\calTbtaun \subset \calTrhobtaun$, $\PP\{ e_{12} > 0\}$ if and only if there exists $\Cjk \in \calTrhobtaun\setminus \calTbtaun$. In other words, $\PP\{ e_{12} > 0\}$ if and only if there exists $\Cjk \in \calTrhobtaun$ such that $\hrho(\Cjk) < d/n$ and $\Deltajk > 2^{-j}b\tau_n$. Therefore,
\begin{align*}
&\PP\{e_{12}>0\} 
\le \sum_{\Cjk \in \calTrhobtaun} \PP \{\hrho(\Cjk) <d/n \text{ and } \Deltajk > 2^{-j}b\tau_n\}
\\
& 
\le \sum_{\Cjk \in \calTrhobtaun} \PP \left\{\hrho(\Cjk) <d/n \text{ and } \rho(\Cjk) > \tfrac{4b^2\tau_n^2}{9\theta_2^2} \right\}
\quad \left(\text{since }\Deltajk \le \tfrac{3}{2}\theta_2 2^{-j}\sqrt{\rho(\Cjk)}\right)
\\
& \le 
\sum_{\Cjk \in \calTrhobtaun} \PP \left\{ |\hrho(\Cjk) - \rho(\Cjk)| > \rho(\Cjk)/2  \text{ and } \rho(\Cjk) > \tfrac{4b^2\tau_n^2}{9\theta_2^2} \right\}
\\
&
\qquad (\text{for $n$ large enough so that }  2b^2\kappa^2 \log n > 9\theta_2^2 d)
\\
&
\le 
\sum_{\Cjk \in \calTrhobtaun}
2e^{-\frac{3}{28} n \cdot \frac{4 b^2 \kappa^2 \log n }{9\theta_2^2 n} }
\le 
 2 n^{-\frac{b^2\kappa^2}{21\theta_2^2}} 
  \#\calTrhobtaun.
\end{align*}
The leaves of $\calTrhobtaun$ satisfy $\rho(\Cjk) > 4b^2\tau_n^2/(9\theta_2^2)$. Since $\rho(\calM )=1$, there are at most $9\theta_2^2/(4b^2\tau_n^2)$ leaves in  $\calTrhobtaun$. Meanwhile, since every node in $\calT$ has at least $\amin$ children, $\# \calTrhobtaun \le 9\theta_2^2 \amin/(4b^2\tau_n^2)$. Then for a fixed but arbitrary $\nu>0$,
$$
\PP\{e_{12}>0\} 
\le 
\tfrac{18\theta_2^2 \amin}{4b^2\tau_n^2} n^{-\frac{b^2\kappa^2}{21\theta_2^2}} 
\le
\tfrac{18\theta_2^2 \amin}{4b^2\kappa^2} n^{1-\frac{b^2\kappa^2}{21\theta_2^2}} 
\le
 C(\theta_2,\amax,\amin,\kappa) n^{-\nu},
$$
if $\kappa$ is chosen such that $\kappa > \kappa_1$ where $b^2\kappa_1^2/(21\theta_2^2) = \nu+1$. 


\subsection{Proof of Lemma \ref{lemma4}}
We first prove \eqref{lemma4eq1}.
Introduce the intermediate variable
$$\bDeltajk :=\|\calQjk\|_n= \left\| (\calP_j -\calP_{j+1})\chijk X \right\|_n$$
and then observe that
\begin{align}
\PP\left\{\hDeltajk \le \eta \ \text{ and } \
\Deltajk \ge b \eta  \right\}
& \le 
\PP\left\{\hDeltajk \le \eta \ \text{ and } \
\bDeltajk \ge (\amax+2)\eta  \right\}
\nonumber
\\
& + \PP\left\{\bDeltajk \le (\amax+2)\eta \ \text{ and } \
\Deltajk \ge (2\amax+5)\eta  \right\}.
\label{lemma4eq11}
\end{align}

The bound in Eq. \eqref{lemma4eq1} is proved in the following three steps. In Step One, we show that $\Deltajk \ge b\eta$ implies $\rho(\Cjk) \ge \mathcal{O} (2^{2j} \eta^2)$. Then we estimate $\PP\left\{\hDeltajk \le \eta \ \text{ and } \ \bDeltajk \ge (\amax+2)\eta  \right\}$ in Step Two and $ \PP\left\{\bDeltajk \le (\amax+2)\eta \ \text{ and } \
\bDeltajk \ge (2\amax+5)\eta  \right\}$ in Step Three.

\noindent{\bf{Step One:}}
Notice that $\Deltajk \le \frac{3}{2}\theta_2 2^{-j}\sqrt{\rho(\Cjk)}$.
As a result, $\Deltajk \ge b\eta$ implies 
\beq
\label{lemma4eq3}
\rho(\Cjk) \ge \frac{4b^2 2^{2j}\eta^2}{9\theta_2^2}.
\eeq

\noindent{\bf{Step Two:}}
\beq
\label{lemma4eq5}
\PP\left\{\hDeltajk \le \eta \ \text{ and } \
\bDeltajk \ge (\amax+2)\eta  \right\}
\le \PP\left\{|\hDeltajk - \bDeltajk| \ge (\amax+1)\eta \right\}.
\eeq
We can write
\begin{align}
& |\hDeltajk - \bDeltajk|
 \le
\left\|(\calPjk - \hcalPjk ) \chijk X\right\|_n
+
 \sum_{C_{j+1,k'} \in \Child(\Cjk)}  \left\|(\hcalP_{j+1,k'} - \calP_{j+1,k'})\mathbf{1}_{j+1,k'} X\right\|_n 
 \nonumber
\\
&
\le 
\underbrace{\left(\|\cjk-\hcjk\| + \theta_2 2^{-j} \|\proj_{\Vjk} -\proj_{\hVjk}\|
\right)\sqrt{\hrho(\Cjk)}}_{e_1}
\nonumber
\\
& \qquad 
+\underbrace{\sum_{C_{j+1,k'} \in \Child(\Cjk)}
\left(\|c_\jkp-\hc_\jkp\| + \theta_2 2^{-(j+1)} \|\proj_{V_\jkp} -\proj_{\hV_\jkp}\|
\right)\sqrt{\hrho(C_{j+1,k'})}}_{e_2}.
\label{lemma4eq6}
\end{align}

\noindent{\bf{Term $e_1$:}} We will estimate $\PP\{e_1 > \eta\}$. Conditional on the event that $\{|\hrho(\Cjk) -\rho(\Cjk)| \le \frac 1 2\rho(\Cjk) \}$, we have
$e_1 \le \frac{3}{2}\left(\|\cjk-\hcjk\| + \theta_2 2^{-j} \|\proj_{\Vjk} -\proj_{\hVjk}\|
\right)\sqrt{\rho(\Cjk)}.$
A similar argument to the proof of Lemma \ref{thm1} along with \eqref{lemma4eq3} give rise to 
$$\PP\left\{
\frac 3 2\left(\|\cjk-\hcjk\| + \theta_2 2^{-j} \|\proj_{\Vjk} -\proj_{\hVjk}\|
\right)\sqrt{\rho(\Cjk)} 
> \eta
\right\} \le \tgamma_1 e^{-\tgamma_2 2^{2j}n\eta^2}
$$
where $\tgamma_1 := \tgamma_1(\theta_2,\theta_3,d)$ and $\tgamma_2 := \tgamma_2(\theta_2,\theta_3,\theta_4,d)$;
otherwise
$
\PP\left\{|\hrho(\Cjk) -\rho(\Cjk)| > \frac 1 2\rho(\Cjk)\right\}  \le 2 e^{-\frac{3}{28}n\rho(\Cjk)} 
\le 2 e^{-\frac{b^2 2^{2j}n\eta^2}{21\theta_2^2}}.
$
Therefore
\beq
\label{lemma4eq7}
\PP\{e_1>\eta\} \le \max(\tgamma_1,2) e^{ -\min(\tgamma_2,\frac{b^2}{21\theta_2^2})2^{2j}n\eta^2
}
\eeq

\noindent{\bf{Term $e_2$:}} We will estimate $\PP\{e_2>\amax\eta\}$. Let 
$\Lam^{-} = \left\{C_\jkp \in \Child(\Cjk) : \rho(C_\jkp)\le \frac{2^{2j}\eta^2}{8\theta_2^2}\right\}$ and $\Lam^+ = \Child(\Cjk)\setminus \Lam^-.$
For every $C_\jkp\in\Lam^-$, when we condition on the event that $\left\{\rho(C_\jkp)\le \frac{2^{2j}\eta^2}{8\theta_2^2} \text{ and } \hrho(C_\jkp)\le \frac{2^{2j}\eta^2}{4\theta_2^2}\right\}$, we obtain 
\begin{align}
&
\sum_{C_{j+1,k'} \in \Lam^-}
\left(\|c_\jkp-\hc_\jkp\| + \theta_2 2^{-(j+1)} \|\proj_{V_\jkp} -\proj_{\hV_\jkp}\|
\right)\sqrt{\hrho(C_{j+1,k'})} 
\nonumber
\\
&
\le 
\sum_{C_{j+1,k'} \in \Lam^-} \theta_2 2^{-j}\sqrt{\hrho(\Cjk)}
\le \amax \eta/2;
\label{lemma4eq8}
\end{align}
otherwise,
\begin{align}
&\PP\left\{
\rho(C_\jkp)\le \tfrac{2^{2j}\eta^2}{8\theta_2^2} \text{ and } 
\hrho(C_\jkp)> \tfrac{2^{2j}\eta^2}{4\theta_2^2}
\right\}
\nonumber
\\
&
\le 
\PP\left\{ \rho(C_\jkp) \le \tfrac{ 2^{2j} \eta^2}{8\theta_2^2}
\ \text{ and }
\ |\hrho(C_\jkp)-\rho(C_\jkp)| \ge \tfrac{ 2^{2j} \eta^2}{8\theta_2^2}
\right\}
\nonumber
\\
&
\le
 2 e^{-\left({3n \left(\frac{ 2^{2j} \eta^2}{8\theta_2^2}\right)^2}\right)\big/\left({6\rho(C_\jkp)+2 \frac{ 2^{2j} \eta^2}{8\theta_2^2}}\right)}
\le
 2 e^{-\frac{3\cdot 2^{2j} n\eta^2}{64\theta_2^2 }}.
 \label{lemma4eq9}
\end{align}
For $C_\jkp \in \Lam^+$, a similar argument to $e_1$ gives rise to
\begin{align}
&\PP\left\{\sum_{C_\jkp \in \Lam^+}
\left(\|c_\jkp-\hc_\jkp\| + \theta_2 2^{-(j+1)} \|\proj_{V_\jkp} -\proj_{\hV_\jkp}\|
\right)\sqrt{\hrho(C_\jkp)} > \amax \eta/2 \right\}
\nonumber
\\
&
\le
\sum_{C_\jkp \in \Lam^+}\PP\left\{
\left(\|c_\jkp-\hc_\jkp\| + \theta_2 2^{-(j+1)} \|\proj_{V_\jkp} -\proj_{\hV_\jkp}\|
\right)\sqrt{\hrho(C_\jkp)} 
\ge \eta/2
\right\} 
\nonumber
\\
&
\le \tgamma_3 e^{-\tgamma_4 2^{2j}n\eta^2}
\label{lemma4eq10}
\end{align}
where $\tgamma_3 := \tgamma_3(\theta_2,\theta_3,\amax,d)$ and $\tgamma_4 := \tgamma_4(\theta_2,\theta_3,\theta_4,\amax,d)$.

Finally combining \eqref{lemma4eq5}, \eqref{lemma4eq6}, \eqref{lemma4eq7}, \eqref{lemma4eq8}, \eqref{lemma4eq9} and \eqref{lemma4eq10} yields
\begin{align}
&\PP\left\{\hDeltajk \le \eta \ \text{ and } \
\bDeltajk \ge (\amax+2)\eta  \right\}
\le  \PP\left\{|\hDeltajk - \bDeltajk| \ge (\amax+1)\eta \right\}
\nonumber
\\
&\qquad\qquad\qquad 
\le \PP\{e_1 > \eta \} + \PP\{e_2>\amax\eta\}
 \le \tgamma_5 e^{-\tgamma_6 2^{2j} n \eta^2}
\label{lemma4eq12}
\end{align}
for some constants $\tgamma_5 := \tgamma_5(\theta_2,\theta_3,\amax,d)$ and $\tgamma_6 := \tgamma_6(\theta_2,\theta_3,\theta_4,\amax,d)$.

\noindent{\bf{Step Three:}}
The probability $\PP\left\{\bDeltajk \le (\amax+2)\eta \ \text{ and } \
\Deltajk \ge (2\amax+5)\eta  \right\}$ is estimated as follows. For a fixed $\Cjk$, we define the function
$$
f(x) = \left\|\left(\calP_j - \calP_{j+1} \right) \chijk x\right\|, \ x\in\calM.
$$
Observe that $|f(x)| \le \frac 3 2 \theta_2  2^{-j}$ for any $x\in \calM$. We define
$\|f\|^2 = \int_\calM f^2(x)d\rho$ and $\|f\|_n^2 = \frac 1 n \sum_{i=1}^n f^2(x_i)$. 
Then
\begin{align}
&\PP\left\{\bDeltajk \le (\amax+2)\eta \ \text{ and } \
\Deltajk \ge (2\amax+5)\eta  \right\}
\nonumber\\
&\le
\PP\left\{ \Deltajk - 2\bDeltajk \ge \eta \right\}
= 
\PP\left\{ \|f\| - 2\|f\|_n \ge \eta \right\}
 \le 3 e^{-\frac{2^{2j}n \eta^2}{648 \theta_2^2 }}
\label{lemma4eq13}
\end{align}
where the last inequality follows from \citet[Theorem 11.2]{GKKW_book}.
Combining \eqref{lemma4eq11}, \eqref{lemma4eq12} and \eqref{lemma4eq13} yields \eqref{lemma4eq1}.

Next we turn to the bound in Eq. \eqref{lemma4eq1}, which corresponds to the case that $\Deltajk \le \eta$ and $\hDeltajk \ge b\eta$. In this case we have 
$\hDeltajk \le \frac 3 2 \theta_2 2^{-j} \sqrt{\hrho(\Cjk)}$
which implies 
\beq
\label{lemma4eq14}
\hrho(\Cjk) \ge \frac{4 b^2 2^{2j} \eta^2}{9\theta_2^2},
\eeq
instead of \eqref{lemma4eq3}. We shall use the fact that $\rho(\Cjk) \ge ({2 b^2 2^{2j} \eta^2})/({9\theta_2^2})$ given \eqref{lemma4eq14} with high probability, by writing
\begin{align}
\PP\left\{ \Deltajk \le \eta\ \text{ and }\ \hDeltajk \ge b\eta\right\}
 \le 
 &
 \PP\left\{ \Deltajk \le \eta\ \text{ and }\ \hDeltajk \ge b\eta\, \big|\, \rho(\Cjk) \ge \tfrac{2 b^2 2^{2j} \eta^2}{9\theta_2^2}\right\}
\nonumber
\\
 \quad & +
\PP\left\{ \rho(\Cjk) \le \tfrac{2 b^2 2^{2j} \eta^2}{9\theta_2^2}
\ \text{ and }
\ \hrho(\Cjk) \ge \tfrac{4 b^2 2^{2j} \eta^2}{9\theta_2^2}
\right\}
\label{lemma4eq15}
\end{align}
where the first term is estimated as above and the second one is estimated through Eq. \eqref{eqrho1} in Lemma \ref{lemma2}:
\begin{align*}
& \PP\left\{ \rho(\Cjk) \le \tfrac{2 b^2 2^{2j} \eta^2}{9\theta_2^2}
\ \text{ and }
\ \hrho(\Cjk) \ge \tfrac{4 b^2 2^{2j} \eta^2}{9\theta_2^2}
\right\}
\\
\le
& 
\PP\left\{ \rho(\Cjk) \le \frac{2 b^2 2^{2j} \eta^2}{9\theta_2^2}
\ \text{ and }
\ |\hrho(\Cjk)-\rho(\Cjk)| \ge \frac{2 b^2 2^{2j} \eta^2}{9\theta_2^2}
\right\}
\\
\le
& 2 e^{-\left({3n (\frac{2 b^2 2^{2j} \eta^2}{9\theta_2^2})^2}\right)\big/\left({6\rho(\Cjk)+2 \frac{2 b^2 2^{2j} \eta^2}{9\theta_2^2}}\right)}
\le  2 e^{-\frac{3b^2 2^{2j} n\eta^2}{36\theta_2^2 }}.
\end{align*}
Using the estimate in \eqref{lemma4eq15}, we obtain the bound \eqref{lemma4eq1} which concludes the proof.


\subsection{Proof of Lemma \ref{thm3_lemma2}}
We will show how Lemma \ref{lemma4} implies Eq. \eqref{hbx}. Clearly $e_2 = 0$ if 
$\hLamtaun\vee \Lambtaun=
\hLamtaun\wedge \Lamtaunb
$, or equivalently $\hcalTtaun\cup \calTbtaun=
\hcalTtaun\cap \calTtaunb
$. In the case $e_2>0$, the inclusion $\hcalTtaun\cap \calTtaunb \subset \hcalTtaun\cup \calTbtaun$ is strict, i.e. there exists $\Cjk \in \calTn$ such that either $\Cjk \in \hcalTtaun$ and $\Cjk \notin \calTtaunb $, or $\Cjk \in \calTbtaun$ and $\Cjk \notin \hcalTtaun$. In other words, there exists $\Cjk \in \calTn$ such that either $\Deltajk <2^{-j}\tau_n/b$ and $\hDeltajk \ge 2^{-j}\tau_n$, or $\Deltajk \ge b2^{-j}\tau_n$ and $\hDeltajk < 2^{-j}\tau_n$. As a result,
\begin{align}
\PP\{e_2>0\}
&  \le 
\sum_{\Cjk \in \calTn} 
\PP \left\{\hDeltajk < 2^{-j}\tau_n \text{ and }\Deltajk \ge b2^{-j}\tau_n 
\right\}
\label{thm3p2}
\\
& \quad
+
\sum_{\Cjk \in \calTn} 
\PP \left\{\Deltajk <2^{-j}\tau_n/b \text{ and }
\hDeltajk \ge 2^{-j}\tau_n
\right\}.
\nonumber
\end{align}
\beq
\label{thm3p3}
\PP\{e_4>0\}
 \le
\sum_{\Cjk \in \calTn} 
\PP \left\{\Deltajk <2^{-j}\tau_n/b \text{ and }
\hDeltajk \ge 2^{-j}\tau_n
\right\}.
\eeq
We apply \eqref{lemma4eq1} in Lemma \ref{lemma4} to estimate the first term in \eqref{thm3p2}:
\begin{align*}
&\sum_{\Cjk \in \calTn}\PP \left\{\hDeltajk < 2^{-j}\tau_n \text{ and }\Deltajk \ge b2^{-j}\tau_n 
\right\}
 \le 
\sum_{\Cjk \in \calTn} \alpha_1 e^{-\alpha_2 n 2^{2j}\cdot 2^{-2j} \kappa^2 \lognn}
\\
& 
 = \alpha_1\#\calTn n^{-\alpha_2 \kappa^2}
\le 
\alpha_1 \amin n n^{-\alpha_2 \kappa^2}
\le
\alpha_1 \amin  n^{1-\alpha_2 \kappa^2}
=\alpha_1 \amin  n^{-(\alpha_2 \kappa^2-1)}.
\end{align*}
Using \eqref{lemma4eq1}, we estimate the second term in \eqref{thm3p2} and \eqref{thm3p3} as follows
\begin{align*}
&\sum_{\Cjk \in \calTn} 
\PP \left\{\Deltajk <2^{-j}\tau_n/b \text{ and }
\hDeltajk \ge 2^{-j}\tau_n
\right\}
\le \sum_{\Cjk \in \calTn} \alpha_1 e^{-\alpha_2 n 2^{2j}\cdot  \frac{2^{-2j}}{b^2}\kappa^2\lognn}
\le \alpha_1 \amin  n^{-(\alpha_2 \kappa^2/b^2-1)}.
\end{align*}
We therefore obtain \eqref{hbx} by choosing $\kappa > \kappa_2$ with $\alpha_2 \kappa_2^2 /b^2 = {\nu}+1$.

\section{Proofs in orthogonal GMRA}

\subsection{Performance analysis of orthogonal GMRA}
\label{appo1}

The proofs of Theorem \ref{thmo2} and Theorem \ref{thmo3} are resemblant to the proofs of Theorem \ref{thm2} and Theorem \ref{thm3}. The main difference lies in the variance term, which results in the extra log factors in the convergence rate of orthogonal GMRA. 
%
Let $\Lam$ be the partition associated with a finite proper subtree $\tcalT$ of the data master tree $\calTn$, and let
$$\calS_\Lam = \sum_{\Cjk\in\Lam} \calSjk\chijk\quad \text{and} \quad
\hcalS_\Lam = \sum_{\Cjk\in\Lam} \hcalSjk\chijk.$$

\begin{lemma}
\label{thmo1}
Let $\Lam$ be the partition associated with a finite proper subtree $\tcalT$ of the data master tree $\calTn$. Suppose $\Lam$ contains $\#_j \Lam$ cells at scale $j$.
Then for any $\eta > 0$,
\beq
\label{thmo1eq}
\PP\{\|\calS_\Lam X - \hcalS_\Lam X\| \ge \eta \}
 \le 
\alpha  d
\left(\sum_{j\ge \jmin}  j \#_j \Lam \right) e^{-
\frac{\beta  n \eta^2}{d^2 \sum_{j\ge \jmin} j^4  2^{-2j}\#_j \Lam}}
\eeq
where $\alpha$ and $\beta$ are the constants in Lemma \ref{thm1}.
\end{lemma}

\label{proofthmo1}
\begin{proof}[Proof of Lemma \ref{thmo1}]
The increasing subspaces $\{\Sjx\}$ in the construction of orthogonal GMRA may be written as 
\begin{align*}
S_{0,x}  & = V_{0,x}
\\
S_{1,x} & =  V_{0,x} \oplus V_{0,x}^\perp V_{1,x}
\\
S_{2,x} & =  V_{0,x} \oplus V_{0,x}^\perp V_{1,x}
\oplus V_{1,x}^\perp V_{0,x}^\perp V_{2,x}
\\
& \cdots \\
S_{j,x} & = V_{0,x} \oplus V_{0,x}^\perp V_{1,x} 
\oplus \ldots \oplus V_{j-1,x}^\perp\cdots V_{1,x}^\perp V_{0,x}^\perp V_{j,x}.
\end{align*}
Therefore $\|\proj_{\Sjx} -\proj_{\hSjx}\| \le  \sum_{\ell =0}^j (j+1-\ell) \|\proj_{V_{\ell,x}} -\proj_{\hV_{\ell,x}}\|$,
and then
\begin{align}
& \PP\left\{\|\proj_{\Sjx} -\proj_{\hSjx}\| \ge t\right\}
 \le \sum_{\ell = 0}^j \PP\left\{\|\proj_{V_{\ell,x}} -\proj_{\hV_{\ell,x}}\| \ge {t}/{j^2} \right\}.
\label{thmo1eq3}
\end{align}
The rest of the proof is almost the same as the proof of Lemma \ref{thm1} in appendix \ref{app1} with a slight modification of \eqref{prop3eq1} substituted by \eqref{thmo1eq3}.
\end{proof}

The corollary of Lemma \ref{thmo1} with $\Lam = \Lam_j$ results in the following estimate of the variance in empirical orthogonal GMRA.

\begin{lemma}
\label{lemmao0}
For any $\eta \ge 0$,
\begin{align}
&\PP\{\|\calS_j X - \hcalS_j X\| \ge \eta \}
 \le 
\alpha d
  j \#\Lam_j  e^{-
\frac{\beta 2^{2j} n \eta^2}{d^2  j^4 \#\Lam_j}},
\label{lemmao0eq1}
\\ 
&
\EE \|\calS_j X - \hcalS_j X\|^2 
\le
\frac{d^2 j^4 \#\Lam_j \log[\alpha d   j \#\Lam_j]}{\beta 2^{2j}n}.
\label{lemmao0eq2}
\end{align}
\end{lemma}

\commentout{

\begin{proof}
\eqref{lemmao0eq1} is a corollary of Lemma \ref{thm1} when $\Lam = \Lam_j$. \eqref{lemmao0eq2} follows by integrating the probability in \eqref{lemmao0eq1} over $\eta$.

\begin{align*}
 \EE \|\calS_j X -\hcalS_j X\|^2
& = \int_{0}^{+\infty}  \eta \PP\left\{ 
\|\calS_j X - \hcalS_j X\| \ge \eta
\right\} d\eta
\\
& 
\le \int_0^{+\infty} \eta \min\left\{1, \alpha d j \#\Lam_j e^{-\beta\frac{2^{2j} n \eta^2}{d^2 j^4 \#\Lam_j}}\right\} d\eta
\\
&
 = \int_{0}^{\eta_0} \eta d \eta
+\int_{\eta_0}^{+\infty}  \eta \alpha d j \#\Lam_j e^{-\beta\frac{2^{2j} n \eta^2}{d^2 j^4 \#\Lam_j}} d\eta 
\\
& = \frac 1 2 \eta_0^2 +\frac{\alpha}{2\beta} \cdot \frac{d^3 j^5 \#\Lam_j^2 }{2^{2j}n} e^{-\beta\frac{2^{2j} n \eta_0^2}{d^2 j^4 \#\Lam_j}}
\\
&\le 
\frac{d^2 j^4 \#\Lam_j \log[\alpha d  j \#\Lam_j]}{\beta 2^{2j}n},
\end{align*}
where $\alpha d j \#\Lam_j e^{-\beta\frac{2^{2j}n\eta_0^2}{d^2 j^4 \#\Lam_j}} = 1$.

\end{proof}
}


\begin{proof}[Proof of Theorem \ref{thmo2}]
\begin{align*}
&\EE \|X-\hcalS_j X\|^2
\le 
 \|X-\calS_j X\|^2
+
\EE \|\calS_j X-\hcalS_j X\|^2
\\
& 
\le |\rho|_{\AS^{\rm o}}^2 2^{-2sj} + 
\frac{d^2 j^4 \#\Lam_j \log[\alpha d  j \#\Lam_j]}{\beta 2^{2j}n}
\le 
|\rho|_{\AS^{\rm o}}^2 2^{-2sj} + 
\frac{d^2 j^4 2^{j(d-2)} }{\theta_1 \beta n}\log\frac{\alpha d  j 2^{jd}}{\theta_1}.
\end{align*}
%
%
When $d\ge 2$, We choose $j^*$ such that
$2^{-j^*}  = 
\mu \left( (\log^5 n)/n\right)^{\frac{1}{2s+d-2}}.$
By grouping $\Lam_{j^*}$ into light and heavy cells whose measure is below or above $\tfrac{28}3(\nu+1)\log^5 n/n$, we can show that the error on light cells is upper bounded by $C((\log^5 n)/n)^{\frac{2s}{2s+d-2}}$ and all heavy cells have at least $d$ points with high probability.

\begin{lemma}
\label{lemmathmo2_1}
Suppose $j^*$ is chosen such that $2^{-j^*}  = \mu \left( \tfrac{\log^5 n}{n}\right)^{\frac{1}{2s+d-2}}$ with some $\mu>0$.
Then 
\begin{align*}
&\|(X-\calP_{\jstar} X)\mathbf{1}_{\{C_{\jstar,k}: \rho(C_{\jstar,k}) 
\le \tfrac{28(\nu+1)\log^5 n}{3n}\}}\|^2 
\le
 \tfrac{28(\nu+1) \theta_2^2 \mu^{2-d}}{3\theta_1} \left( \tfrac{\log^5 n}{n}\right)^{\frac{2s}{2s+d-2}}, 
\\
&\PP\left\{ \forall\, C_{j^*,k} \,:\, \rho(C_{\jstar,k}) >  \tfrac{28(\nu+1)\log^5 n}{3n},\,C_{j^*,k} \text{ has at least $d$ points} \right\} \ge 1- n^{-\nu}.
\end{align*}

\end{lemma}
Proof of Lemma \ref{lemmathmo2_1} is omitted since it is the same as the proof of Lemma \ref{lemmathm2_1}. Lemma \ref{lemmathmo2_1} guarantees that a sufficient amount of cells at scale $j^*$ has at least $d$ points.
The probability estimate in \eqref{thmo2eq21} follows from
\begin{align*}
&\PP\left\{\|\calS_{j^*} X - \hcalS_{j^*} X \| 
\ge C_1 \left (\tfrac{\log^5 n}{n} \right)^{\frac{s}{2s+d-2}}
 \right\} 
\le C_2 \log n
 \left( \tfrac{\log^5 n}{n}\right)^{-\frac{d}{2s+d-2}} e^{-{\beta\theta_1
\mu^{d-2} C_1^2 (2s+d-2)^4 }/d^2 \log n}
\\
& \le C_2 \left(\log n\right) {n^{\frac{d}{2s+d-2}}} n^{-{\beta\theta_1 \mu^{d-2}
C_1^2 (2s+d-2)^4/d^2}}
 \le C_2 n^{1-{\beta\theta_1\mu^{d-2}
C_1^2 (2s+d-2)^4/d^2}}
\le C_2  n^{-\nu}
\end{align*}
provided $C_1$ is chosen such that ${\beta\theta_1 \mu^{d-2}
C_1^2(2s+d-2)^4}/d^2-1 > \nu$.
The proof when $d=1$ is completely analogous to that of Theorem \ref{thm2}.
\end{proof}

\subsection{Performance analysis of adaptive orthogonal GMRA}
\label{appo2}

\begin{proof}[Proof of Theorem \ref{thmo3}]
Empirical adaptive orthogonal GMRA is given by 
$\hcalS_{\hLamtauno} = \sum_{\Cjk \in \hLamtauno}
\hcalSjk \chijk.$
Using triangle inequality, we have
$$
\| X - \hcalS_{\hLamtauno} X \| \le 
e_1 + e_2 + e_3 +e_4$$
with each term given by 
\begin{align*}
e_1  := 
\| X -\calS_{\hLamtauno\vee \Lambtauno} X\|
&\qquad e_2 :=
\| \calS_{\hLamtauno\vee \Lambtauno} X
-
\calS_{\hLamtauno\wedge \Lamtaunob} X
\|
\\
e_3 :=
\| 
\calS_{\hLamtauno\wedge \Lamtaunob} X
-
\hcalS_{\hLamtauno\wedge \Lamtaunob} X
\|
& \qquad  e_4 :=
\| 
\hcalS_{\hLamtauno\wedge \Lamtaunob} X
-
\hcalS_{\hLamtauno} X
\|
\end{align*}
where $b ={2\amax+5}$. 
We will prove the case $d \ge 3$. Here one proceeds in the same way as in the proof of Theorem \ref{thm3}. A slight difference lies in the estimates of $e_3$, $e_2$ and $e_4$.

\noindent{\bf Term $e_3$:}
$\EE e^2_3$ is the variance. One can verify that 
$\calTrhotaunob \subset \calT_{j_0}:=\cup_{j \le j_0} \Lam_{j}$
where $j_0$ is the largest integer satisfying $2^{j_0 d} \le  {9b^2\theta_0 \theta_2^2}/(4{\tau_n^o}^2)$. The reason is that $\Deltajk^o \le \frac 3 2 \theta_2 2^{-j}\sqrt{\theta_0 2^{-jd}}$ so $\Deltajk^o \ge 2^{-j}\tau_n^o/b$ implies $2^{j_0 d} \le  {9b^2\theta_0 \theta_2^2}/(4{\tau_n^o}^2)$.
For any $\eta>0$, 
\begin{align*}
\PP\{e_3 > \eta\} 
\le \alpha d  j_0 \# \calTtaunob e^{-\frac{\beta n \eta^2}{j_0^4 \sum_{j\ge \jmin}2^{-2j}\#_j  \calTtaunob }}
\le
\alpha d j_0  \#\calTtaunob e^{
-\frac{\beta n \eta^2}{j_0^4 |\rho|_{\BS^o}^p (\tau_n^o/b)^{-p} }
}
\end{align*}
The estimate of $\EE e^2_3$ follows from 
\begin{align*}
& \EE e_3^2  
= \int_{0}^{+\infty} \eta \PP\left\{e_3 > \eta \right\} d\eta 
= \int_{0}^{+\infty} \eta \min\left(1, \alpha d j_0 \#\calTtaunob e^{
-\frac{\beta n \eta^2}{j_0^4 \sum_{j\ge \jmin}2^{-2j}\#_j \calTtaunob }
} \right) d\eta
\\
&
\le \tfrac{j_0^4 \log \alpha j_0 \#\calTtaunob}{\beta n} \sum_{j\ge \jmin} 2^{-2j} \#_j \calTtaunob
 \le C \tfrac{\log^5 n}{n} (\tau_n^o/b)^{-p}
\le 
C(\theta_0,\theta_2,\theta_3,\amax,\kappa,d,s) \left(\tfrac{\log^5 n}{n} \right)^{\frac{2s}{2s+d-2}}.
\end{align*}

\noindent{\bf Term $e_2$ and $e_4$:}
These two terms are analyzed with Lemma \ref{lemmao4} stated below such that for any fixed but arbitrary $\nu>0$,
$$\PP\{e_2>0\} + \PP\{e_4>0\} \le {\beta_1 \amin}/{d} n^{-\nu}$$
if $\kappa$ is chosen such that $\kappa > \kappa_2$ with $d^4 \beta_2 \kappa_2^2/b^2 = \nu+1$.
%
\begin{lemma}
\label{lemmao4}
$b = 2\amax+5$.
For any $\eta >0$ and any $\Cjk \in \calT$
\begin{align*}
\max\left(\PP\left\{\hDeltajk^o \le \eta \ \text{ and } \
\Deltajk^o \ge b\eta  \right\}
,\PP\left\{\Deltajk^o \le \eta \ \text{ and } \
\hDeltajk^o \ge b\eta  \right\}\right)
&  \le
\beta_1 j e^{-\beta_2 n  2^{2j}  \eta^2 /j^4},
\end{align*}
with positive constants $\beta_1 := \beta_1(\theta_2,\theta_3,\theta_4,\amax,d)$ and $\beta_2 := \beta_2(\theta_2,\theta_3,\theta_4,\amax,d)$.

\end{lemma}


\commentout{
A similar argument to the case of adaptive GMRA yields
\begin{align*}
\PP\{e_2>0\}
&  \le 
\sum_{\Cjk \in \calTn} 
\PP \left\{\hEPjk < j^2 2^{-j}\tau_n^o \text{ and }\EPjk \ge b j^2 2^{-j}\tau_n^o 
\right\}
\\
& \quad
+
\sum_{\Cjk \in \calTn} 
\PP \left\{\EPjk  < j^2 2^{-j}\tau_n^o/b \text{ and }
\hEPjk \ge j^2 2^{-j}\tau_n^o
\right\},
\end{align*}
and
$$\PP\{e_4>0\}
 \le
\sum_{\Cjk \in \calTn} 
\PP \left\{\EPjk < j^2 2^{-j}\tau_n^o/b \text{ and }
\hEPjk \ge j^2 2^{-j}\tau_n^o
\right\}.
$$
Lemma \ref{lemmao4} yields
$$
\sum_{\Cjk \in \calTn}\PP \left\{\hEPjk < j^2 2^{-j}{\tau_n^o} \text{ and }\EPjk \ge b j^2 2^{-j}{\tau_n^o} 
\right\}
\le
 \sum_{\Cjk \in \calTn} \beta_1 j n^{-\beta_2\kappa^2} 
\le \beta_1 a n^{-(\beta_2 \kappa^2-1)}
$$
and
$$
\sum_{\Cjk \in \calTn} 
\PP \left\{\EPjk <2^{-j}{\tau_n^o}/b \text{ and }
\hEPjk \ge 2^{-j}{\tau_n^o}
\right\}
\le  B
\sum_{\Cjk \in \calTn} \beta_1 j  e^{-\beta_2 \kappa^2/b^2}  
\le 
\beta_1 a n^{-(\beta_2 \kappa^2/b^2-1)}.
$$
Therefore we obtain \eqref{thmo3p1} by choosing $\kappa > \kappa_2$ with $\beta_2 \kappa_2^2/b^2 = {\nu}+1$.
}

\end{proof}


\bibliographystyle{plain}
\bibliography{MyPublications,DiffusionBib,BigBib_Before}

\end{document}